\documentclass[12pt,a4paper,oneside]{book}  

\usepackage[ top=2.5cm,bottom=2.5cm,left=3cm,right=2cm ]{ geometry }

\usepackage{subfigure}
\usepackage{caption}
\usepackage{subcaption}
\usepackage{CJKutf8}
\usepackage{indentfirst}
\usepackage{array}
\usepackage{multirow}
\usepackage{amsthm}
\usepackage{amsmath}
\usepackage{amsfonts}
\usepackage{amssymb}
\usepackage{graphicx}
\usepackage{multirow}
\usepackage{setspace}
\usepackage{balance}
\usepackage{textcomp}
\usepackage{color}
\usepackage[ampersand]{easylist}
\ListProperties(Hide=100, Hang=true, Progressive=3ex, Style*=$\bullet$ ,
Style2*=\tiny$\blacksquare$ ,Style3*=$\circ$ ,Style4*=-- )
\usepackage{changepage}
\usepackage{algorithm}
\usepackage{algpseudocode}
\usepackage{graphics}
\usepackage{epsfig}
\floatname{algorithm}{Algorithm}

\usepackage{verbatim}

\usepackage{eqparbox}
\usepackage{textcomp}

\usepackage[T1]{fontenc}
\usepackage{enumerate}

\usepackage[font={rm}]{caption}

\usepackage{verbatim}

\usepackage{indentfirst}

\usepackage{chngcntr}
\counterwithout{figure}{chapter}
\counterwithout{table}{chapter}

\usepackage{color}
\usepackage{colortbl}

\usepackage{listings}
\RequirePackage{fancyhdr}
\usepackage{pdfpages}

\usepackage{transparent}
\usepackage{eso-pic}

\usepackage{amsmath,amsfonts,bm}

\def\eqref#1{equation~\ref{#1}}

\def\1{\bm{1}}

\DeclareMathAlphabet{\mathsfit}{\encodingdefault}{\sfdefault}{m}{sl}
\SetMathAlphabet{\mathsfit}{bold}{\encodingdefault}{\sfdefault}{bx}{n}

\DeclareMathOperator*{\argmin}{arg\,min}

\usepackage{hyperref}
\usepackage{url}
\usepackage{enumitem}
\usepackage{booktabs}
\usepackage[numbers]{natbib}

\NewDocumentCommand{\citepx}{m}{%
  (\citet{#1},~\citeyear{#1})%
}
\NewDocumentCommand{\citetx}{m}{%
  \citet{#1}~(\citeyear{#1})%
}

\usepackage{lettrine}
\usepackage{quotchap}
\usepackage{lipsum}
\usepackage{csquotes}

\usepackage{flushend}
\usepackage{textcomp}

\usepackage{pdflscape}

\theoremstyle{plain}
\newtheorem{theorem}{Theorem}[section]
\newtheorem{proposition}[theorem]{Proposition}
\newtheorem{lemma}[theorem]{Lemma}

\theoremstyle{definition}
\newtheorem{definition}[theorem]{Definition}
\newtheorem{assumption}[theorem]{Assumption}
\theoremstyle{remark}

\usepackage{rotating}

\begin{document}
\begin{CJK*}{UTF8}{bkai}

\newgeometry{top=3cm,bottom=3cm,left=3cm,right=3cm}

\begin{titlepage}
  \thispagestyle{empty}
  \centering

  {\Large National Yang Ming Chiao Tung University\\
  Department of Computer Science\par}

  \vspace{2.5cm}

  {\LARGE\bfseries
  Playstyle and Artificial Intelligence:\\
  An Initial Blueprint Through the Lens of Video Games\par}

  \vspace{1.2cm}

  {\large
  Chiu-Chou Lin\\
  \texttt{dsobscure@outlook.com}\par}

  \vspace{0.6cm}
  {\large Hsinchu 30010, Taiwan\par}

  \vfill

  {\large PhD Dissertation\par}
  {\large July 2025\par}

  \vspace{1.2cm}
  \begin{minipage}{0.9\textwidth}
    \small
    \textit{Note.} This is a public version prepared for arXiv of the original PhD dissertation:  
    Lin, Chiu-Chou (2025). \textit{Playstyle and Artificial Intelligence: An Initial Blueprint Through the Lens of Video Games}.  
    PhD dissertation, National Yang Ming Chiao Tung University.  
    The official university submission may differ in formatting (e.g., watermark, committee/authorization pages).
    \end{minipage}

\end{titlepage}

\restoregeometry




\frontmatter
\pagenumbering{roman}
{\fontfamily{ptm}\selectfont

\chapter*{Abstract}
\addcontentsline{toc}{chapter}{Abstract} 

\normalsize 

Contemporary artificial intelligence (AI) development largely centers on rational decision-making, valued for its measurability and suitability for objective evaluation. Yet in real-world contexts, an intelligent agent’s decisions are shaped not only by logic but also by deeper influences such as beliefs, values, and preferences. The diversity of human decision-making styles emerges from these differences, highlighting that “style” is an essential but often overlooked dimension of intelligence.

This dissertation introduces \textit{playstyle} as an alternative lens for observing and analyzing the decision-making behavior of intelligent agents, and examines its foundational meaning and historical context from a philosophical perspective. By analyzing how beliefs and values drive intentions and actions, we construct a two-tier framework for style formation: the \textit{external interaction loop} with the environment and the \textit{internal cognitive loop} of deliberation. On this basis, we formalize style-related characteristics and propose measurable indicators such as \textit{style capacity}, \textit{style popularity}, and evolutionary dynamics.

The study focuses on three core research directions: (1) \textbf{Defining and measuring playstyle}, proposing a general playstyle metric based on discretized state spaces, and extending it to quantify strategic diversity and competitive balance; (2) \textbf{Expressing and generating playstyle}, exploring how reinforcement learning and imitation learning can be used to train agents exhibiting specific stylistic tendencies, and introducing a novel approach for human-like style learning and modeling; and (3) \textbf{Practical applications}, analyzing the potential of these techniques in domains such as game design and interactive entertainment.

Finally, the dissertation outlines future extensions, including the role of style as a core element in building artificial general intelligence (AGI). By investigating stylistic variation, we aim to rethink autonomy, value expression, and even offer a tangible perspective on the ultimate philosophical question: \textit{What is the soul?}

\textbf{Keywords:} Playstyle, Artificial Intelligence, Decision-Making Style, Video Game, Style Measurement, Balance and Diversity, Agent Behavior
\newpage



\addcontentsline{toc}{chapter}{Contents} \tableofcontents \newpage

\renewcommand{\numberline}[1]{Figure~#1\hspace*{1em}}
\addcontentsline{toc}{chapter}{List of Figures} \listoffigures \newpage

\renewcommand{\numberline}[1]{Table~#1\hspace*{1em}}
\addcontentsline{toc}{chapter}{List of Tables} \listoftables \newpage

\mainmatter
\pagenumbering{arabic} 



\chapter{Introduction}
\section{The Beginning of the Beginning}

Have you ever pondered the purpose or meaning of your life?
While such questions may
appear overwhelmingly grand or confined to philosophical speculation, they can be decomposed into more concrete forms:
“What do you want to do?” “What are you hoping will happen?”
These are not merely occasional reflections, but continuous and pervasive questions—consciously or unconsciously posed—as long as awareness persists. They lie at the heart of what we call intelligence: how we define intentionality, decision-making, and agency. 

This deep-rooted questioning reveals an essential structure within intelligence itself: the ability to make decisions.
Even in the absence of explicit deliberation, every action, intentional or incidental, reflects a choice shaped by cognition and awareness \citepx{russell2020aima}.
Decision-making, in this sense, underlies both the trivial and the profound aspects of existence.

Consider a trivial plan such as deciding to eat later in the day—it still reflects motivation, whether arising from hunger, habit, or implicit expectation.
Even a day spent doing “nothing in particular” is, in fact, a sequence of choices—active or passive—that collectively shape the trajectory of one’s life.
Ultimately, life itself can be seen as the cumulative outcome of all such decisions, conscious or unconscious.
Viewed this way, even random or passive acts become implicit answers to the most fundamental questions of existence.

If life is a series of choices, and if every choice implies a direction, a value, or a stance toward the world, then these actions inevitably raise a deeper question:
\textbf{What kind of world do we live in that makes such choices possible and meaningful?}
This line of reflection leads us into the domain of metaphysics, often regarded as the “first philosophy” \citepx{aristotle_metaphysics}.
Metaphysics seeks to understand the fundamental nature of reality, within which \textit{ontology} stands as a central branch:
the systematic investigation of what exists, how entities are categorized, and what it means \textit{to be}.
Ontology thus provides the conceptual foundation for understanding existence—not only what is, but also why and how entities act, interact, and persist.

At this juncture, two oft-cited aphorisms—emerging from classical and existentialist thought, come to mind:

\begin{displayquote}
\textit{“To be is to do.”}\footnote{Attributed to Socrates (c. 470–399 BCE); though not a direct quotation, the phrase captures the ethical, action-centered core of Socratic philosophy.}

\medskip

\textit{“To do is to be.”}\footnote{Attributed to Jean-Paul Sartre (1905–1980), reflecting the existentialist belief that existence is realized through action.}
\end{displayquote}

Though not canonical quotations, these phrases serve to highlight a central insight:
\textbf{being and doing are inseparable.}
To act is to affirm existence; to exist is, in some sense, to choose.

This returns us to the question of motivation, which lies at the intersection of ontology and agency:
\textbf{Why do we act?} \textbf{What compels one choice over another?}

In asking why we act, one may be tempted to ask further: “Why that reason?” “Why that desire?”
But such inquiry soon reaches a limit—a stopping point beyond which no further justification seems possible. In theological metaphysics—most notably in the writings of Thomas Aquinas—this recalls the concept of a “first cause”:
an unmoved mover that requires no prior explanation \citepx{aquinas_summa}. Likewise, human decision-making can often be traced back to irreducible foundations: needs, desires, and goals—
internal sources of motion that resist further reduction.

Even in mundane decisions—what to eat, where to go, what to pursue—each act reveals fragments of an underlying structure of intention.
These small moments echo the same questions we've traced through philosophy: What drives action? How is agency realized? What does it mean to choose?

Seen in this light, the study of playstyle and artificial intelligence becomes, at its core, an inquiry into decision-making, motivation, and the embodied expression of being through action.

Artificial intelligence, long envisioned as a means of emulating or extending human decision-making, is deeply rooted in the very questions explored above.
Early systems followed fixed instructions, but modern AI agents perceive, reason, and act—driven by internal objectives, shaped by environmental feedback, and refined through adaptation \citepx{russell2020aima}.
More recently, even these objectives are no longer treated as given.
Increasingly, AI systems are tasked not merely with choosing actions, but with inferring goals and aligning with human values \citepx{russell2022value_alignment}—an evolution that reopens the most fundamental philosophical questions about intention, agency, and moral reasoning.

Thus, this dissertation begins from these conceptual roots and expands outward.
While models, algorithms, and empirical validations are indispensable, the core of this inquiry lies in four enduring yet under-defined concepts: \textbf{style}, \textbf{diversity}, \textbf{balance}, and \textbf{creativity}.

Ultimately, I hope this journey offers new perspectives on the timeless philosophical questions,\footnote{Popularized by Paul Gauguin’s 1897 painting \textit{D'où venons-nous ? Que sommes-nous ? Où allons-nous ?}}

\begin{displayquote}
\textit{Who am I?}
\textit{Where did I come from?}
\textit{Where am I going?}
\end{displayquote}

—and perhaps, through the study of decision-making styles and reflective agency in artificial forms, we may offer insight toward what might be considered an understanding of the soul.

From this perspective, playstyle is not merely an external behavioral pattern.
Rather, it is a reflective projection of the internal architecture of beliefs, values, and intentions that constitute an intelligent agent.

\section{Belief, Value, and the Self}

In our ongoing effort to trace the origins of behavior, motivation, and intentionality, we eventually encounter a conceptual boundary—a point beyond which further questioning seems impossible.
This epistemic limit recurs across disciplines: in mathematics, as axioms accepted without proof; in physics, as the problem of initial conditions or the origin of the universe; in ethology, as the layered hierarchy of behavioral causation \citepx{tinbergen1963aims}; and in theology or metaphysics, as the notion of a “first cause.”

Each of these fields arrives, in its own way, at something foundational: a point of origin that halts infinite regress and enables structured explanation.
In the case of intelligent agents, I propose that such a point corresponds to \textbf{belief}: a foundational and irreducible premise that underlies decision-making.
Belief acts as the internal starting point—something an agent accepts as true or useful—which permits action to begin without further justification.

Mathematics offers a compelling analogy. To avoid logical paradoxes such as Russell’s Paradox, formal systems like the Peano axioms define natural numbers based on a finite set of unprovable yet accepted assumptions \citepx{peano_axioms}. These axioms—such as a defined base element and a rule of succession—are not adopted because they are provable, but because they furnish a necessary and practical foundation for reasoning.

Yet even the most rigorous axiom systems have intrinsic limits. Gödel’s incompleteness theorems demonstrate that any sufficiently expressive formal system inevitably contains true statements that cannot be proven within the system itself \citepx{godel1931}. In other words, some truths always lie beyond formal justification; all reasoning ultimately rests on unprovable foundations—assumptions that are not derived, but posited.

This necessity of unprovable assumptions is not unique to mathematics. In cognitive terms, it directly parallels the role of belief in intelligent behavior. Just as formal systems require axioms, intelligent agents require beliefs—fundamental commitments that arrest the regress of justification and enable coherent action within bounded resources. In practice, such foundational beliefs need not be rational, justified, or even correct; it is their presence—not their correctness—that permits action to begin.

Beliefs serve as repositories of pre-validated cognitive units from which decisions can be rapidly drawn.
For instance, when deciding what to eat for dinner, individuals typically select from familiar options without extensive deliberation, relying on beliefs shaped by past experience.

To formalize this notion in artificial systems, knowledge representation offers a useful analogy \citepx{russell2020aima}.
Structured encodings of belief allow machines to explicitly store, retrieve, and reason about information—an operational proxy for human-like belief systems.

This formal approach is not new. Early symbolic AI systems employed propositional and first-order logic to encode knowledge as declarative sentences, enabling logical manipulation and inference \citepx{russell2020aima}.
This tradition reflects a long-standing intuition: that intelligent behavior emerges from the structured processing of beliefs.

Despite its formal rigor, symbolic AI faces a critical challenge: the construction of knowledge structures has historically depended on intensive human effort.
Expert-built knowledge bases often fail to align with the richness of lived human experience or the fluid complexity of real-world dynamics.
In contrast, machine learning systems autonomously extract patterns from data—whereas symbolic AI lacked scalable mechanisms to acquire and adapt beliefs.

Thus, while knowledge representation provided a formal framework, it fell short in addressing how beliefs are formed and adapted.
Recognizing this limitation is essential: beliefs are not static entities, but dynamic constructs that evolve through continuous interaction with the environment.
The central challenge, then, is not merely how to represent beliefs, but how to build systems capable of continually updating and refining them in response to real-world feedback.

This difficulty is further compounded by the limitations of language itself as a representational medium.
While language enables the explicit encoding of knowledge, it also imposes constraints—ambiguities, contextual dependencies, and semantic variability across individuals.
As Wittgenstein later argued in his \textit{Philosophical Investigations} \citepx{wittgenstein1953philosophical}, meaning is not fixed by syntactic form alone, but arises from use within particular “language games.”
The same expression can carry different meanings depending on the speaker, context, or shared background assumptions.
In this light, symbolic knowledge systems that rely on rigid logical structures may fail to capture the fluid, situated, and socially grounded nature of belief and meaning.

To address this challenge, I propose that the interrelation among belief, value, and the self can be visualized as a dynamic feedback loop, as shown in Figure~\ref{fig:belief_value_self}.

\begin{figure}[htbp]
\centering
\includegraphics[width=\textwidth]{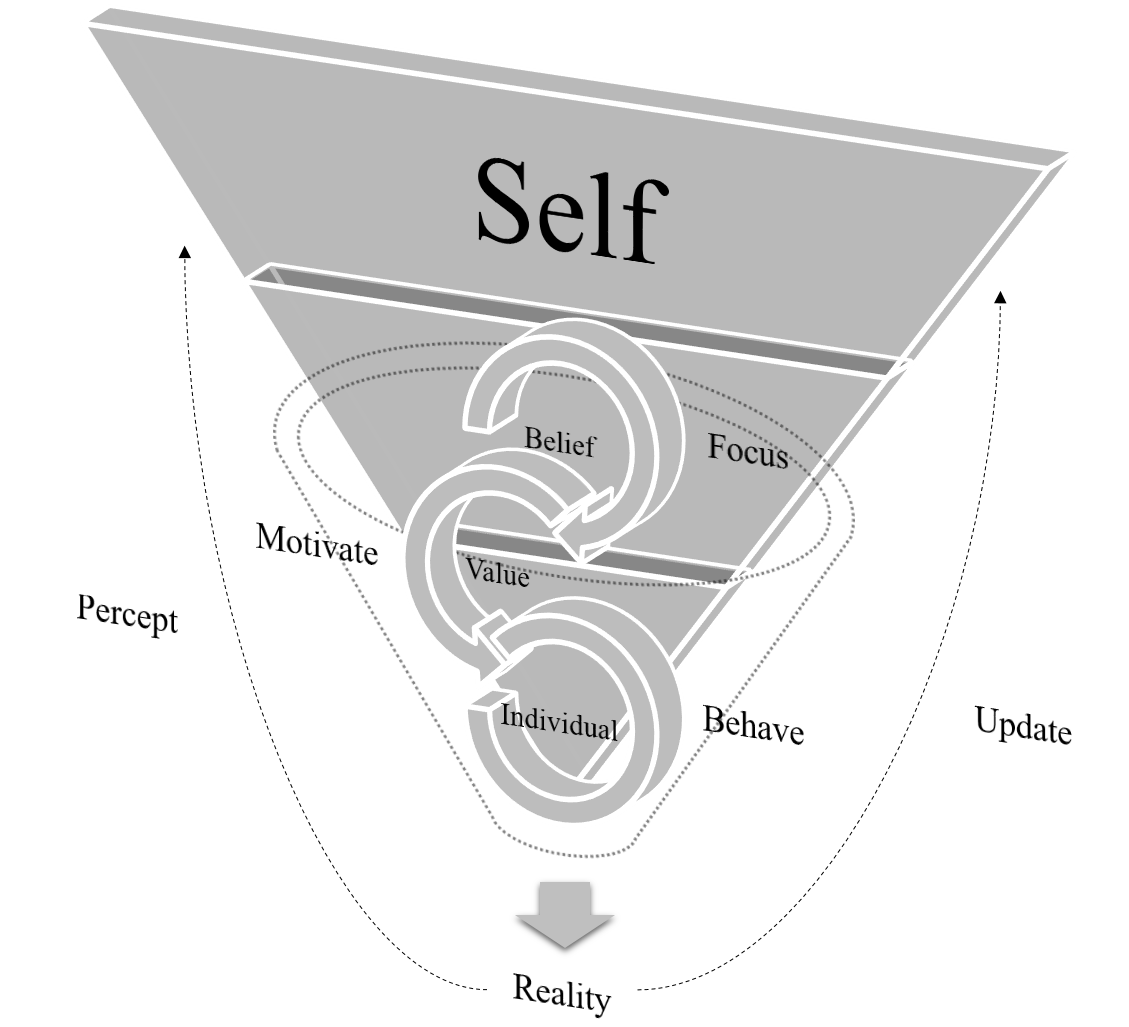}
\caption[The Dynamic Interplay of Belief, Value, and the Self]{This diagram illustrates the internal structure and outward expression of an intelligent agent.
At its core, the \textit{Self} comprises a layered architecture of \textbf{Belief}, \textbf{Value}, and \textbf{Individual}, through which decisions are motivated and behaviors are enacted.
These components form a dynamic feedback loop: beliefs inform values, values shape the agent’s individual actions, and behaviors interact with \textbf{Reality}.
Perception and feedback from the environment guide the continual updating of beliefs, sustaining the evolution of the self.
Arrows labeled \textit{Focus}, \textit{Behave}, \textit{Update}, \textit{Percept}, and \textit{Motivate} trace the ongoing cycle through which the agent engages with the world.
}
\label{fig:belief_value_self}
\end{figure}

Building on the belief–decision connection, I argue that intelligence fundamentally seeks to realize value—and that value, in essence, serves to validate belief. From this perspective, intelligent existence may be understood as a continuous process of belief validation through value generation.

Consider an academic example: novice researchers often emphasize technical novelty or the quantity of results. Yet experienced reviewers tend to ask a deeper question: What problem are you trying to solve? This concern reflects an implicit search for motivation—not as an afterthought, but as a manifestation of belief. Without a clearly stated motivation, technical merit appears adrift. Belief, expressed through motivation, anchors the value of a work.

Upon closer examination, value is not simply the result of solving problems. Though problem-solving appears to produce value, I suggest the causal flow is reversed: belief comes first, value arises from its validation, and problems are constructed as instruments for that validation.

In this framework, beliefs act not as externally justified premises, but as generative anchors of agency. They need not be rational, provable, or even accurate. Their functional role is to initiate behavior. Value emerges as the outcome of successful validation, while problems serve as structured tests through which belief manifests and is refined.

Beyond isolated instances of value, intelligent agents operate within broader value systems—structured hierarchies of evaluative priorities that guide decisions over time. These systems determine which beliefs to validate, how to resolve conflicts among them, and ultimately shape the agent’s behavioral tendencies. Divergence in value systems gives rise to differing preferences and behaviors, even under similar external conditions.

This also clarifies why value cannot be reduced to objective metrics such as money. Currency derives its worth from collective belief in the authority that issues it. In hyperinflationary contexts such as Zimbabwe, money lost its meaning the moment public belief collapsed.

Value, therefore, rests upon belief. Something is deemed worthwhile insofar as it is believed to resolve a need or fulfill a desired purpose.

This leads to a critical realization: the belief–value dynamic forms a self-reinforcing loop, with belief as both its origin and regulator. At the center of this loop lies the concept of the self.

I propose that the self is best understood as the internal system of beliefs maintained by an intelligent agent. It is not a passive repository, but a dynamic structure in which beliefs are formed, revised, and activated to guide behavior. Variations in these internal belief–value structures give rise to distinct decision-making styles, accounting for behavioral diversity across both human and artificial agents.

Clarifying the interdependence among belief, value, and the self thus provides a conceptual foundation for this dissertation. The following chapters revisit this triad across the domains of behavior, style, learning, and creativity—probing not only how beliefs guide intelligent action, but also how artificial systems might acquire a meaningful sense of self through the continual process of belief validation.

As Martin, Sugarman, and Hickinbottom emphasize, understanding selfhood requires more than modeling internal beliefs—it demands attention to the interpretive, purposive, and moral dimensions of agency \citepx{martin2009persons}.
Their critique of reductive psychological models highlights the danger of equating intelligence with computational logic alone.
In contrast, this dissertation approaches the construction of an artificial self not merely as a mechanistic aggregation of functions, but as an evolving center of belief-validated value—a structure through which action becomes meaningful and identity is sustained.

This view implicitly rejects Spinoza’s vision of freedom, which lies not in purposive agency, but in the rational acceptance of necessity \citepx{spinoza1677ethics}.
For Spinoza, to be free is not to construct a coherent self or to pursue value-driven goals, but to align with the deterministic unfolding of nature.
Yet if we aim to model agents capable of choice, learning, and self-directed growth, a value-centered conception of agency becomes not only useful, but necessary.

\section{Intention and Behavior}

Building upon the previous discussion, one might ask:
\textbf{If an intelligent agent possesses a structured set of beliefs and continually refines them through interaction with the world, is this sufficient to account for intelligent behavior?}
Or is there something missing—an internal source of initiative that cannot be reduced to environmental feedback alone?

This missing element is what we call \textbf{agency}.
Without it, an agent would merely react to external stimuli, passively adjusting its internal state without originating any independent action or direction.
The belief–value–self loop, while powerful, risks becoming a closed adaptive mechanism unless complemented by a principle of internally generated intention.

This distinction between reactivity and agency is central to Daniel Dennett’s theory of the \textbf{intentional stance}, which proposes that intelligent behavior is best understood not as a chain of stimulus-response reactions, but as the expression of internal beliefs, desires, and intentions \citepx{dennett1987intentional}.
Similarly, Andy Clark emphasizes that real intelligence involves \textbf{proactive world-modeling} and selective attention to affordances—not merely passive processing of input \citepx{clark2015surfing}.
Both thinkers argue that agency entails the ability to formulate internal perspectives, simulate possibilities, and direct behavior accordingly—not just react to the present.

In contrast, early behaviorist models such as Pavlov’s classical conditioning \citepx{pavlov1927conditioned} and Thorndike’s law of effect \citepx{thorndike1898animal} conceived behavior purely in terms of stimulus-response pairings and reinforcement.
These frameworks offered powerful tools for studying observable behavior but left no room for internal deliberation, competing goals, or autonomous intention.
As a result, they fail to capture the kind of self-originated activity that characterizes agents capable of learning not just what is rewarded, but what is meaningful.

This brings us to a deeper philosophical question:
\textbf{Is intelligence merely reactive, or must it inherently involve self-originated intention?}
In other words, can we explain intelligent action purely through environmental adaptation, or must we also account for the agent's internal source of purpose?

The classical view, most famously expressed by Ren\'e Descartes in his proposition \textit{Cogito, ergo sum} ("I think, therefore I am") \citepx{descartes_meditations}, asserts that conscious thought—thinking that originates from within—is the foundation of existence.
For Descartes, the self is not passive but initiatory: to exist is to begin from a thinking subject.
In contrast, Thomas Hobbes conceived human behavior as fundamentally reactive, governed by mechanistic principles of self-preservation and external causation \citepx{hobbes1651leviathan}.
In this view, intention is not a generative force but a byproduct of environmental necessity.

Edmund Husserl, extending and refining the Cartesian tradition of Descartes, proposed that \textbf{consciousness is inherently intentional}—that it is always directed toward some object, goal, or meaning \citepx{husserl_ideas}.
This idea shifts the emphasis from mere thinking to \textbf{directedness}: a conscious agent is not just aware, but aims, chooses, and orients toward something.
From this perspective, intentionality is not an optional enhancement of intelligence—it is its defining structure.

To operationalize this insight in the context of artificial agents, I propose that belief systems function not as a single reactive loop, but as two interacting cycles:
\begin{itemize}
\item An \textbf{external loop}, where beliefs are updated through perception, action, and feedback from the environment; and
\item An \textbf{internal loop}, where beliefs interact, evaluate, and regulate one another—even in the absence of external input.
\end{itemize}

These loops jointly constitute the agent’s capacity for intention.
The external loop enables learning from the world, while the internal loop enables directedness toward goals and the selection of action—even before or without external prompting.

The external loop governs how agents learn from their environment: they perceive stimuli, choose actions, receive feedback, and adapt their beliefs accordingly.
Yet this perception is not direct access to reality. In standard models of artificial agents, the environment is only partially observable: agents receive incomplete, noisy, and often ambiguous signals.
To act meaningfully, an agent must construct what we may call a \textbf{subjective perceived state}—a filtered internal representation of the world, grounded in limited observation and structured by prior beliefs \citepx{russell2020aima}.

This framework mirrors a long-standing philosophical tradition.
Empiricists such as \citetx{locke1689essay} and \citetx{berkeley1710treatise} argued that all knowledge originates from experience, not from innate structures.
In contrast, \citetx{kant1781critique} proposed that experience itself is \textbf{co-constructed} by sensory input and a priori cognitive frameworks—such as space, time, and causality—through which we interpret the world.
From this view, agents do not perceive the world “as it is,” but always \textbf{through a lens} structured by prior categories of understanding.

Leibniz offered a complementary picture: each agent is a \textbf{monad}, a self-contained center of perception that reflects the universe from its own internal logic \citepx{leibniz1714monadology}.
Though abstract and metaphysical, this notion resonates with modern agent-based AI systems: each agent maintains a \textbf{private, bounded, and structured perspective} on the environment, updating it continuously but never accessing totality.

In this light, action selection is not based on objective knowledge, but on interpretation, shaped by perception, prior belief, cognitive bias, and uncertainty.
The environment, in turn, evolves in response to the agent’s actions, providing feedback that enables further belief revision.
This cycle of perception, belief update, and feedback constitutes the agent’s \textbf{external loop}: the interactive process through which the world becomes knowable and actionable.

While external interaction drives belief formation, \textbf{the internal loop governs belief activation}.
Beliefs are not all used simultaneously. In practice, agents face competing goals, limited resources, and ambiguous priorities.
When multiple beliefs conflict or cannot be jointly acted upon, agents must resolve tension through \textbf{context-sensitive selection}.
This process of prioritization gives rise to \textbf{intention}: the commitment to one course of action over others, based on internal criteria.

This idea was foreshadowed by William James, who proposed that volition emerges not from singular willpower, but from an \textbf{internal competition of impulses}, with attention functioning like a spotlight that selects among them \citepx{james1890principles}.
In this view, action is not merely the product of dominant stimuli, but the outcome of \textbf{internal deliberation} across coexisting tendencies.
Jerome Bruner deepened this idea by emphasizing that meaning is not discovered, but constructed—through the \textbf{selective activation of interpretive frameworks} shaped by cultural and experiential contexts \citepx{bruner1990acts}.

These ideas are reflected in psychological models such as Maslow’s hierarchy of needs \citepx{maslow1943}, where motivations are structured into layers and selectively activated based on situational salience.
Although Buck later critiqued its linear rigidity \citepx{buck1988human}, the essential insight remains: \textbf{motivational systems are dynamic}, and behavior emerges from real-time modulation of competing internal drivers.

Rather than analyzing agents by their full belief system, it is often more informative to focus on \textbf{localized decision events}—each defined by a specific subset of beliefs activated by the present context.
This view aligns with Bernard Baars’ \textbf{global workspace theory}, which portrays consciousness as a spotlight: only a limited portion of mental content becomes globally accessible, enabling deliberate action \citepx{baars1988cognitive}.
In this sense, intention arises from a structured selection process—not a single belief, but a network of prioritized beliefs activated to guide behavior.

Each such decision-making instance can be analyzed through three interrelated components:
\begin{enumerate}
\item \textbf{Intention}: the internal motivation and preference that initiate the deliberation process.
\item \textbf{Expected Behavior}: the action an agent anticipates it will take, given its current beliefs and internal models.
\item \textbf{Actual Behavior}: the executed action and its observable impact in the environment.
\end{enumerate}

This triadic structure captures the essential gap between intention and realization: it allows us to analyze not only what the agent did, but what it aimed to do—and how belief and context shaped the difference. This alignment between intention and action becomes especially critical in environments with uncertainty, resource constraints, or conflicting preferences, where agents must choose among trade-offs.
As \citetx{ellul1990reason} warns, a civilization that pursues technical efficiency and observable outcomes while neglecting alignment with deeper intention risks falling into \textbf{futility}—what the Book of Ecclesiastes calls \textit{hevel} (often translated as “vanity”)—a state not born of inaction, but of action disconnected from purpose. Such disconnection may yield prolific behavior yet leave no enduring meaning, as the driving force behind action becomes divorced from what truly matters.

Such an event-level decomposition forms the analytical basis for modeling \textbf{playstyle}.
Since playstyle is not defined by a single action, but by \textbf{patterns} of decision-making, we infer it by aggregating many such instances over time.
By observing which beliefs are repeatedly activated, which behaviors are preferred or avoided, and how intentions align (or misalign) with outcomes, we can reconstruct an agent’s \textbf{macro-level profile of stylistic tendencies}.

Yet this approach reveals a critical challenge: the \textbf{opacity of preference orderings}.
Unlike binary propositions, preferences are not fixed logical entities—they are constructed through the interaction of multiple beliefs, goals, and situational constraints.
A single behavior may be consistent with multiple possible preference structures, making it difficult—even impossible—to unambiguously reconstruct internal priorities from external observations.

This difficulty is not merely technical, but philosophical.
As \citetx{davidson1963actions} argued, reasons for action function both as explanations and as justifications—but they are never fully transparent to outside observers.
\citetx{nagel1986nowhere} deepened this insight by distinguishing between the internal, subjective frame of reasons and the external, objective stance he termed “the view from nowhere.”
From this perspective, intention can never be fully inferred from behavior alone; it remains partially hidden within the structure of selfhood.

This brings us to a broader question: what does it mean to possess a self?
Rather than viewing the self as a fixed set of traits or responses, contemporary psychological theorists such as \citetx{martin2009persons} argue that selfhood consists in a dynamic system of values, goals, and interpretive frameworks.
These cannot be reduced to mere neurobiological or behavioral patterns, but must be understood as evolving structures of meaning and agency.

In summary, this section offers a key refinement:
the style of an intelligent agent is not a static reflection of its belief system, but a \textbf{dynamic pattern of selective expression}—shaped by intention, governed by internal and external constraints, and embedded within a fluid architecture of value.
Understanding this requires attending to the agent’s dual-loop structure:
one loop of interaction with the world, and another of inward deliberation and belief regulation.

\textbf{We will return to this dual-loop model in Chapter~5, where it will be formalized into a general framework for modeling playstyle and agent individuality.}

\section{From Philosophy to AI: Bridging Foundations}

By this point in our discussion, it should be clear that many seemingly abstract philosophical issues are deeply entangled with the foundations of artificial intelligence. Indeed, whenever we attempt to step outside familiar assumptions and re-examine the very nature of AI, we find ourselves inevitably drawn into philosophical and cognitive scientific territory. These questions are not remote or esoteric; rather, they form the implicit scaffolding beneath the algorithms, architectures, and benchmarks that define AI practice—often unnoticed, yet profoundly shaping its development.

Much of contemporary AI research focuses on what AI can do and how well it performs specific tasks. However, when we revisit the classical fourfold categorization of AI, as presented in standard textbooks \citepx{russell2020aima}, we immediately encounter the deeper complexities at play. Revisiting these categories also raises a teleological question: What are AI systems ultimately for? What kind of intelligence do we wish to cultivate?

The four traditional definitions of AI are:
\begin{itemize}
\item \textbf{Thinking Humanly}: Emulating human-like thinking processes.
\item \textbf{Thinking Rationally}: Reasoning in a logically consistent manner.
\item \textbf{Acting Humanly}: Behaving in ways indistinguishable from human actions.
\item \textbf{Acting Rationally}: Choosing actions that achieve the best outcomes based on given goals and information.
\end{itemize}

Among these, \textbf{acting rationally} is the most straightforward to formalize mathematically, through tools such as decision theory, optimization methods, and reinforcement learning frameworks. \textbf{Acting humanly}, while more complex, still allows for empirical modeling based on observable behavior. A classic milestone in this direction is the Turing Test \citepx{Turing1950}, which evaluates intelligence based on whether a machine’s behavior is indistinguishable from that of a human in conversational interaction.

The categories involving “thinking” are far more contentious. \textbf{Thinking rationally} requires logical coherence and internal consistency—yet real-world human cognition is often shaped by heuristic shortcuts, cognitive biases, and bounded rationality.
A historical example of this category appears in the Soviet scientific programme on AI \citepx{kirtchik2023soviet}, which framed artificial intelligence primarily as an extended control tool. From this view, a machine cannot truly “think”; any reasoning it performs is ultimately the product of the human designer’s own thinking, embedded into the system’s rules and procedures. Such a stance aligns with \textbf{thinking rationally}, as the emphasis is on executing human-specified logical and control processes whose correctness derives from formal reasoning. It is not \textbf{acting rationally}, because the system is not expected to autonomously determine optimal actions from goals and observations; rather, it follows a pre-defined reasoning structure, prioritizing interpretability and controllability over autonomous performance. While this approach was coherent in the era of expert systems and symbolic search—where reasoning steps were explicit and traceable—it becomes harder to sustain in modern AI, where statistical and deep learning models often operate as opaque black boxes, beyond full explanation even by their creators.

\textbf{Thinking humanly}, on the other hand, plunges deeper into philosophical controversy, as our scientific understanding of human cognition—spanning psychology, neuroscience, and cognitive science—remains incomplete and fragmented. It invokes enduring debates about the mind-body problem, the nature of intentionality, and the explanatory gap between physical substrates and subjective experience. Key issues such as consciousness, intuition, motivation, and subjective experience continue to resist unified modeling approaches.

If we interpret "thinking humanly" as the ambition to simulate human cognitive architecture, we observe two principal trajectories of research:
\begin{itemize}
\item \textbf{Cognitive modeling}: Constructing symbolic reasoning systems and abstract knowledge structures that mirror conceptual processes.
\item \textbf{Biological modeling}: Emulating the functional properties of the brain, most prominently via artificial neural networks.
\end{itemize}

Both directions ultimately converge on the same fundamental question: \textit{What is the true nature of human intelligence?} Understanding playstyle—our central focus in this work—is deeply entangled with this question.

Accordingly, this dissertation is titled \textit{Playstyle and Artificial Intelligence: An Initial Blueprint Through the Lens of Video Games} not merely to explore how AI can define, measure, and utilize playstyle as a behavioral pattern, but to advance a deeper claim: \textbf{playstyle, as a manifestation of decision-making style, bridges the structural basis of human intelligence with the generative mechanisms of artificial intelligence.}

Throughout our investigation, we will:
\begin{itemize}
\item Analyze the underlying beliefs, values, and motivations that drive human decision-making;
\item Examine how AI systems can learn, imitate, and express distinctive playstyles;
\item Explore the emergence, plasticity, and creativity of playstyle across different contexts and tasks.
\end{itemize}

Choosing video games as our primary domain of observation is both deliberate and significant. Games—particularly those that afford strategic diversity and creative freedom—offer a uniquely fertile environment for the expression of human intelligence. Within these constructed environments, players operate with relative autonomy, often free from real-world institutional or social constraints, allowing their internal decision-making tendencies and stylistic preferences to manifest with clarity.

Thus, video games serve as a natural laboratory for examining the logical structures, expressive variations, and stylistic patterns of intelligent behavior. By leveraging AI tools to model, interpret, and predict these variations, we aim to deepen not only our understanding of playstyle, but also our broader grasp of decision-making, learning, creativity, and the evolving interplay between human and artificial intelligence.

\section{Behavior, Goal, and Style}

Before proceeding to a concrete exploration of playstyle, it is essential to articulate several foundational questions—spanning philosophy, cognitive science, and artificial intelligence—that shape the conceptual scaffolding of this dissertation. These questions form the thematic thread that weaves through the chapters that follow.

\textbf{First, what is a goal?}
Building on the reasoning developed earlier, we may interpret a goal as the expected outcome associated with validating a localized subset of the self—i.e., a belief configuration that motivates action.
From an engineering perspective, goals can be operationalized as optimization targets (e.g., in reinforcement learning), preference models (e.g., in inverse reinforcement learning), or structured constraints in planning. Later chapters will explore whether such formalizations can preserve both cognitive grounding and computational tractability, and how they may be communicated or transferred between agents.

\textbf{Second, if a goal can be defined, can it be taught or learned by artificial agents?}
Beyond pursuing predefined tasks, an agent could potentially internalize goal structures that encode preference hierarchies or motivational priors.
Possible methods include preference learning, reward modeling, and goal-conditioned policies, which would allow style to serve as a medium for transmitting cognitive architectures or value systems across agents.

\textbf{Third, if agents can hold multiple goals, how should they be reconciled?}
When goals coexist, the key is to understand their underlying \textbf{preference structure}—how goals are prioritized, weighted, or conditioned on context.
Preferences may be stable or dynamic, explicit or implicit, and can be inferred from observed behavior or embedded in the agent’s design.
From a computational perspective, this can be modeled through multi-objective optimization, preference ranking, or context-sensitive weighting.
Later chapters will examine how multi-objective trade-offs are addressed in \textbf{human-like decision-making}, and how differences in preference structure can shape \textbf{game strategy} and stylistic patterns.

\textbf{Fourth, what is the relationship between goals and style?}
Goals primarily describe the intention and the final outcome of a decision, whereas style, while considering the outcome, also attends to the process by which the outcome is achieved.
In this sense, style can be seen as a more general construct than a goal—one that encompasses both the end state and the path taken.

\textbf{Finally, what is style itself?}
As introduced in Chapter\~2, style is defined as the mapping from an environment in which there exist two or more recognized decision strategies or utility functions, to their corresponding outcomes.
This definition captures style as a structural property of decision spaces, enabling both measurement and comparison across agents.

Though these questions may appear distinct, they are deeply interwoven. Together, they form the central thread of this dissertation. In the chapters that follow, we will engage them systematically—beginning with philosophical inquiry, and progressing toward computational methodologies for modeling, measuring, and applying playstyle in intelligent systems.

\section{Organization Preview}

In the preceding sections, we have established a conceptual foundation for understanding playstyle as a form of intelligent decision-making—deeply intertwined with belief structures, intention, selfhood, and the architecture of artificial agents. The remainder of this book develops a systematic blueprint for modeling, measuring, and applying playstyle, moving from conceptual analysis to technical realization.

The dissertation is organized into four main parts, each addressing a distinct dimension of the problem:

\bigskip

\noindent \textbf{Part I: Playstyle — Concepts and Blueprint}\\
We begin by examining style in daily life, strategic behavior, and human cognition. Playstyle is introduced as a concrete form of decision-making style, and we propose an analytical framework to categorize and differentiate styles across domains. This part also clarifies how playstyle relates to, and differs from, phenomena such as artistic style or generative style in AI systems. Serving as the philosophical and structural basis, it equips readers with the key conceptual tools for the rest of the book.

\bigskip

\noindent \textbf{Part II: Essence of Playstyle}\\
Here we formally define and quantify playstyle using methods ranging from heuristic rules to data-driven modeling. We develop metrics for assessing stylistic difference, agent diversity, and behavioral stability. We also examine the rationality and dynamics of multiple playstyles, analyzing diversity and balance not only as theoretical properties but also as practical concerns in game design and strategic systems.

\bigskip

\noindent \textbf{Part III: Expression of Playstyle}\\
This part explores how playstyle can be learned, imitated, and creatively generated by artificial agents. We review decision-making systems—including deep learning and reinforcement learning architectures—highlighting how stylistic individuality may emerge from different learning paradigms. Special attention is given to imitation learning and the boundaries of stylistic innovation, bridging human playstyle and machine behavior.

\bigskip

\noindent \textbf{Part IV: Applications and Future}\\
We conclude by examining the real-world applications of playstyle analysis: in games (e.g., player modeling, matchmaking, personalization), in social systems (e.g., interaction patterns, recommendation engines), and in broader user experience (UX) contexts. The final chapter returns to foundational questions: Can playstyle serve as a structural pathway toward general intelligence—or even as a proxy for selfhood or soul in artificial agents?

\bigskip

Taken together, these four parts aim not only to construct a methodology for analyzing playstyle, but also to initiate a broader dialogue across artificial intelligence, cognitive science, and philosophy. Through the lens of playstyle, we revisit enduring questions: What constitutes intelligence? How is individuality expressed? And what might it mean to build agents with preferences, values, and recognizable behavioral identity? In doing so, each part progressively transforms these questions into formal models, measurable quantities, and practical approaches.

\part{Playstyle: Concepts and Blueprint}
\chapter{Understanding Style}
\noindent\textbf{Key Question of the Chapter:} \
\textit{What exactly is style?}

\noindent\textbf{Brief Answer:}  
In this dissertation, we take the position that \textbf{style arises when multiple viable options exist, and the distinguishable differences among these options—together with their effects on outcomes—constitute the corresponding style}. This framing emphasizes that style is meaningful only in the presence of choice, and that it can be understood both through the available alternatives and the consequences they produce.

\section{Styles in Daily Life}

We live in a world saturated with styles—patterns of expression that shape how we see, act, and relate to others. From the visual elegance of Gothic cathedrals \citepx{von2020gothic} and the expressive boldness of Fauvism \citepx{millard1976fauvism}, to the rhythmic complexity of jazz music and the structured discipline of classical music \citepx{burkholder2019history, caplin1998classical}, stylistic diversity defines our cultural and aesthetic experiences. Styles also manifest in dynamic activities: in sports, individualistic playing styles contrast sharply with team-oriented tactics, such as those observed in basketball \citepx{oliver2004basketball} and football strategies \citepx{wilson2013inverting}. Even in the mundane aspects of daily life—such as cuisine, clothing, hairstyle, makeup, architecture, and modes of transport—stylistic variation is omnipresent \citepx{civitello2011cuisine, young2019art, steele2015berg}.

Fashion and color choices, in particular, offer salient examples of everyday style. In clothing, individuals continuously navigate between expressing individuality and conforming to social norms \citepx{simmel1957fashion, steele2015berg}. Fashion styles not only embody aesthetic preferences but also serve as markers of identity, affiliation, and aspiration. They operate simultaneously as personal signals and social codes—linking subjective preference with shared cultural meaning. Likewise, color selection—whether in apparel, design, or environment—plays a central role in stylistic expression and emotional experience \citepx{itten1973art, valdez1994effects}.

These examples reveal a generalizable property: \textbf{style exists when multiple feasible options are available, and when these options lead to discernible differences in expression or outcome.} This applies equally to tangible artifacts like architecture and to intangible patterns like performance tempo or color composition.

In some domains, styles are clearly delineated through formal criteria. Gothic architecture, for instance, is characterized by pointed arches, flying buttresses, and verticality—reflecting an orderly medieval conception of space \citepx{von2020gothic}. In contrast, styles in other areas may be highly subjective and open to interpretation. Modern literary forms such as magical realism and minimalism illustrate how distinct aesthetic and narrative choices can coexist, with readers perceiving different stylistic cues from the same work.

Styles may be mutually exclusive, as seen in the contrast between classical symphonic music and the aggressive energy of heavy metal \citepx{weinstein2000heavy, walser2015running}. Yet they can also hybridize, creating new forms—such as the fusion of rock and hip-hop into rap-rock \citepx{watkins2005hip, miller1981rolling}.
From a decision-making perspective, such exclusivity and hybridity correspond to different structural relationships between strategies—either selecting one path at the exclusion of others, or blending elements from multiple strategies into a coherent whole.

While the acknowledgment of diverse styles is almost inevitable, a deeper question arises: \textbf{What is the value of style?} Is style merely a proliferation of options, or does it fulfill more fundamental functional, emotional, or cultural roles? Historical patterns suggest that style cannot persist merely through novelty. To endure, it must resonate—anchored in underlying beliefs or shared meaning.

The case of Fauvism is illustrative. Despite initial rejection for its vivid colors and simplified forms, its radical experimentation later proved deeply influential in the evolution of modern art \citepx{millard1976fauvism, hellmanzik2009artistic}. Such examples suggest that behind every lasting style lies a \textbf{belief}—a conviction that this particular form of expression is meaningful, valuable, or in some sense “right.”

In the following sections, we will explore the significance and necessity of style, proposing a hierarchical classification to organize different types of style and understand their relationship to decision-making contexts. Discussions of style innovation will be reserved for Chapter\~10, after a solid foundation on definition and manifestation has been laid.

\section{Why Do Styles Exist?}

Given our intuitive recognition of stylistic diversity in daily life, a fundamental question naturally arises: Why do we need styles, and why are they worth valuing?

A simple way to approach this is through negation: What would the world look like if styles did not exist? Style presupposes that multiple acceptable options are available—if every problem had only a single optimal solution, style would be irrelevant. In other words, \textbf{style emerges only when a situation admits more than one viable path to success}.

Once such alternatives exist, choice becomes possible, and those choices are filtered through an agent’s \textbf{beliefs and preferences}. A key driver is \textbf{perceived usability}, which combines functional adequacy with subjective satisfaction \citepx{norman2013design}. From this perspective, perceived usability is a projection of belief: a solution is deemed usable because it fits the agent’s expectations, needs, or values.

Belief-based validation need not rely on formal proof. Choices can stem from rational calculation, emotional resonance, cultural identification, or ingrained habit. As discussed in Chapter~1, \textbf{belief serves as a stopping condition for deliberation}: once a belief sufficiently supports a choice, action proceeds without further scrutiny.

Over time, repeated belief-driven selections shape what we recognize as \textbf{individual personality}. Different constellations of beliefs and values naturally yield different choices, and the accumulation of such differences manifests as stylistic divergence. This is evident across domains—from artistic schools and culinary traditions to martial arts and fashion—where styles endure not just through novelty but through their symbolic, emotional, or philosophical alignment with their adherents’ values.

Thus, \textbf{style is not arbitrary}; it represents each agent’s implicit answer to the deeper question: \textit{What is worth choosing?} In aesthetics, \citetx{goodman1978ways} framed stylistic systems as tools of “worldmaking,” shaping how meaning and reality are constructed. In sociology, \citetx{bourdieu1984distinction} showed that styles operate as markers of symbolic capital, embedding preferences within social structures \citepx{bennett2009culture}.

These insights lead directly to the next step: if styles arise from belief-driven choices among multiple valid options, how can we systematically describe their relationships? In the next section, we introduce two key dimensions—\textit{capacity} and \textit{popularity}—as structural measures for mapping and analyzing the space of styles.

\section{Dimensions and Types of Styles}

To understand and compare different styles more concretely, it is useful to represent stylistic variation within a conceptual space. Here, I propose two principal dimensions that shape the structure of this space: \textbf{Capacity} and \textbf{Popularity}. These dimensions allow us to position styles not as isolated categories, but as locations within a continuous space, offering a geometric intuition for their relationships and evolution.

The first dimension, \textbf{Capacity}, reflects the expressive richness or internal flexibility of a style. A high-capacity style permits a wide range of behaviors, configurations, or interpretations under its umbrella, while a low-capacity style enforces tighter restrictions. For instance, a general strategic archetype in a game such as “rush” or “control” may allow for diverse implementations, whereas a highly specific team composition or weapon build limits player decisions to a narrow path. From a set-theoretic perspective, styles can be seen as defined by constraint sets, and greater capacity corresponds to looser or fewer constraints.

The second dimension, \textbf{Popularity}, refers to how widely a style is recognized, adopted, or practiced by agents in a given environment. Styles with high popularity are broadly accepted and frequently used, while low-popularity styles remain niche or emergent. In multiplayer games, for example, a champion or strategy frequently picked in ranked matches reflects high popularity, whereas off-meta builds or unconventional tactics occupy the lower end of this axis. While popularity may seem straightforward to quantify by user counts or pick rates, it often reflects deeper social, historical, or cultural dynamics, and can vary significantly across player communities.

These two dimensions—\textit{how much expression a style permits} and \textit{how much a style is used}—together define a 2D space in which we can analyze and visualize the distribution of playstyles. This is not a ranking of superiority, but a structural mapping of stylistic variation.

Notably, this 2D structure aligns with the dual-loop framework proposed earlier. \textbf{Capacity} captures the extent to which a playstyle can manifest distinct observable outcomes in the external loop—i.e., whether differences in style are empirically detectable through interaction. In contrast, \textbf{Popularity} reflects how observers recognize or gravitate toward a playstyle, rooted in the internal loop of subjective interpretation and belief-driven salience.

Together, these two orthogonal dimensions are sufficient to partition the space of styles into meaningful quadrants, providing a compact yet expressive tool for comparative analysis. While more complex taxonomies may be possible, this framework captures a key philosophical and functional distinction: the difference between styles that \textit{are} different, and those that \textit{feel} different. This justifies why a two-dimensional formulation is not arbitrary, but grounded in the fundamental duality of perception and expression.

With these two dimensions established, we can begin to systematically compare different styles not in terms of superiority, but through their structural properties and functional roles. For example, when evaluating two gameplay archetypes—such as a highly flexible support role versus a narrowly defined burst-damage build—we can ask: Which allows more stylistic variation (capacity)? Which is more frequently adopted in real play (popularity)?

To further analyze such comparisons, we can interpret styles as constraint-defined subsets within a broader behavior space. Here, tools from set theory provide useful formal intuitions: a style that introduces additional restrictions over another (e.g., limited weapon types, fixed formation rules) can be seen as a proper subset of the more general style, and thus possesses lower capacity. This allows us to infer relative expressiveness without needing to quantify absolute style space volume.

Popularity, by contrast, is often more straightforward to approximate—via pick rates, user preferences, or adoption curves—but may conceal complex internal diversity. Two styles with equal player counts may differ drastically in geographic distribution, skill level, or subcultural context. These comparative ambiguities are addressed in later chapters on style balancing and population diversity.

\begin{figure}[ht]
\centering
\includegraphics[width=\textwidth]{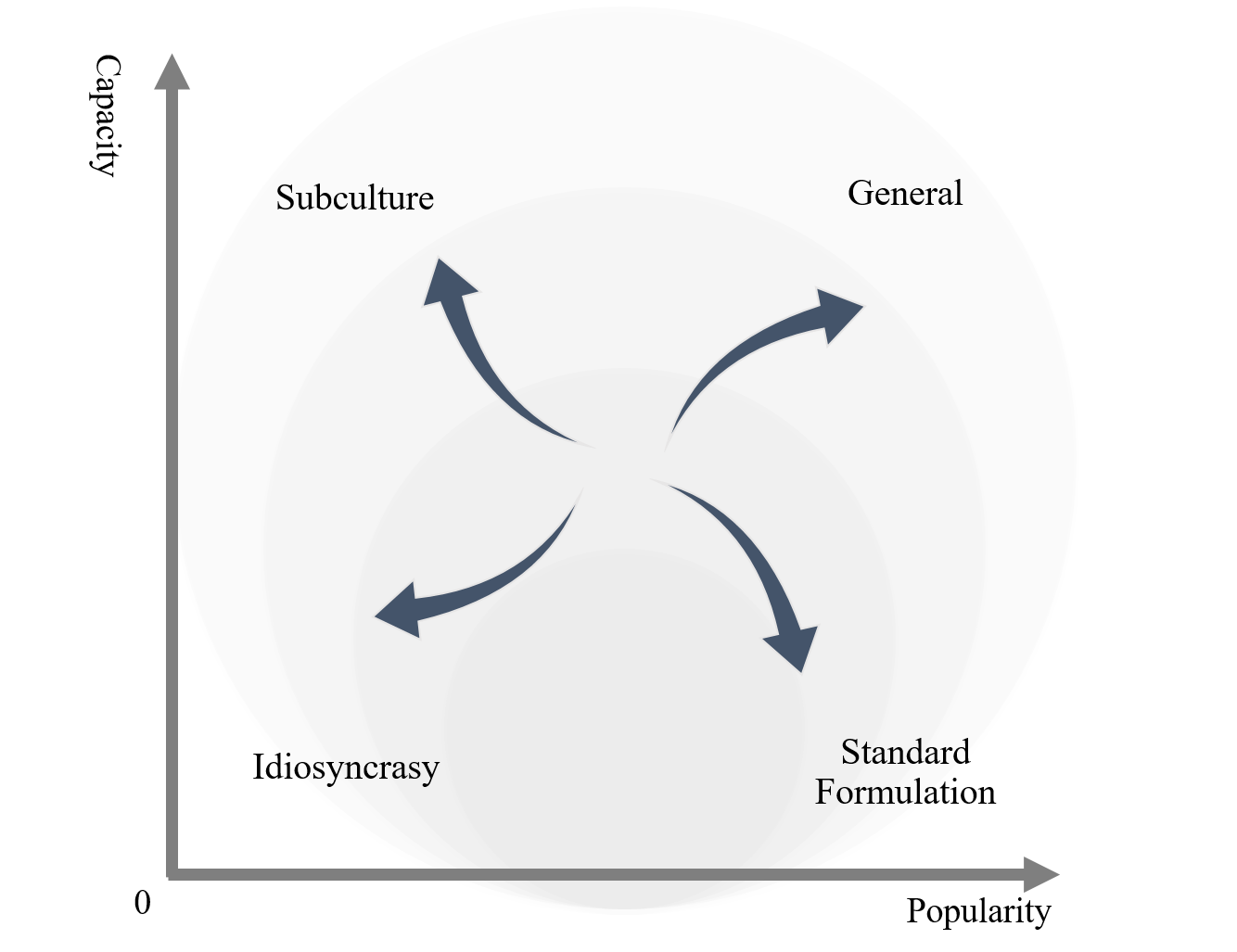}
\caption[Style Dimensions and Types]{A conceptual style map based on Capacity and Popularity. Each vector denotes a stylistic shift or transformation path.}
\label{fig:capacity_popularity}
\end{figure}

\begin{itemize}
    \item \textbf{General} (high capacity, high popularity): These styles are highly expressive and widely accepted. They offer broad flexibility, allowing players to pursue diverse strategies or expressions within a widely recognized framework. \\
    \textit{Example:} Minecraft \\
Minecraft is an open-world sandbox game where players build, explore, survive, and create freely. Its design accommodates a vast range of playstyles—creative, competitive, exploratory, or narrative-driven—making it a paradigm of stylistic openness with broad cultural acceptance.

    \item \textbf{Subculture} (high capacity, low popularity): These styles allow deep strategic expression but are mainly appreciated within niche communities. They often require domain-specific knowledge or preferences to be appreciated. \\
    \textit{Example:} Turn-based Strategy Games (e.g., Fire Emblem, Into the Breach)\\
	These games emphasize spatial planning, turn sequencing, and unit synergy. Players adopt distinct tactical preferences (e.g., sacrificial attacks, defensive walls, hit-and-run tactics), but the style tends to be less accessible to casual audiences due to its cognitive complexity and pacing.

    \item \textbf{Idiosyncrasy} (low capacity, low popularity): These styles are highly constrained and rarely adopted. They reflect extreme or unusual preferences, often requiring high skill, patience, or a taste for unconventional challenges. \\
    \textit{Example:} I Wanna Be The Guy\\
	This cult-classic platformer is infamous for its punishing design, unpredictable traps, and frequent restarts. It appeals to a small community that embraces pain as part of the experience. The resulting style—relentless trial-and-error perseverance—is unique but largely inaccessible to most players.

    \item \textbf{Standard Formulation} (low capacity, high popularity): These styles are simple, stable, and widely used. Despite their limited expressive space, they gain popularity through standardization, efficiency, or cultural convention.\\
    \textit{Example:} Sudoku\\
	Sudoku is a logic-based number puzzle with fixed rules and minimal stylistic room. Its predictability and clear evaluative criteria make it a universal form of structured reasoning, valued more for its clarity than personal expression.
\end{itemize}

This classification framework does not imply value judgment. Each region of the style space has functional significance—some styles act as scaffolds for creativity, others as stabilizers of coordination. Understanding where a style lies helps clarify its affordances, social dynamics, and evolutionary potential.

\textit{Crucially, this conceptual framing is not merely philosophical—it provides the foundation for a computational approach to style.} In Chapter~\ref{ch:discrete_measurement}, we show how \textbf{capacity} and \textbf{popularity} can be integrated into learning objectives to measure and generate stylistic diversity. This reinforces our central claim: style is not just an aesthetic phenomenon, but a structurally analyzable and technically actionable construct.

\section{Styles in Generative Applications}

Having established a conceptual framework for understanding style in terms of \textbf{capacity} and \textbf{popularity}, we now extend this analysis to styles that emerge not from human agents, but from artificial intelligence systems. Can such \textit{machine-generated styles} be interpreted using the same dimensions? Do these styles merely replicate patterns from training data, or can they reflect deeper structures—preferences, constraints, or even a primitive form of intention?

To explore these questions, we introduce the notion of a \textbf{style boundary}: a conceptual region that defines the range of variation a system considers stylistically valid. This boundary is not a fixed formula, but a flexible zone of plausible outputs—shaped by training distributions, architectural priors, and emergent dynamics.

\medskip

Early generative models such as Variational Autoencoders (VAEs) \citepx{vae} aimed to reconstruct inputs via latent feature spaces. Though their sampling enabled limited interpolation, their outputs remained largely tethered to known data distributions. Generative Adversarial Networks (GANs) \citepx{gan}, in contrast, introduced a competitive learning dynamic in which a generator produces outputs to fool a discriminator, which in turn learns to distinguish real from fake \citepx{conditional_gan, wgan}.

We reinterpret GANs as implicitly learning a \textbf{style boundary}—a flexible frontier of plausibility in high-dimensional space. Any sample within this frontier, even if novel, is treated as stylistically coherent. This reframing allows us to analyze not only what AI systems generate, but \textit{where} in the style space their generations reside.

Later models such as InfoGAN and StyleGAN \citepx{info_gan, style_gan} expanded this boundary through disentangled latent factors and hierarchical control, enabling fine-grained exploration of stylistic variation. Users gained greater access to distinct stylistic dimensions, suggesting that style boundaries are not only learned, but modular and navigable.

Style transfer and domain translation models like UNIT \citepx{unit} and CycleGAN \citepx{cycle_gan} can be interpreted as learning \textbf{mappings between style boundaries}, even in the absence of paired data. These mappings suggest that boundaries are transformable and that style space itself possesses a latent geometry amenable to structured traversal.

Diffusion models \citepx{diffusion} offer yet another perspective: by iteratively denoising random inputs, they trace smooth trajectories through style space—preserving local coherence while allowing for flexible movement. Models such as Stable Diffusion \citepx{stable_diffusion} demonstrate how prompt-based conditioning can steer generation along culturally meaningful axes of style.

Large-scale pre-trained models like GPT \citepx{gpt1, gpt3, rlhf}, DALL-E \citepx{dalle1, dalle3}, and Midjourney also operate within learned stylistic boundaries shaped by massive corpora of human-authored data. Their outputs align with culturally accepted patterns, suggesting they navigate high-popularity regions of style space. Cross-modal interfaces like CLIP \citepx{clip} and Imagen \citepx{imagen} further allow users to probe and traverse these spaces through language.

This leads to a central insight: the \textbf{capacity} of a generative system corresponds to the breadth and granularity of its style boundary, while \textbf{popularity} reflects how the system’s outputs resonate socially or culturally. For example, “Ghibli-style” visuals have gained traction not merely due to model capability, but because of communal amplification—a feedback loop between generative capacity and social reception.

In contrast, rule-based \textbf{Procedural Content Generation (PCG)} \citepx{pcg} represents a more rigid form of stylistic design. Here, boundaries are manually engineered through templates and constraints, offering precision but limited expressive emergence. PCG, in this view, is \textit{explicit boundary crafting}, whereas learning-based models engage in \textit{implicit boundary discovery}.

\medskip

Taken together, these developments validate our analytic framework. The \textbf{style boundary} concretizes capacity as the expressive range of valid outputs, while popularity reflects patterns of cultural uptake. This unified lens allows us to interpret generative systems not as mere statistical mimics, but as active participants in the construction, negotiation, and circulation of style.

\medskip

Yet as generative systems exhibit growing diversity and novelty, deeper questions arise: Can machines possess stylistic preferences? Might their outputs reflect not only consistency, but \textbf{agent-like identity}? What constitutes \textit{intention} in such systems, and at what point do we attribute “style” not just to outputs, but to the models themselves?

These philosophical and technical questions will return in later chapters—particularly in discussions of playstyle measurement, imitation, and creative agency. For now, the capacity-popularity-style boundary framework provides our bridge between human stylistic expression and the emergent behaviors of generative AI.

\chapter{Decision-Making with Style}
\noindent\textbf{Key Question of the Chapter:} \
\textit{What exactly is the style of decision-making?}

\noindent\textbf{Brief Answer:}  
In a given decision-making environment dynamic, different preferences correspond to different utility functions, thereby shaping the outcomes of decision-making.
When these preferences remain consistent throughout the decision process—meaning that, from the agent’s perspective in each encountered state, the chosen action continues to align with its own evaluation—this consistency of decision-making constitutes a decision-making style.
In Chapter~5, we will formalize this property mathematically.

\section{Goals and Decision Outcomes}

Building on the previous chapter's set-theoretic framework for understanding style, we now turn our attention to a specific class of styles—those that arise in the process of decision-making. Unlike the vivid and immediately perceptible styles of painting, music, or vocal performance, decision-making styles are often subtle and indirect. They do not reside in observable forms, but instead manifest through the choices an agent makes and the resulting outcomes over time.

At the core of this perspective is the recognition that every style must ultimately relate to some \textbf{goal or problem}. That is, style does not emerge in a vacuum; it arises from variation in how agents approach decision problems. Given that different agents hold different preferences, beliefs, or heuristics, there will naturally be multiple recognizable ways to respond to the same challenge—thus giving rise to stylistic variation.

Formally, a decision-making scenario involves an agent selecting an action given a particular state. At any given moment, the only variable the agent can directly control is its choice. While the state may encode traces of the agent’s previous decisions (and hence its style), it is itself fixed at the moment of decision. In sequential decision problems, the agent can only influence future states indirectly through its actions.

This leads to a foundational question: \textit{What is the objective of decision-making?} In artificial intelligence, this is typically modeled through the maximization of a \textbf{utility function}—a mathematical representation of the agent's preferences over possible outcomes \citepx{russell2020aima}. However, utility functions are often unknown, nontrivial to define, or computationally expensive to evaluate in practice. In many domains—such as search, planning, or game-playing—agents instead rely on an \textbf{evaluation function}, which approximates the utility of intermediate states and guides decision-making under limited resources.

Designing effective evaluation functions is both critical and challenging. Even if a search or planning algorithm is theoretically optimal, it cannot outperform the quality of the signal it receives from its evaluation function. This issue is starkly illustrated by the historical contrast between AI performance in chess and Go.

In chess, several well-defined and quantifiable features—such as material advantage, positional control, and piece activity—enabled the construction of relatively accurate handcrafted evaluation functions. Coupled with high-speed search algorithms like alpha-beta pruning, these evaluations allowed systems such as Deep Blue to achieve superhuman performance, culminating in its victory over the world champion in 1997 \citepx{campbell2002deep}.

By contrast, the game of Go posed a far greater challenge. The high branching factor, long-term dependencies, and abstract value of territory made it extremely difficult to hand-engineer evaluation functions that could reliably assess board positions. For decades, even amateur-level human players easily outperformed Go programs.

This gap was only bridged with the advent of \textbf{Monte Carlo Tree Search (MCTS)} \citepx{mcts}, a paradigm shift in game AI. Rather than relying on explicit evaluation formulas, MCTS estimates state values through randomized playouts—simulations that project the likely outcomes of different actions. This approach was further refined by the \textbf{Upper Confidence Bound applied to Trees (UCT)} algorithm \citepx{uct}, which integrated exploration–exploitation tradeoffs from multi-armed bandit theory. These innovations enabled systems to evaluate decisions not by static heuristics, but by dynamic statistical rollouts—even in domains with deep, abstract structure such as Go \citepx{uct_go}.

Parallel to these advances, researchers also sought to quantify decision-making styles in Go using empirical methods. \citetx{go_mm_elo}, for instance, proposed learning Elo-like ratings over local move patterns, effectively modeling decision effectiveness as a function of stylistic motifs. This bridged the divide between handcrafted expert features and data-driven learning, enabling models to interpret style as both descriptive and predictive.

A major breakthrough came with \textbf{AlphaGo} \citepx{alpha_go}, which integrated deep neural networks into both move selection and value estimation. Trained on expert human games and coupled with MCTS, AlphaGo surpassed the strongest human players. Its successor, \textbf{AlphaGo Zero} \citepx{alpha_go_zero}, dispensed with human data altogether—learning purely from self-play and reinforcement learning, it constructed its own evaluation function from scratch and discovered superhuman strategies inaccessible to human intuition.

These developments reveal a crucial truth: the \textbf{utility function}—or more precisely, its implicit estimation—is central to intelligent behavior. Yet in most real-world domains beyond games, agents do not have access to well-defined reward structures or clear victory conditions. How, then, can agents construct evaluation functions aligned with meaningful goals?

Here, the conceptual groundwork laid in earlier chapters offers guidance. We argued that value stems from belief, and that decision-making is the process of validating beliefs through behavior. From this perspective, \textbf{preferences} emerge as ordered beliefs—structured commitments to one outcome over another. These preferences form the latent substrate from which an internal evaluation function can be inferred.

Indeed, recent developments in AI have begun to operationalize this view. \textbf{Reinforcement Learning from Human Feedback (RLHF)} \citepx{rlhf}, for example, allows agents to learn by observing ranked preferences or trajectory comparisons provided by human evaluators. Rather than relying on hand-specified rewards, agents align their behavior with human value structures—effectively learning an evaluation function from social feedback.

In this light, decision-making style is not just a pattern of choice—it is a signature of internal structure. It encodes what an agent values, how it estimates success, and which trade-offs it is willing to accept. To analyze style is to reverse-engineer these implicit architectures.

In the sections that follow, we discuss this notion and explore how decision-making styles can be modeled, measured, and learned—particularly within artificial agents. Through this, we aim to show that style is more than behavior: it is a window into the agent’s evaluative core, a trace of belief expressed through choice.

\section{Playstyle as a Decision Style}

In this work, I refer to decision-making style as \textit{Playstyle}, not merely because the discussion is situated within the domain of video games, but because the term \textit{Play} emphasizes intentionality, agency, and the freedom to choose. This framing directs our attention to an agent’s active engagement with alternatives, situating playstyle as a reflection of how decisions are formed and expressed within interactive environments.

To clarify this focus, consider a video game context: if changes in background music, graphical aesthetics, or atmospheric design do not influence a player’s actual choices, then these elements—while contributing to the overall experience—do not constitute playstyle \citepx{playstyle_distance}. Our analysis centers strictly on those features that shape the decision-making process itself.

This perspective aligns with foundational thinking in game design. In \textit{The Art of Game Design} \citepx{art_of_game_design}, games are described as deeply versatile media, capable of encompassing nearly every form of human activity. From a broader philosophical lens, one might even adopt the view that reality resembles a structured simulation \citepx{bostrom2003are}, in which case understanding playstyle becomes synonymous with understanding how agents act, choose, and express values in any structured decision environment.

\medskip

\noindent\textbf{Examples of Playstyle in Action.}
Consider a few canonical examples:

\begin{itemize}
    \item In chess: Should I begin with a pawn to e4 or a knight to f3? How do I adapt after the opponent's first move?
    \item In Rock-Paper-Scissors: Based on past plays, what should I throw next to outguess my opponent?
    \item In poker: Should I raise, call, or fold, given my hand and the betting context?
\end{itemize}

In each case, multiple viable decisions are available. The \textit{Playstyle} reflects how the agent navigates this decision space, exhibiting tendencies such as aggression, caution, unpredictability, or pattern exploitation.

\medskip

\noindent\textbf{Distinguishing Playstyle from General Style.}
While this decision-centric definition is intuitive, it risks being overly inclusive. Without constraints, even a generative image model that chooses pixel values one-by-one might be seen as exhibiting playstyle. To avoid such dilution, we introduce an additional criterion: \textbf{Playstyle must involve interaction with an external entity}—an environment, another agent, or a structured system capable of providing feedback.

That is, the agent must:
\begin{enumerate}
    \item make a decision, and
    \item receive a response from a system not fully under its control.
\end{enumerate}

The agent’s behavior is thus evaluated externally; the meaning and consequence of its decision are interpreted by the world. This separates playstyle from purely internal or generative styles, which may reflect aesthetic or structural preferences, but lack environmental consequence or feedback.

\medskip

\noindent\textbf{Core Definition.}
We may now define:
\begin{quote}
\textit{Playstyle is the decision-making style of an agent, expressed through its interaction with a responsive environment, where multiple viable choices exist and external consequences provide structure and meaning to the agent's actions.}
\end{quote}

\noindent\textbf{Consistency and Evolution.}  
In this work, we treat the agent as the principal of playstyle formation. Internally, the agent defines its \textbf{beliefs} and \textbf{preferences}, which together shape its decision-making process. We make a working assumption that when an agent exhibits a given playstyle, there is \textit{temporal consistency}—that is, each decision the agent makes is one it would still endorse if it could re-evaluate the same situation from its current state of knowledge.

In reality, only the agent itself can know whether this consistency is upheld. Playstyle can also be \textit{dynamically evolving}: as the agent expands its cognition or acquires new information, it may adjust its beliefs and preferences. As long as it continues to endorse its past decisions—even when viewed through its updated perspective—we may still consider this to be the same playstyle.

This property directly influences how playstyles evolve. When an agent attempts to repudiate previously held beliefs or preferences, it breaks this consistency, often incurring a \textbf{psychological cost} in changing style. In human behavior, this cost is observable in the cognitive rigidity of belief revision: younger individuals tend to adapt more easily because they have fewer entrenched decisions to overturn, whereas older individuals may find it harder to accept changes that imply rejecting long-trusted values.

\medskip
Thus, playstyle captures not only the agent’s expressive pattern in interactive decision-making, but also the \textit{continuity or rupture} in how its belief–preference structure is maintained or revised over time.

\section{Importance and Applications}

Given that a playstyle reflects an agent’s belief–preference structure and its consistent expression across decision points, a natural question follows: \textit{Why does playstyle matter?} Beyond surface-level variation, what is its practical and philosophical significance?

\subsection{In Games: Expression, Diversity, and Engagement}

Consider the realm of competitive games—such as esports strategy, character builds in role-playing games (RPGs), or tactics in real-time combat. Imagine a game in which a universally optimal solution is known and strictly dominates all others. In such a setting, meaningful decision-making vanishes. The validation of internal beliefs through action collapses, and the space for expressing individual values and preferences disappears \citepx{player_modeling}. Even when players retain agency in name, the game reduces to mere execution.

A canonical example is unrestricted Gomoku. Once the perfect black-piece winning strategy is discovered, the game may still hold initial interest as players test and verify its logic. But over time, play devolves into repetition—showcasing only brute-force precision or memorization. The opportunity for exploration, stylistic divergence, or personal innovation fades. In such cases, the act of choosing ceases to be expressive—decisions no longer reflect beliefs, only compliance.  
By contrast, games such as chess or Go, while theoretically solvable, remain intractable for human computation, preserving vast spaces for stylistic difference. These domains allow strategic preferences to flourish even under optimal-play theory.

Games that support multiple viable playstyles offer a richer and more enduring experience. These environments empower agents—human or artificial—to validate internal values through behavior. One could argue that the very act of play becomes meaningful only when such expressive space exists. This view aligns with research emphasizing that behavioral diversity enhances exploration, adaptability, and engagement \citepx{diayn, policy_similarity_metric, trajectory_diversity}.

\subsection{In Game Development and AI Systems}

Game designers have long recognized the importance of stylistic variation. From the earliest console generations, a proliferation of genres—platformers, strategy, stealth, simulation—emerged to satisfy diverse player preferences \citepx{game_develop_essential}. Over time, technologies such as dynamic AI difficulty, adaptive NPC behavior, and flexible matchmaking further embraced the value of stylistic variation \citepx{art_of_game_design}.

Crucially, the ability to offer distinct playstyles supports not only engagement but also identity formation. Players are drawn to signature ways of playing—not simply for efficiency, but to reflect who they are. This diversity underpins a global industry approaching 200 billion USD in valuation—making it comparable in size to the film or music sectors \citepx{gaming_industry_report_2025}. These practical implications will be revisited in Chapter~7 and Chapter~11, where playstyle modeling supports personalization, player profiling, and balanced game ecosystems.

\subsection{In Broader Decision-Making Contexts}

Beyond gaming, the importance of playstyle extends to general AI. As discussed in Chapter~1, classical paradigms like “acting rationally” and “acting humanly” can be reframed through the lens of playstyle. When agents face multiple valid decisions under uncertainty, the pattern of their choices reflects underlying beliefs, heuristics, and values—that is, their decision-making style.

In reinforcement learning, stylistic diversity across policies is increasingly recognized as essential for generalization, robustness, and emergent capabilities \citepx{gdi, diayn, agent57, lbc}. Multi-agent systems benefit from stylistic variety to promote exploration and adaptive synergy. RL frameworks like \textbf{Reinforcement Learning from Human Feedback (RLHF)} \citepx{rlhf} rely precisely on subtle preference modeling—capturing the playstyle of human evaluators as a guiding signal for alignment.

Moreover, playstyle modeling is central to human-centered AI. In recommendation engines, dynamic pricing, adaptive interfaces, and personalized content generation, systems must infer and respond to individual behavioral tendencies. Modeling playstyle enables systems to resonate with human users—not just functionally, but meaningfully.

Emerging fields such as behavioral stylometry provide additional validation. For instance, researchers have shown that patterns in chess moves can reveal cognitive traits and even psychological dispositions \citepx{chess_style}. Such approaches can be naturally connected with the \textbf{capacity} and \textbf{popularity} framework introduced in Chapter~2, extending style analysis into non-game domains while preserving its structural interpretability.

\subsection{Conclusion}

In summary, playstyle is not merely an aesthetic label—it is a functional interface between belief, behavior, and identity. It reflects how agents, human or artificial, express value systems through interactive decision-making. By modeling how decisions encode structure, preference, and trade-off resolution, we gain a powerful tool for understanding cognition, enhancing interaction, and building systems that are not only intelligent, but individually expressive.

\chapter{Blueprint of Playstyle}
\noindent\textbf{Key Question of the Chapter:} \\
\textit{What are the key issues in playstyle worth exploring?} \\
\textbf{Brief Answer:} Playstyle research can be organized into three mutually reinforcing components—Definition and Measurement, Expression and Innovation, and Applications and Future—forming a blueprint for systematic investigation.

\section{An Overview of the Playstyle Components}

\begin{figure}[htbp]
    \centering
    \includegraphics[width=0.6\textwidth]{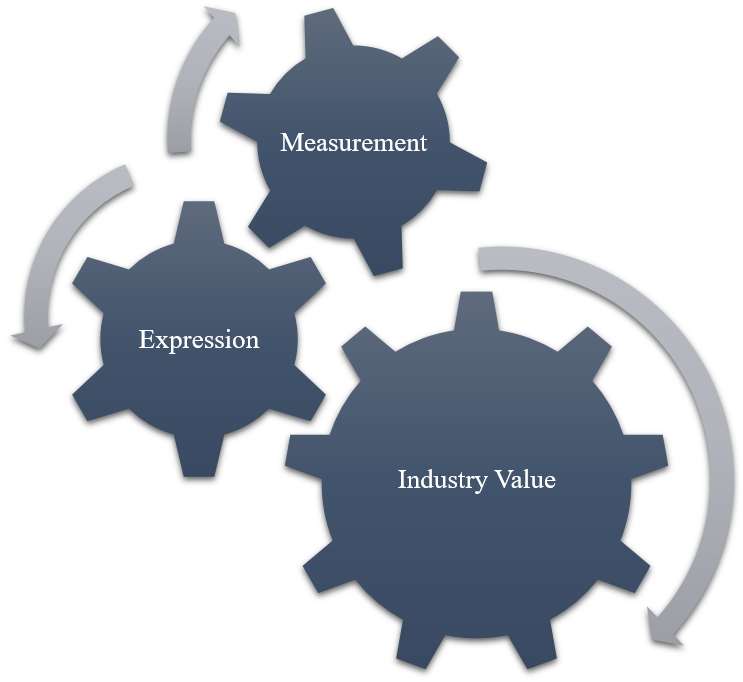}
    \caption[The interplay between the three key components of playstyle]{The interplay between the three key components of playstyle: Measurement, Expression, and Industry Value. Measurement enables Expression; Expression expands applications toward Industry Value; Industry Value, in turn, raises new questions that refine Measurement.}
    \label{fig:playstyle_framework}
\end{figure}

Before advancing further into the study of playstyle, it is essential to clarify the core components that structure this field and render it suitable for systematic, scientific investigation. This chapter introduces a conceptual \textbf{blueprint} that divides the inquiry into three tightly interconnected components:
\begin{enumerate}
    \item \textbf{Definition, Measurement, and Rationality}
    \item \textbf{Expression and Innovation through AI}
    \item \textbf{Applications and Future Outlook}
\end{enumerate}

\noindent Each component serves a distinct role, yet all mutually reinforce one another in shaping a unified research framework.

First, we establish the foundations of \textbf{definition and measurement}. While previous chapters approached playstyle from philosophical and intuitive perspectives, such approaches lack the precision required for scientific progress and technical application. This part develops formal tools to define and quantify style, enabling rigorous analysis. On this basis, key concerns such as diversity, balance, and rational decision-making naturally emerge. Later, in Chapter~\ref{ch:discrete_measurement}, we demonstrate how abstract concepts such as \textit{capacity} and \textit{popularity} can be embedded into loss functions for style-sensitive models. This concretely shows that our conceptual framework is not merely theoretical—it translates directly into practical modeling strategies.

Second, we turn to \textbf{expression and innovation}, focusing on how playstyle can be manifested and extended by intelligent systems, particularly AI agents. This includes not only reproducing known styles but also generating novel stylistic behaviors. The discussion naturally connects to deeper questions about creativity: What enables the emergence of new styles? How can we distinguish between variation and innovation? While the primary focus is on machine behavior, the analysis remains grounded in human analogs to better understand the boundaries of artificial creativity.

Third, we examine the \textbf{practical value and industrial applications} of playstyle. We begin with video games—the most immediate and intuitive domain—exploring how style modeling informs player behavior, AI design, community dynamics, and streaming culture. We then expand the discussion to other areas where preferences, styles, and individuality matter: recommendation systems, personalized content generation, adaptive pricing, and beyond.

These three pillars are mutually reinforcing. Robust definitions and measurement techniques are prerequisites for identifying and evaluating stylistic patterns. Without them, expressive modeling becomes arbitrary, and innovation lacks grounding. Likewise, industrial application provides a reality check, ensuring that our methods retain practical relevance and social utility. The feedback from real-world deployments also raises new research questions that loop back into refining measurement and theory.

In sum, this chapter provides a structural overview of the rest of this book. It functions as a bridge between the philosophical foundations of Chapter~1 and the formal, technical investigations to come. By organizing the study of playstyle around measurement, expression, and application, we aim to offer not only a coherent research agenda, but also a general framework for advancing this emerging field.

\section{Foundations of Definition, Measurement, and Rationality}

The first conceptual gear in the blueprint of playstyle research concerns the formalization of \textbf{definition}, the development of methods for \textbf{measurement}, and the articulation of a coherent framework for \textbf{rational evaluation}.

We begin with an intuitive premise: humans can effortlessly recognize stylistic differences—in art, music, gameplay, or behavior. This intuitive capacity forms the entry point for a more rigorous investigation. In this part, we examine how such recognition can extend beyond surface features to more latent, decision-based playstyles. We also draw a clear distinction between stylistic variation in generative outputs and playstyle as a form of intentional, interactive decision-making.

Next, we present a structured overview of existing methods for measuring playstyle. These range from heuristic, rule-based approaches to advanced data-driven models, culminating in a discrete-state modeling method introduced in this work. Each approach is analyzed based on its assumptions, expressive power, and applicability. By organizing these techniques into distinct categories, we aim to map the methodological landscape and clarify how playstyle can be operationalized across domains.

Building on these measurement foundations, we turn to two central but often vaguely defined concepts: \textbf{diversity} and \textbf{balance}. While both terms are widely used in game design and AI, they are rarely accompanied by formal definitions or actionable metrics. Here, we offer precise formulations: diversity is defined as the range of viable stylistic expressions within a system; balance is defined as the rational comparability and coexistence of such styles in a shared environment. These formulations support both theoretical analysis and practical control of style dynamics in competitive or collaborative systems.

Taken together, this part lays the critical conceptual and computational groundwork for what follows. Without clear definitions, the existence of playstyle cannot be verified; without measurement, it cannot be analyzed or improved; and without rational criteria, style comparisons become arbitrary. These foundations enable the study of playstyle as a scientific and design-oriented discipline, bridging philosophy, machine learning, and interactive systems.

\section{Foundations of AI Expression and Innovation}

Once the definition and measurement of playstyle have been established, the next logical step is to examine how playstyle can be expressed, manipulated, and extended—particularly through artificial intelligence.

We begin by revisiting the foundations of decision-making in AI. An intelligent system must first identify its goals and develop strategies to achieve them. The chapter explores the evolution of deep learning and reinforcement learning not merely as a chronological sequence of techniques, but as a reflection of the shifting motivations and design constraints that have shaped the field. Rather than presenting a catalog of methods, we focus on the broader trends that have guided the transition from passive models to adaptive agents, from perception to decision, and from task completion to style-sensitive behavior.

On this foundation, we analyze how AI systems can replicate human playstyles—whether through statistical modeling, trajectory imitation, or latent style inference from expert demonstrations. This naturally leads to the question of \textbf{human-likeness}: Can AI behave not only effectively, but with intention, expressiveness, or even personality? When does imitation become identity?

More provocatively, we explore whether AI can move beyond imitation to exhibit genuine \textbf{creativity}. While innovation is traditionally associated with human agents, we argue that creativity need not be limited to biological systems. Under appropriate conditions—namely, the presence of a structured style space, the capacity to traverse it meaningfully, and mechanisms for evaluating novelty—AI systems can generate stylistic outcomes that are not only novel but valuable. From this perspective, playstyle innovation becomes not only possible but \textit{measurable}.

This work takes a clear stance: AI is capable of creative expression, provided that we construct suitable evaluative frameworks. Indeed, current technologies are rapidly approaching the threshold at which AI-generated innovations are not merely plausible, but observable. By the time this dissertation is read, the question of whether AI can be creative may no longer be philosophical—it may be empirical.

In sum, this section positions playstyle at the frontier of computational creativity. It forges a conceptual bridge between decision-making models and stylistic innovation, presenting playstyle as something not merely reproduced or recognized, but actively \textit{generated}—by both humans and machines.

\section{Foundations of Applications and Future}

The final part of this blueprint shifts from theoretical foundations to practical implementation and long-term vision, focusing on the \textbf{applications of playstyle technologies} and their role in shaping future interactions between humans and intelligent systems.

We begin with the domain most inherently aligned with playstyle research: video games. This includes applications in game design, matchmaking, player profiling, community behavior analysis, and livestreaming culture. Across these contexts, playstyle analysis proves valuable not only for enhancing player experience but also for improving system-level balance and engagement. By translating subjective behavioral patterns into actionable design signals, playstyle becomes a powerful tool for both developers and players.

Yet the relevance of playstyle extends well beyond the boundaries of games. Wherever decision-making patterns, stylistic preferences, or interactive behaviors hold significance: education, robotics, creative industries, marketing, consumer analytics—the principles of playstyle modeling can be meaningfully applied. When treated as a structured expression of intention and preference, style becomes a medium for interpreting and influencing behavior across diverse domains.

In this light, playstyle is more than a diagnostic lens—it emerges as a design principle and communication interface. This section explores how playstyle-aware models can be embedded within industrial and technological pipelines to enable personalization, fairness, adaptability, and user-centric optimization. Moreover, as these systems are deployed and adopted, they offer feedback loops that refine our conceptual understanding of behavior and identity, deepening the integration of theory and practice.

Looking further ahead, we return to the philosophical foundations established earlier in this work. As artificial intelligence systems become more sophisticated—capable not only of solving tasks, but of exhibiting consistent and recognizable behavioral patterns—questions of identity, agency, and value inevitably resurface. What does it mean for a machine to have a style? Can playstyle function as a proxy for intention or even a rudimentary selfhood? In the context of Artificial General Intelligence (AGI), might playstyle offer a new language for describing digital individuality?

This final section thus reframes playstyle not only as a subject of academic and engineering interest, but as a conceptual lens for interpreting the evolution of intelligent systems. It bridges abstract constructs, real-world applications, and metaphysical inquiry—completing the arc of this blueprint and setting the stage for the deeper explorations in the chapters that follow. In particular, Chapter~5 begins by detailing the definitional and measurement foundations, providing the formal tools required to operationalize playstyle analysis across domains.

\newpage

\part{Essence of Playstyle}
\chapter{Aspects of Defining Playstyle}
\noindent\textbf{Key Question of the Chapter:} \\
\textit{How can playstyle be formally defined?} \\
\textbf{Brief Answer:} We first establish how to demonstrate the existence of playstyle differences, then formalize the concept of \textit{playstyle consistency} in mathematical terms. Building on this, we construct equations linking intention and behavior, thereby creating the mathematical language that underpins subsequent methods for measuring and expressing playstyle.

\bigskip

What exactly constitutes a playstyle, and how can it be defined in a way that is both conceptually meaningful and empirically measurable? This chapter develops a formal framework for understanding playstyle, from conceptual foundations to measurable structures, and establishes the theoretical distinction between playstyle and generative style.

Playstyle is often recognized intuitively: we can sense when two players approach the same game in fundamentally different ways. Yet this intuition masks a deeper complexity. Playstyle is not encoded in static rules, nor fully captured by surface behavior. Rather, it emerges as a relational construct—an expression of internal preferences, intentions, and strategies, made visible through interactions with an environment.

To study playstyle scientifically, we must first ask what it is that we seek to define and measure. Do we locate playstyle in the agent's internal beliefs and value structures? In the policy it follows? Or in the empirical distribution of its actions across observable states?

This chapter addresses these foundational questions through a structured progression. We begin by clarifying the ontological status of playstyle: under what conditions can we say that a style exists, and how can we distinguish one from another? We then introduce two complementary perspectives—\textit{process-based} and \textit{outcome-based}—as formal lenses for measuring stylistic divergence.

Building upon these perspectives, we identify the set of observable elements that serve as inputs to playstyle measurement. These include not only the agent's actions, but also states, rewards, preferences, and full behavioral trajectories. Recognizing and unifying these inputs allows us to define a general observable space for systematic comparison.

Finally, we construct a formal framework in which internal belief systems and preference structures generate decision policies, which in turn induce distributions over observable behavior. While this formulation is inspired by the Partially Observable Markov Decision Process (POMDP) and belief MDP framework \citepx{kaelbling1998planning, russell2020aima}, our treatment of belief deviates from traditional assumptions. Rather than directly specifying a fixed belief state space, we allow belief to be realized as a latent internal structure—constructed through agent-specific mechanisms such as an interpretation function $\mathcal{I}(o)$ from the given observation $o$—and leave the implementation of abstract preference functions through belief states to the agent.

We close the chapter by examining the boundary between playstyle—as a reflection of decision-making—and generative style—as a product of content creation. This distinction is critical for establishing the scope of playstyle analysis and for positioning our formalism within broader discussions of creativity and expression.

In sum, this chapter transitions our discussion from philosophical intuition to mathematical formulation. It provides the conceptual and technical foundation required for the systematic study of playstyle as a measurable construct.

\section{Existence of Playstyle}

As discussed in Chapter~3, playstyle can be understood as a form of decision-making style—an emergent pattern rooted in how agents make choices across interactive contexts. At its core, each decision reflects an underlying \textit{utility function}, which in turn is shaped by the agent’s \textit{beliefs} about the world and its \textit{preferences} over possible outcomes. From this perspective, a playstyle is instantiated by a particular \textbf{belief–preference system} that governs how an agent evaluates alternatives and commits to actions.

A common intuition is to bind a style directly to its \textbf{operator}: for example, assuming that one person corresponds to one style, or that one complete strategy represents one style. However, the subject of a style need not be fixed in this way. A single person can maintain multiple styles and freely switch between them; multiple individuals can share the same style; and styles can even be blended into hybrid forms. What is essential is that the operator is \textit{aware} of switching or merging styles, such that these transitions are intentional rather than accidental.

Because the belief–preference system is latent—hidden within the agent’s internal mechanisms—it cannot be directly observed unless the full decision-making process is transparent. In practice, this is rarely the case. We must therefore reason about playstyle primarily from \textbf{external evidence}: observable traces such as action choices, state transitions, and eventual outcomes. It is through these manifestations that we attempt to infer the structure and stability of an agent’s goals and decision logic.

This indirect relationship reveals a critical insight: \textbf{differences in playstyle are not necessarily equivalent to differences in policy}. Even if two agents possess divergent beliefs and preferences, their resulting policies may appear identical if the decision space is severely constrained. For instance, in high-risk or survival-critical scenarios within adventure games, agents with different intentions may all converge on the same “optimal” response due to environmental pressure. In such cases, the absence of observable difference does not imply the absence of internal variation—it merely reflects limited behavioral expressiveness.

Conversely, even when agents share the same beliefs and utility structures, differences in training dynamics, architecture, or exploration strategies may result in distinct policies. For example, in deep reinforcement learning, agents trained with the same reward function and algorithm can produce diverging behaviors due to random initialization, sampling variance, or stochasticity in policy updates \citepx{playstyle_distance}. Even when value functions (e.g., $Q$-functions) are held constant, the choice of deterministic versus stochastic policy realization can lead to stylistic divergence in execution.

These examples reveal a many-to-many mapping between belief systems, utility functions, and observable policies. As a result, playstyle cannot be defined solely through policy representation. Instead, we are compelled to adopt a relational and evidential stance: \textit{playstyle exists insofar as structured, consistent variation in decision behavior can be observed across agents under comparable conditions}.

Such an evidential perspective does not require complete observability or full causal inference of beliefs. It instead demands sufficient regularity to distinguish meaningful variation from implementation noise, randomness, or environmental bias. Accordingly, the challenge of defining playstyle is not simply about isolating behaviors, but about discerning whether those behaviors reflect intentional, belief-driven differences rather than external constraints or stochastic fluctuations.

In the next section, we operationalize this evidential approach by introducing two principal strategies for identifying stylistic differences—\textbf{process-based} and \textbf{outcome-based}—while noting that these are not the only possible perspectives. This sets the stage for a broader discussion of how playstyle can be recognized, measured, and verified in both practical and theoretical settings.

\section{Process vs. Outcome: Perspectives on Defining Playstyle}

Since playstyle represents a form of indirect style, our task is to observe and quantify its differences through indirect yet systematic means. Two fundamental perspectives emerge for relative measurement: \textbf{process-based differences}, which reflect how decisions are made \textit{within} the decision-making process (thus ``internal'' to the agent), and \textbf{outcome-based differences}, which reflect the external results \textit{produced by} such decisions as observed by others (thus ``external'').

These perspectives hinge on the types of data available. When we can observe observation--action pairs, we may directly assess decision tendencies—how an agent acts under specific contexts. When such internal traces are absent, or only aggregate consequences are reported, we must rely on an outcome-based view.

\subsection{Process Difference (Internal Difference)}

The process-based perspective investigates playstyle through differences in decision-making behavior: how agents map observations to actions. If two agents receive the same input but choose different actions, this indicates a divergence in their underlying policies, beliefs, or preferences.

Let \( \pi_1(a \mid o) \) and \( \pi_2(a \mid o) \) be policies over the observation space \( \mathcal{O} \) and action space \( \mathcal{A} \). Then:

\[
\exists o \in \mathcal{O},\ \exists a \in \mathcal{A},\ \pi_1(a \mid o) \neq \pi_2(a \mid o)
\quad \Rightarrow \quad \text{evidence of process-based playstyle difference}.
\]

This perspective underpins early work such as Playstyle Distance \citepx{playstyle_distance}, which discretizes state spaces to compare action distributions, and Playstyle Similarity \citepx{playstyle_similarity}, which incorporates multiscale representations and perceptual kernels to enhance sensitivity.

Such divergence may arise from differences in internal representations, reward generalization, exploration behavior, or inductive bias. For example, agents trained on the same objective but with different seeds, architectures, or sampling schedules may yield distinct process-level tendencies.

Conversely, the absence of any divergence:

\[
\forall o \in \mathcal{O},\ \forall a \in \mathcal{A},\ \pi_1(a \mid o) = \pi_2(a \mid o)
\quad \Rightarrow \quad \text{no detectable process-based difference}.
\]

suggests that stylistic differences---if any---are not reflected in the decision traces, though latent factors may still differ.

In complex environments, true decision tendencies may be distorted by partial observability, multi-agent interaction, or structural constraints. Consider:

\begin{itemize}
    \item Environmental randomness causes the same playstyle to produce different observed states.
    \item Two distinct playstyles are forced into similar distributions by dominant external dynamics.
\end{itemize}

\subsection{Outcome Difference (External Difference)}

The outcome-based perspective focuses on \emph{the effects} of decision-making, without assuming access to its internal logic. Rather than comparing actions given inputs, we compare the summarized results of policy--environment interaction.

Formally, for an environment \( e \), let the report function \( e(\pi) \) produce a summary outcome \( r \sim e(\pi) \in \mathbb{R}^k \), such as win rate, score, user rating, or metric vector. We say:

\[
\mathbb{E}_{r \sim e(\pi_1)}[r] \neq \mathbb{E}_{r \sim e(\pi_2)}[r]
\quad \Rightarrow \quad \text{evidence of outcome-based playstyle difference}.
\]

This is especially useful when internal observation--action traces are unavailable. For example:

\begin{enumerate}
    \item In games, if two agents with different strategies show different win rates or score distributions, this suggests outcome-level stylistic divergence \citepx{agent_relay}.
    \item In preference-based evaluation, if human users systematically prefer one agent’s behavior across tasks, this too reflects a meaningful difference.
\end{enumerate}

Importantly, outcome difference requires no stepwise alignment, only that the environment produces interpretable summaries.

Nonetheless, outcome analysis faces ambiguity:

\begin{itemize}
    \item Agents may differ internally but appear similar due to low environment sensitivity.
    \item Outcome noise or complexity may exaggerate superficial differences.
\end{itemize}

Moreover, even when outcome vectors match, human observers often identify stylistic nuances. For example, in competitive games or sports, players with similar performance ratings may still differ in strategic pacing, risk preference, or tactic selection. This phenomenon is particularly evident in professional esports and board games such as Go, where individual playstyle is widely recognized despite comparable Elo ratings.\footnote{\url{https://www.goratings.org}}

\subsection{The Importance of Combining Both Perspectives}

Each perspective captures a distinct dimension:

\begin{itemize}
    \item \textbf{Process-based difference} reflects how agents make decisions given comparable contexts.
    \item \textbf{Outcome-based difference} reflects the resulting external consequences of those decisions.
\end{itemize}

Crucially, these two perspectives are not logically equivalent, and neither implies the other:

\begin{itemize}
    \item \textbf{Process difference does not imply outcome difference}: Two agents may exhibit clear differences in how they respond to observations—i.e., their policies differ at the decision level—yet still achieve similar outcomes. This can occur when the environment masks stylistic variation due to strong convergence pressures, sparse reward signals, or outcome measures that are too coarse to reflect nuanced behavioral differences. Moreover, limited evaluation episodes may fail to capture the long-term effect of process-level divergence.

    \item \textbf{Outcome difference does not imply process difference}: Two agents may achieve different outcomes despite using the same observable policy. This can happen when outcome variation is caused by factors outside the agent’s control, such as differences in starting state distributions, partner behaviors in multi-agent settings, or stochastic transitions. In such cases, the divergence arises from environmental conditions rather than differences in the decision-making process itself.
\end{itemize}

These observations motivate a combined perspective: by examining both decision-level and outcome-level evidence, we gain a more complete understanding of how playstyle manifests and under what conditions differences become observable. In particular, process-based comparison helps trace the generative source of behavior, while outcome-based comparison grounds that behavior in its external consequences.

In practice, most real-world assessments fall between the two extremes. Partial trajectories, state–action summaries, and aggregated results often yield hybrid measurements that reflect both decision tendencies and environment interactions. Understanding which elements of a playstyle difference stem from internal factors—and which result from external context—is key to accurate interpretation.

Robust analysis of playstyle thus requires integrating both perspectives, evaluating what types of signals are available, and remaining aware of the limits of each lens in isolation.

While the process–outcome framework provides two rigorous and widely applicable lenses, it is not exhaustive. In principle, any reliable source of evidence can be used to distinguish playstyles. For example, if an agent can be directly queried about whether it has changed its style—and we are prepared to take its self-report at face value—this constitutes a form of identification. Such approaches relax the need for formal verification and instead rely on trust, context, or domain-specific conventions.

This highlights an important distinction between \textbf{identification capability} and \textbf{verification capability}: the former concerns whether a difference in style can be recognized or labeled at all (even informally), while the latter demands evidential grounding that is reproducible and resistant to deception or noise. In high-stakes or scientific settings, verification is essential; in exploratory, narrative, or creative contexts, identification may suffice.

Moreover, when agents can intentionally maintain, switch, or merge multiple playstyles, the boundary between distinct styles may itself be fluid. In such cases, process-based and outcome-based signals may intermittently overlap, diverge, or blend, depending on when and how the agent transitions between styles. Understanding and accommodating this fluidity is part of building a robust theory of playstyle.

In short, process and outcome differences offer strong foundations for measurement, but practical playstyle analysis often requires a broader toolkit—capable of incorporating hybrid signals, self-reports, and context-specific interpretations—while remaining explicit about the evidential standard being applied.

\section{Available Elements for Defining Playstyle}

To define and measure playstyle, we must begin by asking: \textit{What can we observe?} Since playstyle is not explicitly encoded in any parameter or rule, its analysis must rely on empirical evidence derived from agent--environment interaction. These signals—whether concrete or abstract, process- or outcome-oriented—constitute the foundation for all inference about style.

To unify these diverse sources under a single abstraction, we introduce the term \textbf{Observable Element}. An Observable Element refers to any information that can be accessed, recorded, or inferred during or after agent execution. This includes not only traditional components such as states and actions \citepx{rl_book}, but also rewards, preference signals, trajectories, behavior summaries, or terminal outcomes.

This abstraction helps avoid premature assumptions about what “counts” as observable. In practice, available signals vary by domain: some environments expose rich state variables but no reward; others log preferences or episode-level reports but hide low-level transitions. By adopting a flexible definition, we permit any accessible evidence to inform stylistic analysis.

We classify Observable Elements into two broad categories:

\subsection{External Observables}

These are directly accessible signals produced through agent--environment interaction. They form the basis of most empirical comparisons.

\begin{enumerate}
    \item \textbf{State (Observation)}: \\
    The environment's current condition at each timestep. Typically encoded as structured or high-dimensional vectors, these define the decision context under which the agent operates.

    \item \textbf{Action}: \\
    The agent’s decision in response to the current state. Action logs are essential for recovering or approximating policies and process-based differences.

    \item \textbf{Reward or Preference Signal}: \\
    Evaluative feedback from the environment or an external source. Rewards are often numeric and task-driven; preferences may be human-derived, as in inverse RL or RLHF \citepx{irl_theory, rlhf}.

    \item \textbf{Trajectory Fragments or Summaries}: \\
    Sequences or compressions of behavior across time, capturing temporal structure such as strategic planning or stylistic rhythm. These are useful when individual decisions are uninformative in isolation.

    \item \textbf{Terminal and Semantic Signals}: \\
    Episode-level outcomes or high-level indicators, such as win/loss flags, achievement markers, or annotated goals. These serve as interpretable anchors for outcome-based analysis.
\end{enumerate}

These elements form the primary input to empirical methods for playstyle assessment. Their usefulness depends on how well they reflect the agent’s meaningful behavioral structure and variation.

\subsection{Internal Factors}

While Observable Elements are externally accessible, \textbf{internal factors} refer to the latent variables and mechanisms that shape an agent’s behavior but are not directly visible. They influence which Observable Elements are generated but cannot be measured directly.

\begin{enumerate}
    \item \textbf{Memory and History Dependence}: \\
    Agents with memory (e.g., RNNs, LSTMs) may condition actions on long histories rather than immediate observations, enabling stylistic patterns that unfold over time \citepx{lstm, gru}.

    \item \textbf{Latent Objectives or Biases}: \\
    Agents may act according to implicit tendencies—such as risk aversion, novelty seeking, or safety prioritization—that deviate from reward maximization \citepx{iqn}. These give rise to behavioral variation even under identical external incentives.

    \item \textbf{Modeling Formulation}: \\
    Differences in algorithm design (e.g., stochastic vs. deterministic, model-based vs. model-free) constrain the types of policies and stylistic expressions an agent can exhibit.

    \item \textbf{Parametric Settings}: \\
    Training configurations such as learning rate, regularization, exploration noise, or loss weighting can cause systematic behavioral divergence across agents, even with the same task setup.
\end{enumerate}

Although internal factors are not observable per se, they shape the structure of playstyle and are crucial for explaining behavioral differences that are not evident in surface-level metrics alone. When internal factors are hypothesized but not directly observable, \textbf{verification capability} becomes critical to avoid misattribution of style differences to noise or external constraints.

\subsection{Summary}

Observable Elements define the empirical boundary of playstyle measurement. They range from low-level state--action pairs to high-level outcome summaries. Internal factors, while unobservable, play an essential generative role. Notably, a single agent may exhibit multiple playstyles by selectively activating different subsets or patterns of Observable Elements, enabling intentional switching or fusion of styles. A comprehensive understanding of playstyle must therefore combine what can be seen with informed hypotheses about what cannot. This unified framing prepares the ground for the formal model of playstyle introduced in the next section.

\section{Playstyle Formalization}

Building on the observable elements defined in the previous section, we present a formal framework rooted in belief-driven decision-making and preference-based evaluation to rigorously define playstyle. We begin by establishing the foundational assumption that enables this formulation: \textbf{Playstyle Consistency}.

\subsection{The Assumption of Playstyle Consistency}

Playstyle is fundamentally about preference: an agent chooses actions that lead to belief states it prefers, according to some internal value system. However, directly comparing future belief states $b'_1, b'_2$ may not yield a consistent ordering—especially under pairwise evaluation. The \textit{Condorcet paradox} \citepx{condorcet_paradox} demonstrates that even consistent local comparisons may fail to produce a total ordering globally.

To address this, we introduce the assumption of \textbf{Playstyle Consistency}, which anchors preference evaluation to the current belief state. Given a current belief $b$ and a set of possible actions $\mathcal{A}$, each action $a \in \mathcal{A}$ leads to an expected future belief state $b' = \tau(b, a)$, where $\tau$ denotes the agent's belief transition dynamics.

We define a preference function $\mathrm{PF}: \mathcal{B} \times \mathcal{B} \rightarrow \mathbb{R}$ such that:

\[
\mathrm{PF}(b', b) > \mathrm{PF}(b'', b) \quad \Longleftrightarrow \quad \text{$b'$ is preferred over $b''$, given current belief $b$}.
\]

Here, the function $\mathrm{PF}(b', b)$ measures the desirability of future belief $b'$ relative to current belief $b$. A higher value indicates stronger preference.
This anchors all comparisons to a shared origin $b$, avoiding cyclic inconsistencies across unanchored belief comparisons. Importantly, this makes the agent's internal loop consistent: decision-making becomes the process of selecting the action $a^*$ that leads to the most preferred future belief:

\[
a^* = \arg\max_{a \in \mathcal{A}} \mathrm{PF}(\tau(b, a), b).
\]

This assumption does not require a global or transitive ordering among all belief states, only that comparisons from the same anchor $b$ are well-defined. It also does not preclude an agent from switching or blending styles; as long as each decision step remains anchored to its current belief, the evaluation remains self-consistent.

\subsection{The Formal Tuple of Playstyle Definition}

We define a \textbf{Formal Playstyle System} via two coupled tuples: one for actual interaction with the environment (\textbf{Actual Interaction Form}), and one for internal decision-making (\textbf{Agent-Expected Form}).

\paragraph{Actual Interaction Form (Environment-facing):}
\[
\mathcal{M}_{\text{env}} = (S, A, T, \Omega, O)
\]
\begin{itemize}
    \item $S$: set of latent world states.
    \item $A$: set of actions available to the agent.
    \item $T(s' \mid s, a)$: environment transition probability.
    \item $\Omega$: set of observations received by the agent.
    \item $O(o \mid s')$: observation probability given state $s'$.
\end{itemize}

\paragraph{Agent-Expected Form (Internal decision-making):}
\[
\mathcal{M}_{\text{agent}} = (B, \tau, \mathrm{PF}, \mathcal{I})
\]
\begin{itemize}
    \item $B$: space of belief states, represented as probabilistic distributions over latent factors or internal configurations.
    \item $\tau(b' \mid b, a)$: belief transition function, modeling the agent’s subjective model of dynamics.
    \item $\mathrm{PF}(b', b)$: preference function comparing successor belief $b'$ to current belief $b$, outputting a scalar preference score.
    \item $\mathcal{I}(o)$: interpretation function mapping observations $o$ to internal belief $b = \mathcal{I}(o)$. This may vary by agent.
\end{itemize}

\subsection{Policy Definition and Dual Decision Loops}

Based on the above tuples, we define two decision loops that govern policy construction:

\paragraph{Agent-Expected Policy (Internal Loop):}
\[
\pi(b, o; \mathrm{PF}, \tau, \mathcal{I}) = \arg\max_{a \in A} \mathbb{E}_{b' \sim \tau(b'| b, a)} \left[ \mathrm{PF}(b', b) \right]
\]
This policy selects actions that are expected to lead to preferable belief states, according to the agent’s internal model and current interpretation.

\paragraph{Actual-Environment Policy (External Loop):}
\[
\pi(s; \mathrm{PF'}, \mathcal{I'}) = \arg\max_{a \in A} \mathbb{E}_{s' \sim T(s' \mid s, a)} \mathbb{E}_{o' \sim O(o' \mid s')} \left[ \mathrm{PF'}(\mathcal{I'}(o'), \mathcal{I'}(o)) \right]
\]
This policy captures the actual execution result, where outcomes are determined by the real environment dynamics. It also shows how PF can indirectly shape external behaviors via internal representations.

\subsection{Measurement and Expression of Playstyle}

Given the formalization of playstyle as a function of belief-based decision-making, we now define how it can be empirically measured and expressed based on \textbf{Observable Elements} $o \in \Omega$. These definitions serve as a foundation for all downstream comparisons and evaluations of style.

\paragraph{Playstyle Measurement}

We consider two types of measurement functions, both relying on a distance metric $D$ over observable elements:

\begin{itemize}
    \item \textbf{Absolute Measurement} \\
    $D : \Omega \rightarrow \mathbb{R}^d$ \\
    We say an observation $o$ is $\epsilon$-close to a target style $t$ if:
    \[
    D(o \| t) \leq \epsilon
    \]

    \item \textbf{Relative Measurement} \\
    $D : \Omega \times \Omega \rightarrow \mathbb{R}^d$ \\
    We say two observations $o, o' \in \Omega$ are $\epsilon$-close if:
    \[
    D(o \| o') \leq \epsilon
    \]
\end{itemize}

In both cases, the distance metric $D$ may be instantiated using various similarity or divergence measures depending on the structure of the observable space $\Omega$. These implementations will be discussed in Chapter~6.

\paragraph{Playstyle Expression}

Once measurement is defined, we can characterize playstyle expression in terms of observable divergence between policies.

\begin{itemize}
    \item \textbf{Expression via Absolute Measurement} \\
    Let $\pi$ be a policy and $t$ a target style. We say $\pi$ expresses style $t$ within tolerance $\epsilon$ if:
    \[
    \mathbb{E}_{o \sim \pi}[ D(o \| t) ] \leq \epsilon
    \]

    \item \textbf{Expression via Relative Measurement} \\
    Let $\pi$ and $\pi_t$ be two policies. We say $\pi$ expresses a playstyle similar to $\pi_t$ if:
    \[
    \mathbb{E}_{o \sim \pi,\ o' \sim \pi_t}[ D(o \| o') ] \leq \epsilon
    \]
\end{itemize}

These definitions support both fixed-target evaluations and pairwise comparisons between agents. Importantly, they allow us to treat playstyle as a \textit{measurable phenomenon} grounded in observed behavior, without requiring direct access to the agent’s internal parameters or beliefs.

\paragraph{Summary}

To conclude, playstyle measurement provides the necessary link between internal preferences and observable behavior. By introducing formal $\epsilon$-closeness over observables, we obtain a precise vocabulary for:

\begin{itemize}
    \item determining when a policy can be said to instantiate a particular style,
    \item comparing different agents in terms of stylistic similarity,
    \item and evaluating how changes in internal preference functions manifest in behavior.
\end{itemize}

These definitions also support cases where style identity is fluid, allowing comparison of partial or transitional trajectories. This framework sets the stage for Chapter~6, where we examine concrete implementations and applications of these measurement paradigms.

\subsection{Conclusion of the Formalization}

This formalism establishes a rigorous pipeline from beliefs to action policies, from preferences to behavior, and from internal models to observable outcomes. It makes explicit the assumptions needed for inference (via Playstyle Consistency), and lays the groundwork for measurement and comparison in later chapters.

In the next section, we turn to a key clarification: how playstyle differs from generative style, and why interaction-based decision-making necessitates a distinct treatment.

\section{Difference Between Generative Style and Playstyle}

In assessing style, a natural starting point is to consider existing evaluation techniques developed for generative models.  
However, a fundamental distinction separates \textbf{generative style evaluation} from the needs of \textbf{playstyle evaluation}.

\subsection{Fundamental Conceptual Differences}

In generative modeling, outputs—whether images, audio, or text—are directly produced by the model itself.  
In contrast, in playstyle settings, outputs (states) emerge from interactions between an agent's decisions and an external environment.

Critically, the \textbf{environment} acts as a \textbf{perfect generator}:  
states are not imperfect samples needing fidelity correction but are inherently valid by their nature.  
Thus, any evaluation framework based on detecting "fake" versus "real" outputs, such as in generative adversarial networks (GANs) \citepx{inception_score, fid, prd}, is fundamentally misaligned with playstyle analysis.

Furthermore, while generative models typically assume control over local output contexts (e.g., consecutive frames or sentences), playstyle trajectories are highly sensitive to environmental transitions.  
Small changes—such as switching rooms, encountering dynamic obstacles, or unexpected NPC behaviors—can cause abrupt shifts, invalidating assumptions of local smoothness often relied upon in video, audio, or text generation evaluation \citepx{flow_net, tcc}.

Thus, both the ontological nature of outputs and the structural dynamics of context differ significantly between the two domains.

\subsection{Evaluation Methods in Generative Models}

Generative style evaluations can be grouped into several main categories:

\subsubsection{Perceptual Quality Metrics}

\begin{itemize}
\item \textbf{Structural Similarity Index Measure (SSIM)} \citepx{ssim}: Quantifies perceived quality by comparing structural information between generated and reference images.
\item \textbf{Peak Signal-to-Noise Ratio (PSNR)} \citepx{image_process_book, psnr_vs_ssim}: Measures pixel-level reconstruction fidelity based on mean squared error.
\item \textbf{Learned Perceptual Image Patch Similarity (LPIPS)} \citepx{lpips}: Assesses perceptual similarity using deep network feature activations.
\item \textbf{Inception Score (IS)} \citepx{inception_score}: Evaluates the visual quality and diversity of generated images based on predicted class distributions from an Inception model.
\end{itemize}

These methods assume an authoritative reference or a trained classifier and focus on appearance-level evaluation.  
Since playstyle outputs are valid environmental responses and lack fixed references, these metrics are unsuitable.

\subsubsection{Distributional Divergence Metrics}

\begin{itemize}
\item \textbf{Fréchet Inception Distance (FID)} \citepx{fid}: Measures distributional similarity between real and generated samples in a feature space.
\item \textbf{Precision and Recall for Distributions (PRD)} \citepx{prd}: Separately assesses sample precision and coverage between two distributions.
\item \textbf{Likeness Score} \citepx{likeness_score}: Directly analyzes divergence between generated image sets without relying on pretrained models.
\end{itemize}

Although these metrics are effective for generative quality assessment, they do not capture the interactive, environment-dependent nature of playstyle.  
Notably, in prior work on Playstyle Distance \citepx{playstyle_distance}, adapting FID-style evaluations showed only random-level discrimination and incurred high computational costs.

\subsubsection{Temporal Consistency Metrics}

\begin{itemize}
\item \textbf{Temporal Cycle-Consistency (TCC)} \citepx{tcc}: Measures consistency across video frames by learning correspondences without supervision.
\item \textbf{Optical Flow Estimation (FlowNet)} \citepx{flow_net}: Provides pixel-wise motion estimation between frames; often used to compute warping-based consistency metrics.
\end{itemize}

While sequential coherence is desirable, playstyle trajectories are often disrupted by dynamic environmental transitions, limiting the reliability of temporal smoothness assumptions.

\subsubsection{Human Preference and Semantic Consistency Metrics}

\begin{itemize}
\item \textbf{CLIPScore} \citepx{clip_score}: Scores generated content by matching it to text prompts in a shared vision-language embedding space.
\item \textbf{ArtFID} \citepx{art_fid}: Evaluates the style quality of generated artworks relative to human-created references.
\item \textbf{ArtScore} \citepx{art_score}: Learns to predict the "artistic quality" of AI-generated images based on human annotations.
\end{itemize}

These metrics suggest promising directions for playstyle evaluation involving human preference modeling, although substantial adaptation is necessary to handle action–state trajectories.

\subsubsection{Generative Model Internal Metrics}

\begin{itemize}
\item \textbf{Evidence Lower Bound (ELBO)} \citepx{vae}: Estimates the quality of probabilistic reconstructions in variational autoencoders.
\item \textbf{Perceptual Path Length (PPL)} \citepx{style_gan}: Measures the smoothness of latent space interpolations in GANs.
\end{itemize}

As playstyle lacks stable latent representations and depends heavily on reactive interactions, these internal metrics offer limited applicability.

\subsubsection{NLP-Oriented Metrics}

\begin{itemize}
\item \textbf{Bilingual Evaluation Understudy (BLEU)} \citepx{bleu}: Assesses machine translation quality via n-gram overlap with reference texts.
\item \textbf{Perplexity} \citepx{language_book}: Measures the predictive confidence of a language model over a sequence.
\end{itemize}

Both BLEU and Perplexity presume static ground-truth references and do not address the interactive decision-making complexity intrinsic to playstyle.

\subsection{Why Direct Application Fails for Playstyle}

The key obstacles can be summarized as follows:

\begin{itemize}
\item \textbf{Absence of Fake Samples}:  
  Environment outputs are always valid; there is no notion of generation failure to detect.

\item \textbf{High Interactivity}:  
  State trajectories result from complex agent-environment feedback loops, not unilateral generation.

\item \textbf{Context Fragility}:  
  Environmental transitions frequently cause abrupt state changes, invalidating assumptions of local smoothness.

\item \textbf{Evaluation Target Shift}:  
  Playstyle evaluation must capture the \textit{decision-making process} and \textit{interaction patterns}, not merely output distributions.

\item \textbf{Feature Space Mismatch}:  
  In Playstyle Similarity experiments \citepx{playstyle_similarity}, cosine similarity metrics provided no meaningful advantage over Euclidean distances on the FID like approach, reflecting no structured angular semantic guarantees in playstyle spaces.
\end{itemize}

Thus, while generative evaluation frameworks offer useful conceptual inspirations, direct application risks mischaracterizing the nature of playstyle.

\subsection{Insights and Inspirations}

Despite these challenges, some insights remain valuable:

\begin{itemize}
\item \textbf{Preference Modeling}:  
  Preference-guided evaluations suggest methods for subjective, human-aligned playstyle assessment.

\item \textbf{Soft Distributional Comparisons}:  
  Rather than rigid matching, trajectory-level statistical modeling could inform style similarity estimation.

\item \textbf{Context-Aware Consistency}:  
  Temporal evaluation methods, if conditioned on environment-induced changes, could enhance nuanced playstyle modeling.
\end{itemize}

In Chapter~6, we adapt these inspirations into domain-specific metrics that operate over Observable Elements and account for both process- and outcome-based differences.

\chapter{Playstyle Measurement}
\noindent\textbf{Key Question of the Chapter:} \\
\textit{What methods are available for measuring playstyle?}

\noindent\textbf{Brief Answer:} 
Playstyle can be measured by quantifying structured patterns in decision-making, as expressed through \textbf{Observable Elements}, using divergence metrics over their distributions to identify genuine stylistic differences while filtering out noise or environmental artifacts.

\bigskip

Building upon the formalization of playstyle as belief-driven decision-making expressed through observable elements, we now turn to the central task of \textbf{quantifying these expressions}. Understanding and measuring playstyle is a foundational step toward analyzing agent individuality, supporting human-aligned AI design, and enabling the training of diverse and interpretable artificial agents.

As established in previous chapters, playstyle refers not merely to outcome-level differences (such as win rates), but to structured patterns in how agents make decisions—reflecting latent preferences, contextual beliefs, and interactive strategies. However, since these internal structures are unobservable, playstyle must be inferred from behavior—through what we have defined as \textbf{Observable Elements} $o \in \Omega$.

This inference task brings forth several key challenges:

\begin{itemize}
    \item \textbf{What elements can we observe and reliably measure?}
    \item \textbf{What constitutes a significant stylistic difference under these observations?}
    \item \textbf{How can we ensure that measurements reflect genuine agent-level variation rather than noise or environmental artifacts?}
\end{itemize}

To address these questions, we introduced in Chapter~5 a formal framework grounded in $\epsilon$-closeness over distributions of observable elements. This chapter now focuses on the \textbf{practical realizations} of this framework: How can we construct such measurements? What assumptions or structures must be imposed to enable comparison across agents?

We categorize the existing methods for playstyle measurement into four primary classes, each grounded in a different epistemic stance toward what can or should be measured:

\begin{enumerate}
    \item \textbf{Heuristic Rules and Feature-Based Approaches} \\
    Leverage human domain knowledge to define interpretable features and handcrafted criteria for distinguishing playstyles.

    \item \textbf{Data-Driven Learning Approaches} \\
    Use supervised, unsupervised, or contrastive learning to discover latent style representations from data, without manual feature design.

    \item \textbf{Policy and Action Distribution Approaches} \\
    Analyze differences in decision distributions directly—e.g., comparing $\pi(a|s)$ or entire trajectory distributions—to infer preferences from observed actions.

    \item \textbf{Discrete State-Based Approaches} \\
    Project continuous observations into symbolic states to enable alignment and comparison across agents, especially useful in high-dimensional or multi-agent settings.
\end{enumerate}

These approaches differ in their level of abstraction, reliance on prior knowledge, and sensitivity to environmental variation. Throughout this chapter, we will examine each category in turn, with reference to the formal principles of style measurement established earlier. In particular, we will evaluate how each method instantiates the measurement function $f(\pi) = P(o|\pi)$ and what forms of divergence metrics $D$ are used to define stylistic similarity or difference.

\section{Measurement Approaches Overview}

To operationalize the formal notion of playstyle as observable variation in policy-induced behavior, we classify existing measurement approaches into four major types. Each reflects a different assumption about what constitutes meaningful stylistic evidence—either in the form of absolute differences in behavior, or relative comparisons conditioned on context.

\begin{itemize}
    \item \textbf{Heuristic Rules and Feature-Based Measurements}
    \begin{itemize}
        \item \textit{Core Idea:} Use domain-specific knowledge and hand-engineered rules to define stylistic signatures, such as total combat engagement, average resource usage, or route choices.
        \item \textit{Measurement Type:} \textbf{Absolute measurement}—based on predefined thresholds over $D(o)$, where $o \in \Omega$ are observable elements.
        \item \textit{Formulation:} $\pi$ is said to be $\epsilon$-close to a reference style $t$ if $D(o \| t) \leq \epsilon$ for $o \sim \pi$.
        \item \textit{Advantages:} High interpretability and precision when targeting specific properties.
        \item \textit{Limitations:} Strong dependence on handcrafted engineering; limited transferability across domains.
    \end{itemize}

        \item \textbf{Data-Driven Learning Approaches (SL / UC / CL)}
    \begin{itemize}
        \item \textit{Core Idea:} Learn stylistic patterns from data using supervised classification (SL), unsupervised clustering (UC), or contrastive learning (CL), each depending on different forms of supervisory signals.
        \item \textit{Measurement Type:} 
        \begin{itemize}
            \item \textbf{Absolute}: SL directly predicts style labels $y$ via $g(o) = y$, where $o \in \mathcal{O}$ are observable elements.
            \item \textbf{Relative}: UC and CL infer latent distances or group structures among behavior samples without ground-truth categories.
        \end{itemize}
        \item \textit{Formulation:} 
        \begin{itemize}
            \item SL: Learn $g: \mathcal{O} \rightarrow \mathcal{Y}$ such that $g(o) \approx y$ minimizes cross-entropy over labeled data.
            \item CL: Learn embedding $f(o)$ to satisfy $d(f(o_1), f(o_2)) < \epsilon$ if $(o_1, o_2)$ is a similar pair, and $> \epsilon$ if dissimilar.
        \end{itemize}
        \item \textit{Advantages:} 
        \begin{itemize}
            \item SL: Clear interpretability when labels are well-defined; effective for known style taxonomies.
            \item CL: Strong ability to learn nuanced differences; flexible use of structured pairwise constraints.
        \end{itemize}
        \item \textit{Limitations:} 
        \begin{itemize}
            \item SL: Requires curated, consistent labels; hard to scale across domains with ambiguous categories.
            \item UC: Sensitive to the quality of feature representation; may produce semantically meaningless clusters if behavioral features do not encode style-relevant information.
            \item CL: Depends on quality of pairwise labels; not truly label-free—often requires similarity judgements.
        \end{itemize}
    \end{itemize}

    \item \textbf{Policy / Action Distribution-Based Measurements}
    \begin{itemize}
        \item \textit{Core Idea:} Compare policies $\pi_1(a|s)$ and $\pi_2(a|s)$ across shared states, or compare their induced observable distributions $f(\pi)$.
        \item \textit{Measurement Type:} \textbf{Relative measurement} over policy-induced action or trajectory distributions.
        \item \textit{Formulation:} $\pi_1$ and $\pi_2$ are $\epsilon$-close if $d(f(\pi_1), f(\pi_2)) \leq \epsilon$, where $f(\pi)$ maps to distribution over observable elements.
        \item \textit{Advantages:} Captures implicit decision-making tendencies without manual feature design.
        \item \textit{Limitations:} Requires access to full policies or dense behavior logs; aligning distributions across agents may be difficult.
    \end{itemize}

    \item \textbf{Discrete State Measurements}
    \begin{itemize}
        \item \textit{Core Idea:} Discretize continuous or high-dimensional observations into symbolic states $z \in \mathcal{Z}$, enabling alignment and conditional comparison of behavior across agents.
        \item \textit{Measurement Type:} \textbf{Relative measurement}, focusing on differences across aligned symbolic contexts.
        \item \textit{Formulation:} For trajectories $o$, $o'$, they are $\epsilon$-close if $D(o | o') \leq \epsilon$ under aligned $z$.
        \item \textit{Advantages:} Label-free, domain-agnostic, and sample-aligned; supports multi-agent comparison via shared anchors.
        \item \textit{Limitations:} Strongly dependent on quality of discretization; nontrivial symbolic abstraction design.
    \end{itemize}
\end{itemize}

\vspace{0.5em}

\begin{table}[ht]
\centering
\caption{Summary of Playstyle Measurement Approaches}
\begin{tabular}{|p{2.7cm}|p{2.7cm}|p{2.7cm}|p{2.7cm}|p{2.7cm}|}
\hline
\textbf{Measurement \newline Approach} & \textbf{Applicable \newline Scenarios} & \textbf{Advantages} & \textbf{Limitations} & \textbf{Primary \newline Information \newline Sources} \\
\hline
Heuristic Rules \& Features & Small-scale, \newline Expert-driven Settings & High \newline Interpretability, \newline High Precision & Manual Design Effort, \newline Poor \newline Generalization & Environment Info, \newline Handcrafted Features \\
\hline
Supervised/\newline Unsupervised/ \newline Contrastive Learning & Large-scale \newline Behavioral Data, \newline Rapid \newline Prototyping & Automation, \newline Scalability, \newline Flexible \newline Supervision & Label Quality (SL); \newline Feature-sensitive, Interpretability (UC); \newline Pair Selection (CL) & States, Actions, \newline Rewards or \newline Preferences \\
\hline
Policy/Action Distribution & Access to \newline Policy Models \newline or Detailed \newline Logs & Theoretically Sound, \newline Focuses on \newline Decisions & Requires \newline Policy Access, \newline Distribution Alignment Challenges & Action \newline Distributions, \newline Decision States \\
\hline
Discrete State & General-purpose, \newline Domain-Agnostic \newline Settings & Sample-aligned \newline Comparison, \newline No Label \newline Requirement & Sensitive to \newline Discretization \newline Quality & Symbolic States, \newline Conditional \newline Action  \newline Distributions \\
\hline
\end{tabular}
\end{table}

\vspace{0.5em}

In the following sections, we will examine each of these approaches in greater detail, presenting examples, algorithms, and design considerations for analyzing and comparing playstyles across various domains.

\section{Measuring Playstyle with Heuristic Rules and Features}

\subsection*{Core Idea}

Heuristic and feature-based methods rely on expert knowledge to handcraft measurable indicators that are assumed to reflect stylistic tendencies. These indicators—often called \textit{playstyle features}—are manually selected and evaluated using interpretable rules or thresholds.

\subsection*{Measurement Type}

This approach is generally an \textbf{absolute measurement}, as it interprets an agent's behavior in relation to predefined rules or thresholds, rather than through relative comparison with other agents.

\subsection*{Formalization}

Let $o \in \mathcal{O}$ denote an observable element, such as a trajectory fragment, action count, or performance summary. A heuristic playstyle metric can be formulated as:
\[
h(o) = v \in \mathbb{R}
\]
where $h$ is a handcrafted measurement function that maps observable elements to scalar values (e.g., kill rate, average step length). Multiple heuristic features may be composed as a vector $\mathbf{h}(o) \in \mathbb{R}^d$.

\subsection*{Advantages}

\begin{itemize}
    \item \textbf{High Interpretability}: Metrics have clear semantic meaning and are understandable by designers or analysts.
    \item \textbf{Precision for Specific Styles}: Useful for detecting known patterns or stylistic archetypes.
    \item \textbf{Minimal Data Requirements}: Does not rely on large datasets or training pipelines.
\end{itemize}

\subsection*{Limitations}

\begin{itemize}
    \item \textbf{Manual Design Effort}: Requires extensive domain expertise and iterative tuning.
    \item \textbf{Limited Expressiveness}: Often fails to capture emergent or subtle style nuances.
    \item \textbf{Poor Generalization}: Rules are highly domain-specific and hard to transfer across games or tasks.
\end{itemize}

\subsection*{Examples Across Domains}

Representative use cases of heuristic-based measurements include:

\begin{itemize}
    \item \textbf{Shooting Games:} Kill/death ratios, headshot frequency, reload timing.
    \item \textbf{Real-Time Strategy Games:} Resource spending efficiency, unit diversity, tech tree preference.
    \item \textbf{Competitive Sports:} Passing accuracy, movement heatmaps, possession time.
    \item \textbf{Social Media:} Post timing regularity, sentiment fluctuation, retweet diversity.
    \item \textbf{Visual Art \& Music:} Color vibrancy, brushstroke size, BPM variance, harmony density.
    \item \textbf{Symbolic Frameworks:} MBTI classifications \citepx{myers1998mbti}, personality  mappings~\citepx{player_personality}.
    \item \textbf{Physical Motion:} Gait cycle frequency, turn smoothness, collision rates.
\end{itemize}

\subsection*{Human-Likeness Constraints}

Heuristic rules also help enforce human-likeness in agent behavior. Common constraints include:
\begin{itemize}
    \item Penalizing non-human patterns like excessive twitching, spinning, or erratic camera shaking~\citepx{shaking_spinning_cost}.
    \item Enforcing biological limits: reaction latency, motion speed bounds, noisy actuation~\citepx{biological_constrains}.
    \item Integration into large-scale systems such as AlphaStar, which restricts APM and reaction time to human-like thresholds~\citepx{alpha_star}.
    \item Procedural frameworks such as CAPS (Conditioning for Action Policy Smoothness) to maintain smoothness and realism~\citepx{caps}.
\end{itemize}

\subsection*{Higher-Level Playstyle Concepts}

More abstract heuristic frameworks, such as player typologies or design archetypes, offer categorical representations of style:
\begin{itemize}
    \item \textbf{Bartle Taxonomy}~\citepx{bartle_taxonomy}: Categorizes players into Achievers, Explorers, Socializers, and Killers.
    \item \textbf{Playstyle Personas}: Designer-defined play personas that guide feature tuning or game content targeting \citepx{player_modeling}.
\end{itemize}

\subsection*{Discussion}

While heuristic rules are foundational in early style modeling and remain useful in interpretable applications, they face inherent limitations. Their dependency on expert assumptions constrains scalability, and their rigidity makes them ill-suited to capture style emergence in modern, high-dimensional decision environments. These challenges motivate the rise of automated, data-driven techniques in subsequent sections.

\section{Measuring Playstyle by Learning from Data}

Moving beyond handcrafted rules, data-driven approaches aim to \textbf{automatically extract latent playstyle representations} from behavioral observations. These methods learn from the distribution of \textbf{Observable Elements} $\mathcal{O}$ associated with an agent's interaction history, without assuming explicit access to the agent's preference function $\mathrm{PF}$. Depending on the degree of supervision and the learning signal structure, these methods are categorized into three major paradigms: supervised learning (SL), unsupervised clustering (UC), and contrastive learning (CL).

\subsection*{Supervised Learning (SL)}

Supervised learning relies on datasets annotated with predefined style labels, such as strategic archetypes, player roles, or personality categories. Models (e.g., classification networks) are trained to predict these style labels from behavioral input---often processed into state--action trajectories or domain-specific features~\citepx{supervised_style_classification}. This provides \textbf{absolute playstyle measurement}, with learned mappings from behavior to style classes.

However, this paradigm \textbf{inherits the inductive bias and limitation of the labeling schema}. Labels are often noisy, subjective, or domain-locked, and may fail to capture the full spectrum or granularity of playstyle variation. In addition, the assumption of categorical style classes contradicts the more fluid and continuous formulation of style proposed in Chapter~5.

\subsection*{Unsupervised Clustering (UC)}

Unsupervised clustering techniques attempt to discover latent groupings in behavioral data without requiring predefined labels. Studies have employed self-organizing maps~\citepx{player_modeling_self_organization} and clustering algorithms such as DBSCAN~\citepx{dbscan}, k-means, and hierarchical clustering to identify distinct behavioral modes across various games.

A critical prerequisite for meaningful clustering is the construction of a suitable \textbf{feature space}. Without an effective representation of behavior—whether hand-engineered or learned—the resulting clusters may reflect noise or irrelevant variance rather than genuine stylistic differences. Feature extraction may rely on statistical summaries, trajectory encodings, or embedding spaces learned via auxiliary objectives.

This approach is particularly useful in exploratory scenarios or when no labels are available. However, a key challenge lies in interpreting the semantic meaning of the resulting clusters, especially when the features are high-dimensional, noisy, or lack clear stylistic anchors.

\subsection*{Contrastive Learning (CL)}

Contrastive learning learns a \textbf{style-sensitive embedding space} by contrasting similar and dissimilar agent behaviors~\citepx{policy_similarity_metric}. Unlike SL and UC, which operate on class labels or global distance, CL focuses on \textbf{local relational structure}, optimizing for an embedding $z = \phi(o)$ such that:
\[
d(z^+, z^-) > d(z, z^+)
\]
for behavior pairs $(o, o^+)$ deemed stylistically similar and $(o, o^-)$ deemed different.

A key strength of CL lies in its \textbf{label-efficient learning}: it leverages weak signals such as win/loss outcomes~\citepx{rd, udm}, trajectory similarity~\citepx{chess_style}, or self-play pairing structure. In domains where clear labels are scarce but relative feedback is abundant, contrastive objectives can \textbf{preserve stylistic semantics} while accommodating noise and context variance.

Nonetheless, CL's effectiveness depends on the \textbf{quality of positive/negative sampling}. Poorly defined pairings may collapse the style space or introduce non-stylistic noise. Moreover, contrastive embeddings often lack interpretability unless coupled with downstream probing or visualization.

\subsection*{Discussion}

Data-driven learning enables scalable and flexible playstyle measurement---especially when traditional metrics or handcrafted rules fall short in high-dimensional or emergent environments. The three paradigms differ fundamentally in supervision type, semantic clarity, and inductive bias:

\begin{itemize}
  \item \textbf{SL} assumes strong priors from human-defined categories.
  \item \textbf{UC} relies entirely on the statistical structure of behavior.
  \item \textbf{CL} uses local pairwise constraints to learn generalizable embeddings.
\end{itemize}

From the perspective of Chapter~5, these approaches correspond to different ways of inferring or comparing distributions $f(\pi)$ over observable elements $o \in \mathcal{O}$, and ultimately reflect indirect approximations of latent preference structures $\mathrm{PF}$.

The remainder of this chapter will delve into more direct measurement frameworks, including policy action distributions and symbolic discrete-state analysis, which offer greater alignment with the formal framework of playstyle consistency and belief-based reasoning.

\section{Measuring Playstyle from Policy Action Distributions}

While previous approaches emphasized external observables—outcomes, trajectories, or handcrafted features—this section shifts focus to the agent’s \textbf{internal decision process}. Specifically, we analyze playstyle via the probabilistic structure of policies: how action probabilities vary across decision contexts.

By comparing action distributions induced by different policies over sampled states, this method captures stylistic differences rooted in choice tendencies rather than environment-specific observables. It aligns with the internal loop introduced in Chapter~5, reflecting how agents express preferences when situated in similar belief or observation states.

This perspective is especially useful when raw trajectories or outcomes do not sufficiently differentiate behavior, offering a more fine-grained and theoretically grounded way to study intention and policy diversity.

\subsection*{Distributional Metrics}

Distributional divergences quantify how two action distributions differ at the probabilistic level. Common choices include:

\begin{itemize}
    \item \textbf{Kullback-Leibler (KL) Divergence} measures the information loss incurred when one distribution is used to approximate another. Given two discrete distributions $P$ and $Q$, it is defined as $D_{KL}(P \parallel Q) = \sum_x P(x) \log \frac{P(x)}{Q(x)}$.

    \item \textbf{Jensen-Shannon (JS) Divergence} is a symmetrized and smoothed variant of KL divergence: $D_{JS}(P \parallel Q) = \frac{1}{2} D_{KL}(P \parallel M) + \frac{1}{2} D_{KL}(Q \parallel M)$ where $M = \frac{1}{2}(P + Q)$. It remains finite even when the supports of $P$ and $Q$ do not overlap.

    \item \textbf{Wasserstein Distance} (or Earth Mover’s Distance) measures the minimum cost of transforming one distribution into another, based on a ground distance between outcomes. It is widely used in reinforcement learning for measuring policy similarity, as demonstrated in works such as Measuring Policy Distance for Multi-Agent Reinforcement Learning~\citepx{mdps} and Primal Wasserstein Imitation Learning~\citepx{pwil}.
\end{itemize}

Among these, Wasserstein distance is often preferred in practice due to its robustness under distribution mismatch and smoother optimization properties.

\subsection*{Trajectory-Based Similarity}

Rather than comparing action distributions state-by-state, one may analyze entire trajectories—sequences of states and actions—as holistic representations of style.

For example, edit distance or Levenshtein distance can be applied to symbolic action sequences (e.g., opening moves in board games), while \textbf{Dynamic Time Warping} (DTW) is a well-established method for aligning temporally varying sequences. These methods are particularly useful when temporal alignment is non-trivial or when decision timing itself reflects stylistic preference \citepx{driver_dtw}.

\subsection*{Imitation Learning Constraints}

In imitation learning, policy similarity is often operationalized through prediction loss between an expert and a learner.

\begin{itemize}
    \item \textbf{Behavioral Cloning (BC)} minimizes supervised loss between agent predictions and expert demonstrations.
    \item \textbf{Generative Adversarial Imitation Learning (GAIL)}~\citepx{gail} extends this by training a discriminator to distinguish expert from agent trajectories; increased confusion implies greater policy similarity.
    \item \textbf{InfoGAIL}~\citepx{info_gail} introduces latent variables to uncover interpretable behavioral modes, bridging policy alignment and style representation learning.
\end{itemize}

Other imitation strategies such as DQfD or Ape-X DQfD~\citepx{dqfd, apex_dqfd}, while not explicitly measuring playstyle, apply similar constraints on agent behavior to preserve expert-like decisions. These will be discussed further in Chapter~9 on imitation learning.

\subsection*{Multi-Agent Policy Distance}

In multi-agent systems, evaluating policy diversity is essential for understanding cooperation, competition, and team-level strategy.

\begin{itemize}
    \item \textbf{Measuring Policy Distance for Multi-Agent RL}~\citepx{mdps} proposes integration-based metrics for quantifying inter-agent divergence across collaborative or adversarial settings.
    \item \textbf{System Neural Diversity (SND)}~\citepx{snd} introduces a suite of behavioral diversity metrics grounded in neural representation, enabling large-scale analysis of population-wide variation.
\end{itemize}

These approaches emphasize the importance of capturing fine-grained policy differences in environments with strategic interplay.

\subsection*{Discussion}

Action distribution-based measurements offer a precise lens into an agent’s decision-making structure. Unlike outcome-based or feature-driven methods, these approaches reflect \textit{what the agent intends to do} in a given context, even when outcomes may vary due to environment dynamics.

However, such precision comes with challenges: estimating policy distributions requires either access to internal models or extensive logging of state–action pairs. Moreover, aligning states across agents—especially in high-dimensional or unstructured environments—is nontrivial and may obscure true differences.

Despite these limitations, action-centric analysis offers deep insight into the expressive space of decision-making. It helps distinguish stylistic variance even when surface behavior appears similar, and plays a central role in the theoretical formulation of playstyle presented in Chapter~5. In the next section, we synthesize this action-based perspective with the symbolic alignment offered by discrete state modeling.

\section{Measuring Playstyle via Discrete States}
\label{ch:discrete_measurement}

While prior approaches have separately leveraged data-driven feature extraction and policy/action distribution modeling, both face inherent challenges when applied to high-dimensional, interaction-driven environments. In this section, we introduce a hybrid framework that combines the interpretability and alignment of symbolic features with the decision-centric expressiveness of policy comparison—by projecting agent behavior into a space of \textit{discrete symbolic states}.

This method begins from the observation that \textbf{policy-based playstyle measurement}, while theoretically grounded, suffers from two major limitations:

\begin{itemize}
    \item \textbf{Access to Action Distributions:} Accurate estimation of action distributions across all relevant states is often infeasible—particularly when modeling human players or black-box agents.
    \item \textbf{State Alignment Problem:} Even if two policies are available, there is no guarantee that their action distributions are defined over shared or meaningful states. States sampled from one agent's trajectory may lie outside the experiential support of another, resulting in unreliable or undefined comparisons.
\end{itemize}

To address these issues, we propose a structural shift: instead of operating over continuous or high-dimensional raw states, we discretize the observation space into a set of interpretable, symbolic states shared across agents. These \textbf{discrete states} act as anchors for conditional comparisons of decision patterns. By aligning trajectories in this symbolic space, we enable robust comparisons of how different agents behave under equivalent situational conditions—effectively integrating the strengths of data-driven abstraction and action-based evaluation.

\textbf{Philosophically}, this approach also resonates with our earlier formalization of playstyle as a form of belief-based preference over belief space transitions. Just as beliefs partition the internal expectation space, discrete states allow us to partition and interpret the external interaction space. From this perspective, a discrete state paired with its chosen action can be treated as an \emph{observable manifestation} of the agent’s underlying preference at that moment—providing a measurable bridge between internal decision dynamics and external behavior.

In the sections that follow, we:
\begin{itemize}
    \item Introduce methods for learning discrete state abstractions from trajectories.
    \item Define \textit{Playstyle Distance} as a policy-based comparison over shared discrete states.
    \item Extend to multiscale representations to capture perceptual granularity.
    \item Generalize beyond intersection-based comparisons to probabilistic and continuous forms.
    \item Conclude with a discussion on playstyle spectra as a foundation for diversity and population-level analysis.
\end{itemize}

\subsection{Discrete State Projection}

In high-dimensional observation spaces, such as those in modern video games, even imperceptible visual or spatial changes can result in substantially different raw observations.  
This sensitivity renders traditional tabular reinforcement learning inapplicable~\citepx{rl_book}, and makes direct comparison of states across agents infeasible, as the underlying state space lacks semantic alignment.  
In the context of playstyle measurement—where the goal is to evaluate and compare decisions made under similar conditions—this alignment issue becomes particularly critical.  
Without a shared and semantically meaningful representation, it is not straightforward to define effective intersections or establish reliable comparisons between different players' behaviors.

Early approaches in deep reinforcement learning attempted to approximate discrete states via hashing techniques.  
For example, static hash functions and learned embeddings from autoencoders have been applied in count-based exploration frameworks~\citepx{sim_hash, count_based_exploration}.  
However, while useful for exploration, these techniques offer no guarantees of semantic consistency or cross-agent comparability.

Seeking a more principled solution, we turned to \textit{Vector Quantized Variational AutoEncoders} (VQ-VAE)~\citepx{vq_vae}, which combine the k-nearest-neighbor idea from clustering with deep learning to generate discrete latent representations.  
VQ-VAE projects inputs into a learned codebook through nearest-neighbor quantization, and employs a straight-through gradient estimator to enable end-to-end training.  
Yet, standard VQ-VAE remains limited in two critical ways:  
(1) its large codebook often results in an excessively fine-grained symbolic space, where overlap across agents becomes sparse;  
(2) it lacks explicit constraints to ensure that the learned states encode gameplay-relevant decision features.

To address these limitations, we developed the \textit{Hierarchical State Discretization} (HSD) framework~\citepx{playstyle_distance}.  
HSD introduces two key innovations:
First, it organizes discrete codes into a compositional multiscale hierarchy, where each level captures different abstraction granularity and is combined via soft weighting coefficients $\alpha$.  
Second, it incorporates dual decoding objectives—reconstruction loss and policy prediction loss—that jointly steer the discrete representation toward two core properties:

\begin{itemize}
  \item \textbf{Capacity}: The reconstruction loss ensures that perceptual structure in the input is preserved, enabling the model to distinguish different environments or contexts even in compact spaces.
  \item \textbf{Popularity}: The policy loss enforces that gameplay-relevant decision patterns are retained, allowing different agents’ behaviors to be meaningfully compared under a common symbolic representation.
\end{itemize}

These two dimensions directly correspond to the fundamental requirements for playstyle measurement:  
Capacity supports the differentiation of behavioral diversity across contexts, while Popularity ensures that decision tendencies are consistently encoded for comparative analysis.

\begin{figure}[ht]
  \centering
  \includegraphics[width=\textwidth]{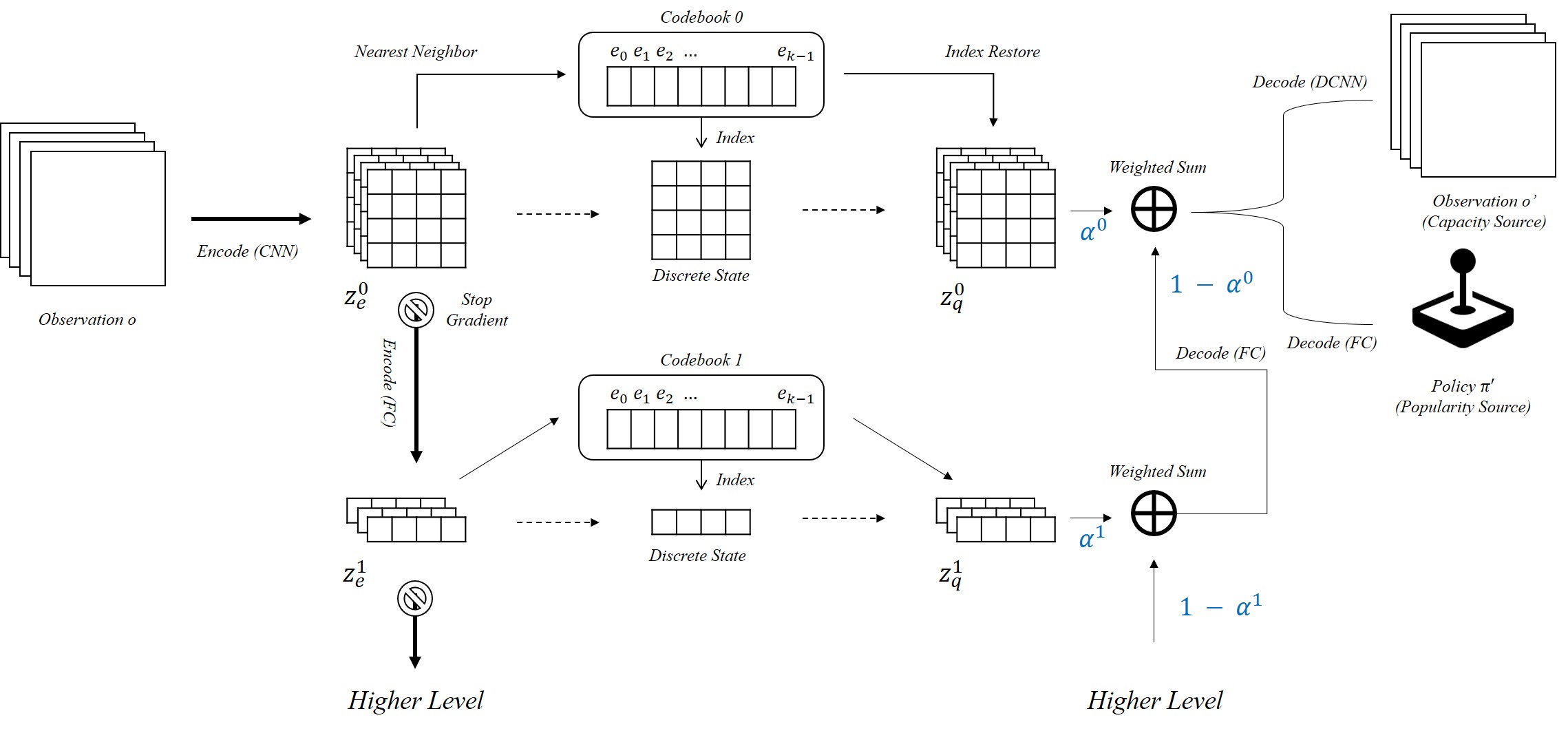}
  \caption[Illustration of the Hierarchical State Discretization (HSD) framework]{Illustration of the Hierarchical State Discretization (HSD) framework. Each hierarchy is quantized separately and composed through weighted gating ($\alpha$) to form a flexible, multiscale symbolic state. Policy supervision ensures gameplay-relevant semantics are retained.}
  \label{fig:hsd}
\end{figure}

\vspace{1em}

\subsubsection{Verification of Discrete Representation on CIFAR-10}

To empirically verify that the discrete representations produced by HSD remain useful and discriminative even under severe compression, we conducted supplementary experiments on the CIFAR-10 dataset~\citepx{cifar10}.  
We treat the image classification task as a proxy decision-making problem, where each image classification corresponds to a distinct decision context.

A modified Wide Residual Network (Wide ResNet)~\citepx{wide_resnet} served as our backbone architecture.  
The model used a depth of 16 and a widening factor of 4, followed by an additional CNN layer that reduced the feature dimension to 128. This layer was bounded within $[-1, 1]$ by a hyperbolic tangent activation.  
Both the observation decoder and the policy decoder were trained separately, following the standard VQ-VAE structure for image reconstruction and classification prediction, respectively.

\begin{table}[ht]
    \centering
    \caption[CIFAR-10 classification accuracy with HSD discrete representations]{Classification accuracy on CIFAR-10 using Wide ResNet and HSD discrete representations. The average accuracy and standard deviation are reported over three independent runs with different random seeds.}
    \label{Table:CIFAR-10-Accuracy}
    \begin{tabular}{|l|r|}
    \hline
    \quad & Test Accuracy \\
    \hline
    Human Benchmark & 93.91\% \\
    \hline
    Wide ResNet (Baseline) & 94.92\% $\pm$ 0.05\% \\
    \hline
    HSD ($2^{16}$ state space) & 91.03\% $\pm$ 0.30\% \\
    \hline
    \end{tabular}
\end{table}

As shown in Table~\ref{Table:CIFAR-10-Accuracy}, despite compressing observations into an ultra-small symbolic space of just $2^{16}$ states, HSD achieved a test accuracy exceeding 91\%, approaching the human benchmark of 93.91\%.  
This result confirms that the learned discrete states preserve decision-relevant features under compression—and do not merely reflect class identity.

\subsubsection{State Space Utilization Analysis}

We further examined the utilization of symbolic states across different HSD configurations.  
Table~\ref{Table:CIFAR-10-Used-State-Count} reports the number of unique discrete states used in the CIFAR-10 test set, where $B$ denotes the block size and $K$ the codebook size.

\begin{table}[ht]
    \centering
    \caption[Used discrete states on CIFAR-10]{Number of used discrete states under different HSD configurations.}
    \label{Table:CIFAR-10-Used-State-Count}
    \begin{tabular}{|r|r|r|r|}
    \hline
    \multicolumn{4}{|c|}{\bf State Space: $2^{16}$} \\
    \hline
    B2-K256 & B4-K16 & B8-K4 & B16-K2 \\
    \hline
    1863 & 1910 & 2228 & 2911 \\
    \hline
    \multicolumn{4}{|c|}{\bf State Space: $2^{20}$} \\
    \hline
    B4-K32 & B5-K16 & B10-K4 & B20-K2 \\
    \hline
    4011 & 4992 & 6160 & 6757 \\
    \hline
    \end{tabular}
\end{table}

Even within the limited $2^{16}$ state space, HSD was able to utilize thousands of distinct states to differentiate among the 10,000 CIFAR-10 test samples.  
This confirms that the symbolic abstraction is expressive enough to preserve intra-class variation and encode meaningful structure.  
Moreover, increasing block granularity ($B$) improves expressiveness even under a fixed symbolic space size.

\subsubsection{Sensitivity to Observation Perturbations}

We next evaluated the robustness of discrete representations to perceptual noise.  
By adding Gaussian noise of varying standard deviations to the input images, we measured how often the resulting discrete state remained unchanged.

\begin{table}[ht]
    \centering
    \caption[State stability under noise (CIFAR-10)]{Percentage of unchanged discrete states under Gaussian noise perturbation.}
    \label{Table:CIFAR-10-State-Stability}
    \begin{tabular}{|l|c|c|c|c|}
    \hline
    \quad & Std 1 & Std 2 & Std 4 & Std 8 \\
    \hline
    B2-K256 ($2^{16}$) & 89.19\% & 79.68\% & 63.94\% & 42.10\% \\
    B16-K2 ($2^{16}$) & 84.24\% & 71.15\% & 51.37\% & 27.96\% \\
    \hline
    \end{tabular}
\end{table}

As shown in Table~\ref{Table:CIFAR-10-State-Stability}, discrete state encodings remain stable under low noise levels and degrade predictably with larger perturbations—consistent with human perceptual thresholds.  
This property highlights the robustness of HSD to input variations while retaining sensitivity to meaningful differences.

\begin{figure}[ht]
    \centering
    \includegraphics[width=\linewidth]{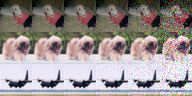}
    \caption[Examples of CIFAR-10 images with Gaussian noises]{Examples of CIFAR-10 images with increasing levels of Gaussian noise (Std = 0, 2, 4, 8, 16, 32).}
    \label{Figure:CIFAR-10-Noise}
\end{figure}

\subsubsection{Ablation Study on Policy Supervision}

Finally, we investigated the importance of policy supervision in learning useful discrete representations.  
We compared HSD encoders trained with and without a policy decoder, using the TORCS racing environment~\citepx{torcs}.  
The evaluation task involved recognizing five distinct driving styles under a speed dimension and five under an action noise dimension.

\begin{table}[h]
    \centering
    \caption[Ablation study on policy supervision]{Impact of removing the policy decoder in discrete representation learning (TORCS environment).}
    \label{Table:PolicyDecoderAblation}
    \begin{tabular}{|l|c|c|}
    \hline
    \quad & Speed Task Accuracy & Noise Task Accuracy \\
    \hline
    With Policy Decoder & 89.60\% (intersection: 187.53) & 72.67\% (intersection: 187.23) \\
    \hline
    Without Policy Decoder & 78.13\% (intersection: 1.40) & 40.93\% (intersection: 1.40) \\
    \hline
    \end{tabular}
\end{table}

Table~\ref{Table:PolicyDecoderAblation} clearly shows that excluding the policy decoder severely reduces both prediction accuracy and the number of stable intersection states.  
This ablation confirms that policy supervision is essential to ensure that discrete states retain decision-critical attributes—beyond merely reconstructing observations.

\subsection{Playstyle Distance}
\label{sec:playstyle_distance}

Building on the discrete state representations introduced in the previous section, we now define a concrete metric for quantifying differences in playstyle: the \textit{Playstyle Distance}.

The core philosophy is straightforward:  
Rather than comparing behavior across the entirety of the state space—which would be vulnerable to sparse coverage and semantically invalid mismatches—we restrict our attention to the \textbf{intersection} of discrete states that two agents or players have both encountered.  
Within these shared contexts, we can meaningfully compare decision tendencies by analyzing the respective action distributions.

This design is well-grounded in the epistemological premise of style:  
\textit{Playstyle manifests in how different agents choose among available options when facing similar situations.}  
Absent a shared context, any comparison of actions risks arbitrariness and fails to reflect genuine stylistic differences.

Thus, \textbf{Playstyle Distance} is defined as the divergence between action distributions conditioned on intersected discrete states.  
This formulation offers several key advantages:
\begin{itemize}
  \item It relies only on observable behavior (states and actions), without requiring access to internal policy parameters.  
  \item It accommodates partial trajectory overlap—agents need not visit identical state distributions.  
  \item It avoids misleading comparisons over unshared or out-of-distribution states.
\end{itemize}

In the following subsections, we formally define the metric, present experimental validations across diverse domains, and discuss methodological variants such as alternative divergence measures and weighting schemes.

\subsubsection{Formal Definition of Playstyle Distance}

Let $A$ and $B$ be two playing datasets, each consisting of observation-action pairs.  
Let $O_A$ and $O_B$ denote the observation sets in $A$ and $B$, respectively.  
Each observation $o$ is mapped to a discrete symbolic state $\overline{s}$ using a mapping function $\phi: O \to \overline{S}$.

We first define the \textbf{intersection set} of discrete states as:
\[
\overline{S}_\phi(A,B) = \phi(O_A) \cap \phi(O_B),
\]

For each $\overline{s} \in \overline{S}_\phi(A,B)$, let $\pi_A(\cdot|\overline{s})$ and $\pi_B(\cdot|\overline{s})$ denote the empirical action distributions from datasets $A$ and $B$, respectively.

The \textbf{local policy distance} at each shared state is defined as:
\[
\overline{d}_\phi(A,B|\overline{s}) = W_2\big(\pi_A(\cdot|\overline{s}), \pi_B(\cdot|\overline{s})\big),
\]
where $W_2$ denotes the 2-Wasserstein distance.

Two global aggregation schemes are proposed:

\paragraph{Uniform Averaging:}
\[
    d_\phi^{\text{uniform}}(A,B) = \frac{1}{|\overline{S}_\phi(A,B)|} \sum_{\overline{s} \in \overline{S}_\phi(A,B)} \overline{d}_\phi(A,B|\overline{s}).
\]

\paragraph{Expectation-Based Averaging:}
 \[
    d_\phi^{\text{expected}}(A,B) = \frac{1}{2}\left( d_\phi(A|B) + d_\phi(B|A) \right),
    \]
    where
    \[
    d_\phi(X|Y) = \mathbb{E}_{o \sim O_Y, \, \phi(o) \in \overline{S}_\phi(X,Y)}\left[ \overline{d}_\phi(X,Y|\phi(o)) \right].
\]

To increase robustness, define a filtered intersection set:
\[
\overline{S}'_\phi(A,B,t) = \left\{ \overline{s} \in \overline{S}_\phi(A,B) \ \middle|\ F_\phi(\overline{s},A) \geq t, \ F_\phi(\overline{s},B) \geq t \right\}.
\]

\subsubsection{Discussion on Consistency}

The Playstyle Distance framework satisfies two criteria:
\begin{itemize}
  \item \textbf{Generality}: Requires no task-specific labels or domain features; works across tasks and agent types.
  \item \textbf{Consistency}: If an oracle metric ranks similarity as $A \sim C < B \sim C$, then the Playstyle Distance preserves that order:
  \[
  d_\phi(A,C) < d_\phi(B,C)
  \]
\end{itemize}

\subsubsection{Evaluation on TORCS: Impact of Discrete Representations and State Space Sensitivity}

We begin by evaluating Playstyle Distance in the TORCS racing simulation environment, using datasets collected from rule-based AI agents configured along two manually controlled style dimensions \citepx{torcs, gym_torcs}:
\begin{itemize}
    \item \textbf{Target Driving Speed}: Five speed settings (60, 65, 70, 75, and 80 km/h).
    \item \textbf{Driving Noise Level}: Five levels of Gaussian noise applied to steering and acceleration actions, simulating different levels of motor control smoothness.
\end{itemize}

A total of 25 datasets were generated, each corresponding to a unique (speed, noise) configuration.  
Matching query datasets were collected for playstyle retrieval evaluation using the metric described in Section~6.5.2.

The discrete state encoder $\phi$ used in computing Playstyle Distance was trained on human gameplay rather than rule-based agents.  
This design choice ensures that the resulting symbolic representation captures realistic decision tendencies and generalizes across different agent sources.  
Training was conducted via the Hierarchical State Discretization (HSD) framework introduced earlier.

\paragraph{State Space and Threshold Selection.}

To evaluate the robustness of Playstyle Distance across different parameter settings, we varied both the discrete state space size and the intersection threshold $t$.  
Results are reported in Table~\ref{Table:TORCS-Style-Accuracy-Search}.

\begin{table}[ht]
    \centering
    \caption[TORCS HSD State Space Search]{Comparison of different discrete state space sizes and thresholds for evaluating playstyle accuracy in TORCS.}
    \label{Table:TORCS-Style-Accuracy-Search}
    \begin{tabular}{|l|c|c|c|}
    \hline
    5 Speed Styles & $2^{16}$ & $2^{20}$ & $2^{24}$ \\
    \hline
    $t=1$ & 89.33\% & 89.60\% & \textbf{91.47\%} \\
    $t=2$ & 89.73\% & \textbf{91.20\%} & 89.60\% \\
    $t=4$ & 90.13\% & \textbf{90.27\%} & 86.67\% \\
    \hline
    \end{tabular}
    \vspace{1mm}
    
    \begin{tabular}{|l|c|c|c|}
    \hline
    5 Noise Styles & $2^{16}$ & $2^{20}$ & $2^{24}$ \\
    \hline
    $t=1$ & \textbf{74.13\%} & 72.67\% & 73.60\% \\
    $t=2$ & 78.27\% & \textbf{81.33\%} & 78.67\% \\
    $t=4$ & 84.53\% & \textbf{85.07\%} & 82.40\% \\
    \hline
    \end{tabular}
\end{table}

These results suggest the following:

\begin{itemize}
    \item For 5-speed classification, the largest state space ($2^{24}$) yields the best accuracy at $t=1$ but becomes less reliable as $t$ increases, likely due to sparsity in high-dimensional representations.
    \item For 5-noise classification, $2^{20}$ consistently achieves high accuracy across thresholds, indicating a good balance between capacity and stability.
    \item Threshold $t=2$ provides a reliable compromise—filtering unstable states without overly shrinking the intersection set.
\end{itemize}

We therefore adopt $2^{20}$ as the default symbolic state space and $t=2$ as the threshold for all subsequent experiments.

\paragraph{Comparison of Discrete Representations.}

With the above configuration fixed, we compared HSD against several baseline representations:

\begin{itemize}
    \item \textbf{Pixel}:  
    Raw observations (four stacked game frames) used directly as symbolic states.
    \item \textbf{Low-Resolution Discretization (LRD)}:  
    Downsampled grayscale images ($8\times8$) with quantized pixel values (16 levels).
    \item \textbf{Continuous Metric (FID-like)}:  
    Playstyle similarity computed in a continuous latent space using an FID-style Wasserstein distance between feature means.
\end{itemize}

\begin{table}[ht]
    \centering
    \caption[TORCS Playstyle Distance Accuracy]{Playstyle accuracy on TORCS using different discrete representation methods.}
    \label{Table:TORCS-Style-Accuracy-Baseline}
    \begin{tabular}{|l|c|c|c|}
    \hline
    Representation & 25 Styles & 5 Speeds & 5 Noise Levels \\
    \hline
    Pixel & 4.00\% & 20.00\% & 20.00\% \\
    LRD & 4.00\% & 20.00\% & 20.00\% \\
    HSD (ours) & \textbf{69.73\%} & \textbf{91.20\%} & \textbf{81.33\%} \\
    Continuous (FID-like) & timeout & 27.60\% & 23.20\% \\
    \hline
    \end{tabular}
\end{table}

HSD significantly outperforms all baselines.  
Pixel and LRD fail due to the near absence of intersection states, while the continuous method is not only less discriminative but also computationally expensive.  
These results validate the importance of meaningful symbolic state alignment in playstyle measurement.

\subsubsection{Evaluation on RGSK: Playstyle Measurement with Human Players}

We further evaluate Playstyle Distance on human-generated data using the Realistic Car Physics Racing Game Starter Kit (RGSK), a Unity-based racing environment available from the Unity Asset Store~\citepx{unity_ml_agent}.  
In this setup, six human participants were instructed to maintain distinct playstyles across four manually defined dimensions.

Following the settings determined in the TORCS experiments, we use a discrete state space of $2^{20}$ and an intersection threshold $t=2$ throughout this evaluation.

The four evaluated style dimensions are:

\begin{itemize}
    \item \textbf{Nitro Usage (Nitro)}:  
    Preference and timing in activating the $N_{2}O$ boost system.
    \item \textbf{Surface Preference (Surface)}:  
    Tendency to drive primarily on paved roads versus grass surfaces.
    \item \textbf{Track Positioning (Position)}:  
    Consistency in maintaining an inner or outer trajectory along the racing line.
    \item \textbf{Corner Handling (Corner)}:  
    Decision strategy when turning — drifting versus braking.
\end{itemize}

Table~\ref{Table:RGSK-Trajectory-Style-Accuracy} shows that Playstyle Distance using HSD outperforms baseline representations by a wide margin in all style dimensions.

\begin{table}[h]
    \centering
    \caption[RGSK Playstyle Distance Accuracy]{Playstyle accuracy on RGSK across four human-defined style dimensions.}
    \label{Table:RGSK-Trajectory-Style-Accuracy}
    \begin{tabular}{|l|c|c|c|}
    \hline
    Style Dimension & Pixel & LRD & HSD (ours) \\
    \hline
    Nitro Usage & 16.67\% & 33.33\% & \textbf{93.11\%} \\
    Surface Preference & 16.67\% & 16.67\% & \textbf{99.83\%} \\
    Track Positioning & 16.67\% & 33.17\% & \textbf{99.67\%} \\
    Corner Handling & 16.67\% & 42.50\% & \textbf{99.56\%} \\
    \hline
    \end{tabular}
\end{table}

These results demonstrate the robustness and generality of the Playstyle Distance framework.  
Even without access to explicit style labels or feature engineering, the framework reliably differentiates nuanced behavior patterns in human gameplay.

\vspace{1em}

\subsubsection{Evaluation on Atari: Playstyle Analysis of Learning-Based Agents}

We further apply Playstyle Distance to seven Atari 2600 games from the Arcade Learning Environment~\citepx{atari}, selected to represent a range of gameplay dynamics: Asterix, Breakout, MsPacman, Pong, Qbert, Seaquest, and Space Invaders.

Agents were trained using four reinforcement learning algorithms — DQN~\citepx{dqn}, C51~\citepx{c51}, Rainbow~\citepx{rainbow}, and IQN~\citepx{iqn} — all implemented via the Dopamine framework~\citepx{dopamine}.  
For each algorithm, five models were trained from independent seeds, producing a total of 20 agents considered as distinct playstyles.

We use the same experimental setup as in prior evaluations, with $2^{20}$ discrete states and threshold $t=2$.

Table~\ref{Table:Atari-Model-Accuracy} presents the playstyle prediction accuracy across different discrete state representations.

\begin{table}[ht]
    \centering
    \caption[Atari Playstyle Distance Accuracy]{Playstyle accuracy on Atari 2600 games using different state representations.}
    \label{Table:Atari-Model-Accuracy}
    \begin{tabular}{|l|c|c|c|}
    \hline
    Game & Pixel & LRD & HSD (ours) \\
    \hline
    Asterix & 67.10\% & 92.20\% & \textbf{98.65\%} \\
    Breakout & 59.05\% & 59.85\% & \textbf{93.83\%} \\
    MsPacman & 70.95\% & \textbf{96.75\%} & 93.85\% \\
    Pong & 34.20\% & \textbf{94.15\%} & 91.55\% \\
    Qbert & 30.00\% & 83.70\% & \textbf{94.58\%} \\
    Seaquest & 14.95\% & 42.70\% & \textbf{93.87\%} \\
    Space Invaders & 30.35\% & 56.10\% & \textbf{79.72\%} \\
    \hline
    \end{tabular}
\end{table}

HSD consistently achieves the highest prediction accuracy across most games, showing strong generalization even in deterministic or low-randomness environments like Atari.  
While LRD slightly outperforms HSD in MsPacman and Pong, the performance gaps are small and context-dependent.

Together with the TORCS and RGSK results, these findings validate the broad applicability of Playstyle Distance and support its use as a general-purpose metric across domains, agents, and playstyle sources.

\subsubsection{Uniform or Expected Distance}

In defining Playstyle Distance, two aggregation strategies over shared discrete states are considered:

\begin{itemize}
    \item \textbf{Uniform Averaging}:  
    Each intersected state contributes equally to the overall distance.
    \item \textbf{Expected Averaging}:  
    Each state's contribution is weighted by its empirical observation frequency.
\end{itemize}

\paragraph{Empirical Comparison.}

We evaluate these two aggregation strategies on both the TORCS and RGSK datasets.  
Results are summarized in Tables~\ref{Table: Uniform vs Expected - TORCS} and~\ref{Table: Uniform vs Expected - RGSK}.

\begin{table}[ht]
    \centering
    \caption[TORCS Uniform vs Expected Distance]{Comparison of uniform vs expected distance on TORCS (playstyle prediction accuracy).}
    \label{Table: Uniform vs Expected - TORCS}
    \begin{tabular}{|l|c|c|}
    \hline
    Aggregation Strategy & 5 Speed Styles & 5 Noise Styles \\
    \hline
    Uniform ($t = 2$) & 76.00\% & 65.20\% \\
    Uniform ($t = 4$) & 83.47\% & 74.40\% \\
    Expected ($t = 2$) & \textbf{91.20\%} & \textbf{81.33\%} \\
    Expected ($t = 4$) & \textbf{90.27\%} & \textbf{85.07\%} \\
    \hline
    \end{tabular}
\end{table}

\begin{table}[ht]
    \centering
    \caption[RGSK Uniform vs Expected Distance]{Comparison of uniform vs expected distance on RGSK (playstyle prediction accuracy).}
    \label{Table: Uniform vs Expected - RGSK}
    \begin{tabular}{|l|c|c|c|c|}
    \hline
    Aggregation Strategy & Nitro & Surface & Position & Corner \\
    \hline
    Uniform ($t=2$) & 92.33\% & 99.28\% & 99.33\% & 99.50\% \\
    Uniform ($t=4$) & 86.94\% & 96.44\% & 95.56\% & 95.61\% \\
    Expected ($t=2$) & \textbf{93.11\%} & \textbf{99.83\%} & \textbf{99.67\%} & \textbf{99.56\%} \\
    Expected ($t=4$) & 88.78\% & 98.39\% & 98.33\% & 98.44\% \\
    \hline
    \end{tabular}
\end{table}

From these results, we observe the following:

\begin{itemize}
    \item In TORCS, expected averaging consistently outperforms uniform averaging across both speed and noise style dimensions.
    \item In RGSK, the difference is relatively minor, likely due to denser and more structured human demonstrations, but expected averaging still yields the best performance.
    \item Increasing the threshold $t$ mitigates instability from rare states, but expected averaging remains robust even under lower thresholds.
\end{itemize}

Accordingly, we adopt \textbf{expected averaging} as the default aggregation strategy for Playstyle Distance in all subsequent experiments.

\vspace{1em}

\subsubsection{Alternative Distribution Distance Metrics}

While Playstyle Distance primarily uses the 2-Wasserstein distance ($W_2$) to quantify divergences between action distributions, the framework is flexible and supports alternative metrics.

We evaluate the following:

\begin{itemize}
    \item \textbf{Wasserstein-1 distance ($W_1$)}:  
    Uses $L_1$ norm instead of $L_2$; often more robust in sparse settings.
    
    \item \textbf{Kullback–Leibler divergence ($D_{KL}$)}:  
    A common but asymmetric metric measuring information loss.

    \item \textbf{Average KL divergence ($M_{KL}$)}:
The average of $D_{KL}(P | Q)$ and $D_{KL}(Q | P)$, providing a symmetric but not bounded divergence.
While similar in form to Jensen-Shannon divergence, it is not equivalent.
\end{itemize}

The comparison results on RGSK are summarized in Table~\ref{Table: W1 vs W2 vs LK vs AvgKL - RGSK}.

\begin{table}[ht]
    \centering
    \caption{Comparison of distance metrics for action distributions in RGSK.}
    \label{Table: W1 vs W2 vs LK vs AvgKL - RGSK}
    \begin{tabular}{|l|c|c|c|c|}
    \hline
    Metric & Nitro & Surface & Position & Corner \\
    \hline
    $W_1$ & \textbf{95.28\%} & \textbf{100.00\%} & \textbf{99.83\%} & \textbf{99.83\%} \\
    $W_2$ & 93.11\% & 99.83\% & 99.67\% & 99.56\% \\
    $D_{KL}$ & 37.44\% & 74.28\% & 58.06\% & 56.33\% \\
    $M_{KL}$ (approx. JS) & 76.78\% & 96.06\% & 91.94\% & 87.33\% \\
    \hline
    \end{tabular}
\end{table}

Key observations:

\begin{itemize}
    \item Wasserstein distances ($W_1$ and $W_2$) vastly outperform KL-based metrics.
    \item $W_1$ achieves slightly better accuracy than $W_2$ in this setting, though the margin is small.
    \item Symmetrized KL ($M_{KL}$) improves considerably over $D_{KL}$, but still falls behind Wasserstein metrics.
\end{itemize}

Despite the slight edge of $W_1$ in RGSK, we adopt $W_2$ as the default for the following reasons:

\begin{itemize}
    \item It provides smoother gradients, which are useful for optimization-based applications.
    \item It incorporates covariance structure in continuous action spaces.
    \item It is widely used in generative modeling and behavioral distribution comparisons.
\end{itemize}

\paragraph{Summary.}

To summarize, Playstyle Distance is computed over intersected discrete states using an HSD-trained symbolic encoder.  
We measure local action divergence using the 2-Wasserstein distance and aggregate across states using expected averaging.  
This combination yields strong performance and broad applicability, forming the foundation for the multiscale and probabilistic extensions introduced in the following sections.

\subsection{Multiscale States}
\label{sec:multiscale_states}

The performance of playstyle measurement fundamentally depends on the quality and coverage of intersected states used to compare decision patterns. However, any single discrete state space often imposes a trade-off: coarser abstractions increase intersection coverage but reduce stylistic resolution, while finer discretizations offer more expressive power but suffer from sparse overlaps. To address this limitation, we introduce a multiscale formulation of \textit{Playstyle Distance}, which integrates multiple symbolic abstractions simultaneously.

\subsubsection{Multiscale Formulation}

We generalize the original mapping function $\phi$ to a set of discrete encoders $\Phi = \{\phi_1, \phi_2, ..., \phi_k\}$, where each $\phi$ projects observations into a different level of symbolic granularity (e.g., coarse, mid-level, fine). The multiscale expected Playstyle Distance is then defined as:

\begin{equation}
\label{equation:multiscale_playstyle_distance}
d_\Phi(A, B) = \frac{1}{2} \left[ \frac{1}{|\Phi|} \sum_{\phi \in \Phi} d_\phi(A | B) + \frac{1}{|\Phi|} \sum_{\phi \in \Phi} d_\phi(B | A) \right],
\end{equation}

where each $d_\phi(X | Y)$ is the expected action distribution distance under encoder $\phi$:

\begin{equation}
d_\phi(X | Y) = \mathbb{E}_{o \sim O_Y,\phi(o) \in \overline{S}_\phi(X,Y)}[\overline{d}_\phi(X,Y|\phi(o))].
\end{equation}

In our experiments, we adopt a three-level hierarchy:
\begin{enumerate}
    \item $\phi_1$: singleton mapping (i.e., all observations map to the same symbolic state)
    \item $\phi_2$: HSD encoder with a $2^{20}$-sized discrete space (intermediate resolution)
    \item $\phi_3$: base hierarchy of HSD (fine-grained encoding, denoted as $256^{\text{res}}$)
\end{enumerate}

This composite formulation eliminates the need for strict intersection thresholds and enables robust decision alignment even when symbolic representations vary across datasets or environments.

\subsubsection{Empirical Results and Benefits}

We empirically evaluate the multiscale version of \textit{Playstyle Distance} against single-scale baselines across three types of environments: TORCS (rule-based racing agents), RGSK (human driving behaviors), and Atari (learning-based agents). Table~\ref{Table:multiscale_state_space_efficacy} summarizes playstyle prediction accuracies using different discrete state configurations.

Compared to fixed-resolution settings (e.g., $2^{20}$ or base-level $256^\text{res}$), the multiscale formulation (denoted as ``mix'') consistently achieves competitive or superior results. Notably:

\begin{itemize}
  \item In environments like TORCS and Atari (e.g., Asterix, Breakout), the multiscale approach improves robustness, reducing the need to tune the threshold $t$ for stable measurement.
  \item In human-play data (RGSK), where behavior is more varied, the multiscale method remains on par with the best single-scale setting while simplifying hyperparameter selection.
\end{itemize}

These findings confirm that multiscale abstraction mitigates the resolution--alignment trade-off, providing a more reliable and general-purpose tool for cross-style comparison across diverse domains.

\begin{table}[ht]
\centering
\caption[Multiscale Playstyle Prediction Accuracy]{Playstyle prediction accuracy (\%) $\pm$ standard deviation across different discrete state mappings, thresholds $t$, and environments. ``mix'' denotes the multiscale configuration using $\{1, 2^{20}, 256^{\text{res}}\}$.}
\label{Table:multiscale_state_space_efficacy}
\rotatebox{90}{
\begin{tabular}{l|c|cc|cc|cc}
\toprule
\textbf{Environment} & \textbf{1} & \textbf{$2^{20}$ ($t=2$)} & ($t=1$) & \textbf{$256^{\text{res}}$ ($t=2$)} & ($t=1$) & \textbf{mix ($t=2$)} & ($t=1$) \\
\midrule
TORCS & 35.1 $\pm$ 9.1 & 73.3 $\pm$ 8.2 & 66.5 $\pm$ 7.9 & 4.3 $\pm$ 3.1 & 60.9 $\pm$ 9.4 & \textbf{77.3 $\pm$ 7.4} & \textbf{77.5 $\pm$ 7.9} \\
\midrule
RGSK & 81.0 $\pm$ 7.2 & 79.2 $\pm$ 7.9 & \textbf{93.7 $\pm$ 4.7} & 5.7 $\pm$ 2.5 & 25.6 $\pm$ 7.2 & 78.8 $\pm$ 7.5 & \textbf{93.5 $\pm$ 4.3} \\
\midrule
Asterix & 25.2 $\pm$ 9.0 & \textbf{99.9 $\pm$ 0.5} & \textbf{100.0 $\pm$ 0.0} & 49.6 $\pm$ 7.7 & 32.7 $\pm$ 8.0 & \textbf{100.0 $\pm$ 0.0} & \textbf{100.0 $\pm$ 0.0} \\
Breakout & 32.7 $\pm$ 9.2 & \textbf{99.4 $\pm$ 1.6} & \textbf{99.9 $\pm$ 0.6} & 65.9 $\pm$ 8.5 & 29.9 $\pm$ 9.4 & \textbf{99.8 $\pm$ 1.1} & \textbf{99.9 $\pm$ 0.2} \\
MsPacman & \textbf{100.0 $\pm$ 0.0} & \textbf{99.9 $\pm$ 0.5} & \textbf{100.0 $\pm$ 0.0} & 92.8 $\pm$ 4.0 & \textbf{100.0 $\pm$ 0.0} & \textbf{100.0 $\pm$ 0.0} & \textbf{100.0 $\pm$ 0.0} \\
Pong & 49.9 $\pm$ 9.7 & 92.1 $\pm$ 2.7 & 92.3 $\pm$ 2.6 & 50.7 $\pm$ 9.5 & 52.2 $\pm$ 9.9 & \textbf{93.1 $\pm$ 3.2} & 92.4 $\pm$ 2.6 \\
Qbert & \textbf{99.9 $\pm$ 0.5} & \textbf{100.0 $\pm$ 0.0} & \textbf{100.0 $\pm$ 0.0} & 90.1 $\pm$ 5.3 & 91.6 $\pm$ 4.6 & \textbf{99.9 $\pm$ 0.5} & \textbf{100.0 $\pm$ 0.0} \\
Seaquest & 82.0 $\pm$ 7.6 & \textbf{99.7 $\pm$ 1.2} & \textbf{99.9 $\pm$ 0.6} & 17.1 $\pm$ 5.2 & 16.7 $\pm$ 4.9 & \textbf{99.9 $\pm$ 0.3} & \textbf{99.9 $\pm$ 0.2} \\
Space Invaders & 73.1 $\pm$ 8.5 & 98.7 $\pm$ 2.3 & \textbf{99.7 $\pm$ 1.2} & 50.4 $\pm$ 5.7 & 49.6 $\pm$ 8.4 & \textbf{99.9 $\pm$ 0.5} & \textbf{99.9 $\pm$ 0.6} \\
\bottomrule
\end{tabular}
}
\end{table}

\clearpage

\subsection{Perception of Similarity}
\label{subsec:perception_of_similarity}

While distance-based playstyle metrics offer a principled foundation for quantifying behavioral differences, they suffer from an important limitation: distance measures primarily encode dissimilarity, not the degree of similarity. Although distance is a common measure for determining similarity, a larger distance value conveys primarily that two entities are different, without giving much insight into the degree of their similarity. For example, given a point in 2D space, the candidate points with the same distance to the given point form a circle. As the distance increases, the size of this candidate circle also increases, and the similarity information is diluted as illustrated in Figure~\ref{figure:2d_circle_example}. This phenomenon has been observed in human decision-making as the \textit{Magnitude Effect}, suggesting diminished sensitivity to larger numbers \citepx{magnitude_effect}. This aligns with the \textit{Weber–Fechner Law} in psychophysics \citepx{weber_fechner_law}, which models the relationship between stimulus intensity and perceived change as logarithmic. In other words, as the absolute difference increases, perceptual sensitivity to that difference declines—a pattern mirrored in human similarity judgments.

\begin{figure}[ht]
    \centering
    \includegraphics[width=0.5\textwidth]{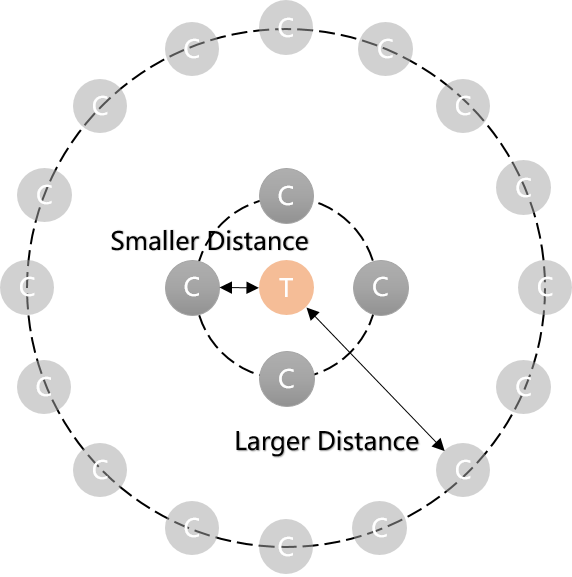}
    \caption[Degree of Similarity]{Degree of Similarity: This demonstrates how multiple candidate points C can share identical distance values from a target point T, emphasizing that as distance increases, the degree of similarity information diminishes.}
    \label{figure:2d_circle_example}
\end{figure}

\subsubsection{Perceptual Similarity Function}

We propose a probability-based model for similarity. In this model, greater similarity (i.e., smaller distance) corresponds to a probability closer to 100\%, while lesser similarity (larger distance) approaches 0\%. This proposed probability function aligns with the logarithmic human perceptual sensitivity to differences. Specifically, we use the exponential kernel to describe the probability of similarity, with the mapping function given by $P(d) = \frac{1}{e^d}$, where $d$ is the distance value from the policy distance function $D(\pi_X,\pi_Y)$ with 2-Wasserstein distance. This perceptual relation is the only relation under our assumptions from human cognition and probability. The original paper by \citetx{playstyle_similarity} provides a proof using differential equations in the Appendix.

This exponential transformation can also be found in the radial basis function \citepx{rbf} and Bhattacharyya coefficient \citepx{bhattacharyya}. The Bhattacharyya coefficient $BC(P,Q)$ measures the similarity between two probability distributions $P$ and $Q$, and it is related to the overlapping region between these two distributions. It is defined as $BC(P,Q) = \int_\mathcal{X} \sqrt{P(x)Q(x)}dx $. The Bhattacharyya distance, derived from the coefficient, is $ D_{B}(P,Q) = -ln(BC(P,Q)) $, and the inversion is $BC(P,Q) = exp(-D_{B}(P,Q))$.

Thus, we define a new playstyle measure $PS_\Phi^{\cap}(M_A,M_B)$ with probability of similarity in Equation~\ref{equation:perceptual_similarity}:
\begin{equation}
\label{equation:perceptual_similarity}
\begin{aligned}
& PS_\Phi^{\cap}(M_A,M_B) = \frac{\sum _ {s \in \bigcup_{\phi \in \Phi}\phi (M_A) \cap \phi(M_B)} P(D_\Phi ^M(\pi_{M_A}(s),\pi_{M_B}(s)))}{|\bigcup_{\phi \in \Phi}\phi (M_A) \cap \phi(M_B)|}, \\ & \textrm{where } D_{\Phi}^M(\pi_X,\pi_Y) = \frac{D(\pi_X,\pi_Y)}{\overline{D}_{\Phi}^{M}}
\end{aligned}
\end{equation}

The measure has been simplified by adopting a uniform average distance instead of an expected value. This not only streamlines calculations but also underscores the significance of encoder granularity. In particular, an intricate encoder with a vast state space may be accorded greater weight, especially if the intricate encoder reveals more intersection states. To match our probabilistic framework, we rescale the distances with a constant, $\overline{D}_{\Phi}^{M}$, ensuring the expected distance converges to 1. The constant $\overline{D}_{\Phi}^{M}$ can be calculated by averaging all observed distance on each discrete state in comparisons. Collectively, our revamped measure provides a probabilistic lens to interpret similarity, firmly rooted in cognitive theory and tailored for human comprehension. In practical terms, it transforms a raw distance metric into a probabilistic score interpretable as a likelihood of behavioral similarity.

\subsubsection{The Meaning of Different Distance Metrics}

Recalling our earlier discussion in Chapter~3 on \textit{style evolution} and the resistance to changing styles, the Wasserstein distance can be interpreted as the "effort" required to transition between different playstyle action distributions.  
If we imagine starting from a neutral or uniform distribution, the cost to reach a target style might be smaller than starting from an existing extreme style and moving to its opposite.  
From a state-conditioned perspective, this is like having to first remove an existing bias (subtracting one) and then add the opposite bias (adding one)—a process that incurs more cost than directly adding from a neutral state.  
This makes Wasserstein distance a natural measure for quantifying the resistance or friction in changing playstyles, whether in human execution or policy adaptation.

The Bhattacharyya distance, in contrast, is not about this "effort".  
Instead, it gauges the likelihood that two playstyles will result in the same action, due to its relation to the overlapping regions between two distributions.  
Thus, while \textit{Perceptual Similarity} is built on the idea of the effort needed to change playstyles, the \textit{BC Similarity} (or its variant \textit{BD Similarity}) is built on the frequency of identical actions.  
A player might care more about the effort involved in changing styles, while an observer may focus more on the actions they witness.

Formally, the Bhattacharyya distance is defined through the Bhattacharyya coefficient $BC$ with value range $[0,1]$, and the corresponding distance $D_{B} = -\ln(BC)$.  
For discrete probability distributions:  
\[
BC(P,Q) = \sum_{x \in \mathcal{X}} \sqrt{P(x)Q(x)}.
\]  
For continuous probability distributions, such as the action space in racing games like TORCS, it involves integration over probability density functions:  
\[
BC(P,Q) = \int_{x \in \mathcal{X}} \sqrt{p(x)q(x)}.
\]  
We adopt the multivariate normal formulation of $D_B$ \citep{bhattacharyya}, where $p_i = \mathcal{N}(\mu_i,\Sigma_i)$:
\begin{equation}
\begin{aligned}
    & D_{B}(p_1,p_2) = \frac{1}{8}(\mu_1-\mu_2)^T\Sigma^{-1}(\mu_1-\mu_2)
    +\frac{1}{2}\ln\left(\frac{\det \Sigma}{\sqrt{\det \Sigma_1 \det \Sigma_2}}\right), \\
    & \text{where} \quad \Sigma = \frac{\Sigma_1+\Sigma_2}{2}.
\end{aligned}
\end{equation}
We clip $D_B$ to a maximum of 10 to prevent extremely large values from affecting average scaling ($\frac{1}{e^{10}} \approx 0\%$).  
A small $\epsilon = 10^{-8}$ is added when computing determinants to handle singular matrices.

\subsubsection{Empirical Results}

To elucidate the benefits of this modification, we examine the relationship between accuracy and dataset size of the sampled observation-action pairs. These pairs are evaluated under a single discrete state space $\{2^{20}\}$, without employing a sample count threshold, to provide a clear assessment of the transformation from distance to similarity. Further comparisons with different discrete state spaces can be found in the Appendix of our original paper \citepx{playstyle_similarity}.

We evaluate several measures in this comparison:
\begin{itemize}
\item \textit{Playstyle Distance}: $-d_\Phi$
\item \textit{Playstyle Intersection Similarity}: $PS_\Phi^{\cap}$
\item \textit{Playstyle Inter BD Similarity}: $PS_\Phi^{\cap BD}$, a variant of $PS_\Phi^{\cap}$ that employs the Bhattacharyya distance in place of the 2-Wasserstein distance
\item \textit{Playstyle Inter BC Similarity}: $PS_\Phi^{\cap BC}$, the Bhattacharyya coefficient version, which omits the scaling coefficient before the perceptual kernel $\frac{1}{e^d}$
\item \textit{Random}: A uniform random baseline that is a common result from supervised learning or contrastive learning if there is no style label or group (like self and others) information in the training data.
\end{itemize}

Results presented in Figure~\ref{figure:perceptual_similarity} suggest that probabilistic similarity can be a good alternative to distance-based similarity, offering improved explainability in terms of measure values. Among the methods evaluated, the 2-Wasserstein distance with a perceptual kernel and the Bhattacharyya coefficient emerge as superior candidates. The intention behind using probabilistic similarity is that it provides a consistent measure of similarity across different games (via likelihood). For distance similarity, understanding the property's and distribution of distance is essential to interpret the measure value. Besides the explainability of similarity values, the transformed similarity value can be incorporated with the Jaccard index as described latter. The evidence shows that results with probabilistic similarity are not worse than distance similarity and are slightly better on TORCS, which includes slightly different playstyles. The upcoming experiments also support the idea of probabilistic similarity under slightly changed playstyles, such as rule-based TORCS agents and Atari game agents trained with the same algorithm.

\begin{figure}
	\centering
	\includegraphics[width=\textwidth]{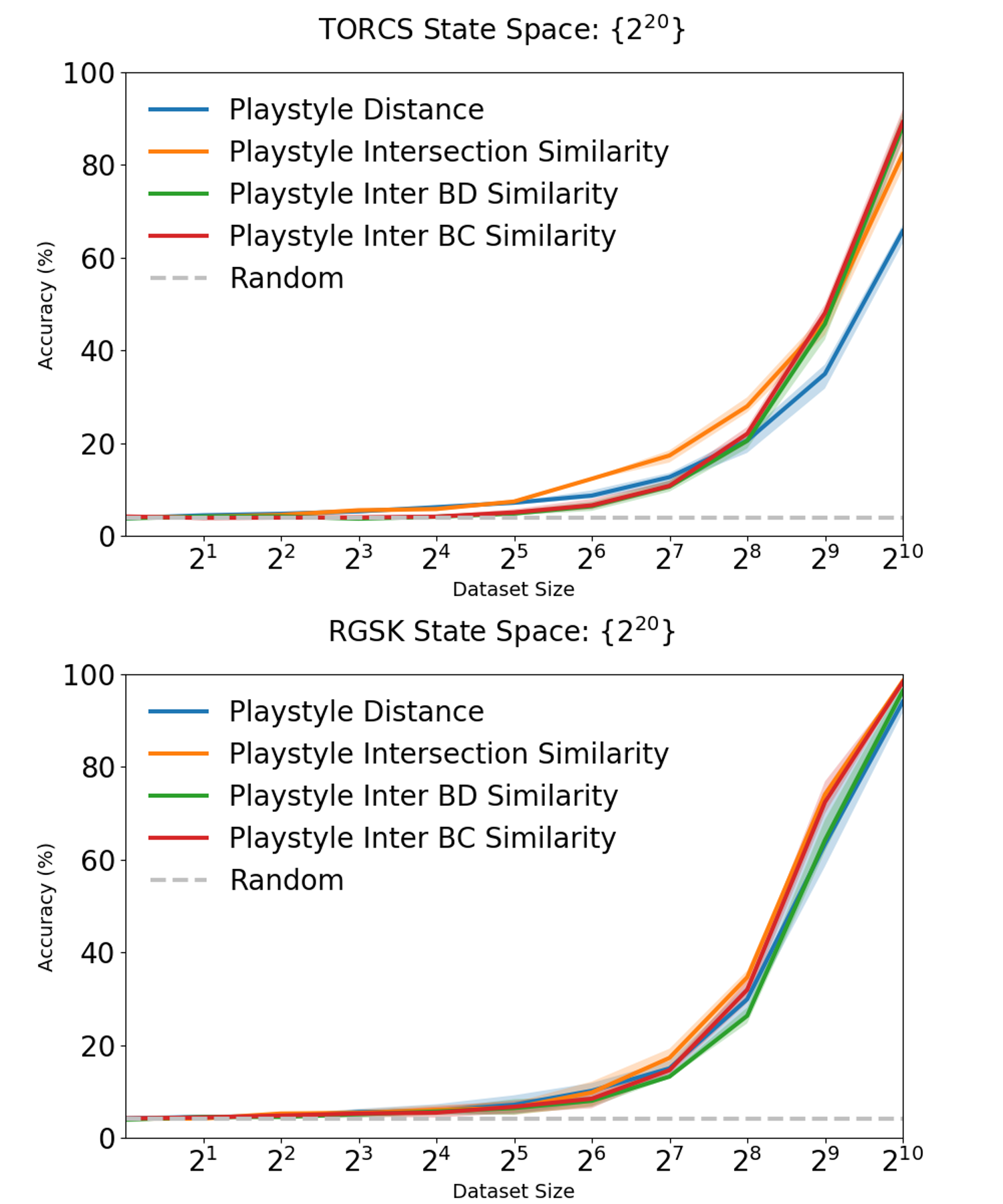}
	\caption[Experiment on Probabilistic vs. Distance Approaches]{Comparison of Efficacy: Probabilistic vs. Distance Approaches. The plot illustrates the relationship between accuracy (Y-axis) and size of the sampled observation-action pairs (X-axis). The shaded area indicates the range between min and max accuracy among three encoder models.}
    \label{figure:perceptual_similarity}
\end{figure}

\clearpage

\subsection{Beyond Intersection} \label{subsec:beyond_intersection}

Before introducing our final playstyle similarity metric, it is essential to revisit a key structural element of \textit{Playstyle Distance}: the reliance on intersected symbolic states. This intersection determines the contexts under which action distributions can be meaningfully compared. However, when the number of intersected states is small, the resulting measurements can become unstable or uninformative. A small intersection may stem from two fundamentally different causes: it may reflect genuinely divergent playstyles that explore different parts of the state space, or it may simply be the result of stochastic factors such as environment randomness, task variability, or other agents’ behavior.

To address this ambiguity, it is valuable to consider not only the absolute count of intersected states, but also their relative proportion with respect to the union of observed states. A natural candidate for capturing this proportion is the \textit{Jaccard index} \citepx{jaccard_index}, also known as Intersection over Union (IoU), widely used in information retrieval and set similarity contexts.

In the context of playstyle analysis, the Jaccard index can serve as a lightweight proxy for similarity—especially in deterministic or low-entropy environments where symbolic states provide clear playstyle cues. However, this approach has limitations. In highly stochastic domains, or in tasks where state coverage is near-universal (e.g., in games with simple dynamics or single-state settings like K-armed bandits~\citepx{rl_book}), the Jaccard index becomes insensitive to stylistic nuance.

Nevertheless, our observations suggest that in structured environments with ample state granularity and controlled randomness, the Jaccard index can still be a practical and interpretable tool for style measurement. We formalize the multiscale variant of Jaccard similarity as:

\begin{equation} \label{equation:jaccard_index} J_\Phi(M_A,M_B)=\frac{|\bigcup_{\phi \in \Phi}\phi (M_A) \cap \phi(M_B)|}{|\bigcup_{\phi \in \Phi}\phi (M_A) \cup \phi(M_B)|} \end{equation}

This formulation allows Jaccard-based similarity to be evaluated over symbolic abstractions at multiple scales, increasing robustness in partially aligned or sparse data regimes.

\clearpage

\subsection{Playstyle Similarity} 
\label{subsec:playstyle_similarity}

Throughout our investigation, we have introduced and evaluated several discrete playstyle similarity metrics. Synthesizing these insights, we propose a unified measure termed \textit{Playstyle Similarity}, denoted $PS_\Phi^{\cup}(M_A, M_B)$. This metric combines perceptual similarity over intersected states with an intersection-over-union weighting, yielding:

\begin{equation}
\begin{aligned}
& PS_\Phi^\cup(M_A,M_B) = J_\Phi(M_A,M_B) \times PS_\Phi^{\cap}(M_A,M_B) \\
& = \frac{\sum _ {s \in \bigcup_{\phi \in \Phi}\phi (M_A) \cap \phi(M_B)} P(D_\Phi ^M(\pi_{M_A}(s),\pi_{M_B}(s)))}{|\bigcup_{\phi \in \Phi}\phi (M_A) \cup \phi(M_B)|}
\end{aligned}
\end{equation}

Unlike a pure Jaccard index, which gives each intersected state equal weight, this formulation incorporates a similarity-aware weighting—assigning values between 0 and 1 based on the degree of action distribution similarity. This not only enhances interpretability but also provides more fine-grained resolution in distinguishing playstyles.

Additionally, the formulation naturally handles cases where states are not shared. These are treated as maximally dissimilar (i.e., contributing 0 to the numerator), aligning well with intuitive interpretations of playstyle divergence. This integration of perceptual similarity and state overlap is illustrated in Figure~\ref{figure:playstyle_similarity_example}.

\begin{figure}[ht] 
\centering 
\includegraphics[width=0.8\textwidth]{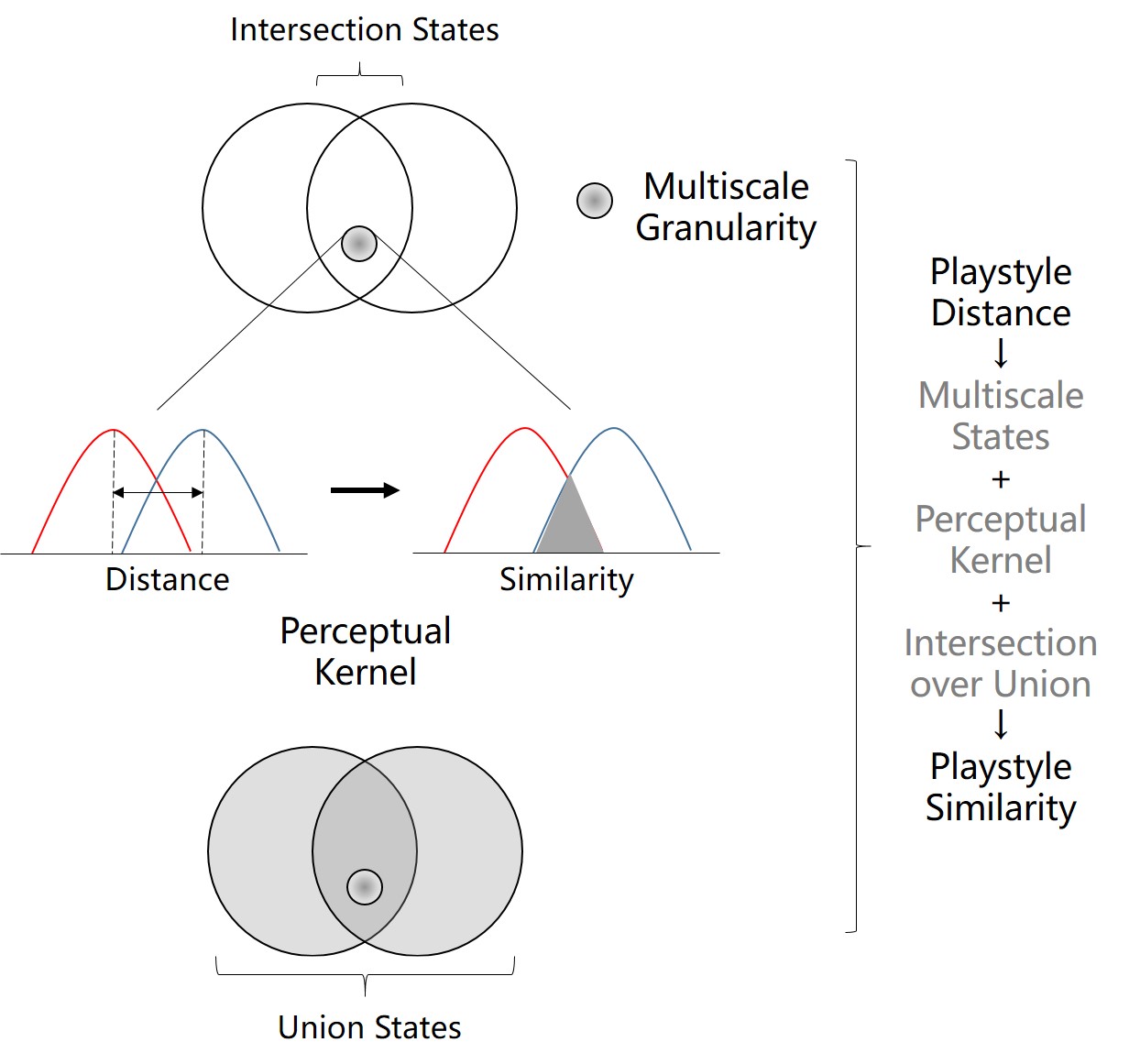} 
\caption[From Playstyle Distance to Playstyle Similarity]{Transforming Playstyle Distance into Playstyle Similarity. The pipeline begins by projecting observations into multiscale discrete states. Each intersected state’s action distribution distance is converted via a perceptual kernel, and the result is aggregated using a Jaccard-weighted formulation over all visited states.} 
\label{figure:playstyle_similarity_example} 
\end{figure}

\subsubsection{Empirical Results}

To evaluate the proposed metric, we conduct full-dataset experiments across three environments: TORCS, RGSK, and a unified Atari platform.

For Atari, we pool gameplay traces from 20 DRL agents across 7 games, treating each agent-game pair as a distinct playstyle (totaling 140 styles). Instead of aligning action semantics across games, we pad all action vectors to the maximal Atari action space size (18). Each game contributes two discrete encoders, plus one shared cross-game encoder, resulting in 15 total state encoders used for multiscale evaluation.

The compared methods include: 
\begin{itemize}
\item \textit{Playstyle Distance}: $-d_\Phi$
\item \textit{Playstyle Intersection Similarity}: $PS_\Phi^{\cap}$
\item \textit{Playstyle Inter BC Similarity}: $PS_\Phi^{\cap BC}$
\item \textit{Playstyle Jaccard Index}: $J_\Phi$
\item \textit{Playstyle Similarity}: $PS_\Phi^{\cup}$
\item \textit{Playstyle BC Similarity}: $PS_\Phi^{\cup BC}$, the union version of \textit{Playstyle Inter BC Similarity}
\item \textit{Random}: A uniform random baseline.
\end{itemize}

As shown in Figure~\ref{figure:full_data_eval}, \textit{Playstyle Similarity} consistently achieves the highest accuracy across environments. Jaccard index, though simple, still performs competitively—especially in low-randomness settings like Atari. Most notably, even with only 512 observation-action pairs, our metric surpasses 90\% accuracy, indicating the potential for real-time playstyle detection before an episode concludes.

\begin{figure}[ht] 
\centering 
\includegraphics[width=0.75\textwidth]{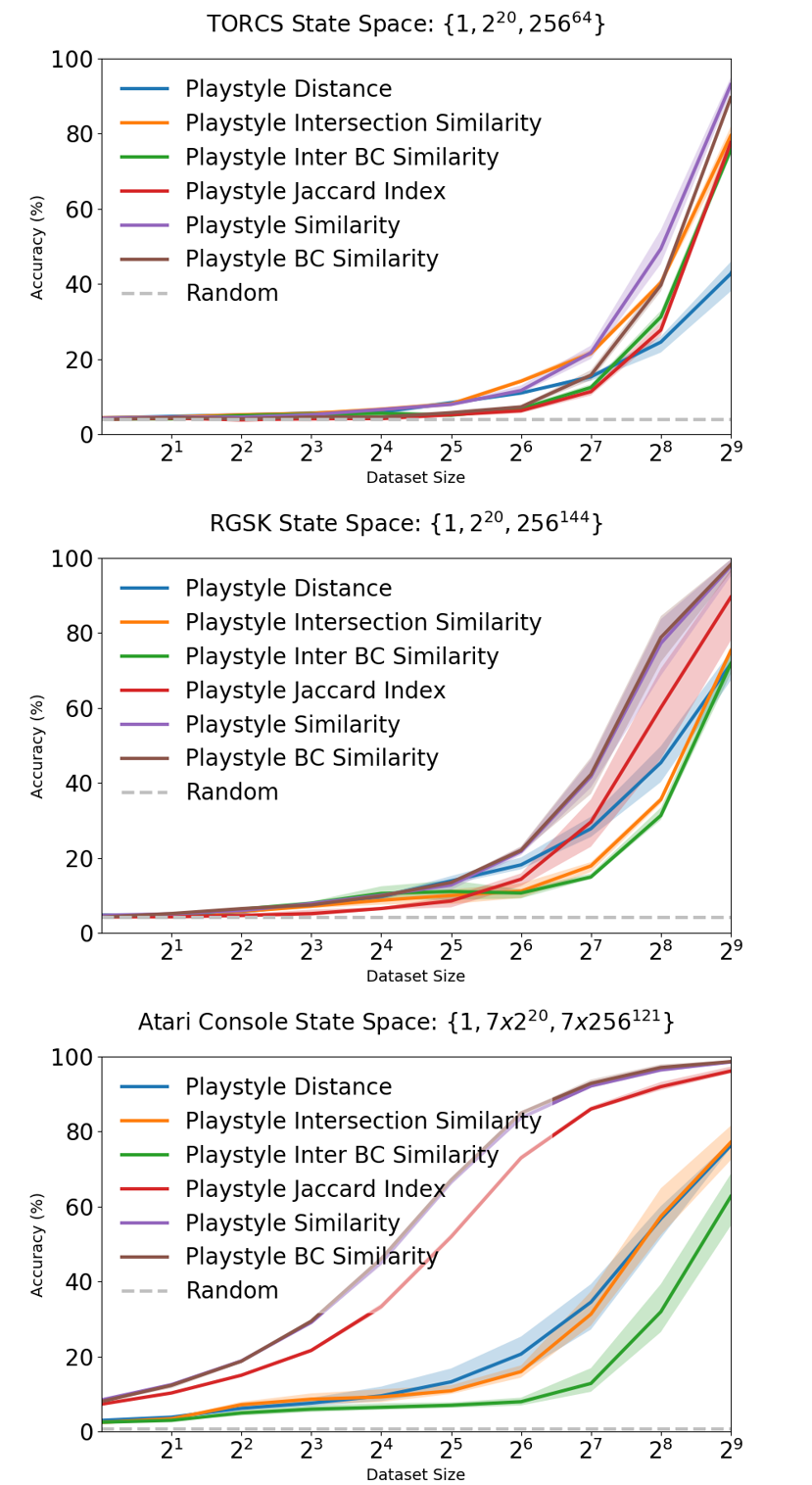}
\caption[Playstyle Measure Evaluation]{Evaluation of different playstyle metrics in TORCS, RGSK, and Atari. Accuracy is computed via retrieval of matching styles under different sample sizes. Shaded areas show variation across encoder models.} 
\label{figure:full_data_eval} 
\end{figure}

\subsection{Continuous Playstyle Spectrum} 
\label{subsec:continuous_playstyle}

While traditional playstyle analysis often focuses on discrete distinctions between players or agents, many real-world scenarios exhibit a continuous spectrum of behavioral variation. In this section, we evaluate whether playstyle similarity measures can reliably reflect subtle, progressive differences in behavior. Specifically, we examine whether the similarity values are consistent with expected gradations in playstyle and whether these values decrease or increase smoothly along known behavioral dimensions.

We conduct this evaluation in the TORCS environment, which provides two controllable axes of playstyle variation: 
\begin{itemize} 
	\item \textbf{Target Speed} (60, 65, 70, 75, 80 km/h) 
	\item \textbf{Action Noise Level} (N0 to N4) 
\end{itemize} 
These parameters define a $5\times5$ grid of agent configurations, forming a natural continuous spectrum of driving styles.

\subsubsection{Experimental Design}

We consider two types of reference cases: 
\begin{itemize} 
	\item \textbf{Corner Case}: where the target agent is at the top-left of the behavior spectrum (Speed60N0). We expect similarity to decrease as speed or noise increases. 
	\item \textbf{Center Case}: where the reference agent (Speed70N2) is surrounded by both similar and dissimilar variants. We test whether similarity is symmetric and decreases away from the center. 
\end{itemize}

Each measure is evaluated over 100 rounds of random subsampling, where each sampled dataset contains 512 observation-action pairs. For each reference agent, we compute the similarity to the other 24 variants and evaluate whether the similarity values respect the expected ordering — i.e., whether the closest variants yield the highest similarity, and distant variants the lowest. A similarity sequence is considered consistent if it follows a strict monotonic trend along one axis (e.g., increasing speed).

\subsubsection{Visual Illustration: Corner Case}

Figure~\ref{figure:continuous_playstyle_similarity_heatmap} provides a heatmap visualization of similarity scores from the Playstyle Similarity (mixed state space) measure, using Speed60N0 as the query point. The smooth gradient from top-left to bottom-right demonstrates that similarity degrades coherently with both increasing speed and noise. This confirms the measure’s ability to detect fine-grained behavioral divergence.

\begin{figure}[ht] 
\centering 
\includegraphics[width=0.7\textwidth]{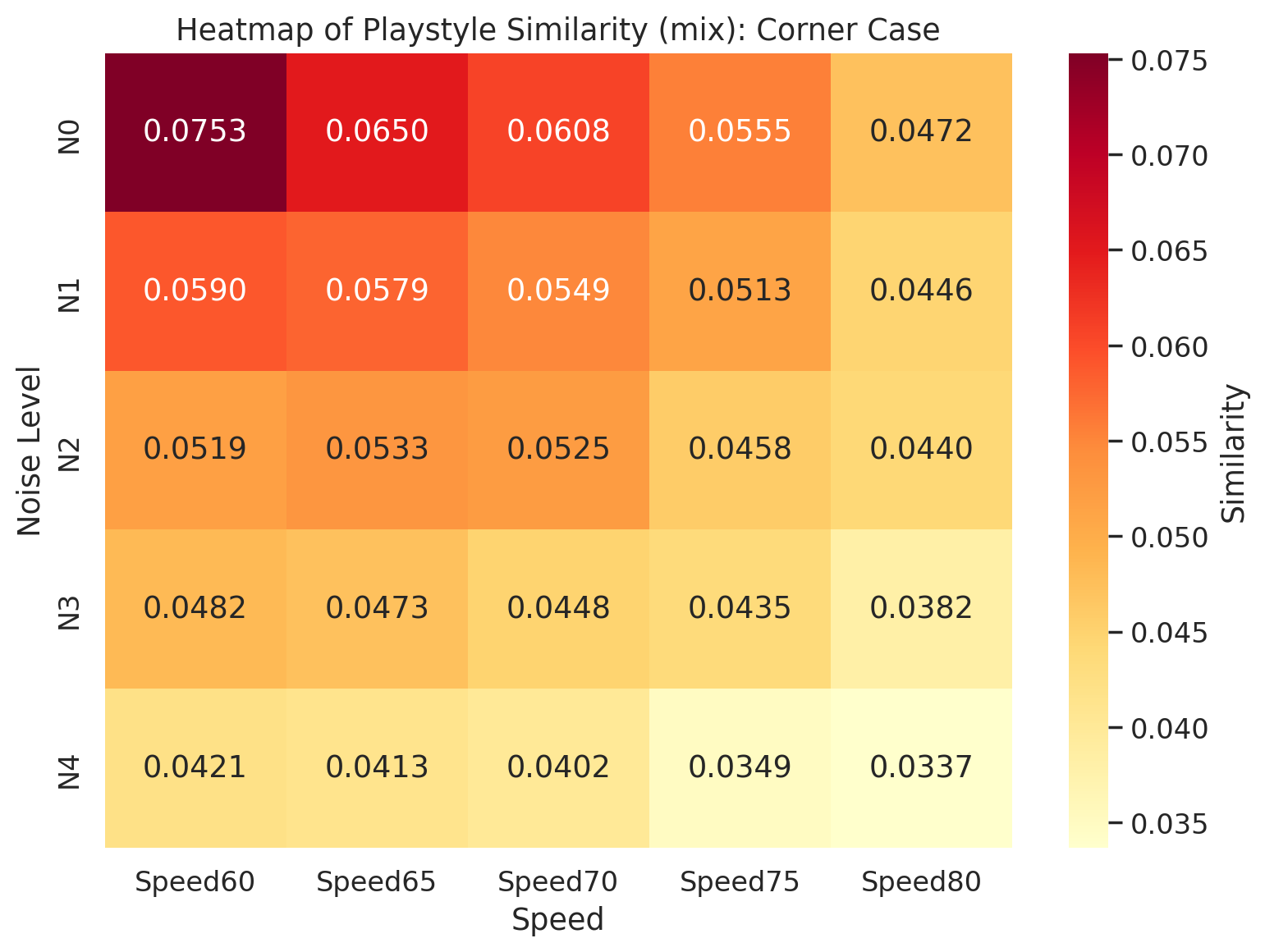} \caption[Heatmap of Playstyle Similarity]{Heatmap of Playstyle Similarity (mix) using Speed60N0 as the reference. Rows indicate noise levels (N0 to N4), and columns represent target speeds (60 to 80 km/h). Darker colors denote higher similarity.} \label{figure:continuous_playstyle_similarity_heatmap} 
\end{figure}

\subsubsection{Summary of Consistency Results}

Table~\ref{table:continuous_playstyle_summary} summarizes the number of consistent trend sequences observed under each metric, for both the Corner Case and Center Case. The proposed \textit{Playstyle Similarity (mix)} measure achieves the highest consistency across both settings, confirming its robustness and smooth response over continuous playstyle variations.

\begin{table}[ht] 
\centering 
\caption{Consistency count of two continuous playstyle spectrum cases.} 
\label{table:continuous_playstyle_summary} 
\begin{tabular}{l|ll} 
\toprule \quad & Corner Case & Center Case \\ 
\midrule 
Playstyle Distance ($2^{20}$) & 3 & 2\\
Playstyle Distance (mixed) & 4 & 2\\ 
Playstyle Intersection Similarity (mixed) & 8 & 2\\ 
Playstyle Jaccard Index (mixed) & 5 & 1\\ 
\textbf{Playstyle Similarity (mixed)} & \textbf{9} & \textbf{3}\\
\bottomrule 
\end{tabular} 
\end{table}

\subsubsection{Discussion}

The results highlight two key insights: 
\begin{enumerate} 
	\item Measures based solely on raw distances (e.g., $W_2$) or state co-occurrence (e.g., Jaccard index) can be unstable or insensitive when facing gradual behavioral transitions. 
	\item Probabilistic similarity, especially when integrated with multiscale state discretization, provides both stability and perceptual fidelity across subtle style changes. 
\end{enumerate}

These findings reinforce the necessity of perceptual transformation and state hierarchy in designing general-purpose playstyle metrics. More detailed consistency matrices and center case visualizations can be found in our original Appendix \citepx{playstyle_similarity}.

\clearpage

\section{Discussion} 

Having established a set of discrete, perceptually grounded playstyle similarity measures and evaluated them across multiple environments and metrics, we now turn to a broader discussion on their positioning relative to other unsupervised alternatives, and their applicability to complex or high-uncertainty scenarios. This section consolidates prior experiments, compares to latent-feature-based approaches, and highlights important game-dependent distinctions.

\subsection{Comparison with Latent-Based Unsupervised Similarity Measures}

While the proposed playstyle similarity measures are designed with explicit behavioral alignment and epistemological principles in mind, one may wonder how they compare to more generic similarity metrics that are popular in unsupervised learning contexts. Notably, latent feature similarity has been widely used in generative modeling (e.g., GANs) and stylometric analyses, relying on embeddings extracted from observations to define style. In this subsection, we evaluate whether such approaches can serve as effective baselines for playstyle analysis.

\subsubsection*{Latent Feature-Based Measures}

We investigate two well-known latent similarity metrics: \begin{itemize} 
	\item \textbf{Euclidean Distance (L2)}: Measures the $\ell_2$ distance between mean latent vectors computed from the observation set. 
	\item \textbf{Cosine Similarity}: Measures the angular similarity between the same mean latent vectors, commonly used in natural language processing and behavior modeling \citepx{chess_style, cosine_similarity_example}. 
\end{itemize}

In our experiments, latent features are extracted from the continuous bottleneck (pre-VQ) layer of the HSD encoder used in the discrete methods. Each agent’s dataset is summarized by a single 500-dimensional average feature vector. No action distribution is used, and no explicit alignment of state contexts is enforced.

\subsubsection*{Results and Insights}

Figure~\ref{figure:unsupervised_comparison1} shows the prediction accuracy of playstyle identification using latent metrics versus discrete playstyle methods in TORCS and RGSK. As seen in the TORCS results (Figure~\ref{figure:unsupervised_comparison1}), latent similarity performs near random. This is expected—playstyle variation in TORCS primarily affects hidden behavioral traits such as noise tolerance or target speed, which have limited expression in raw observation frames.

In contrast, the RGSK results (Figure~\ref{figure:unsupervised_comparison1}) show that latent-based metrics perform moderately well, sometimes outperforming even Playstyle Similarity at low sample sizes. This difference arises because visual traits in RGSK—like nitro usage or off-track driving—are visibly encoded in the frame (e.g., color of exhaust flames), allowing visual models to capture stylistic variance.

\begin{figure}[ht] 
\centering 
\includegraphics[width=\textwidth]{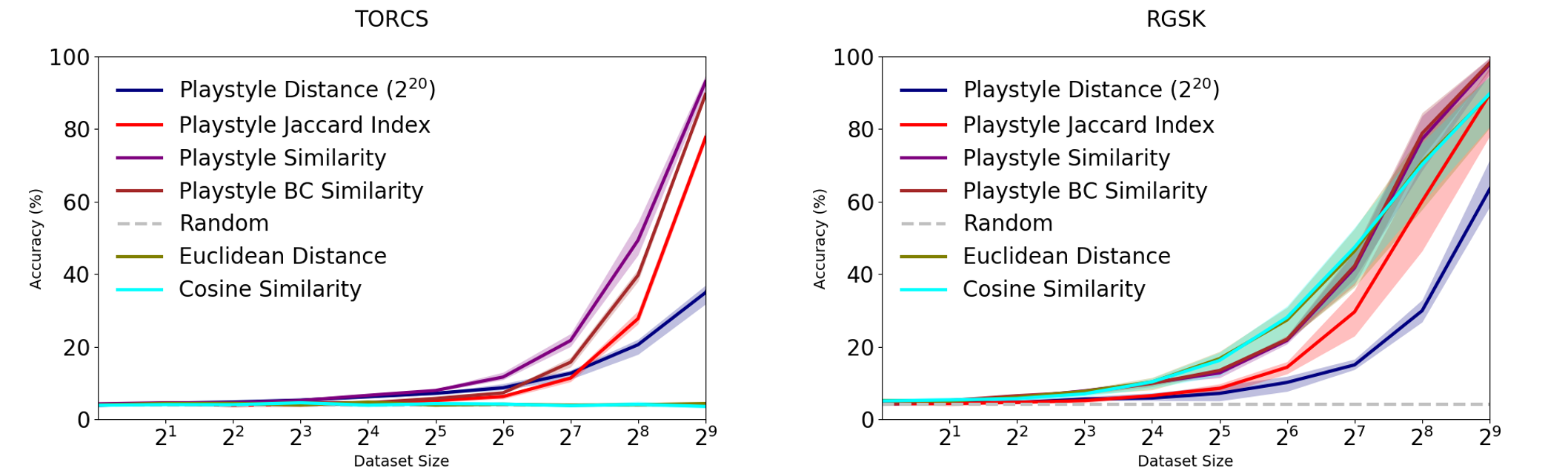}
\caption[Unsupervised Similarity Comparison (TORCS and RGSK)]{Accuracy comparison of potential unsupervised similarity measures on TORCS and RGSK. Discrete playstyle metrics outperform latent-based ones in behavior-dominant tasks (TORCS), while latent similarity performs better when visual cues are strong (RGSK).} \label{figure:unsupervised_comparison1} 
\end{figure}

To evaluate generality, we also compare these metrics in the Atari benchmark (Figure~\ref{figure:unsupervised_comparison2}). Here, latent similarity shows inconsistent performance across games. In contrast, Playstyle Similarity and its BC variant consistently yield higher accuracy beyond 256 samples, even without any explicit playstyle labels. This confirms that discrete action-based comparisons—especially when filtered through multiscale state abstractions—offer greater robustness and interpretability across diverse environments.

These results suggest that while latent feature similarity may suffice in tasks where visual traits dominate or where sample sizes are large, it lacks the behavioral grounding, semantic resolution, and cross-domain consistency needed for general playstyle analysis. Playstyle Similarity, in contrast, is more principled, interpretable, and empirically robust.

\begin{figure}[ht] 
\centering 
\includegraphics[width=\textwidth]{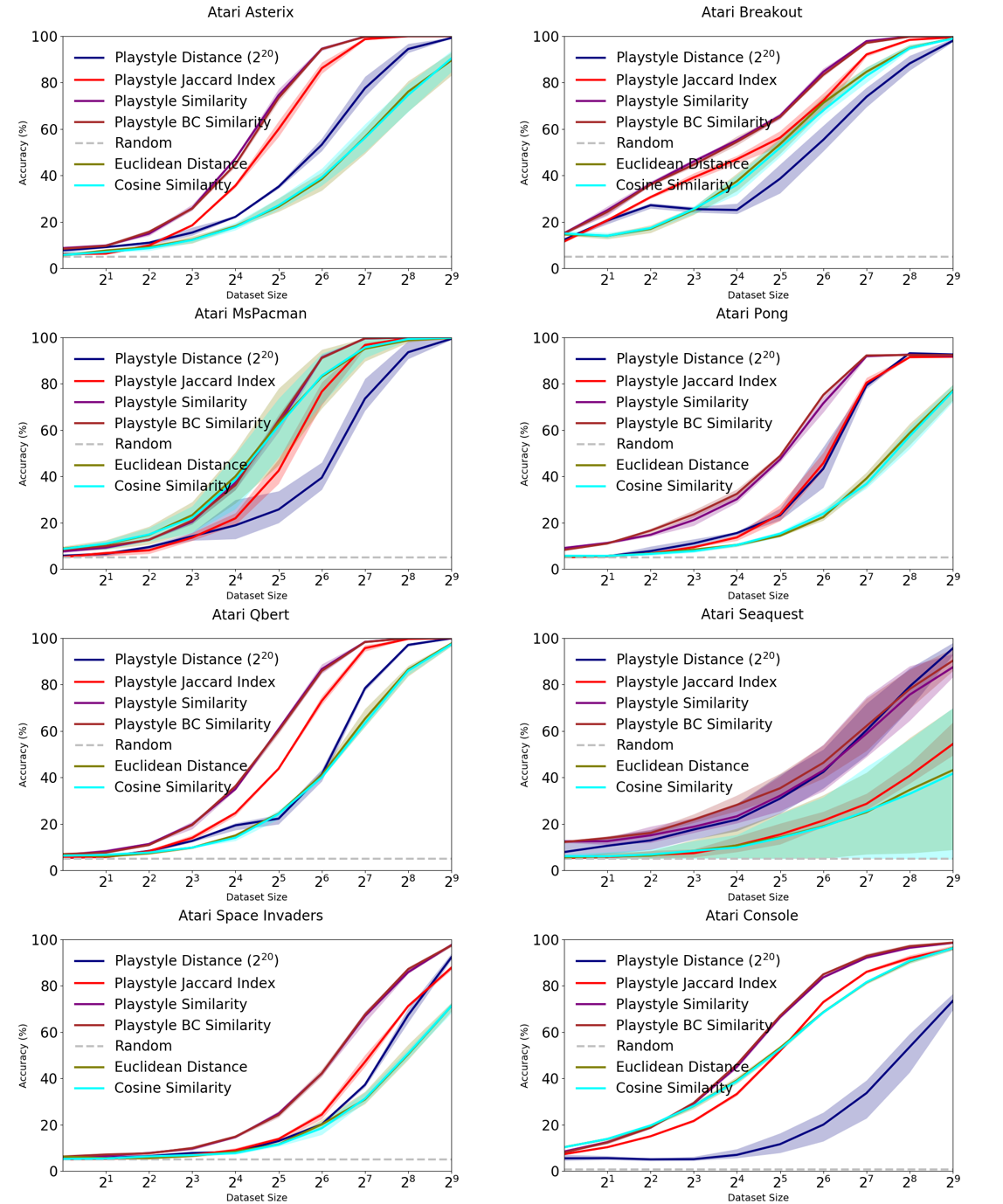}
\caption[Unsupervised Similarity Comparison (Atari Games)]{Accuracy comparison of potential unsupervised similarity measures on Atari games.} \label{figure:unsupervised_comparison2} 
\end{figure}

\clearpage

\subsection{Game Platform Overview}
\label{subsec:game_overview}

To contextualize our evaluation and demonstrate the generality of the proposed playstyle measures, we provide a unified overview of the five distinct game platforms used throughout this dissertation. These include both single-player and multi-agent settings, varying in action space complexity, state representation, agent type, and style diversity.

\begin{figure}[ht]
    \centering
    \includegraphics[width=0.95\textwidth]{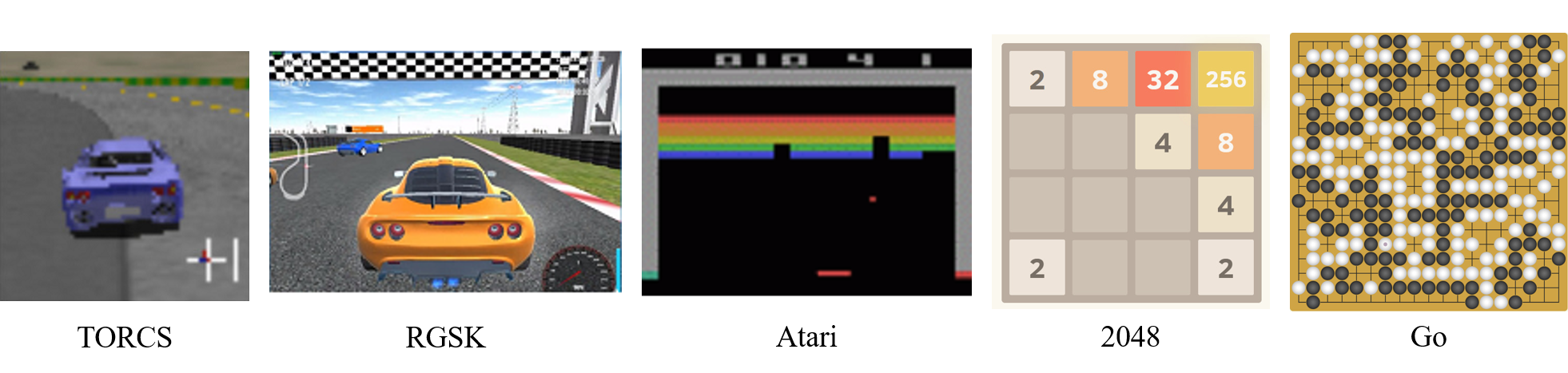}
    \caption[Playstyle Experiment Game Screens]{Game environments used in this study. From left to right: TORCS, RGSK, Atari, 2048, and Go. These platforms span from real-time simulations to turn-based strategic settings, offering varied challenges for playstyle analysis.}
    \label{fig:game_screens}
\end{figure}

As illustrated in Figure~\ref{fig:game_screens}, the games differ significantly in their interaction dynamics and visual feedback. Table~\ref{table:game_details} summarizes key properties relevant to playstyle measurement:

\begin{itemize}
    \item \textbf{TORCS}: A rule-based autonomous racing environment with continuous control and low observation variability. Playstyles differ primarily in target speed and noise control, which are not visually obvious, making latent-based similarity methods ineffective.
    
    \item \textbf{RGSK}: A Unity-based human racing simulator where player behavior directly influences visible effects like drift, nitro usage, or track adherence. Visual features align closely with playstyle dimensions, which is why latent features perform competitively here.
    
    \item \textbf{Atari}: A diverse set of DRL-trained agents operating in low-resolution environments. The wide action and state spaces offer both visual and behavioral diversity, allowing thorough testing of measure generality.
    
    \item \textbf{2048}: A single-agent stochastic puzzle game where randomness dominates short trajectories. Although the observation is fully discrete, the style signal is weak per episode, posing a strong test of sampling robustness.
    
    \item \textbf{Go}: A highly strategic two-player game with immense state space complexity and partial observability of intention. Human players exhibit consistent preferences in opening and midgame styles, but these are difficult to identify without strong representation models.
\end{itemize}

\begin{table}[ht]
\centering
\caption{Playstyle game details.}
\label{table:game_details} 
\begin{tabular}{ |l|l|l|l|l| } 
\hline
Game Platform & Agent Type & Style Count & Observation Size & Action Space\\
\hline
TORCS & Rule-based Agent & 25 & [4, 3, 64, 64] & Continuous 2D \\
RGSK & Human Player & 24 & [4, 3, 72, 128] & Discrete 27 \\
Atari & DRL Agent & 20 & [4, 1, 84, 84] & Discrete 4–18 \\
2048 & RL Agent & 10 & [4, 4] & Discrete 4 \\
Go & Human Player & 200 & [18, 19, 19] & Discrete 362 \\
\hline
\end{tabular}
\end{table}

Together, these environments represent a comprehensive testbed for evaluating playstyle metrics under diverse conditions—ranging from low-dimensional, deterministic control tasks to high-dimensional, stochastic and strategic decision-making.

\clearpage

\subsection{2048: A High-Randomness Puzzle Game} \label{section:game2048}

To evaluate the robustness of discrete playstyle measures in highly stochastic environments, we conduct experiments using the puzzle game \textbf{2048}. Unlike real-time or turn-based games with stable dynamics, 2048 introduces randomness at every step—each player move triggers the insertion of a new tile in a random position with a random value. This makes it nearly impossible to generate identical trajectories, even with deterministic agents.

In our experiment, we trained a reinforcement learning (RL) agent for 10 million episodes, saving the model every 1 million episodes to create 10 different "players." For each of these 10 models, we collected 1,000 episodes, using the first 500 as candidate datasets and the remaining 500 for query datasets, forming 5,000 query instances in total. This setup simulates a practical scenario: identifying the source model from a small, noisy sample—only a few moves long—without relying on visual or handcrafted features.

We adopted the raw $4 \times 4$ game board (without compression or visual encoding) as the discrete state representation. Since the game is inherently discrete and deterministic on player inputs, using raw boards as symbolic anchors is both natural and effective. The small state space also ensures sufficient overlap between episodes despite randomness.

\begin{table}[ht] 
\centering 
\caption[2048 Playstyle Accuracy]{Accuracy of model identification in 2048 using different playstyle measures.} 
\label{table:2048} 
\begin{tabular}{l|c} 
\toprule 
Measure & Accuracy \\ 
\midrule Playstyle Distance & \textbf{98.90\%} \\ 
Playstyle Intersection Similarity & \textbf{98.52\%} \\ 
Playstyle Inter BC Similarity & \textbf{98.90\%} \\ 
Playstyle Jaccard Index & 49.22\% \\ 
Playstyle Similarity & 71.26\% \\ 
Playstyle BC Similarity & 71.26\% \\ 
\bottomrule \end{tabular} 
\end{table}

\textbf{Discussion.}
The results in Table~\ref{table:2048} demonstrate that even under severe randomness, discrete measures based on action distributions—especially those comparing distributions on intersected symbolic states—can accurately identify agent-specific behavior. However, the inclusion of the \textit{Jaccard Index} significantly lowers performance due to sparse and unstable overlaps in such settings. This supports a key insight: \textbf{when observation diversity is high and playstyle-specific features are subtle}, measures that rely directly on overlapping symbolic states may falter unless additional structural constraints or normalization techniques are applied.

\clearpage

\subsection{Go: Measuring Playstyles in a Complex Multi-Agent Setting} 
\label{section:go_experiment}

To further evaluate the generality of our framework in highly complex and interactive settings, we consider the game of \textbf{Go}, a two-player board game renowned for its vast state space and strategic depth \citepx{board_game_complexity}. Unlike single-agent environments, Go introduces opponent behavior as an additional, often unpredictable factor that significantly influences state visitation and action patterns, making it an ideal testbed for evaluating the robustness of playstyle measurement.

\subsubsection{Experimental Setup}

We designed a playstyle identification task using human gameplay data. The encoder used in this study is a variant of the Hierarchical State Discretization (HSD) model \citepx{playstyle_distance}, trained not with reconstruction but using win prediction objectives akin to AlphaZero-style training \citepx{alpha_zero}. The training dataset consists of 45,000 games from 9-Dan level players on the Fox Go platform \citepx{foxgo}, provided by the MiniZero team \citepx{minizero}. Importantly, these games are labeled only with moves and outcomes—no explicit playstyle labels were used.

The evaluation dataset contains gameplay from 200 human players (1-Dan to 9-Dan), with 100 games each for query and candidate sets. The discrete encoder supports three levels of state granularity: ${4^8, 16^8, 256^{361}}$, and we also evaluate their multiscale combination ("mix").

Two evaluation protocols were used:
\begin{enumerate}
	\item Opening phase (first 10 moves): where stylistic differences in opening choices are expected to dominate.
	\item Full games: covering mid- and late-game decisions, potentially influenced by broader strategies or positional judgments.
\end{enumerate}

\subsubsection{Results and Interpretation}

\begin{table}
\centering
\caption[Accuracy of Go 200 player identification.]{Accuracy of Go 200 player identification with M games as the query set and also M games as the candidate set. The full comparison with different numbers of query and candidate sets can be found in the original paper \citepx{playstyle_similarity}. The discrete encoder is trained from a variant of HSD~\citep{playstyle_distance}, and the available state spaces in this encoder are $\{4^8, 16^8, 256^{361}\}$. The measures with "mix" notation imply using all three available state spaces simultaneously with our multiscale modification.}
\label{table:go_mxm}
\begin{tabular}{lllllll}
\toprule
\textbf{Only First 10 Moves} & M=5 & M=10 & M=25 & M=50 & M=75 & M=100\\
\midrule
Playstyle Distance ($4^8$) & 11.5\% & 33.5\% & 71.0\% & 87.0\% & 93.0\% & \textbf{97.0\%} \\
Playstyle Distance ($16^8$) & 10.0\% & 41.5\% & 74.0\% & 85.0\% & 92.5\% & 95.5\% \\
Playstyle Distance ($256^{361}$) & 8.0\% & 38.0\% & 75.0\% & 87.0\% & 92.0\% & 94.5\% \\
Playstyle Distance (mix)  & 14.0\% & \textbf{46.5\%} & \textbf{75.5\%} & 87.5\% & 94.5\% & 96.5\% \\
Playstyle Inter. Similarity (mix) & 15.0\% & 42.0\% & 68.5\% & 84.0\% & 94.5\% & 94.5\% \\
Playstyle Inter BC Similarity (mix) & 14.0\% & 42.5\% & 72.0\% & \textbf{89.0\%} & \textbf{96.5\%} & 95.5\% \\
Playstyle Jaccard Index (mix) & 9.5\% & 21.5\% & 49.0\% & 78.0\% & 88.0\% & 95.0\% \\
Playstyle Similarity (mix) & 21.5\% & 40.0\% & 65.0\% & 87.5\% & 94.0\% & \textbf{97.0\%}  \\
Playstyle BC Similarity (mix) & \textbf{24.0\%} & 45.5\% & 70.5\% & 88.0\% & \textbf{96.5\%} & \textbf{97.0\%}  \\
\bottomrule
\toprule
\textbf{Full Game Moves} & M=5 & M=10 & M=25 & M=50 & M=75 & M=100\\
\midrule
Playstyle Distance ($4^8$) & 5.5\% & 11.5\% & 20.0\% & 31.0\% & 39.5\% & 50.0\% \\
Playstyle Distance ($16^8$) & 10.0\% & 44.5\% & 74.0\% & 86.0\% & 93.5\% & 96.5\% \\
Playstyle Distance ($256^{361}$) & 8.0\% & 39.5\% & \textbf{75.5\%} & 88.5\% & 93.5\% & 96.5\% \\
Playstyle Distance (mix) & 17.0\% & 33.0\% & 56.5\% & 76.0\% & 83.0\% & 90.0\% \\
Playstyle Inter. Similarity (mix) & \textbf{19.0\%} & 38.0\% & 66.0\% & 86.0\% & 95.0\% & 95.0\% \\
Playstyle Inter BC Similarity (mix) & 16.5\% & 37.5\% & 66.0\% & 90.0\% & 96.0\% & 97.0\% \\
Playstyle Jaccard Index (mix) & 5.0\% & 8.5\% & 22.5\% & 51.0\% & 66.0\% & 80.5\% \\
Playstyle Similarity (mix) & 15.5\% & 29.5\% & 56.5\% & 81.5\% & 90.5\% & 94.0\% \\
Playstyle BC Similarity (mix) & 22.0\% & \textbf{45.5\%} & 73.0\% & \textbf{92.0\%} & \textbf{97.5\%} & \textbf{97.5\%} \\
\midrule
Behavior Stylometry* & \textbf{27.5\%} & \textbf{46.0\%} & 67.5\% & 84.0\% & 90.5\% & 90.5\% \\ Behavior Stylometry Improved* & \textbf{46.5\%} & \textbf{70.0\%} & \textbf{89.5\%} & \textbf{95.0\%} & \textbf{98.0\%} & \textbf{98.5\%} \\
\bottomrule
\end{tabular}
\end{table}

Table~\ref{table:go_mxm} summarizes the identification accuracy under varying query sizes. Several key observations emerge:

\begin{itemize} 
	\item In the opening-only setting, even small state spaces ($4^8$) suffice to achieve up to \textbf{97.0\%} accuracy with 100-game queries. This validates the intuition that opening styles are highly discriminative and can be effectively captured by compact symbolic representations. 
	\item In the full-game setting, however, the performance of smaller state spaces drops significantly. This highlights the importance of finer representations (e.g., $256^{361}$) or multiscale strategies in capturing nuanced long-term behavior patterns. 
	\item Measures incorporating the \textbf{Bhattacharyya coefficient} (BC) consistently outperform others when the query size is large, especially under full-game settings. This may stem from BC’s sensitivity to overlapping high-probability regions in the action distributions, making it well-suited for modeling consistent behavioral tendencies. 
	\item \textbf{Playstyle Similarity} and \textbf{BC Similarity} using multiscale state spaces achieve strong performance in both settings, offering a good balance of precision and generality.
	\item For comparison, we also report results using the \textit{Behavior Stylometry} approach from \citepx{chess_style}, extended and improved for Go by a separate thesis project \citepx{thesis_of_Chen_Chun_Jung}. While their best model reaches slightly higher top-end accuracy, it requires specialized clustering and contrastive supervision pipelines, whereas our method remains fully unsupervised and domain-agnostic. 
\end{itemize}

\textbf{Conclusion.}
These results demonstrate that discrete playstyle measures can scale to multi-agent games like Go, where behavior is shaped both by internal style and external opponents. When provided with sufficient data, these methods can reliably identify individual players, even in the absence of predefined style categories. Moreover, multiscale symbolic representations and overlap-sensitive metrics like BC enable accurate measurement across both early-game and full-game settings.

\section{Summary and Transition}

In this chapter, we provided a comprehensive investigation into the problem of \textbf{measuring playstyle}, one of the most foundational components of playstyle modeling. We examined four families of methods—heuristic rule-based, data-driven, policy-distribution-based, and discrete-state-based approaches—and proposed a unified framework that not only categorizes these methods but also extends them with theoretically grounded and empirically validated improvements.

Our proposed framework centers around the idea of \textbf{discrete symbolic state representations}, particularly through the Hierarchical State Discretization (HSD) method, which enables semantically meaningful comparisons of decision-making behavior across agents and environments. From this foundation, we introduced the concept of \textbf{Playstyle Distance} and its perceptually motivated transformation into \textbf{Playstyle Similarity}, incorporating both local decision divergence and global state space overlap. This yields a similarity measure that is: 
\begin{itemize} 
	\item \textbf{Theoretically interpretable}: rooted in perceptual and cognitive principles; 
	\item \textbf{Empirically robust}: shown to outperform baselines across five distinct game platforms (TORCS, RGSK, Atari, 2048, and Go); 
	\item \textbf{General and unsupervised}: requiring no style labels and operating across different observation types and action formats; 
	\item \textbf{Scalable}: capable of handling hundreds of unique agents, large observation spaces, and long horizons. 
\end{itemize}

\subsection{Connecting Back to Belief and Preference}

This chapter also marks a return to the philosophical perspective introduced in earlier sections. Our notion of \textit{belief}, as the agent’s internal representation of the world, is now made concrete through symbolic discretization: each discrete state encodes a coarse-grained belief abstraction about the current context. When paired with an observed or estimated \textit{action}, this allows us to recover an agent’s \textit{preference}—expressed not as isolated behavior, but as conditional responses given a belief. In this sense, playstyle becomes the emergent trace of belief–preference alignment.

Viewed through this lens, each measurement family reveals its assumptions:
\begin{itemize}
    \item \textbf{Heuristic rules} encode the designer’s belief and preference explicitly, baking them into the measurement tool itself.
    \item \textbf{Data-driven methods} infer patterns from experience, with the dataset and model architecture implicitly shaping the agent’s assumed belief space.
    \item \textbf{Action-distribution methods} focus on capturing preferences through policies, but may lack alignment in the belief representation over which those policies act.
    \item \textbf{Discrete-state methods}, by contrast, define shared symbolic anchors that stabilize both belief and preference comparisons, making them particularly suited for cross-agent evaluation.
\end{itemize}

This perspective reinforces a central thesis of our work: because playstyle is an indirect and emergent construct, understanding \textit{the intention behind a measurement method}—its assumptions about belief, preference, and representation—is critical. Measurement is not merely technical; it is interpretive.

\subsection{From Measurement to Diversity and Balance}

While measurement forms the bedrock of any style-based analysis, it also naturally raises broader questions about the implications of style:
\textit{How diverse are the styles across a population?}
\textit{Can we quantify imbalance, dominance, or gaps between them?}
\textit{How do style differences relate to fairness, challenge, or innovation?}

These questions shift the analytical lens from individual behavior to population-level phenomena—specifically, the emergent structure of \textbf{diversity} and \textbf{balance}. In the next chapter, we turn our attention to these broader structural concerns. We will formalize what it means for playstyles to be diverse, how to measure the richness or fragmentation of style distributions, and how to evaluate whether a system or environment supports equitable and engaging strategic variety.

In doing so, we build upon the foundations laid here—especially \textit{Playstyle Distance} and \textit{Playstyle Similarity}—to construct tools for \textbf{style-aware system design}, spanning matchmaking, meta-balancing, content generation, and beyond.

\chapter{Rationality in Multiple Playstyles}
\noindent\textbf{Key Question of the Chapter:} \\
\textit{What makes a playstyle meaningful or worth preserving?}

\noindent\textbf{Brief Answer:} 
A playstyle is worth preserving when it holds \textbf{popularity}—supported by widely recognized conceptual beliefs, such as conservative reasoning under uncertainty or rational decisions based on objective outcomes. These foundations naturally connect to \textbf{diversity} and \textbf{balance}, giving the style enduring value.

\bigskip

\section{Belief and Intrinsic Motivation}
\label{sec:belief_intrinsic_motivation}

Now that we have explored how playstyle can be measured, a critical question emerges: \textit{what makes a playstyle rationally valid}? No matter how diverse or distinct a playstyle may appear, it is unlikely to be acknowledged as a legitimate style unless it demonstrates some form of effectiveness. This leads us to examine the rational justification of a style—not just its difference (\textbf{Capacity}), but its recognition and acceptance (\textbf{Popularity}) \citepx{brunswik1956perception}.

Earlier measurement techniques aimed to delineate the boundaries of Capacity—what constitutes similar versus different playstyles. In this chapter, our attention shifts to the second axis: Popularity. Why does a playstyle deserve to exist and persist? The answer lies in \textit{recognition}—by other agents, by the environment, or by broader social and strategic systems. But what drives this recognition? What compels us to accept new styles into our strategic repertoire or cultural value system?

At its root, such recognition stems from \textbf{belief}—especially in rationality. From early education, many of us are taught to act rationally and to trust scientific reasoning. But this belief does not emerge in a vacuum. It is shaped by repeated patterns of feedback: rewards and punishments, observed utility, and confirmation in real-world outcomes. Over time, rationality becomes a culturally embedded belief system \citepx{magnitude_effect, gigerenzer2007gut}.

This belief system is deeply intertwined with another: the awareness of \textbf{finiteness}. Time, energy, and resources are limited. Under such constraints, rationality naturally leads us to value diversity—not as an aesthetic preference, but as a strategic necessity.

Consider biological diversity. A species with little genetic variation may be wiped out by a single disease. The Gros Michel banana, once globally popular, was nearly eradicated by Panama disease due to its genetic uniformity \citepx{stover1962fusarial, ploetz2005panama}. Diversity, in this case, is rational insurance against unknown threats.

A similar logic applies in games. If one strategy dominates, rational players will converge upon it, leading to stagnant and predictable gameplay. Strategic diversity—especially when counter-relationships exist—is essential to preserving meaningful decision-making. When the strategy space collapses into a single dominant point, exploration ceases \citepx{von1947theory}.

To restore balance and sustain diversity, the system must introduce new styles—not just stronger ones, but ones that counter existing dominant strategies without becoming dominant themselves. This induces counterplay, which revives the need for thoughtful decision-making. Rational players then have a reason to deviate, and their belief in their own insight can be tested.

From this perspective, the pursuit of diversity and balance is not just functional—it is rational. And at the heart of this rationality lies \textbf{intrinsic motivation}: an internal drive to validate, challenge, or expand one's belief system. This notion aligns with concepts of curiosity and exploration in psychology and neuroscience \citepx{loewenstein1994psychology, berlyne1966curiosity}.

Curiosity, though not itself a foundational belief, emerges as a behavior driven by uncertainty reduction and epistemic validation. Psychology distinguishes between \textit{perceptual curiosity} (triggered by novel stimuli) and \textit{epistemic curiosity} (seeking knowledge). These two forms mirror two motivations behind playstyle exploration: discovering new styles and validating existing models.

Within playstyle contexts, curiosity drives players to explore new strategies even when existing ones suffice. Boredom, often arising from over-validation, becomes a signal for diversive exploration. This too reflects intrinsic motivation—a rational desire to expand the boundaries of one’s strategic world.

These philosophical motivations also clarify the intent behind various playstyle measurement methods introduced in the previous chapter. Heuristic rule-based methods encode designer beliefs and preferences directly into the metric. Data-driven methods reflect beliefs inferred from empirical distributions and the assumptions embedded in the learning algorithms. Action-distribution-based measures express preference models directly from policies, though sometimes at the cost of precise belief alignment. Finally, discrete-state-based comparisons grounded in symbolic abstraction (e.g., HSD) offer the clearest route for representing both belief (via state context) and preference (via action selection), thus providing a more faithful approximation of rational style.

In sum, rationality, finiteness, and curiosity form a coherent motivational basis for playstyle diversity. A playstyle becomes meaningful not merely by existing, but by being believed in—and by being actively pursued through the rational testing of internal models. This recognition forms the foundation of Popularity. And our measurement of Capacity, when paired with this lens of recognition and intent, provides a full picture of style: not just \textit{what it is}, but \textit{why it matters}.

\section{Trade-offs: Exploration vs. Exploitation}
\label{sec:exploration_vs_exploitation}

In the previous section, we examined how intrinsic motivation and curiosity drive the emergence of new playstyles and promote diversity. These phenomena naturally connect to one of the most fundamental trade-offs in both psychology and artificial intelligence: the balance between \textit{exploration} and \textit{exploitation}.

While exploitation refers to selecting the best-known option according to current beliefs, exploration entails deliberately trying something unfamiliar---potentially suboptimal in the short term---in the hope of uncovering better long-term outcomes. At first glance, such behavior may appear irrational: why sacrifice a sure reward for uncertainty? The answer lies in a critical underlying assumption---a belief often left unstated, but essential for the legitimacy of exploration itself.

\textbf{To explore is to admit: "I might be wrong."} Exploration presupposes the belief that one’s current policy, behavior, or understanding is incomplete, and that better alternatives may exist beyond current knowledge. In other words, it requires an epistemic humility---the recognition that the global optimum likely has not yet been found.

This belief---that there exists something unknown and better---is foundational. Without it, the rational basis for exploration collapses. An agent that is fully confident in its current strategy has no reason to deviate. Thus, the very act of exploring is a declaration of uncertainty, a decision to act not purely for utility, but to challenge and revise one’s belief.

This trade-off is classically illustrated in the study of the multi-armed bandit problem. Exploitation maximizes known returns, while exploration sacrifices immediate reward to reduce long-term regret. Algorithms such as UCB optimize this trade-off by ensuring that regret grows only logarithmically over time \citepx{auer2002finite}. Other methods, such as $\varepsilon$-greedy or Thompson Sampling, incorporate randomness or uncertainty modeling to guide exploration.

In the context of playstyle, each strategy or behavioral pattern can be seen as an arm in the bandit. Exploring a new style implies the belief that it might perform better under some conditions---even if current evidence is lacking. This belief is not trivial; it opens the door to diversity, creativity, and transformation.

Recent advances in reinforcement learning make this principle explicit. Curiosity-driven exploration, count-based novelty bonuses, and intrinsic reward mechanisms \citepx{pathak2017curiosity, bellemare2016unifying, count_based_exploration} are all formalizations of the belief that knowledge itself is valuable. These techniques reward agents for reducing uncertainty or for entering new regions of the state space---not for winning, but for learning.

Yet, exploration without meaningful belief can also fail. If an agent endlessly revisits previously failed strategies or confuses randomness with novelty, it produces activity without progress \citepx{gdi, agent57}. Genuine exploration requires not only action, but direction---a belief that novelty can lead to betterment.

This insight reframes diversity: not merely as variance, but as \textit{directed deviation} grounded in epistemic intent. In this sense, a playstyle becomes meaningful not just when it performs well, but when it reflects an agent’s belief in the possibility of improvement.

In sum, the legitimacy of exploration---and the diversity it produces---rests on a subtle but powerful belief: that what we know is not all there is, and what we do now is not the best we can do. This belief is the true engine behind strategic evolution.

\section{Measuring Playstyle Diversity}

In the previous section, we discussed the rationale for promoting diversity and the benefits it brings to decision-making systems. We now turn to the practical question: \textit{How can diversity be concretely measured and evaluated?} In reinforcement learning and multi-agent environments, what tools are available to quantify diversity in meaningful ways?

\subsection{Classical Approaches and the Role of Exploration}

In deep reinforcement learning (DRL), one of the most fundamental measures of behavioral diversity is \textit{policy entropy}—the degree of randomness in the agent’s action selection. For instance, Proximal Policy Optimization (PPO) often includes entropy regularization to maintain sufficient stochasticity, thus encouraging exploration \citepx{ppo}.

However, merely increasing randomness does not guarantee meaningful diversity. Naïve randomness may introduce behavioral variation, but it often fails to improve policy quality and can significantly reduce learning efficiency. As a result, more principled exploration methods have emerged in recent years:

\begin{itemize}
  \item \textbf{NoisyNet} and \textbf{Bootstrap DQN} replace traditional $\varepsilon$-greedy strategies with structured noise or ensemble-based sampling to facilitate deep and persistent exploration \citepx{noisy_net, bootstrap_dqn}.
  \item \textbf{Count-based exploration} rewards agents for visiting novel states, based on estimated state visitation frequencies \citepx{count_based_exploration, bellemare2016unifying}.
  \item \textbf{Curiosity-based methods}, such as Random Network Distillation (RND) and Never Give Up (NGU), encourage agents to seek states with high prediction error, promoting novelty \citepx{rnd, ngu}.
  \item \textbf{Information-theoretic approaches} like Information-Directed Sampling (IDS) use Bayesian uncertainty estimates to trade off reward and information gain \citepx{ids_exploration}.
\end{itemize}

\subsection{Entropy Maximization and Policy Diversity}

A second family of exploration strategies focuses directly on entropy maximization. Algorithms such as Soft Actor-Critic (SAC) and other energy-based models aim to identify regions of policy space where action distributions are maximally uncertain \citepx{soft_q, sac, extreme_q}.

Interestingly, both paradigms—uncertainty-based exploration and entropy-based optimization—tend to converge in practice. Unfamiliar or novel states naturally induce more stochastic behavior, yielding high-entropy policies that simultaneously support both diversity and learning progress.

\subsection{Diversity as a Performance Driver}

Recent large-scale experiments—such as those conducted under the Generalized Data Distribution Iteration (GDI) framework—have empirically validated the importance of policy diversity \citepx{gdi}. Agents that succeed in maintaining both high performance and high behavioral variety consistently outperform those that overfit to narrow strategies.

Traditional state-counting or t-SNE visualizations may work in low-dimensional tasks but fail in complex environments. Visualizations can offer qualitative insight, but they are not sufficient for robust, quantitative diversity analysis.

To address this, we introduce a simple and intuitive method that directly quantifies diversity at the trajectory level: \textbf{Playstyle Similarity}.

\subsection{Diverse Trajectory Count via Playstyle Similarity}

Playstyle Similarity was introduced in Chapter~\ref{ch:discrete_measurement} as a probabilistic metric to evaluate the stylistic resemblance between trajectories. Here, we apply it as a core building block to measure behavioral diversity in a concrete, empirical setting. By defining a similarity threshold on trajectory-level comparisons, we can determine when a new behavior differs meaningfully from past behaviors.

\begin{algorithm}
\caption{Measuring Trajectory Diversity}
\label{algorithm:trajectory_diversity}
\begin{algorithmic}[1]
\Require Policy $\pi$, Environment $\mathcal{E}$, Similarity measure $M$
\Require Similarity threshold $t$, Number of trajectories $N$
\State Initialize $S$ (trajectory storage), $d \leftarrow 0$ (diversity counter)
\For{$i = 1$ to $N$}
  \State Generate a trajectory $\tau_i \sim \pi,\mathcal{E}$
  \State $is\_diverse \leftarrow \textbf{true}$
  \For{each $\tau_j$ in $S$}
    \If{$M(\tau_i,\tau_j) \geq t$}
      \State $is\_diverse \leftarrow \textbf{false}$
      \State \textbf{break}
    \EndIf
  \EndFor
  \If{$is\_diverse$} $d \leftarrow d + 1$ \EndIf
  \State Store $\tau_i$ in $S$
\EndFor
\Ensure Return $d$ (diverse count), $N$ (total)
\end{algorithmic}
\end{algorithm}

\begin{table}[ht]
\centering
\caption[Atari Diverse Trajectory Count]{Average diverse trajectory count out of 25 for each DRL algorithm (across 7 games, $t = 0.2$).}
\label{table:drl_policy_diversity_main}
\begin{tabular}{llllllll}
\toprule
Algorithm & Asterix & Breakout & MsPacman & Pong & Qbert & Seaquest & SpaceInvaders \\
\midrule
DQN & 6.00 & 6.00 & 5.33 & 4.00 & 6.00 & 11.00 & 25.00 \\
C51 & 6.00 & 7.00 & 7.00 & 6.00 & 7.00 & 21.00 & 25.00 \\
Rainbow & 8.67 & 5.33 & 8.00 & 5.00 & 5.00 & 24.00 & 25.00 \\
IQN & \textbf{25.00} & \textbf{14.00} & \textbf{10.00} & \textbf{9.00} & \textbf{10.00} & \textbf{25.00} & 25.00 \\
\bottomrule
\end{tabular}
\end{table}

This result also supports the broader thesis of this chapter: that diversity and strength are not mutually exclusive. In fact, greater diversity can lead to greater robustness and higher performance.

It is important to note that this method is grounded in \textit{process-based difference}—the variation in how agents behave—even when outcomes such as score remain similar.  
In contrast, \textit{outcome-based difference} focuses on the relative performance of different styles, which directly connects to the question of balance: ensuring that no style gains a disproportionate advantage that undermines the viability of others.  
This shift from process to outcome forms the basis of the next section.

\section{Understanding Balance in Strategy Space}

Having established a foundation for understanding diversity—and its ties to intrinsic motivation, curiosity, and epistemic exploration—we now turn to a complementary question: \textit{what makes the coexistence of playstyles viable, sustainable, and strategically meaningful?} This brings us to the critical notion of \textbf{balance} in the strategy space.

In both game design and multi-agent learning, balance is a central yet often misunderstood concept. It is frequently simplified to the idea of numerical fairness—for instance, equal win rates or mirrored performance metrics. However, this interpretation misses the deeper structural and cognitive roles balance plays in interactive systems.

True balance goes beyond score symmetry. It shapes the competitive dynamics among styles, influences player motivation, and determines whether a system encourages sustained strategic diversity or converges toward uniformity. A balanced system invites exploration by making alternative styles viable; it also fosters fairness by preventing domination; and it enriches gameplay by preserving meaningful counterplay.

Most importantly, balance provides the environmental feedback that validates or challenges a playstyle’s \textit{reason to persist}. In this sense, balance is not only a property of a system, but also a reflection of belief structures—about what works, what should work, and what deserves to be explored. In the next subsections, we examine how balance manifests along different dimensions and how it can be formally understood through strength, counter-relationships, and structural dynamics.

\subsection{Aspects of Balance}

To unpack the concept of balance in strategic environments, we identify four core aspects that collectively characterize playstyle balance:
\begin{enumerate}
  \item \textbf{Fairness}: No strategy should dominate others across all matchups.
  \item \textbf{Challenge}: Players should encounter appropriate levels of difficulty that sustain engagement.
  \item \textbf{Intransitivity}: Counter-relationships should exist to prevent static hierarchies.
  \item \textbf{Matchmaking}: Systems must support viable pairings and team compositions at the population level.
\end{enumerate}

These dimensions reflect different levels of balance—ranging from micro-level fairness in decisions to macro-level dynamics of population play. Interviews with professional game designers reinforce the importance of this multifaceted view \citepx{balance_from_renowned_authors}.

\textbf{Importantly, the overarching goal of balance is to preserve and expand the value of playstyle diversity.} In the competitive space of strategies, balance ensures that \textit{more distinct playstyles are not only possible, but also viable}. This aligns directly with the broader framework introduced in earlier chapters: expanding the \textbf{Capacity} (number of distinct styles) while ensuring that such styles retain sufficient \textbf{Popularity} (likelihood of being chosen or succeeding). A perfectly balanced system is one where diverse strategies coexist with competitive relevance.

\subsubsection{Fairness: Avoiding Overpowered Strategies}

The most immediate notion of balance is fairness—the idea that all strategic options should be comparably viable. When one strategy consistently outperforms all others, the decision space collapses, undermining exploration and discouraging stylistic expression.

In symmetric systems, achieving fairness often implies the presence of intransitive relationships: if no strategy can dominate all others, then some form of circular counterplay must exist. However, intransitivity is not sufficient on its own—there may still exist one strategy that performs well against most others and only loses narrowly to its counters.

Traditional rating systems such as Elo \citepx{elo}, TrueSkill \citepx{true_skill}, and MMR \citepx{mmr} aim to estimate player strength based on match outcomes, providing a basis for fairness evaluation. These models work well in symmetric, outcome-focused contexts but often struggle in multi-agent settings with asymmetric interactions or complex counterplay.

\subsubsection{Challenge: Sustaining Engagement via Dynamic Tension}

Fairness alone does not guarantee engagement. A balanced system should also maintain a sense of challenge that keeps players motivated. This aligns with the psychological notion of \textit{flow} \citepx{mental_flow}, where the best experiences occur when skill and difficulty are in harmony.

This balance is also captured by the \textit{Yerkes–Dodson law} \citepx{arousal_and_performance}, which suggests an optimal challenge zone exists beyond which performance drops. Dynamic difficulty adjustment (DDA) \citepx{dda} is one approach used by designers to maintain this tension. From a playstyle perspective, balance ensures that players can experiment with new strategies without being punished disproportionately.

\subsubsection{Intransitivity: Preserving Meta-Dynamic Depth}

Intransitivity—cyclical relationships like rock-paper-scissors—ensures that strategic dominance is always contextual. This not only sustains the evolution of game metas but also allows niche strategies to thrive under the right conditions.

Frameworks such as mElo \citepx{m_elo}, ND rating \citepx{nd_ability}, and neural counter-category models \citepx{game_balance_analysis} aim to capture these intransitive dynamics. By identifying which strategies beat which others—not just how strong they are—these models allow a deeper understanding of balance stability.

\subsubsection{Matchmaking: Population-Level Balance}

Balance is also shaped by how strategies are paired during play and training. In multi-agent systems, matchmaking influences both learning and evaluation. Self-play, PBT \citepx{pbt}, and PSRO \citepx{psro} are key techniques where opponent selection directly affects performance diversity and balance emergence.

Effective matchmaking encourages not just fairness, but also a dynamic and evolving distribution of viable strategies across the population.

\subsubsection{Conclusion: Balance as a Multi-Faceted Construct}

In summary, balance is not merely about equal win rates—it is about sustaining an ecosystem where \textit{many different playstyles are meaningfully viable}. It safeguards diversity by allowing varied strategies to persist, adapt, and compete on relatively equal footing. This expands the expressive space of strategy, supports creativity, and maintains long-term engagement across competitive systems. In the following sections, we transition from conceptual analysis to practical implementation, presenting neural rating models and structural indicators to measure and preserve balance in complex decision environments.

\subsection{Competitive Strength Measurement}

Why do we need to measure strength at all?

In a system where multiple playstyles interact and compete, understanding \textit{who is stronger} is fundamentally a question of \textbf{preference}---but one mediated by outcomes. Unlike subjective preference derived from internal belief or individual value systems, here the preference is \textit{revealed through interaction results}. That is, the environment (or its designer) implicitly expresses a preference for certain strategies over others by determining their success or failure in competitive play.

Therefore, strength measurement can be seen as an operationalized form of preference modeling:
\begin{itemize}
  \item The \textbf{preference function} is induced by win/loss outcomes.
  \item The \textbf{agent's value} is defined relationally---through its interactions with other agents.
  \item The \textbf{source of preference} is external---encoded by the environment or designer’s reward structure.
\end{itemize}

With this framing, the act of estimating competitive strength is not just a technical task---it is a philosophical stance about how we define ``good'' strategies. This makes strength measurement a critical prerequisite for any meaningful discussion of playstyle balance.

\subsubsection*{Rating Systems and the Bradley--Terry Model}

Many strength estimation methods have been developed to formalize these ideas. Classical rating systems---such as Elo, Glicko, Whole-History Rating (WHR), TrueSkill, and Matchmaking Rating (MMR)---offer scalar skill estimates that improve upon raw win rates by contextualizing outcomes based on opponent strength \citepx{elo, glicko, whr, true_skill, mmr, elo_mmr}.

At the core of many such systems lies the Bradley--Terry model \citepx{bradley_terry}, which models the probability that strategy $i$ defeats strategy $j$ as:
\begin{equation}
P(i > j) = \frac{\gamma_i}{\gamma_i + \gamma_j},
\end{equation}
where $\gamma_i$ and $\gamma_j$ are non-negative real-valued strength parameters.

To improve numerical stability and interpretability, this model is often rewritten in exponential form, setting $\gamma_i = e^{\lambda_i}$:
\begin{equation}
P(i > j) = \frac{e^{\lambda_i}}{e^{\lambda_i} + e^{\lambda_j}}.
\end{equation}

In fact, one commonly used instantiation of this principle---with an alternate base and parameterization---is the expected win probability in Elo rating systems:
\begin{equation}
E_A = \frac{1}{1 + 10^{(R_B - R_A)/400}}.
\end{equation}
This can be derived from the Bradley--Terry model via a change of base and rescaling of $\lambda_i$, offering a practical and interpretable implementation of pairwise win probability.

\subsubsection*{Elo Updates and Online Adaptation}

Elo further introduces an incremental update rule that allows models to adapt to changing data in an online manner. A player’s rating is adjusted after each match based on the prediction error:
\begin{equation}
R'_A = R_A + K (S_A - E_A),
\end{equation}
where $S_A$ is the actual result (1 for win, 0 for loss), $E_A$ is the expected win probability from above, and $K$ is a tunable learning rate.

This update rule provides a practical mechanism for applying the Bradley--Terry framework in real-world systems. It supports immediate adjustments from new results, which is crucial in dynamic environments such as online games, evolving agent systems, or active training pipelines.

\textbf{In summary}, strength measurement grounded in models like Bradley--Terry and Elo provides a way to embed environmental preferences into computational form. While powerful, these scalar models are limited when facing intransitive or context-sensitive strategic landscapes, a topic we turn to next.

\subsection{Counter Relationships}

Not all victories are created equal. A strategy that consistently defeats one class of opponents may still perform poorly against others—despite having a similar overall win rate. This phenomenon is known as a \textbf{counter relationship}, and it reflects a key form of strategic intransitivity: strategy A beats B, B beats C, but C beats A.

Such counter dynamics are widespread in competitive games, including \textit{Rock-Paper-Scissors}, trading card games, MOBAs, and role-based team shooters. They are essential for maintaining a healthy strategic ecosystem—ensuring that no single strategy remains unconditionally dominant, and encouraging players to adapt, read opponents, and diversify their repertoire.

However, counter relationships pose a major challenge for classical scalar strength models like Elo, which are fundamentally transitive. These models assume that if $A$ has a higher rating than $B$, then $A$ should reliably defeat $B$. But this assumption breaks down when outcomes depend heavily on the interaction between specific strategies, rather than aggregate performance.

Consider Figure~\ref{fig:counter_relationships}. Even if Player A has a higher Elo rating than Player B, if A plays Strategy X and B plays Strategy Y—and Y happens to counter X—then B may consistently win. This violates the Elo assumption that higher-rated players win more often, highlighting the inadequacy of scalar models in intransitive settings.

\begin{figure}[t]
    \centering
    \includegraphics[width=0.7\textwidth]{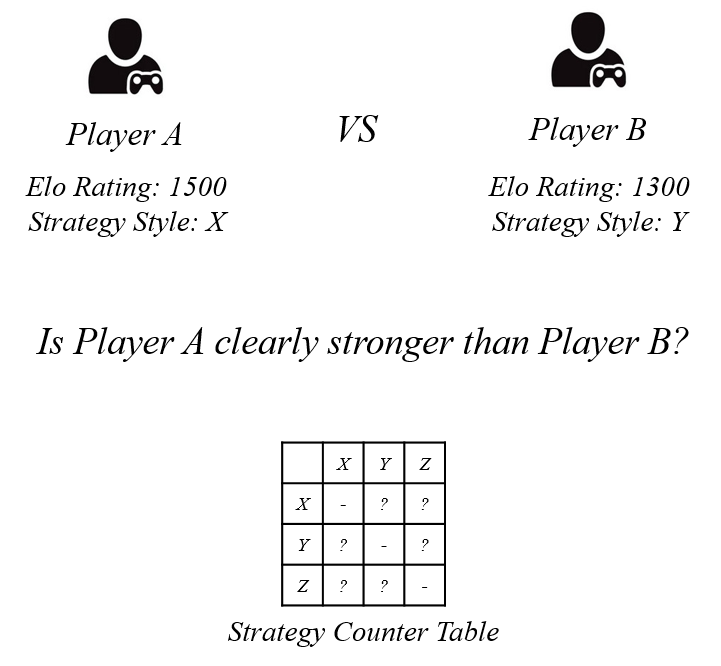}
    \caption[An illustration of counter relationships in games.]{An illustration of counter relationships in games. Player A may have a higher Elo rating than Player B, yet due to the interaction between Strategies X and Y, B consistently wins.}
    \label{fig:counter_relationships}
\end{figure}

To model this kind of context-dependent strength, we need to move beyond global scalar ratings. What is needed are models that explicitly encode \textit{against-whom} a strategy is effective—capturing not just “how strong” a strategy is overall, but “under what conditions” or “against which opponents” it succeeds or fails.

In the following section, we introduce two such extensions to the Bradley--Terry framework: the \textbf{Neural Rating Table (NRT)} and the \textbf{Neural Counter Table (NCT)}. These models offer fine-grained representations of strategy interaction and enable precise balance analysis in complex and intransitive game environments.

\section{Measuring Balance in Complex Strategy Spaces}

As discussed in the previous section, scalar rating systems such as Elo or Bradley--Terry offer a practical foundation for tracking competitive strength. However, they fail to capture the full complexity of real-world strategy spaces—especially when playstyles interact in non-transitive, asymmetric, or context-sensitive ways.

In this section, we move beyond strength estimation to confront the deeper problem: \textbf{how can balance itself be measured?} Our goal is not merely to assign scores to individual strategies, but to understand how the strategic ecosystem behaves as a whole. Which playstyles are viable? Which ones dominate or are dominated? Are counter-relationships structurally supported?

To answer these questions, we introduce a set of methods designed to measure balance in complex strategy spaces. These include:
\begin{itemize}
  \item The \textbf{Neural Rating Table (NRT)}, a neural network model that estimates the joint strength of multi-agent team compositions, capturing nonlinear synergy and interaction effects across roles and components.
  \item The \textbf{Neural Counter Table (NCT)}, which clusters strategies into latent counter categories and learns contextual win probabilities between these groups to reveal emergent counter structures.
  \item Balance-specific indicators such as \textbf{Top-D Diversity} and \textbf{Top-B Balance}, designed to quantify ecosystem viability and counter-structural richness.
  \item An online learning framework that supports adaptive, real-time balance modeling in dynamic multi-agent systems.
\end{itemize}

Together, these tools provide a data-driven foundation for evaluating and maintaining balance in strategically rich environments, from competitive games to policy optimization in complex AI systems.

\subsection{Neural Rating Table and Neural Counter Table}

While traditional scalar rating systems such as Elo offer a compact and interpretable means of estimating competitive strength, they inherently assume transitivity. As discussed previously, this assumption breaks down in the presence of counter relationships and cyclical dominance—conditions that are widespread in real-world games.

To meaningfully measure balance in such complex strategic environments, we require models that go beyond scalar strength and directly encode how strategies interact, combine, and counter one another. In this section, we present two neural extensions to classical rating theory: the \textbf{Neural Rating Table (NRT)} and the \textbf{Neural Counter Table (NCT)}, proposed by \citetx{game_balance_analysis}.

The NRT learns to estimate the strength of entire compositions, capturing synergy and non-linear interactions among components. The NCT complements this by discovering discrete counter categories and learning residual win relationships between them—allowing us to recover intransitive dynamics and structural counterplay from gameplay data. Together, these models provide a foundation for scalable, differentiable, and interpretable balance estimation.

\subsubsection{Modeling Complex Compositions with NRT}
The \textbf{Neural Rating Table (NRT)} extends the Bradley--Terry model to estimate non-linear combinational strength in team-based environments. While the original Bradley--Terry model is designed for two-player zero-sum games, NRT generalizes the idea by treating an entire composition (e.g., a team of heroes or cards) as a single entity and approximates its strength using a neural network.

To predict the win probability between two team compositions $c_i$ and $c_j$, NRT follows the Bradley--Terry formulation:
\begin{equation}
P(i > j) = \frac{e^{\lambda_i}}{e^{\lambda_i} + e^{\lambda_j}},
\end{equation}
where $\lambda_i = \log R_\theta(c_i)$, and $R_\theta$ is a neural encoder that maps composition $c_i$ to a scalar strength rating.

We implement $R_\theta$ using a Siamese neural network architecture with shared weights \citepx{siamese_network}, as shown in Figure~\ref{figure:neural_rating_table}. The network processes two competing compositions in parallel and predicts their relative strength values. The input features can range from simple binary encodings to richer representations capturing internal synergy.

\begin{figure}[t]
    \centering
    \includegraphics[width=0.4\textwidth]{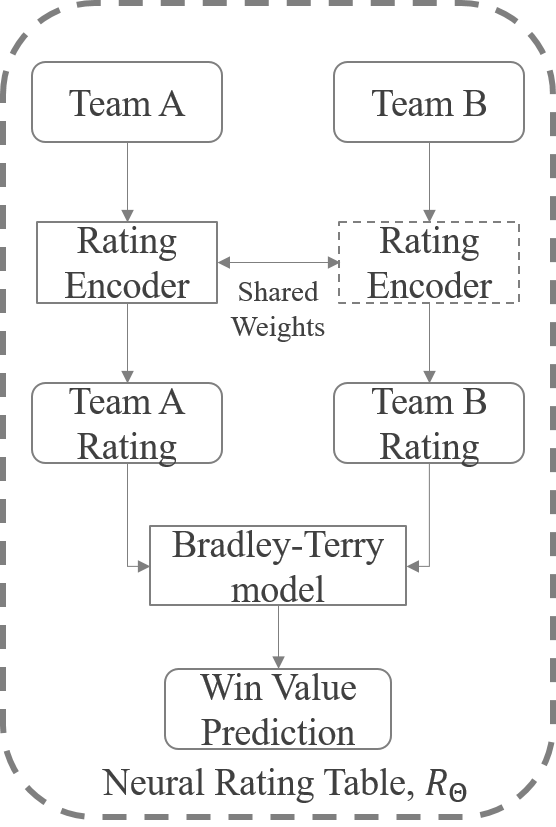}
    \caption[Neural Rating Table Architecture]{Neural Rating Table (NRT) predicts win probabilities between two teams using shared-weight encoders and the Bradley--Terry model, with exponential activation applied in the rating encoder for reparameterization.}
    \label{figure:neural_rating_table}
\end{figure}

To train the network, we minimize the difference between the predicted and actual match outcomes using a loss function $D$, typically mean squared error (MSE):
\begin{equation}
L_{R_\theta} = \mathbb{E}\left[D\left(W, \frac{R_\theta(c_i)}{R_\theta(c_i) + R_\theta(c_j)}\right)\right].
\end{equation}

We initially explored binary cross-entropy as a probabilistic loss, but found it led to unstable convergence and overly large rating magnitudes, similar to hinge loss behavior. MSE provided greater stability and smoother optimization.

This integration allows NRT to estimate the strength of any composition using learned representations, eliminating the need to store or compute a full $N \times N$ win-rate matrix. Ratings scale with the number of compositions as $\mathcal{O}(N)$, making NRT an efficient and scalable solution for modeling combinatorial strength. It also captures the interaction effects and synergy between individual elements within a composition, enabling generalization to unseen combinations.

We now turn to the complementary model that captures intransitive interactions: the \textbf{Neural Counter Table (NCT)}.

\subsubsection{Learning Discrete Counter Relationships with NCT}

The Neural Rating Table (NRT) enables the estimation of strength values for all $N$ possible team compositions with a compact space complexity of $\mathcal{O}(N)$. However, this efficiency comes at a trade-off: the inferred transitive strength relations may be less precise than those obtained via exhaustive pairwise comparisons, which, though more accurate, require $\mathcal{O}(N^2)$ space. While such comparisons can in principle be modeled using neural architectures, the quadratic growth in space makes them difficult to validate or interpret—particularly in large-scale settings.

This limitation becomes even more pronounced when accounting for intransitive relationships or cyclic dominance patterns, which are frequently observed in competitive games. Attempting to capture these relationships explicitly through a full $N \times N$ counter table introduces significant memory and cognitive overhead. Although theoretically informative, such a table would be impractical for both algorithmic learning and human understanding.

To address this, we propose a more tractable alternative: a reduced $M \times M$ counter matrix that approximates the $N \times N$ counter relationships via composition clustering. Here, $M$ denotes a small number of discrete categories—starting from three, which suffices to capture basic cyclic structures in symmetric games. The number of categories can be tuned to balance expressivity with interpretability, making it easier to analyze global balance patterns without sacrificing too much predictive fidelity.

This design naturally reflects how human players mentally structure their understanding of strategic counter relationships—as a finite set of intuitive categories (e.g., rock beats scissors, scissors beat paper, etc.).

For the task of learning discrete categories, we adopt Vector Quantization (VQ) \citepx{vq_vae}, as it enables discrete latent representation learning in an end-to-end neural setting. In contrast to traditional reconstruction tasks in VQ-VAE, our objective is to classify compositions into categories that encode counter-relationship structure via residual win prediction.

Given an observation $o$ and its latent feature $z_e$ from an encoder, we assign it to a codebook entry $e_k$ by nearest neighbor:
\begin{equation}
    q(\overline{s}=k|o)=
    \begin{cases}
    1 & \quad \text{for } k_{i} = \underset{j}{\argmin} ||z_{e}^{i}(o)-e_{j}||^2,\\
    0 & \quad \text{otherwise.}
    \end{cases}
\end{equation}

Given this mapping, we define our counter residual as:
\begin{equation}
\label{eq:counter_table}
W_{res}(c_{m_i}, c_{m_j} | R_{\theta}) = W_m - \frac{R_{\theta}(c_{m_i})}{R_{\theta}(c_{m_i}) + R_{\theta}(c_{m_j})}
\end{equation}

We use a Siamese encoder $Ce_{\theta}$ and a residual decoder $Cd_{\theta}$ to learn this correction over the Bradley--Terry baseline:
\begin{equation}
C_{\theta}(c_{m_i}, c_{m_j}) = Cd_{\theta}(Ce_{\theta}(c_{m_i}), Ce_{\theta}(c_{m_j})).
\end{equation}

The total predicted win value thus becomes:
\begin{equation}
\label{eq:nct_win}
W_{\theta}(c_{m_i}, c_{m_j}) = \frac{R_{\theta}(c_{m_i})}{R_{\theta}(c_{m_i}) + R_{\theta}(c_{m_j})} + C_{\theta}(c_{m_i}, c_{m_j}).
\end{equation}

To train this architecture, we define three loss terms:
\begin{itemize}
  \item $L_{res}$: MSE between predicted and observed residual win values
  \item $L_{vq}$: vector quantization loss between $z_e$ and its quantized version $z_q$
  \item $L_{mean}$: a proposed \textbf{VQ Mean Loss}, which pulls $z_e$ toward the mean codebook embedding
\end{itemize}

The motivation for \textbf{VQ Mean Loss} stems from a common problem in VQ training—low codebook utilization. Even with diverse training data, we found that many embedding vectors $e_k$ in the codebook remained unused. This reduces the expressiveness and coverage of the $M \times M$ counter table. Existing solutions such as stochastic VQ variants \citepx{vq_vae, improved_vqgan, orthogonal_codebook} and techniques like Gumbel softmax sampling or KL-regularized VQ \citepx{reg_vq} aim to address this, but they often introduce additional training complexity. Our proposed VQ Mean Loss offers a lightweight and implementation-friendly alternative:
\begin{equation}
L_{mean} = \mathbb{E}\left[\frac{D(z_e(c_{m_i}), \bar{e}) + D(z_e(c_{m_j}), \bar{e})}{2}\right], \quad \bar{e} = \frac{1}{M}\sum_{k=1}^{M} e_k.
\end{equation}
This term draws latent vectors toward the center of the codebook, increasing the chance of vector assignments being spread more evenly.

\begin{figure}
\begin{center}
\includegraphics[width=\linewidth]{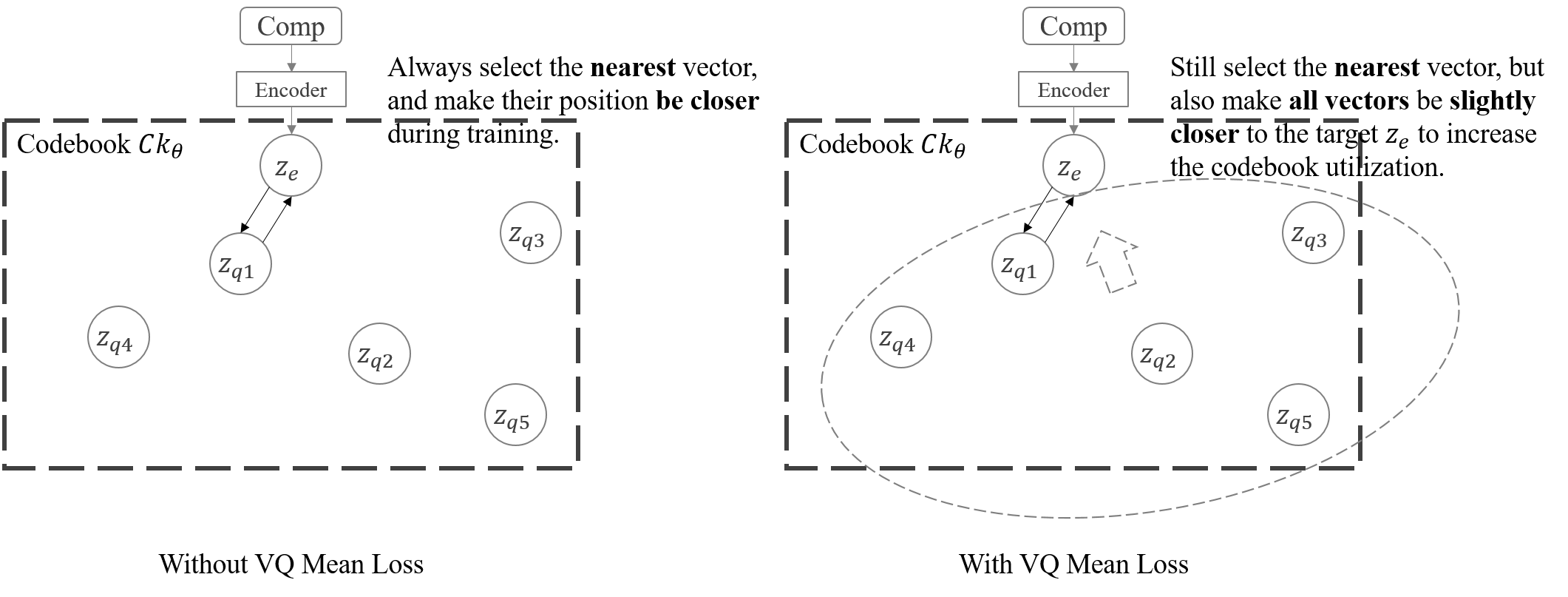}
\end{center}
\caption[VQ Mean Loss Effect]{
Effect of VQ Mean Loss on codebook utilization. Left: Without VQ Mean Loss, the encoder output $\mathbf{z}_e$ aligns with only the closest codeword. Right: With VQ Mean Loss, $\mathbf{z}_e$ is pulled toward the mean of multiple nearby codewords, improving codebook usage and representational flexibility.
}
\label{fig:embedding_problem}
\end{figure}

The gradient flows for the components involved in learning counter relationships are modulated by two hyperparameters, $\beta_N$ and $\beta_M$. The coefficient $\beta_N$ governs the VQ-VAE’s alignment between continuous encoder outputs $z_e$ and their nearest quantized codewords $z_q$, enforcing proximity in latent space. In parallel, $\beta_M$ activates the mean regularization term $L_{mean}$ to guide codebook structuring. The resulting gradients are computed as follows:
\begin{equation}
\begin{aligned}
\nabla Ce_{\theta} &= \frac{\partial L_{res}}{\partial z_q} \cdot \frac{\partial z_e}{\partial Ce_{\theta}} + \beta_N \cdot \frac{\partial L_{vq}}{\partial Ce_{\theta}}, \\
\nabla Ck_{\theta} &= \frac{\partial L_{vq}}{\partial Ck_{\theta}} + \beta_M \cdot \frac{\partial L_{mean}}{\partial Ck_{\theta}}, \\
\nabla Cd_{\theta} &= \frac{\partial L_{res}}{\partial Cd_{\theta}}.
\end{aligned}
\end{equation}

This architecture allows us to reduce the space complexity from $\mathcal{O}(N^2)$ to $\mathcal{O}(N + M^2)$, enabling a more scalable yet interpretable model of counter relationships.

\begin{figure}
\centering
\includegraphics[width=0.6\textwidth]{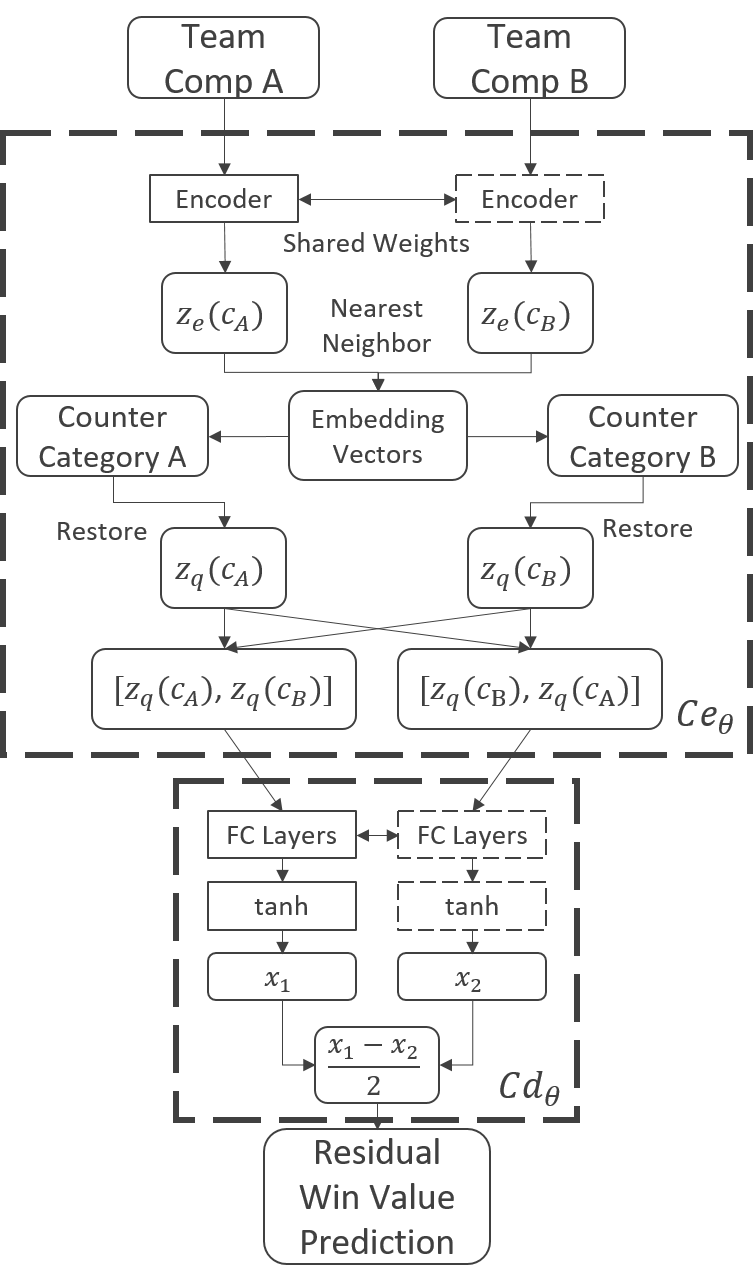}
\caption[Neural Counter Table Architecture]{
Architecture of the Neural Counter Table $C_\theta$. Given two team compositions (Comp A and Comp B), shared-weight encoders produce continuous latent representations $z_e(c_A)$ and $z_e(c_B)$. These are quantized via nearest-neighbor search to obtain discrete embedding vectors $z_q(c_A)$ and $z_q(c_B)$, which are categorized into respective counter classes. The decoded vectors—paired in both orders—are fed into fully connected layers with $\tanh$ activations to produce intermediate scores $x_1$ and $x_2$. The residual win value prediction $W_{res}$ is computed as the average of their differences, providing a symmetric estimate of the counter effect between the two compositions.
}
\label{Figure:neural_counter_table}
\end{figure}

\subsubsection{Learning Pipeline}

The overall methodology for constructing the Neural Rating Table ($R_\theta$) and Neural Counter Table ($C_\theta$) is illustrated in Figure~\ref{Figure:nrt_nct_pipeline}. This framework operates on a dataset comprising match outcomes, where each data point contains two competing team compositions and the corresponding win/loss result. The way compositions are represented is flexible, allowing for anything from simple binary indicators to richer, high-dimensional feature encodings—depending on the design choices of the game system.

A key aspect of this pipeline is that the construction of $C_\theta$ depends on the prior learning of $R_\theta$. This sequential structure reflects the conceptual separation between general team strength estimation and the modeling of residual counter-effects. By first learning composition ratings independently, the framework ensures that $C_\theta$ captures non-transitive and context-specific interactions beyond what $R_\theta$ can represent.

\begin{figure}
\centering
\includegraphics[width=0.8\textwidth]{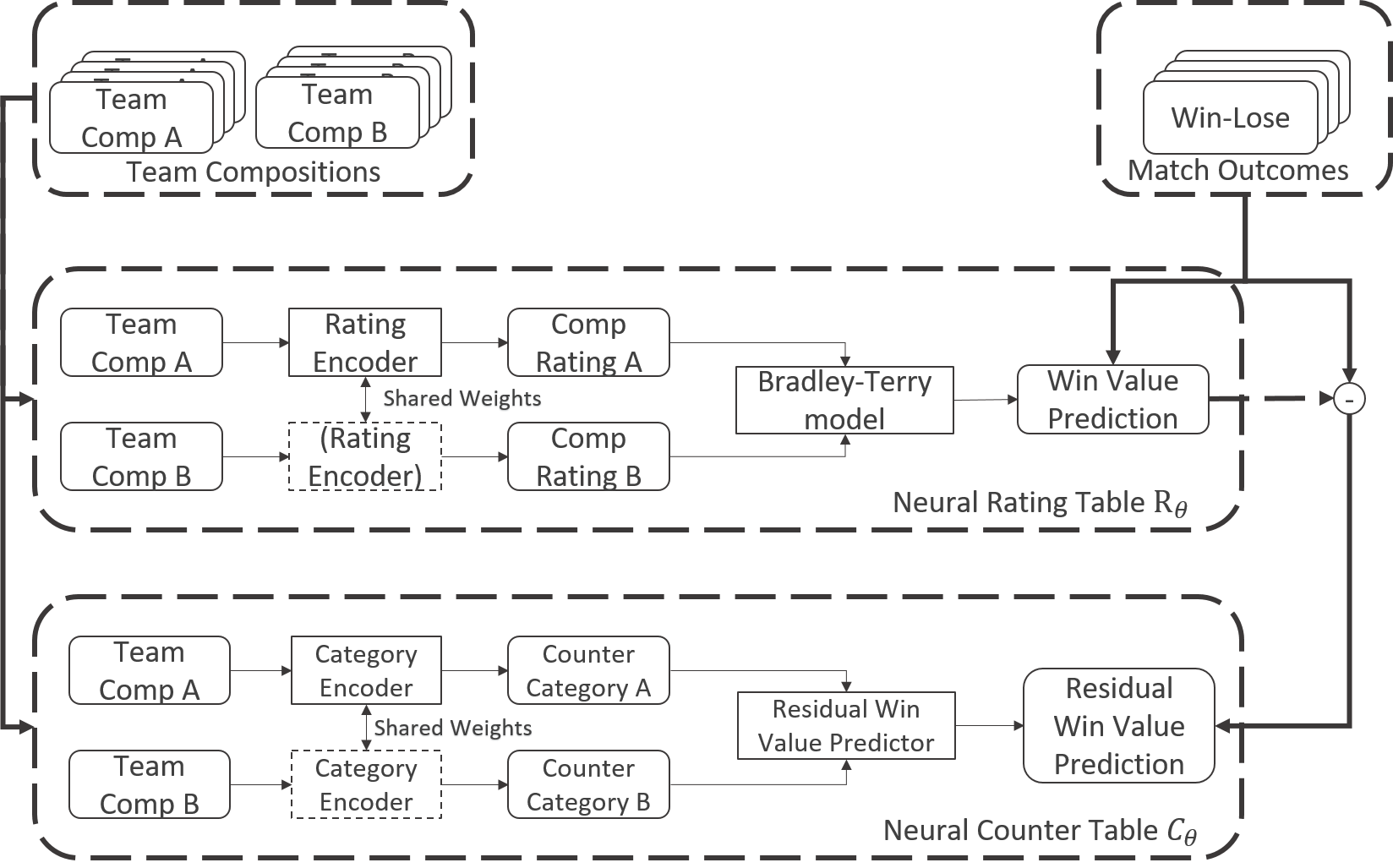}
\caption[NRT and NCT Pipeline]{
Illustration of the learning pipeline for the Neural Rating Table ($R_\theta$) and Neural Counter Table ($C_\theta$). Given match data consisting of team compositions (Comp A and Comp B) and corresponding win/loss outcomes, the pipeline proceeds in two stages. First, compositions are encoded via shared rating encoders to produce composition ratings, which are then passed into a Bradley–Terry model to generate initial win value predictions—constituting $R_\theta$. In the second stage, the same compositions are encoded into discrete counter categories via a shared category encoder. These category pairs are then used to compute residual win values, allowing $C_\theta$ to capture counter relationships beyond transitive ratings and account for cyclic dominance effects.
}
\label{Figure:nrt_nct_pipeline}
\end{figure}

In essence, NRT captures the transitive base strength of compositions, while NCT learns the non-transitive residual deviations—serving as a complementary enhancement over the baseline.

\subsubsection{Evaluation Games and Datasets}

To evaluate the accuracy and generalizability of our rating and counter modeling methods, we conduct experiments across both synthetic and real-world PvP game datasets. These datasets allow us to examine performance under controlled structural properties (e.g., transitivity and intransitivity) and to test scalability and prediction quality in practical competitive settings.

We categorize the test environments as follows:

\begin{itemize}
    \item \textbf{Synthetic games}: including a score-based composition game, the canonical Rock-Paper-Scissors (RPS) scenario, and a hybrid variant embedding RPS rules into a compositional framework.
    \item \textbf{Real-world games}: Age of Empires II (AoE2), Hearthstone, Brawl Stars, and League of Legends (LoL), chosen to reflect a variety of composition scales, strategic depth, and intransitivity levels.
\end{itemize}

\paragraph{Simple Combination Game} This synthetic game features 20 elements, each assigned a numeric score from 1 to 20. A composition consists of three distinct elements, with its strength defined as the sum of its members' scores. Given two compositions $c_1$ and $c_2$ with scores $s_1$ and $s_2$, the win probability is defined as $P(c_1 > c_2) = \frac{s_1^2}{s_1^2 + s_2^2}$. This introduces smooth score-based differences without inherent cyclic structures. The full set contains $C^{20}_3 = 1140$ unique compositions, with 100,000 matches uniformly sampled.

\paragraph{Rock-Paper-Scissors (RPS)} The canonical RPS game includes three strategies with cyclic dominance: Rock beats Scissors, Scissors beats Paper, and Paper beats Rock. Win values are set as 0 (loss), 0.5 (tie), and 1 (win). We sample 100,000 matchups uniformly. This serves as a minimal intransitive benchmark.

\paragraph{Advanced Combination Game} This hybrid setting embeds RPS structure into the combination game above. A composition's base score is still the sum of its elements, but an additional +60 bonus is awarded if it wins the RPS rule determined by $T = s_c \mod 3$, where 0, 1, and 2 correspond to Rock, Paper, and Scissors respectively. The composition space remains $C^{20}_3 = 1140$, with 100,000 matches.

\paragraph{Age of Empires II (AoE2)} AoE2 is a popular real-time strategy game with 45 playable civilizations. We use data from the public aoestats.io dataset (January 2024), which includes 1,261,288 1v1 matches from random map mode. Each composition corresponds to one selected civilization.

\paragraph{Hearthstone} Hearthstone is a collectible card game where compositions correspond to deck archetypes. Using data from HSReplay (January 2024), we extract 91 standardized decks at the Gold ranking level, resulting in 10,154,929 matches.

\paragraph{Brawl Stars} In this MOBA-style game, each team selects three brawlers. The full composition space is large ($C^{64}_3 \times 43 \times 6 \approx 10^7$), factoring in characters, maps, and modes. From the "Brawl Stars Logs \& Metadata 2023" dataset, we sample 179,995 matches, covering 94,235 unique team compositions.

\paragraph{League of Legends (LoL)} LoL features 5v5 team compositions drawn from a pool of 136 champions, yielding over 350 million possible combinations. We use 182,527 solo queue matches from a Kaggle dataset, observing 348,498 unique compositions, indicating high sparsity and diversity.

Together, these datasets span controlled and large-scale scenarios, enabling rigorous testing of strength relation prediction, counter modeling, and balance quantification in both theoretical and practical settings.

\subsubsection{Strength Relation Prediction Accuracy}

To assess whether our models accurately capture strategic dominance, we frame strength prediction as a classification task: given two compositions $c_1$ and $c_2$, does the model correctly predict whether $c_1$ is weaker, equal, or stronger than $c_2$?

Labels are derived from average observed win values. For instance, if Win($c_1$, $c_2$) > 0.501, then $c_1$ is labeled stronger; if < 0.499, then weaker; otherwise, equal. Predicted labels are computed similarly using the model's estimated win values. Accuracy is the fraction of correctly classified comparisons.

We compare five baseline methods:
\begin{itemize}
    \item \textbf{WinValue}: direct win value prediction from composition features.
    \item \textbf{PairWin}: lookup-based pairwise table (ideal but non-generalizable).
    \item \textbf{BT}: linear Bradley-Terry model over component ratings.
    \item \textbf{NRT}: non-linear neural generalization of BT.
    \item \textbf{NCT}: NRT plus counter correction using the learned $M \times M$ counter table.
\end{itemize}

Table~\ref{table:precision_test} summarizes the accuracy of our strength relation prediction models across diverse game environments. The results show that our proposed method, \textbf{NCT} with $M=81$, consistently achieves accuracy on par with \textbf{PairWin}, the oracle-like pairwise baseline, across all evaluated games.

In games with complex and high-dimensional composition spaces, such as \textit{Brawl Stars} and \textit{League of Legends}, scalar or linear models (e.g., WinValue, BT) fall short. The use of neural approximation in \textbf{NRT} significantly improves accuracy, and the additional modeling of counter relationships in \textbf{NCT} further enhances performance. This confirms the necessity of non-linear and intransitive modeling in large-scale competitive environments.

Moreover, in domains where explicit counter relationships are known or expected—such as \textit{Rock-Paper-Scissors}, the \textit{Advanced Combination Game}, and \textit{Age of Empires II}—\textbf{NCT} successfully bridges the large accuracy gap observed between \textbf{PairWin} and \textbf{NRT}, validating its ability to recover underlying intransitive structures. In contrast, direct win value predictors like \textbf{WinValue} often underperform due to their inability to model structural dependencies.

Table~\ref{table:table_size_test} further demonstrates that increasing the counter table size $M$ leads to steady improvements in prediction accuracy. This indicates that richer counter category granularity enables more precise modeling of strategic relationships, while still maintaining scalability.

We also observe the generalization behavior of the models by comparing training and testing accuracy. In most games, the difference remains small, suggesting that learned strength relations are robust and not overfit. However, in \textit{League of Legends}, a notable performance gap suggests overfitting due to the extreme sparsity of the composition space and the added complexity introduced by the ban/pick phase. For such cases, we recommend expanding the training dataset to improve generalizability.

Despite these challenges, the rating and counter models still yield valuable insights from observed data and can guide early-stage game balance analysis even before full generalization is achieved.

\begin{table*}
\centering
\caption[NRT and NCT Accuracy]{Prediction accuracies (\%) on training (top) and testing (bottom) datasets across various games. The results highlight the effectiveness of the proposed \textbf{NCT} method (with $M=81$) in modeling diverse competitive environments.}
\label{table:precision_test}
\begin{tabular}{llllll}
\toprule
\quad & \textbf{WinValue} & \textbf{PairWin} & \textbf{BT} & \textbf{NRT} & \textbf{NCT} M=81 \\
\midrule
Simple Combination & 64.5 & \textbf{71.2} & 63.9 & 64.8 & 66.4\\
\midrule
Rock-Paper-Scissors & 51.3 & \textbf{100} & 51.3 & 51.3 & \textbf{100}\\
\midrule
Advanced Combination & 57.7 & \textbf{83.5} & 56.6 & 57.9 & 79.4\\
\midrule
Age of Empires II & 68.7 & \textbf{97.3} & 68.7 & 68.7 & \textbf{97.7}\\
\midrule
Hearthstone & 81.1 & \textbf{97.8} & 81.4 & 83.4 & \textbf{97.4}\\
\midrule
Brawl Stars & 90.2 & 94.3 & 53.2 & 95.9 & \textbf{97.2}\\
\midrule
League of Legends & 79.6 & 78.9 & 54.0 & 88.2 & \textbf{90.9}\\
\bottomrule
\toprule
Simple Combination & 64.7 & 61.8 & \textbf{65.5} & 64.9 & 63.9\\
\midrule
Rock-Paper-Scissors & 51.1 & \textbf{100} & 51.1 & 51.1 & \textbf{100}\\
\midrule
Advanced Combination & 56.5 & 79.1 & 57.5 & 56.5 & \textbf{79.7}\\
\midrule
Age of Empires II & 64.5 & \textbf{75.7} & 64.5 & 64.5 & \textbf{75.4}\\
\midrule
Hearthstone & 80.9 & \textbf{95.4} & 81.1 & 81.2 & 94.8\\
\midrule
Brawl Stars & 79.7 & 82.4 & 53.0 & 82.8 & \textbf{83.4}\\
\midrule
League of Legends & 51.1 & 50.9 & \textbf{53.6} & 51.1 & 51.0\\
\bottomrule
\end{tabular}
\end{table*}

\begin{table*}
\centering
\caption[NCT Table Size $M$ Search]{Training (top) and testing (bottom) accuracies (\%) across various games using Neural Counter Tables with different category sizes $M$. Larger $M$ values generally lead to higher prediction accuracy, demonstrating improved modeling of counter relationships. The performance gap between \textbf{PairWin} and \textbf{NRT} further indicates the presence of intransitive interactions that cannot be captured by scalar ratings alone.}
\label{table:table_size_test}
\begin{tabular}{lll|llll}
\toprule
\quad & \textbf{PairWin} & \textbf{NRT} & \textbf{NCT} M=3 & M=9 & M=27 & M=81 \\
\midrule
Simple Combination & \textbf{71.2} & 64.8 & 64.8 & 65.2 & 65.8 & 66.4\\
\midrule
Rock-Paper-Scissors & \textbf{100} & 51.3 & \textbf{100} & \textbf{100} & \textbf{100} & \textbf{100}\\
\midrule
Advanced Combination & \textbf{83.5} & 57.9 & 57.9 & 79.4 & 79.8 & 79.4\\
\midrule
Age of Empires II & \textbf{97.3} & 68.7 & 68.7 & 73.1 & 83.8 & \textbf{97.7}\\
\midrule
Hearthstone & \textbf{97.8} & 83.4 & 81.3 & 85.4 & 91.7 & \textbf{97.4}\\
\midrule
Brawl Stars & 94.3 & 95.9 & 96.3 & \textbf{97.3} & \textbf{97.2} & \textbf{97.2}\\
\midrule
League of Legends & 78.9 & 88.2 & 89.5 & 91.6 & \textbf{92.6} & 90.9\\
\bottomrule
\toprule
Simple Combination & 61.8 & \textbf{64.9} & \textbf{64.9} & 64.4 & 64.2 & 63.9\\
\midrule
Rock-Paper-Scissors & \textbf{100} & 51.1 & \textbf{100} & \textbf{100} & \textbf{100} & \textbf{100}\\
\midrule
Advanced Combination & 79.1 & 56.5 & 56.5 & \textbf{79.7} & 80.1 & \textbf{79.7}\\
\midrule
Age of Empires II & \textbf{75.7} & 64.5 & 64.5 & 67.7 & 72.5 & \textbf{75.4}\\
\midrule
Hearthstone & \textbf{95.4} & 81.2 & 81.3 & 85.2 & 91.3 & 94.8\\
\midrule
Brawl Stars & 82.4 & 82.8 & 82.9 & 83.3 & 83.3 & \textbf{83.4}\\		
\midrule
League of Legends & 50.9 & 51.1 & 51.1 & 50.9 & 51.0 & 51.0\\
\bottomrule
\end{tabular}
\end{table*}

\subsubsection{Counter Table Utilization}

The effectiveness of the Neural Counter Table (NCT) relies not only on the accuracy of residual prediction but also on the utilization of its discrete embedding space. When the number of categories $M$ is too large relative to the diversity of latent features, a common issue is \textit{codebook underutilization}, where only a small subset of the available quantized vectors are actually used during training. This degrades both model capacity and interpretability, as unused entries do not contribute to residual modeling.

To address this issue, we introduced an auxiliary loss—\textbf{VQ Mean Loss}—which encourages latent embeddings to stay close to the mean of the codebook vectors. This regularization acts as a soft attraction force that pulls embeddings toward underutilized regions of the codebook, effectively promoting broader category usage without disrupting the nearest-neighbor quantization behavior.

We evaluate this technique using the AoE2 dataset with $M=27$ categories, a setting that exhibits moderate intransitivity and benefits from a medium-sized counter table. Table~\ref{table:beta_comparison1} compares different values of the VQ commitment weight $\beta_N$ under a fixed VQ Mean Loss weight $\beta_M=0.25$, and Table~\ref{table:beta_comparison2} further explores variations in $\beta_M$ while keeping $\beta_N=0.01$.

Our results show that the default setting of $\beta_N = 0.25$ from standard VQ-VAE literature leads to severe underutilization (e.g., only 1 or 2 codes used). Lowering $\beta_N$ and activating the VQ Mean Loss with $\beta_M=0.25$ achieves both higher prediction accuracy and nearly full codebook utilization, validating the importance of this modification.

\begin{table}[tb]
\centering
\caption[VQ $\beta_N$ Search]{
Training accuracy and the number of utilized categories ($M$) in Age of Empires II ($M=27$) under different $\beta_N$ values, with $\beta_M$ fixed at 0.25. The commonly used setting in VQ-VAE, $\beta_N=0.25$, is included for comparison. Utilized $M$ refers to the number of active categories assigned to any composition in the dataset after training.
}
\label{table:beta_comparison1}
\begin{tabular}{lll}
\toprule
\textbf{Setting} & \textbf{Accuracy (\%)} & \textbf{Utilized M} \\
\midrule
\pmb{$\beta_{N}=0.01$} & \textbf{83.8} & \textbf{26.0}\\
$\beta_{N}=0.125$ & 71.4 & 5.4\\
$\beta_{N}=0.25$ & 68.7 & 1.0\\
\bottomrule
\end{tabular}
\end{table}

\vspace{1em}

\begin{table}[tb]
\centering
\caption[VQ Mean $\beta_M$ Search]{
Training accuracy and the number of utilized categories ($M$) in Age of Empires II ($M=27$) under varying $\beta_M$ values, with $\beta_N$ fixed at 0.01. The results show that both excessively small and large $\beta_M$ values lead to suboptimal codebook utilization, while $\beta_M=0.25$ achieves the best overall performance.
}
\label{table:beta_comparison2}
\begin{tabular}{lll}
\toprule
\textbf{Setting} & \textbf{Accuracy (\%)} & \textbf{Utilized M} \\
\midrule
$\beta_{M}=0$ & 76.6 & 13.8\\
$\beta_{M}=0.125$ & 83.6 & 26.0\\
\pmb{$\beta_{M}=0.25$} & \textbf{83.8} & \textbf{26.0}\\
$\beta_{M}=0.5$ & 83.4 & 25.6\\
$\beta_{M}=1.0$ & 81.0 & 21.6\\
\bottomrule
\end{tabular}
\end{table}

\subsubsection{Comparison to Tabular Baselines}

While our proposed NRT and NCT models offer generalization and interpretability, it is instructive to compare them against simpler alternatives that use tabular representations rather than neural approximators.

We implement tabular variants of four common rating systems: \begin{itemize} 
	\item \textbf{WinValue N$\rightarrow$T}: direct averaging of observed win values. 
	\item \textbf{PairWin N$\rightarrow$T}: full composition-by-composition win matrix. 
	\item \textbf{Elo N$\rightarrow$T}: standard scalar Elo rating with online updates. 
	\item \textbf{mElo2}: a two-dimensional extension of Elo capturing weak counter structures. 
\end{itemize}

As shown in Table~\ref{table:tabular_precision_test}, tabular methods excel at memorizing training data (sometimes nearly perfect), but their generalization performance varies. For example, PairWin achieves near-100\% training accuracy in most settings but fails on unseen test matchups due to its lack of function approximation. On the other hand, Elo and mElo2 generalize moderately but underperform in games with strong intransitivity, where scalar or low-dimensional ratings are insufficient.

In contrast, NCT achieves a favorable tradeoff between accuracy and generalizability, especially in complex or counter-heavy games like Hearthstone, AoE2, and Brawl Stars. While it does not outperform PairWin on memorization, it significantly outperforms all tabular methods on held-out test sets.

These results underscore the value of structured neural approximations: they combine the interpretability and update efficiency of rating models with the flexibility of learning from high-dimensional and intransitive environments.

\begin{table}
\centering
\caption[NRT and NCT: Comparison to Tabular Baselines]{
Prediction accuracies (\%) across multiple games under different modeling paradigms. For methods that support both neural and tabular implementations (WinValue, PairWin, Elo), arrows ($\rightarrow$) denote the corresponding performance when transitioning from neural training to tabular inference. The results reveal that while some methods suffer severe degradation in tabular form, the proposed \textbf{NCT} maintains consistently strong accuracy, even without tabular deployment.
}
\label{table:tabular_precision_test}
\rotatebox{90}{
\begin{tabular}{llllll}
\toprule
\quad & \textbf{WinValue N$\rightarrow$T} & \textbf{PairWin N$\rightarrow$T} & \textbf{Elo N$\rightarrow$T} & \textbf{mElo2} & \textbf{NCT} M=81\\
\midrule
Simple Combination & 64.5 $\rightarrow$ 64.5 & 71.2 $\rightarrow$ \textbf{99.9} & 64.8 $\rightarrow$ 64.3 & 56.6 & 66.4\\
\midrule
Rock-Paper-Scissors & 51.3 $\rightarrow$ 51.3 & \textbf{100} $\rightarrow$ \textbf{100} & 51.3 $\rightarrow$ 73.3 & \textbf{100} & \textbf{100}\\
\midrule
Advanced Combination & 57.7 $\rightarrow$ 57.7 & 83.5 $\rightarrow$ \textbf{99.9} & 57.9 $\rightarrow$ 57.1 & 52.7 & 79.4\\
\midrule
Age of Empires II & 68.7 $\rightarrow$ 68.7 & 97.3 $\rightarrow$ \textbf{100} & 68.7 $\rightarrow$ 53.1 & 51.2 & 97.7\\
\midrule
Hearthstone & 81.1 $\rightarrow$ 81.1 & \textbf{97.8} $\rightarrow$ \textbf{98.0} & 83.4 $\rightarrow$ 74.8 & 61.4 & 97.4\\
\midrule
Brawl Stars & 90.2 $\rightarrow$ 95.8 & 94.3 $\rightarrow$ \textbf{99.7} & 95.9 $\rightarrow$ 98.0 & 97.5 & 97.2\\
\midrule
League of Legends & 79.6 $\rightarrow$ \textbf{99.9} & 78.9 $\rightarrow$ \textbf{100} & 88.2 $\rightarrow$ \textbf{100} & 100 & 90.9\\
\bottomrule
\toprule
Simple Combination & \textbf{64.7} $\rightarrow$ \textbf{64.7} & 61.8 $\rightarrow$ 6.5 & \textbf{64.9} $\rightarrow$ \textbf{64.4} & 57.2 & 63.9\\
\midrule
Rock-Paper-Scissors & 51.1 $\rightarrow$ 55.8 & \textbf{100} $\rightarrow$ \textbf{100} & 51.1 $\rightarrow$ 73.6 & \textbf{100} & \textbf{100}\\
\midrule
Advanced Combination & 56.5 $\rightarrow$ 56.7 & 79.1 $\rightarrow$ 8.2 & 56.5 $\rightarrow$ 56.1 & 51.3 & \textbf{79.7}\\
\midrule
Age of Empires II & 64.5 $\rightarrow$ 64.0 & \textbf{75.7} $\rightarrow$ \textbf{75.4} & 64.5 $\rightarrow$ 52.4 & 51.0 & \textbf{75.4}\\
\midrule
Hearthstone & 80.9 $\rightarrow$ 81.5 & \textbf{95.4} $\rightarrow$ \textbf{95.0} & 81.2 $\rightarrow$ 74.9 & 61.1 & 94.8\\
\midrule
Brawl Stars & 79.7 $\rightarrow$ 69.8 & 82.4 $\rightarrow$ 69.3 & 82.8 $\rightarrow$ 77.8 & \textbf{83.1} & \textbf{83.4}\\
\midrule
League of Legends & 51.1 $\rightarrow$ 0.1 & 50.9 $\rightarrow$ 0 & 51.1 $\rightarrow$ 6.3 & 50.2 & 51.0\\
\bottomrule
\end{tabular}
}
\end{table}

\subsection{Top-D Diversity and Top-B Balance}

With these new tools (NRT and NCT), we are now equipped to formally define and measure balance beyond basic win rate statistics. We introduce two new metrics: \textbf{Top-D Diversity}, which measures the number of competitive strategies close in strength to the strongest one, and \textbf{Top-B Balance}, which identifies the number of non-dominated strategies considering counter relationships. Both metrics are designed to be theoretically grounded yet computationally efficient.

We first formalize the notion of dominance using win value estimators:

\begin{definition}
\label{measure_def1}
Let $\text{Win}: c_1, c_2 \to w$ denote a function estimating the probability that composition $c_1$ wins over $c_2$, where $w \in [0, 1]$.
\end{definition}

\begin{proposition}
\label{measure_prop1}
We say composition $c_1$ dominates composition $c_2$ over all opponents $c$ if $\text{Win}(c_1, c) > \text{Win}(c_2, c)$ for every $c$ in the composition space.
\end{proposition}

When Proposition~\ref{measure_prop1} holds, $c_2$ is strictly inferior to $c_1$: across all matchups, $c_1$ performs better. This definition offers a principled foundation for balance assessment. However, checking such relations exhaustively has a cubic complexity of $\mathcal{O}(N^3)$ for $N$ compositions. To address this computational bottleneck, we propose tractable approximations enabled by our learned rating and counter tables, which are introduced in the following subsections.

\subsubsection{Top-D Diversity}

While Proposition~\ref{measure_prop1} identifies a single composition that strictly dominates others, a well-balanced game often supports a broader set of viable strategies. \textbf{Top-D Diversity} quantifies this by counting the number of compositions that are nearly as strong as the top-rated one. This reflects the design intuition that a good balance should encourage variation, not convergence to a single dominant meta choice.

Let $c_{top}$ denote the composition with the highest estimated rating under the NRT model:
\begin{equation}
c_{top} = \arg\max_{c \in \mathcal{C}} R_\theta(c).
\end{equation}

To determine which compositions are acceptably close to $c_{top}$, we ground our metric on three belief-based assumptions:

\begin{assumption}
\label{top_d_assum1}
\textbf{(Belief 1: Fairness Tolerance)} Players are tolerant of small win-rate differences up to a fixed threshold $G \in [0, 1]$, interpreting them as arising from skill variance or randomness. Compositions within this gap are still perceived as competitively viable.
\end{assumption}

\begin{assumption}
\label{measure_assum1}
\textbf{(Belief 2: Rating Validity)} The scalar rating $R_\theta(c)$ approximates the win probability via the Bradley–Terry formulation:
\[
\text{Win}(c_1, c_2) = \frac{R_\theta(c_1)}{R_\theta(c_1) + R_\theta(c_2)},
\]
and is considered a meaningful proxy for composition strength.
\end{assumption}

\begin{assumption}
\label{top_d_assum2}
\textbf{(Belief 3: Local Dominance Generalization)} If a composition $c$ is not clearly dominated by $c_{top}$, then it is unlikely to be dominated by other compositions.
\end{assumption}

Under these beliefs, we derive the following:

\begin{lemma}
\label{top_d_lemma}
A composition $c$ is considered top-tier viable if:
\[
\frac{R_\theta(c)}{R_\theta(c) + R_\theta(c_{top})} + G \geq 0.5.
\]
\end{lemma}

\begin{proof}
By Assumption~\ref{measure_assum1}, this expression estimates the win probability of $c$ against $c_{top}$. If this value plus gap $G$ reaches 0.5, then by Assumption~\ref{top_d_assum1} the composition is still seen as fair. Assumption~\ref{top_d_assum2} extends this local fairness to imply global viability.
\end{proof}

The \textbf{Top-D Diversity} measure is thus defined as the number of compositions satisfying Lemma~\ref{top_d_lemma}. This metric can be computed efficiently in $\mathcal{O}(N)$ time:

\begin{algorithm}[tb]
   \caption{Compute Top-D Diversity Measure}
   \label{alg:top_d_diversity_measure}
\begin{algorithmic}
   \State \textbf{Input:} neural rating table \( R_{\theta} \), top-rated comp \( c_{top} \), acceptable win gap \( G \)
   \State \textbf{Output:} Top-D Diversity Measure \( D \)
   \State Initialize counter \( D \gets 0 \)
   \For{each comp \( c \in \mathcal{C} \)}
   \If{\( \frac{R_\theta(c)}{R_\theta(c) + R_\theta(c_{top})} + G \geq 0.5 \)}
   \State \( D \gets D + 1 \)
   \EndIf
   \EndFor
\end{algorithmic}
\end{algorithm}

\paragraph{Interpretation.}
Top-D Diversity reflects how many strategies are perceived as viable relative to the strongest one. A small $D$ (e.g., $D = 1$) indicates a dominant meta with few competitive alternatives. A large $D$ suggests a richer strategic landscape, more variety in gameplay, and broader viability. Designers can tune $G$ (e.g., $0.02$ to $0.05$) to reflect stricter or looser notions of perceived balance.

Unlike entropy-based pick-rate diversity metrics, which depend on population distribution, Top-D Diversity offers a semantically clear threshold: compositions either qualify or not. More importantly, by leveraging scalar ratings from NRT, it avoids pairwise comparisons and scales linearly, making it suitable for large composition spaces.

We now extend this formulation to capture structural balance in the presence of counter relationships via the \textbf{Top-B Balance} measure.

\subsubsection{Top-B Balance}

While Top-D Diversity identifies compositions that are nearly as strong as the best one, it does not account for contextual counter relationships. \textbf{Top-B Balance} extends the analysis by identifying compositions that are \textit{not dominated by any other}, accounting for contextual matchups through the counter table \( C_\theta \). This reflects a more nuanced view of balance in games with intransitive or cyclical dynamics.

We begin by making a key modeling assumption:

\begin{assumption}
\label{top_b_assum1}
The contextual win probability combining NRT and NCT,
\begin{equation}
\text{Win}_\theta(c_1, c_2) = \frac{R_\theta(c_1)}{R_\theta(c_1) + R_\theta(c_2)} + C_\theta(c_1, c_2),
\end{equation}
offers a trustworthy approximation of true composition advantage.
\end{assumption}

\paragraph{Dominance under counter-aware models.}

We define a composition \( c_1 \) to dominate \( c_2 \) if:
\begin{equation}
\forall c \in \mathcal{C},\quad \text{Win}_\theta(c_1, c) > \text{Win}_\theta(c_2, c).
\end{equation}

This strict definition is computationally expensive to verify over all \( N \) compositions, yielding time complexity \( \mathcal{O}(N^3) \). To make the Top-B measure tractable, we rely on categorical approximations using the discrete counter categories learned in \( C_\theta \).

\paragraph{Category-based dominance shortcuts.}

Let \( \text{cat}(c) \in \{1, \dots, M\} \) be the counter category of composition \( c \). We make two key propositions for category-aware analysis:

\begin{proposition}
\label{top_b_prop2}
(Within-category dominance)

If \( \text{cat}(c_1) = \text{cat}(c_2) \) and \( R_\theta(c_1) > R_\theta(c_2) \), then \( c_1 \) dominates \( c_2 \).
\end{proposition}

\begin{proposition}
\label{top_b_prop3}
(Transitive dominance via category linkage)

Let \( c_1 \), \( c_2 \), and \( c_3 \) be three compositions. If:
\begin{enumerate}
    \item \( c_1 \) dominates \( c_2 \) based on contextual win;
    \item \( \text{cat}(c_2) = \text{cat}(c_3) \);
    \item \( R_\theta(c_2) > R_\theta(c_3) \);
\end{enumerate}
then \( c_1 \) is inferred to dominate \( c_3 \).
\end{proposition}

These two propositions form the basis for efficiently identifying the set of \textit{non-dominated} compositions.

\begin{lemma}
\label{top_b_lemma}
Among the top-rated compositions in each of the \( M \) counter categories, the non-dominated ones under contextual win values form the \textbf{Top-B Balance} set.
\end{lemma}

We now define the Top-B Balance measure as the cardinality of this set, and describe a practical algorithm to compute it.

\begin{algorithm}[tb]
   \caption{Compute Top-B Balance Measure}
   \label{alg:top_b_balance_measure}
\begin{algorithmic}
   \State \textbf{Input:} rating table \( R_\theta \), counter table \( C_\theta \), set of all comps \( \mathcal{C} \)
   \State \textbf{Output:} Top-B Balance Measure \( B \)
   \State Group all comps by counter category \( \text{cat}(c) \)
   \State For each category, select the top-rated comp
   \State Initialize count \( B \gets 0 \)
   \For{each top comp \( c \)}
       \State Assume \( c \) is not dominated
       \For{each other top comp \( c' \ne c \)}
           \State Assume \( c' \) dominates \( c \)
           \For{each top comp \( c'' \)}
               \If{\( \text{Win}_\theta(c', c'') \leq \text{Win}_\theta(c, c'') \)}
                   \State Mark \( c' \) does not dominate \( c \); break
               \EndIf
           \EndFor
           \If{\( c' \) dominates \( c \)} \State Mark \( c \) as dominated \EndIf
       \EndFor
       \If{\( c \) is not dominated} \State \( B \gets B + 1 \) \EndIf
   \EndFor
\end{algorithmic}
\end{algorithm}

\paragraph{Interpretation.}
A high Top-B Balance value indicates that multiple composition archetypes can survive as optimal representatives of their counter category, offering diverse, resilient strategies. Unlike Top-D, which only measures proximity to the best, Top-B ensures contextual robustness and is especially important in games with strong counterplay.

We next turn to concrete examples of applying these metrics.

\subsubsection{Case Study on Age of Empires II}
\label{sec:aoe2_case_study}

To demonstrate the practical application of our proposed balance measures, we conduct an in-depth case study on \textit{Age of Empires II} (AoE2), a real-time strategy game featuring 45 unique civilizations. Using the aoestats.io dataset (January 2024), we analyze 1,261,288 1v1 matches in random map mode.

\paragraph{Top-D and Top-B Evaluation.}
We first compute Top-D Diversity by varying the acceptable win gap threshold $G$. As shown in Table~\ref{aoe2_top_d}, the number of competitively viable civilizations increases with $G$, from 2 at $G=0.01$ to 44 at $G=0.08$, suggesting a diverse yet tiered strategic landscape.

\begin{table}
\centering
\caption[Top-D Diversity and Top-B Balance in Age of Empires II]{Top-D Diversity and Top-B Balance in Age of Empires II under varying conditions.}
\subtable[Top-D Diversity with different thresholds $G$]{
\centering
\label{aoe2_top_d}
\begin{tabular}{ll}
\toprule
\textbf{Win Gap Threshold} $G$ & \textbf{Top-D Diversity} $D$ \\
\midrule
0.01 & 2\\
0.02 & 9\\
0.04 & 25\\
0.08 & 44\\
\bottomrule
\end{tabular}
}
\hfill
\subtable[Top-B Balance under different counter table sizes $M$]{
\centering
\label{aoe2_top_b}
\begin{tabular}{llll}
\toprule
$M$ (Table Size) & Utilized $M$ & Top-B Balance $B$ & Accuracy (\%) \\
\midrule
3 & 1 & 1 & 69.6\\
9 & 9 & 8 & 77.0\\
27 & 26 & 24 & 86.0\\
81 & 45 & 45 & 98.9\\
\bottomrule
\end{tabular}
}
\end{table}

For Top-B Balance, Table~\ref{aoe2_top_b} shows that increasing the counter table resolution $M$ leads to higher predictive accuracy and a more nuanced balance profile. With $M=81$, all 45 civilizations are recovered as non-dominated, reflecting well-distributed strength across strategic archetypes.

\paragraph{Visualizing Rating and Counter Structures.}

Figure~\ref{fig:age_of_empires_ii_rating} shows the scalar strength estimates of each civilization from the NRT. This gives an overview of which civilizations are consistently strong regardless of opponent.

\begin{figure}[ht]
\centering
\includegraphics[width=\textwidth]{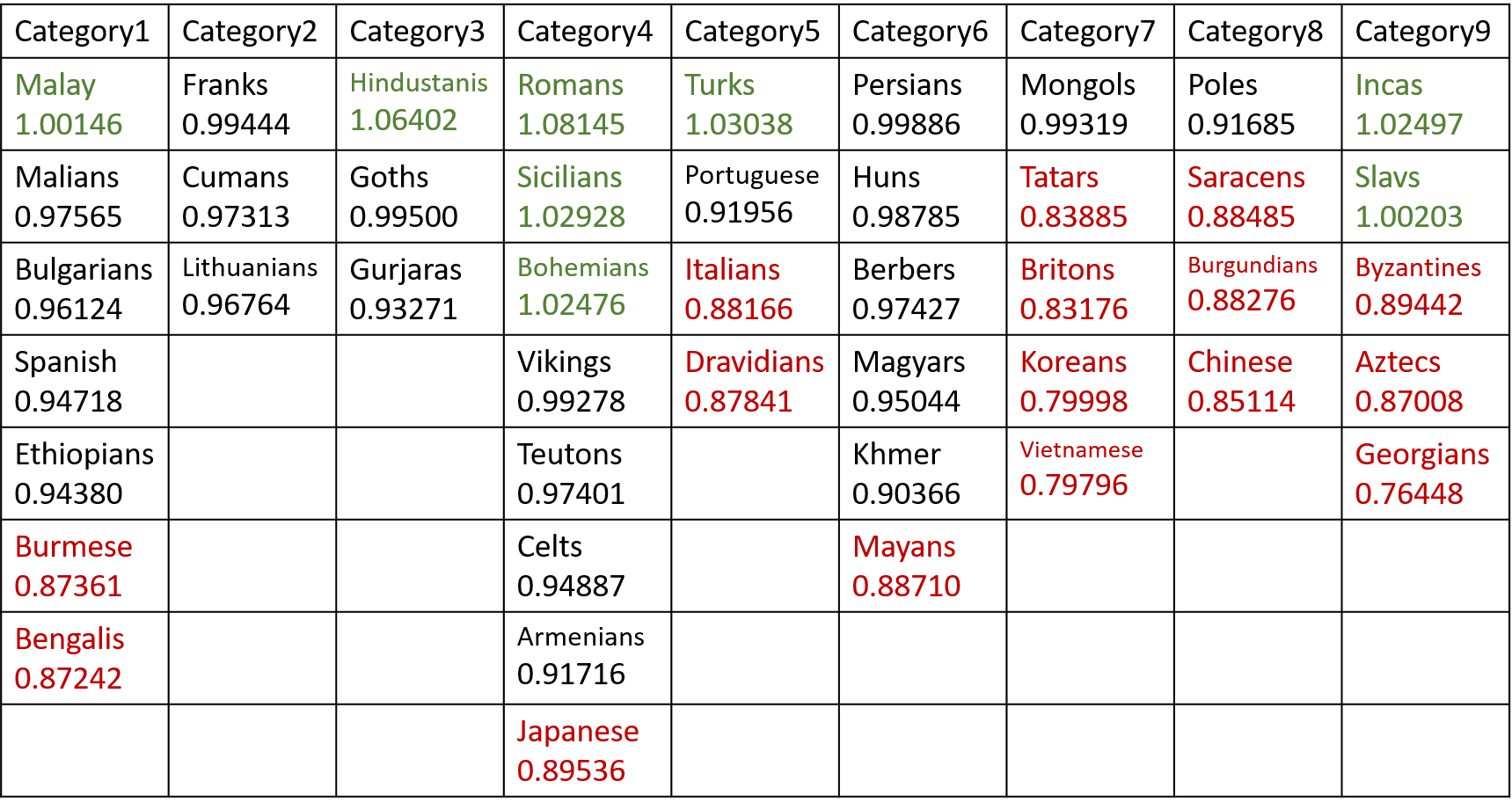}
\caption[Age of Empires II Rating Table]{Scalar rating (NRT) for 45 civilizations in Age of Empires II.}
\label{fig:age_of_empires_ii_rating}
\end{figure}

Complementing this, Figure~\ref{fig:age_of_empires_ii_counter} presents the learned $9 \times 9$ counter table, which clusters civilizations into strategic categories and models their pairwise residual advantages.

\begin{figure}
\centering
\includegraphics[width=0.95\textwidth]{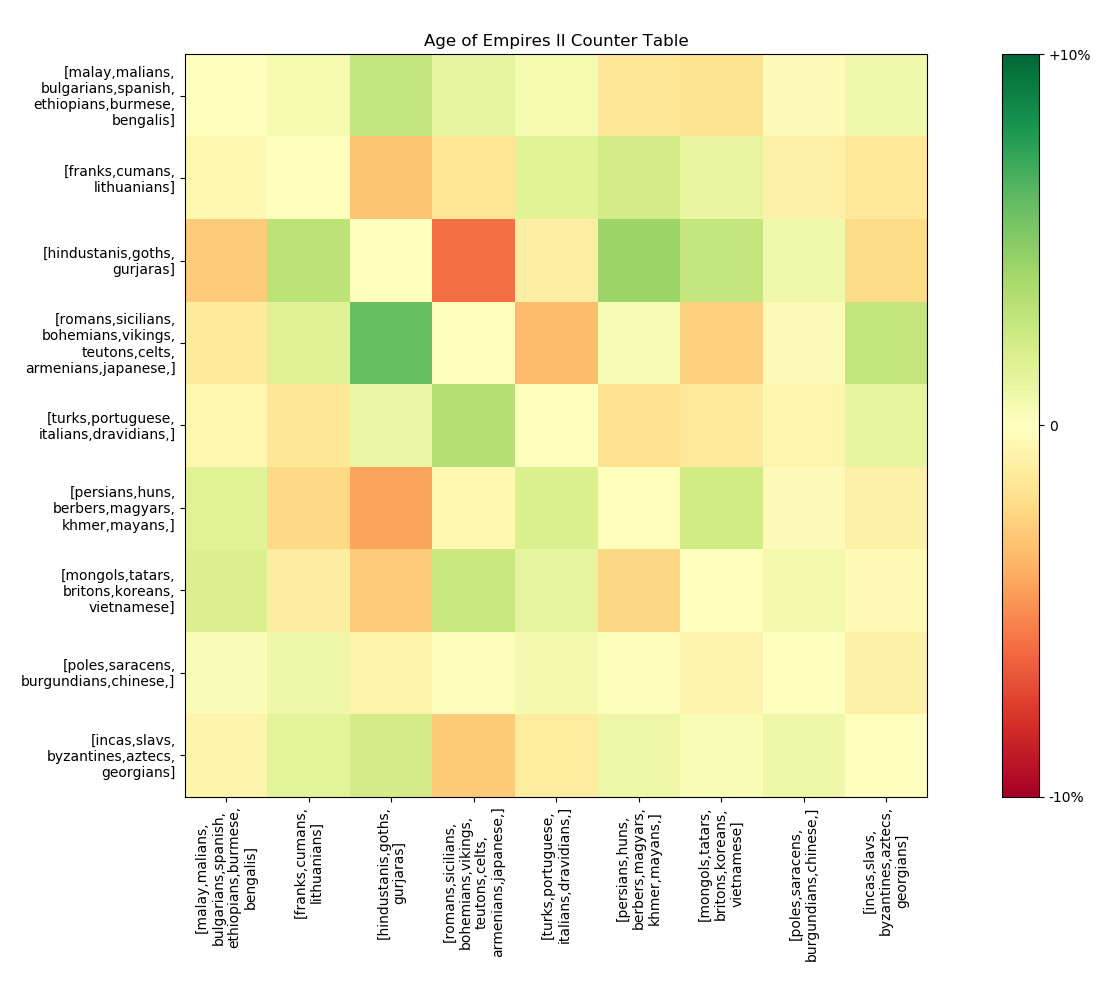}
\caption[Age of Empires II Counter Table]{Learned $9\times9$ counter table for Age of Empires II. Each cell estimates the counter residual between categories.}
\label{fig:age_of_empires_ii_counter}
\end{figure}

\paragraph{Strategic Archetypes and Counter Patterns.}
The nine learned categories capture key gameplay dimensions in AoE2:

\begin{enumerate}
    \item \textbf{Technological and Age Advancement Group}: Civilizations that gain early strategic momentum through accelerated research or age progression. Examples include the Malay (faster aging), Malians (accelerated university techs), and Bulgarians (free militia-line upgrades and reduced Blacksmith/Siege Workshop costs).

    \item \textbf{Heavy Cavalry Group}: Civilizations centered around powerful cavalry units for open-field dominance and mobility, such as the Franks and Lithuanians.

    \item \textbf{Anti-Cavalry and Anti-Archer Group}: Civilizations that excel at countering cavalry and ranged threats through specialized units, such as the Goths (cheap infantry swarm) and Indians (camel-based counters).

    \item \textbf{Infantry Dominance Group}: Civilizations that rely heavily on infantry as their primary military force, often excelling in frontline durability and cost-effective aggression.

    \item \textbf{Gunpowder-Intensive Group}: Civilizations that emphasize gunpowder units for late-game power spikes and high-damage ranged warfare.

    \item \textbf{Cavalry and Horse Archer Group}: Civilizations combining mobility and ranged flexibility, leveraging both cavalry and mounted archers for dynamic raiding and hit-and-run tactics.

    \item \textbf{Archery Group}: Civilizations that prioritize strong archer-line units, excelling in ranged control, positioning, and map pressure.

    \item \textbf{Economic Powerhouses}: Civilizations whose core strength lies in superior economic scaling, enabling stronger military support through sustained resource advantages.

    \item \textbf{Trash and Counter Unit Group}: Civilizations that focus on cost-efficient 'trash units' (e.g., skirmishers, spearmen) and specialized counter tools to maintain battlefield presence and counter specific threats with minimal gold investment.
\end{enumerate}

From the counter table, we observe nuanced dynamics:
\begin{itemize}
    \item The \textbf{Heavy Cavalry} group is vulnerable to \textbf{Anti-Cavalry} civilizations and \textbf{Infantry}-based counters, but outmaneuvers \textbf{Gunpowder} or \textbf{Archery} factions.
    \item \textbf{Goths}, often considered weak in direct matchups, shine against archery-based comps.
    \item  \textbf{Infantry} factions typically struggle against gunpowder but outperform trash-based groups.
    \item  \textbf{Economically oriented} civilizations show fewer extreme counters, reflecting flexible and resilient playstyles.
\end{itemize}

\paragraph{Implications for Balance Analysis.}
This case demonstrates how scalar ratings and counter tables provide complementary insights into game balance, each offering actionable guidance for designers:

\begin{itemize}
    \item \textbf{Top-D Diversity} quantifies how many compositions are perceived as nearly as strong as the best one. In \textit{Age of Empires II}, Top-D values (Table~\ref{aoe2_top_d}) suggest that a large number of civilizations are competitively viable—especially when using a tolerant win gap (e.g., $G=0.08$ yields $D=44$ out of 45). This indicates good diversity and supports the game's replayability across different playstyles.
    
    \item \textbf{Top-B Balance}, by contrast, identifies how many compositions are not dominated after accounting for learned counter relationships. At $M=9$, only 8 of the 9 categories contribute a non-dominated civilization (Table~\ref{aoe2_top_b}). According to our manual interpretation, the dominated category is an economic-focused one, with the \textbf{Poles} being the top yet still dominated comp. This points to a possible need to enhance economic bonuses for that group.

    \item When $M$ is increased to 27, \textbf{Aztecs} and \textbf{Chinese} are also found dominated, suggesting the counter dynamics within their strategic cluster are not favorable, possibly due to shifts in meta or outdated faction strengths.

    \item At the extreme setting $M=81$, every civilization becomes non-dominated. Each one belongs to its own category and wins in at least some specialized matchups. The training accuracy reaches 98.9\%, validating that the game exhibits full counter-based balance in expert-level play. This justifies the game's longstanding strategic depth—but also implies that such granularity might be overwhelming for casual players.
\end{itemize}

Together, these two metrics enable game developers to distinguish between two types of balance flaws: (1) lack of variety in player choices (low Top-D), and (2) entire groups of compositions being obsolete due to systemic domination (low Top-B). Designers can then choose targeted actions—e.g., adjust global strength scaling for Top-D improvements, or restructure category dynamics for Top-B restoration. Traditional balance tuning, such as weakening strong factions or buffing weak ones based on win rates, may miss these subtler structural issues.

This analysis also empowers players to understand that even popular or “meta” civilizations like \textbf{Romans} may have clear counters, reinforcing the importance of strategic scouting and matchup planning.

\subsubsection{Case Study on Hearthstone}
\label{sec:hearthstone_case_study}

To evaluate our proposed metrics in a deck-based game environment, we conduct a case study on \textit{Hearthstone}. Using HSReplay data (January 2024), we filter out decks with fewer than 100 games, resulting in 58 standardized decks for analysis.

\paragraph{Rating and Dominance.}
The NRT model identifies \textbf{Treant Druid} as the strongest deck (rating: 1.48915), significantly outperforming the others (see Figure~\ref{fig:hearthstone_rating}). As shown in Table~\ref{hearthstone_top_d}, the Top-D Diversity is very low across all thresholds, indicating that most decks are perceived as clearly inferior. Even with a 4\% tolerance ($G = 0.04$), only 3 decks fall within the acceptable win gap, signaling poor diversity.

\begin{figure}[ht]
\centering
\includegraphics[width=\textwidth]{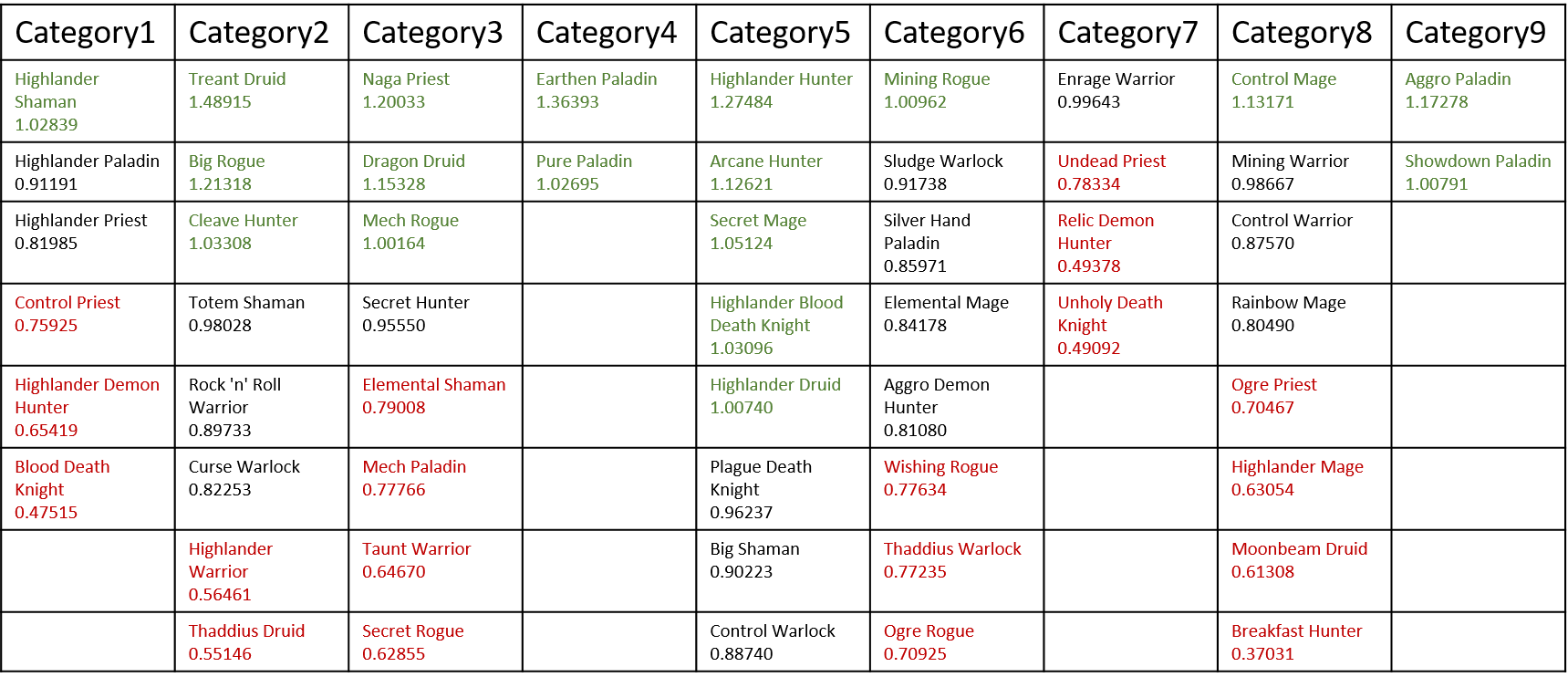}
\caption[Hearthstone Rating Table]{Rating table for Hearthstone decks, showcasing the relative strength of top archetypes.}
\label{fig:hearthstone_rating}
\end{figure}

\begin{table}
\centering
\caption[Top-D Diversity and Top-B Balance in Hearthstone]{Top-D Diversity and Top-B Balance in Hearthstone under varying settings.}
\subtable[Top-D Diversity with different $G$ values]{
\centering
\label{hearthstone_top_d}
\begin{tabular}{ll}
\toprule
\textbf{Threshold $G$} & \textbf{Top-D Diversity $D$} \\
\midrule
0.01 & 1 \\
0.02 & 1 \\
0.04 & 3 \\
0.08 & 9 \\
\bottomrule
\end{tabular}
}
\hfill
\subtable[Top-B Balance with varying $M$ values]{
\centering
\label{hearthstone_top_b}
\begin{tabular}{llll}
\toprule
$M$ (Table Size) & Utilized $M$ & Top-B Balance $B$ & Accuracy (\%) \\
\midrule
3  & 3  & 1  & 81.3 \\
9  & 9  & 9  & 87.0 \\
27 & 25 & 23 & 90.6 \\
81 & 54 & 53 & 97.2 \\
\bottomrule
\end{tabular}
}
\end{table}

\paragraph{Counter Relationships.}
Despite its dominant rating, \textbf{Treant Druid} is countered by \textbf{Aggro Paladin} (rating: 1.17278), as revealed in the learned $9 \times 9$ counter table (Figure~\ref{fig:hearthstone_counter}). This illustrates that counterplay remains viable even when one archetype dominates in scalar rating. Under $M=9$, the Top-B Balance reaches 9, showing that each category contributes a viable top deck. When $M=27$ or $81$, the difference between Used $M$ and Top-B Balance remains small, confirming a healthy counter structure.

\begin{figure}[ht]
\centering
\includegraphics[width=\textwidth]{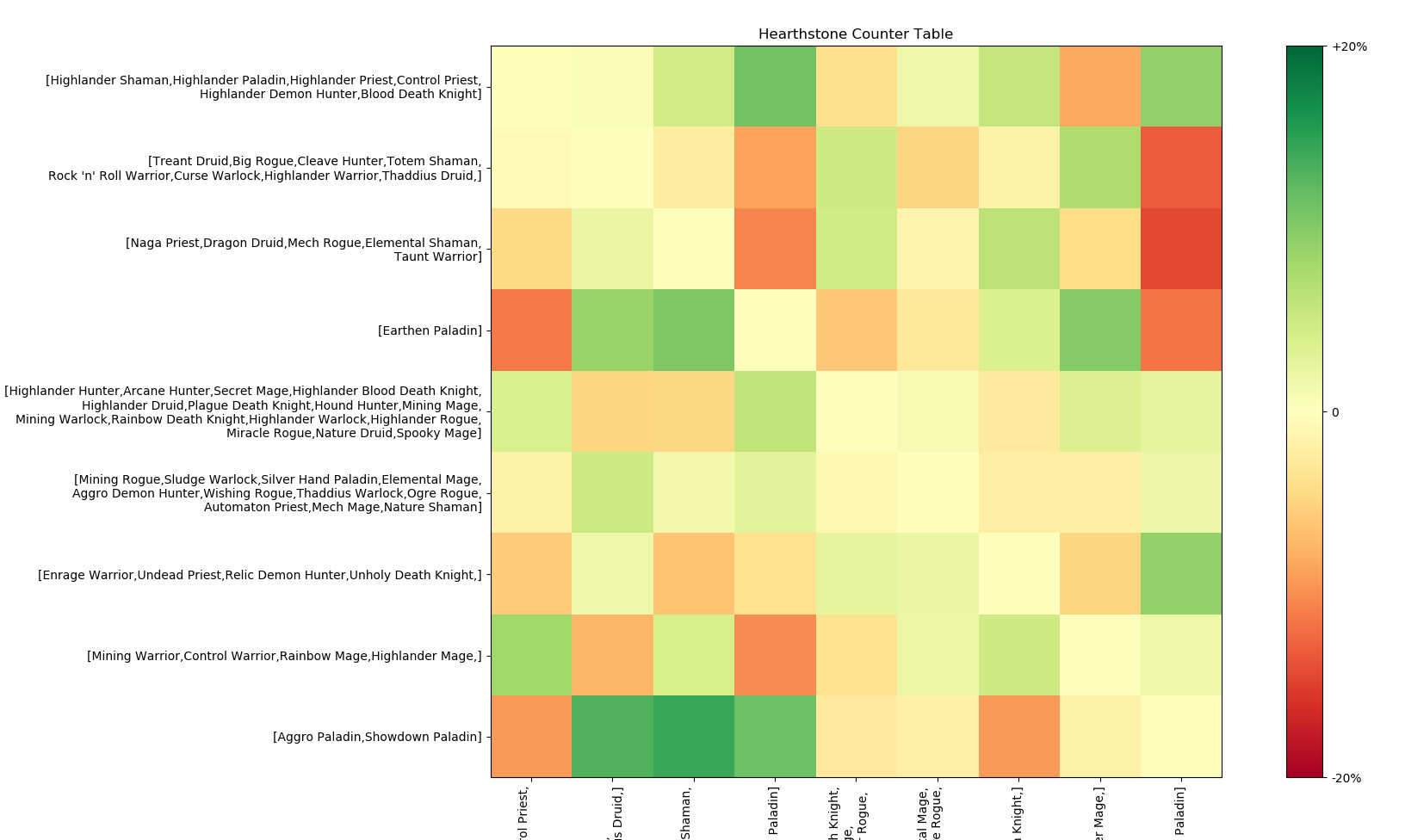}
\caption[Hearthstone Counter Table]{The $9 \times 9$ counter table for Hearthstone, visualizing inter-category counter relationships. Green indicates advantage, red indicates disadvantage.}
\label{fig:hearthstone_counter}
\end{figure}

\paragraph{Category Interpretation.}
The nine learned categories reflect major strategic styles in the current meta:

\begin{enumerate}
    \item \textbf{High-Value Control Archetypes (Category 1)}: Decks that rely on individually powerful cards and value-oriented board clears to achieve long-term advantage. These often include singleton configurations (e.g., Highlander decks) with cards like Reno Jackson. Examples include Highlander Shaman, Paladin, Demon Hunter, Control Priest, and Blood Death Knight.

    \item \textbf{Tempo-Dependent Archetypes (Category 2)}: Decks that rely on efficient mana curve utilization to dominate the board when curving out properly. Treant Druid and Big Rogue exemplify this style, requiring proactive board development and subsequent buffs.

    \item \textbf{Combo-Reliant Archetypes with Limited Draw (Category 3)}: Decks built around key card synergies, but lacking highly efficient draw engines. Examples include Naga Priest (requiring alternating Naga and spells) and Dragon Druid (relying on dragon-in-hand synergies).

    \item \textbf{Midrange Archetypes with Strong Minions (Category 4)}: Decks that maintain board presence through high-stat minions and effective trades. Earthen Paladin is representative, leveraging resilient minions and area-of-effect spells to counter wide boards (e.g., Category 7).

    \item \textbf{Midrange Value Archetypes (Category 5)}: Decks that exert pressure through superior card quality and tempo balance, winning by out-valuing rather than overwhelming.

    \item \textbf{Flexible Midrange Archetypes (Category 6)}: A heterogeneous group of decks that blend multiple midrange strategies. These tend to have more balanced matchups and weaker counter relationships.

    \item \textbf{Aggro and Snowball Archetypes (Category 7)}: Decks that aim to establish early board control and snowball into lethal. These excel against decks lacking early board clears or tempo resets.

    \item \textbf{Value-Control Decks with Combo Finishers (Category 8)}: Decks that sustain long-term control while retaining the potential for a late-game one-turn-kill (OTK). Examples include Control Warrior (via armor stacking burst) and Rainbow Mage (through spell-based OTKs).

    \item \textbf{Aggressive Swarm Paladin Archetypes (Category 9)}: Decks that focus on early board flooding with multiple low-cost threats, aiming to end games quickly before the opponent can stabilize.
\end{enumerate}

Control decks (Category 1) lose to OTK finishers (Category 8) but handle board-centric decks (Category 4) well. Aggro (Category 9) preys on slow or curve-dependent decks (Categories 2–3), while midrange decks oscillate between durability and speed depending on matchup dynamics.

\paragraph{Design Implications.}
This case study illustrates a nuanced view of balance in Hearthstone:
\begin{itemize}
    \item From the Top-D Diversity values (Table~\ref{hearthstone_top_d}), even with a relatively tolerant win gap ($G=0.04$), only 3 decks are considered competitively viable. The dominant deck, \textbf{Treant Druid}, achieves a strength rating of 1.48915 (Figure~\ref{fig:hearthstone_rating}), significantly above others—indicating a lack of diversity and a likely convergence toward a single meta-defining choice.
    \item Top-B Balance at $M=3$ (Table~\ref{hearthstone_top_b}) confirms this imbalance, with \textbf{Treant Druid} dominating all others. However, when the counter table is extended to $M=9$, balance is surprisingly restored. Notably, \textbf{Aggro Paladin}, despite a lower overall strength (1.17278), serves as an effective counter to \textbf{Treant Druid}, and all nine top-category decks become non-dominated. This shows how inter-category counterplay can preserve strategic richness even when scalar dominance exists.
    \item At higher counter resolutions ($M=27$, $M=81$), the game remains stable: the number of non-dominated strategies is close to the number of utilized categories, suggesting well-formed and resilient counter structures.
\end{itemize}

Overall, these results suggest that targeted adjustments to \textbf{Treant Druid}—rather than global rebalancing—could improve perceived diversity without undermining the integrity of the existing counter ecosystem. This highlights the utility of Top-D and Top-B as complementary tools: one flags excessive dominance; the other confirms whether sufficient strategic alternatives still exist.

\subsection{Discussion on Different Types of Balance Measures}

When evaluating game balance, it is essential to consider not only how fairly different strategies perform but also how balance affects the richness and sustainability of gameplay. Our proposed measures—Top-D Diversity and Top-B Balance—are designed to address both concerns: the breadth of viable strategic options and the structural complexity of counterplay.

\vspace{0.5em}
\noindent\textbf{Beyond Win Rates.}
The most common and intuitive measure of balance in player-versus-player (PvP) games is the win rate. Equalizing win rates across strategies—typically aiming for 50\%—is a standard objective in game tuning \citepx{balance_from_renowned_authors}. In practice, this involves buffing underperforming strategies and nerfing dominant ones, as reflected in player communities.

However, enforcing uniform win rates can conflict with another key design goal: preserving the richness of gameplay through stylistic diversity. Players often value the opportunity to express individuality through different playstyles \citepx{intrinsic_motivation}, even if not all strategies are equally optimal. Moreover, strategies close in strength may appear balanced to players due to randomness, skill variance, or personal belief in underdog potential. Our Top-D Diversity measure captures this by quantifying the number of compositions within an acceptable win-rate margin from the best strategy—acknowledging perceived balance rooted in uncertainty and tolerance, not just statistical parity.

\vspace{0.5em}
\noindent\textbf{Limits and Costs of Game-Theoretic Approaches.}  
Beyond simple win rate analysis, game-theoretic approaches offer more principled ways to assess balance. One popular class of methods focuses on identifying \textit{Nash equilibrium} strategies and then maximizing entropy over the equilibrium distribution to encourage unpredictability \citepx{nash_entropy_balancing}. This line of thinking typically recommends buffing rarely chosen strategies and nerfing dominant ones, in an effort to flatten the strategic landscape.

However, this method risks excessive homogenization. If all strategies are tuned to have equal success rates under rational play, they may become indistinguishable in performance, eliminating meaningful differences and reducing strategic depth. In practice, this may lead to what players perceive as “false diversity”: many options that feel the same in outcome, even if thematically different.

To counter this, \citetx{nash_entropy_balancing} also introduces regularization terms to preserve diversity within entropy maximization frameworks—conceptually similar to the \textit{tolerance margin} $G$ we employ in Top-D Diversity.

In addition to entropy-based objectives, other game-theoretic methods like \textbf{$\alpha$-Rank} \citepx{omidshafiei2019alpha} offer a different approach to evaluating strategic ecosystems. Instead of analyzing equilibrium directly, $\alpha$-Rank simulates long-run evolutionary dynamics among strategies and tracks the stationary distribution of strategy frequencies. This provides a principled view of which strategies persist in a population over time and can serve as an implicit balance indicator.

However, $\alpha$-Rank and entropy-based methods share a fundamental requirement: achieve \textit{Nash equilibrium}. While this is tractable in abstract games with few strategies, it becomes infeasible in modern games with vast strategy spaces—such as those involving team compositions, card decks, or custom loadouts.

Even in academic AI research, solving games via full payoff matrix analysis or deep equilibrium learning comes at significant computational cost. For example, large-scale Nash equilibrium computation has been achieved in imperfect-information games like heads-up no-limit Texas Hold’em \citepx{cfr_plus} and Stratego \citepx{deep_nash}, but these require extensive sampling, memory, and tuning over weeks or months of training—often impractical for iterative game development pipelines that demand agility and fast feedback.

In contrast, our Top-D and Top-B measures offer lightweight alternatives:  
\begin{itemize}
    \item \textbf{Top-D} quantifies strategic viability under uncertainty and tolerance, with no need for opponent modeling.  
    \item \textbf{Top-B} maps the structural size of counterplay, highlighting how deep or shallow a game’s balance dynamics are.  
\end{itemize}

Both operate efficiently, scale well with large composition spaces, and make no assumptions about rational equilibrium behavior—making them especially suited to real-world design cycles where tractability, interpretability, and fast iteration are crucial.

\vspace{0.5em}
\noindent\textbf{The Value of Counter Structure.}
Top-B Balance offers a complementary perspective. Instead of averaging over outcomes, it focuses on the topology of domination: how many distinct strategies are non-dominated across the full counter table. This measure aligns with a richer notion of balance: one that favors not just fairness, but strategic depth and reactive decision-making. Increasing the size and complexity of counter relationships challenges players to explore and adapt, enhancing long-term engagement.

In essence, our measures allow designers to promote balance without sacrificing diversity. Top-D offers a scalar approximation of viability range, and Top-B reveals the network of meaningful interactions.

\vspace{0.5em}
\noindent\textbf{Illustrative Examples.}
Consider a simple game like our synthetic combination game. The theoretically optimal composition is (18,19,20). Strategies like (17,19,20) or (13,17,19), although slightly weaker, still exhibit win rates of 49.1\% and 42.5\%, respectively, against the best. Due to sampling variability, players may still view them as viable. Top-D Diversity formally captures this kind of perception-driven diversity, which entropy- or win-rate-based methods might overlook.

\begin{figure}
\begin{center}
\includegraphics[width=\linewidth]{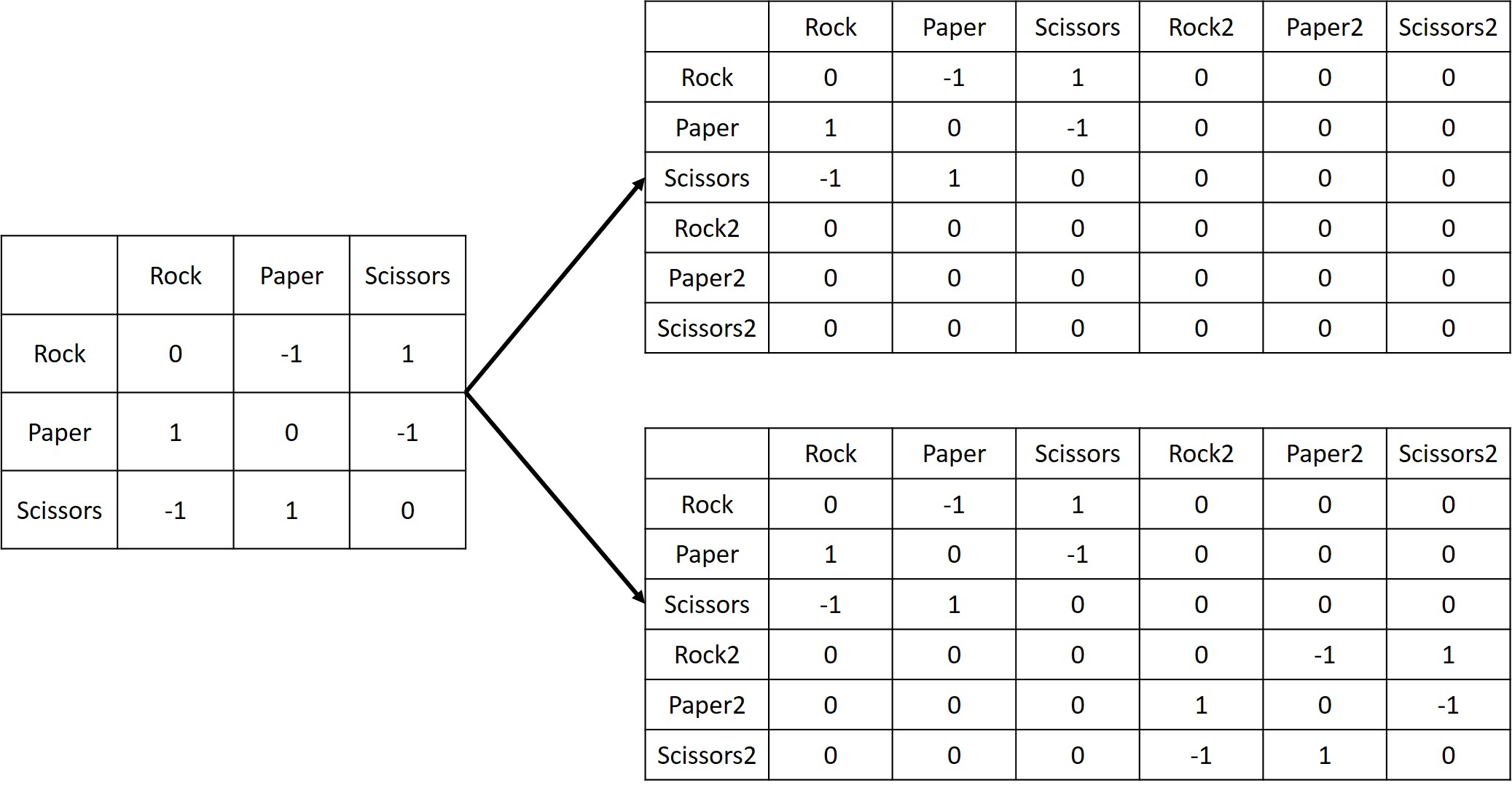}
\end{center}
\caption[Double Rock-Paper-Scissors]{An example of extending the classical Rock-Paper-Scissors to more complex cases.}
\label{fig:double_rps}
\end{figure}

Now consider extended variants of Rock-Paper-Scissors (Figure~\ref{fig:double_rps}). All variants share the same Nash equilibrium strategy and 50\% average-case win rate. However, their underlying complexity differs: the lower variant requires a $6 \times 6$ counter table to describe its structure, while the upper only needs $4 \times 4$. Top-B Balance distinguishes between these cases, quantifying the depth of counter structure invisible to other measures.

\vspace{0.5em}
\noindent\textbf{Design Guidance.}
Balance is just one axis of game design \citepx{art_of_game_design}, and it often coexists with other goals such as thematic consistency, accessibility, and player expressiveness. Our proposed metrics do not replace existing tools—they enrich them. By revealing hidden asymmetries (Top-D) and measuring counter-system complexity (Top-B), these tools empower designers to balance not only outcomes, but also experiences.

Ultimately, we believe balance is not about erasing difference—it is about ensuring that difference remains meaningful. In that sense, both Top-D and Top-B reflect a broader design philosophy: sustaining fairness while celebrating diversity.

\subsection{Online Learning of Counter Categories and Ratings}

Most of the previous discussion assumes access to a complete static dataset of match results. However, in many real-world scenarios such as live games, online ladder systems, or continual training environments for AI agents, composition strengths and counter relationships must be learned and updated incrementally. Previous neural models like NRT and NCT do not naturally support such online updates.

To address this limitation, we introduce an online learning algorithm named \textbf{Elo Residual Counter Category} (Elo-RCC) \citepx{elo_rcc}, which combines the interpretability and simplicity of Elo with the expressiveness of counter categories. The core idea is to maintain player-specific category probabilities and update them using an expectation-maximization (EM) framework \citepx{em_algorithm}.

\subsubsection{Elo-RCC Algorithm}

In the original NCT framework, a pretrained NRT model is required to stabilize strength prediction before computing residuals. This constraint hinders online use where model stability cannot be guaranteed from the start. Elo-RCC circumvents this by directly integrating online residual learning and category refinement into the Elo update process.

Each player maintains:
\begin{itemize}
  \item a scalar rating $R_i$,
  \item a categorical distribution $\mathcal{C}_i \in \Delta^M$ over $M$ counter categories,
  \item a vector $\mathbf{E}_i \in \mathbb{R}^M$ of expected residuals against each category.
\end{itemize}
Here, $\Delta^M$ denotes the $M$-dimensional probability simplex---i.e., the set of all $M$-dimensional vectors with non-negative entries that sum to one.

The algorithm operates through the following four steps for each match:
\begin{enumerate}
  \item \textbf{Elo Rating Update}: Adjust ratings $R_i$ and $R_j$ using the standard Elo formula.
  \item \textbf{Counter Table Update}: Sample counter categories from $\mathcal{C}_i$ and $\mathcal{C}_j$, then update $M \times M$ counter table $\mathbf{T}$ using residuals.
  \item \textbf{Expected Residual Update}: Refine each player's $\mathbf{E}$ using observed residuals.
  \item \textbf{Category Refinement}: Update category distribution $\mathcal{C}_i$ based on distance between $\mathbf{T}$ and $\mathbf{E}_i$.
\end{enumerate}

This procedure effectively implements a soft expectation-maximization (EM) process for counter clustering in real time.

\begin{figure}[t]
    \centering
    \includegraphics[width=0.4\textwidth]{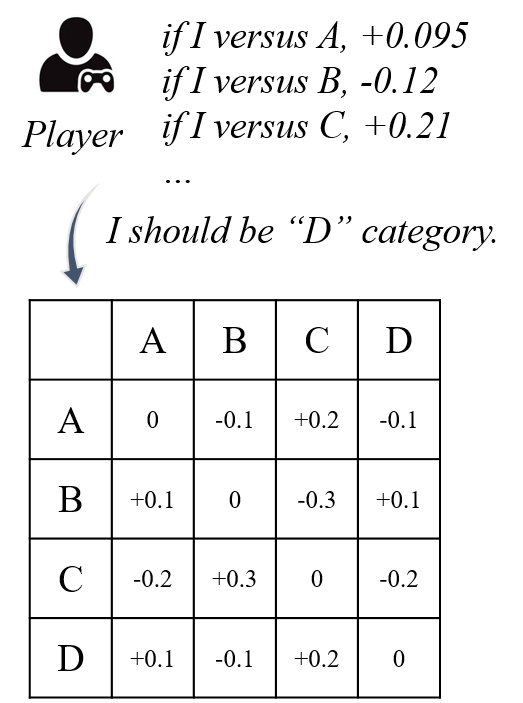}
    \caption[Concept of Elo-RCC]{Illustration of the EM algorithm. Residual win values for all categories are learned iteratively, enabling the best-fitting category for each individual to be identified and refined as the classification label.}
    \label{figure:idea}
\end{figure}

\begin{algorithm}
\caption{Online Update Algorithm for Elo Residual Counter Category}
\label{alg:EloRCC}
\begin{algorithmic}[1]
\Require Initial ratings $\mathcal{R}$, category distributions $\mathcal{C}$, counter table $\mathbf{T}$, expected residual table $\mathbf{E}$, learning rates $\eta_R$, $\eta_T$, $\eta_C$, number of categories $M$.
\Ensure Updated ratings $\mathcal{R}$ and category distributions $\mathcal{C}$.

\For{each match $(i, j)$ with result $O_i$ (1 for win, 0 for loss, 0.5 for tie)}
    \State $P_i \gets \frac{1}{1 + 10^{(R_j - R_i)/400}}$
    \State $R_i \gets R_i + \eta_R (O_i - P_i)$
    \State $R_j \gets R_j + \eta_R ((1 - O_i) - (1 - P_i))$

    \State Sample $c_i \sim \mathcal{C}_i$, $c_j \sim \mathcal{C}_j$
    \State $W_{res} \gets O_i - P_i$
    \State $\mathbf{T}[c_i, c_j] \gets \mathbf{T}[c_i, c_j] + \eta_T (W_{res} - \mathbf{T}[c_i, c_j])$
    \State $\mathbf{T}[c_j, c_i] \gets -\mathbf{T}[c_i, c_j]$

    \State $\mathbf{E}_i[c_j] \gets \mathbf{E}_i[c_j] + \eta_T (W_{res} - \mathbf{E}_i[c_j])$
    \State $\mathbf{E}_j[c_i] \gets \mathbf{E}_j[c_i] + \eta_T (-W_{res} - \mathbf{E}_j[c_i])$

    \State $D_i[c] \gets \sum_{c'} |\mathbf{T}[c, c'] - \mathbf{E}_i[c']|$
    \State $D_j[c] \gets \sum_{c'} |\mathbf{T}[c, c'] - \mathbf{E}_j[c']|$
    \State $c_i^* \gets \arg\min_c D_i[c]$, $c_j^* \gets \arg\min_c D_j[c]$
    \State Update $\mathcal{C}_i$ and $\mathcal{C}_j$ using soft one-hot refinement.
\EndFor
\end{algorithmic}
\end{algorithm}

\subsubsection{Empirical Results}

We evaluate the performance of Elo-RCC in classifying pairwise strength relations like previous sections—specifically, determining whether one composition is stronger, weaker, or indistinguishable from another—based on predicted win values. 

To assess performance under varying granularity, we evaluate Elo-RCC with counter table sizes of $M = 3, 9, 27$, and $81$. These match the configurations used in the NCT baseline, allowing direct comparisons. The following methods are included:

\begin{itemize}
    \item \textbf{NRT}: Neural Rating Table, which models composition strength using neural networks under the Bradley–Terry framework. It captures synergies but not intransitivity.
    
    \item \textbf{NCT}: Neural Counter Table, which extends NRT with an $M \times M$ counter table to model contextual intransitive effects.
    
    \item \textbf{Elo}: The standard Elo rating system with $K = 16$, modeling scalar strength via rating differences.
    
    \item \textbf{mElo2}: A multi-dimensional extension of Elo \citepx{m_elo}, representing ratings as vectors to accommodate intransitive structures.
    
    \item \textbf{Elo (K=0.1)}: A fine-grained variant of Elo, tuned with a smaller update constant $K=0.1$ for slow convergence in finely balanced games.
    
    \item \textbf{Elo-RCC}: Our proposed method combining Elo-style updates with online counter category refinement. The hyperparameters are set to $\eta_R = 0.1$, $\eta_T = 0.00025$, and $\eta_C = 0.01$.
\end{itemize}

Table~\ref{table:precision_test1} summarizes the strength relation classification accuracy on four benchmark games for $M = 81$. As expected, NCT yields the best performance due to its full neural capacity. However, Elo-RCC achieves comparable accuracy in games without deep compositional synergy—such as Rock-Paper-Scissors and Age of Empires II—despite its much simpler and fully online structure. It also consistently outperforms all other non-neural baselines across tasks.

To further examine the effect of discretization size, Table~\ref{table:precision_test2} shows results for $M = 3, 9, 27$. Notably, Elo-RCC slightly outperforms NCT at $M = 3$ and $M = 9$, where the limited category count makes neural quantization less reliable. Elo-RCC directly regresses the residuals in tabular form, enabling it to maintain accuracy under coarse discretizations. For $M = 27$, both models converge to similar performance, confirming Elo-RCC’s robustness even at finer granularity.

Together, these results support the conclusion that Elo-RCC offers a practical trade-off: while it may not always match the offline accuracy of deep neural models like NCT in complex synergy-rich environments (e.g., Hearthstone), it provides strong performance with the added advantages of low computational cost and real-time update capability. This makes it particularly well-suited for use in live game systems, agent training loops, and dynamic balance tracking.

\begin{table*}
\centering
\caption[Elo-RCC Accuracy]{Accuracies (\%) for different rating systems in training (top) and testing (bottom) sets using $M=81$. Results are averaged over 5-fold cross-validation, with $\pm$ one standard deviation. $\dagger$ Results for NRT, NCT, and Elo are reproduced from \citet{game_balance_analysis}. Elo-RCC achieves comparable performance to NCT in games without complex synergy, such as Rock-Paper-Scissors and Age of Empires II, while outperforming other online update methods.}
\label{table:precision_test1}
\rotatebox{90}{
\begin{tabular}{lll|llll}
\toprule
\quad & \textbf{NRT$^\dagger$} & \textbf{NCT$^\dagger$} M=81& \textbf{Elo$^\dagger$} & \textbf{mElo2$^\dagger$} & \textbf{Elo} K=0.1 & \textbf{Elo-RCC} M=81 \\
\midrule
Rock-Paper-Scissors & 51.3 $\pm$ 10.0 & \textbf{100} $\pm$ 0 & 73.3 $\pm$ 10.1 & \textbf{100} $\pm$ 0 & 60.0 $\pm$ 9.9 & \textbf{100} $\pm$ 0\\
\midrule
Advanced Combination & 57.9 $\pm$ 0 & \textbf{79.4} $\pm$  0.8 & 57.1 $\pm$ 0.2 & 52.7 $\pm$ 2.8 & 57.7 $\pm$ 0.1 & 68.0 $\pm$ 0.3\\
\midrule
Age of Empires II & 68.7 $\pm$ 0.8 & \textbf{97.7} $\pm$ 1.0 & 53.1 $\pm$ 4.1 & 51.2 $\pm$ 3.7 & 68.5 $\pm$ 1.0 & \textbf{94.3} $\pm$ 1.4\\
\midrule
Hearthstone & 83.4 $\pm$ 4.6 & \textbf{97.4} $\pm$ 0.3 & 74.8 $\pm$ 2.6 & 61.4 $\pm$ 5.6 & 81.6 $\pm$ 0.5 & \textbf{95.9} $\pm$ 0.9\\
\bottomrule
\toprule
Rock-Paper-Scissors & 51.1 $\pm$ 9.9 & \textbf{100} $\pm$  0 & 73.6 $\pm$ 9.9 & \textbf{100} $\pm$ 0 & 60.2 $\pm$ 10.2 & \textbf{100} $\pm$ 0\\
\midrule
Advanced Combination & 56.5 $\pm$ 0.3 & \textbf{79.7} $\pm$ 0.8 & 56.1 $\pm$ 0.2 & 51.3 $\pm$ 0.6 & 56.4 $\pm$ 0.3 & 65.3 $\pm$ 0.6\\
\midrule
Age of Empires II & 64.5 $\pm$ 1.3 & \textbf{75.4} $\pm$ 0.9 & 52.4 $\pm$ 4.9 & 51.0 $\pm$ 3.1 & 64.7 $\pm$ 1.7& \textbf{75.0} $\pm$ 1.2\\
\midrule
Hearthstone & 81.2 $\pm$ 0.5 & \textbf{94.8} $\pm$ 0.5 & 74.9 $\pm$ 2.6 & 61.1 $\pm$ 5.8 & 81.3 $\pm$ 0.7 & \textbf{94.4} $\pm$ 0.8\\
\bottomrule
\end{tabular}
}
\end{table*}

\begin{table*}[!htb]
\centering
\caption[Elo-RCC Table Size Comparison]{Effect of Category Granularity: Accuracies (\%) for smaller counter category sizes ($M=3, 9, 27$) in training (top) and testing (bottom) sets. Results are averaged over 5-fold cross-validation, with $\pm$ one standard deviation. $\dagger$ Results for NCT are reproduced from \citet{game_balance_analysis}. Elo-RCC demonstrates slightly higher accuracy than NCT for $M=3$ and $M=9$, showcasing its robustness in smaller discretization scenarios.}
\label{table:precision_test2}
\rotatebox{90}{
\begin{tabular}{l|ll|ll|ll}
\toprule
\quad & M=3 \textbf{NCT$^\dagger$} & \textbf{Elo-RCC} & M=9 \textbf{NCT$^\dagger$} & \textbf{Elo-RCC} & M=27 \textbf{NCT$^\dagger$} & \textbf{Elo-RCC} \\
\midrule
Rock-Paper-Scissors & \textbf{100} $\pm$ 0 & \textbf{100} $\pm$ 0 & \textbf{100} $\pm$ 0 & \textbf{100} $\pm$ 0 & \textbf{100} $\pm$ 0 & \textbf{100} $\pm$ 0\\
\midrule
Advanced Combination & \textbf{57.9} $\pm$ 0.0 & \textbf{58.1} $\pm$ 0.8 & \textbf{79.4} $\pm$ 0.6 & 57.7 $\pm$ 0.1 & \textbf{79.8} $\pm$ 0.2 & 57.7 $\pm$ 0.1\\
\midrule
Age of Empires II & \textbf{68.7} $\pm$ 0.8 & \textbf{71.0} $\pm$ 0.5 & \textbf{73.1} $\pm$ 3.9 & \textbf{74.6} $\pm$ 1.2 & \textbf{83.8} $\pm$ 1.4 & \textbf{83.6} $\pm$ 1.4\\
\midrule
Hearthstone & \textbf{81.3} $\pm$ 0.2 & \textbf{81.6} $\pm$ 1.2 & \textbf{85.4} $\pm$ 1.2 & \textbf{85.9} $\pm$ 0.8 & \textbf{91.7} $\pm$ 0.6 & \textbf{91.2} $\pm$ 1.1\\
\bottomrule
\toprule
Rock-Paper-Scissors & \textbf{100} $\pm$ 0 & \textbf{100} $\pm$ 0 & \textbf{100} $\pm$ 0 & \textbf{100} $\pm$ 0 & \textbf{100} $\pm$ 0 & \textbf{100} $\pm$ 0\\
\midrule
Advanced Combination & \textbf{56.5} $\pm$ 0.3 & \textbf{56.4} $\pm$ 0.3 & \textbf{79.7} $\pm$ 0.4 & 56.4 $\pm$ 0.3 & \textbf{80.1} $\pm$ 0.4 & 56.4 $\pm$ 0.3\\
\midrule
Age of Empires II & \textbf{64.5} $\pm$ 1.3 & \textbf{66.5} $\pm$ 2.1 & \textbf{67.7} $\pm$ 2.5 & \textbf{68.7} $\pm$ 1.8 & \textbf{72.5} $\pm$ 1.1 & \textbf{71.3} $\pm$ 0.5\\
\midrule
Hearthstone & \textbf{81.3} $\pm$ 0.5 & \textbf{81.5} $\pm$ 1.5 & \textbf{85.2} $\pm$ 0.9 & \textbf{85.7} $\pm$ 1.3 & \textbf{91.3} $\pm$ 0.8 & \textbf{90.9} $\pm$ 1.1\\
\bottomrule
\end{tabular}
}
\end{table*}

\section{Beyond Balance: Applications, Diversity, and Design Implications}

A core question throughout this chapter has been: \textit{What makes a playstyle meaningful or worth preserving?} Our proposed measures—Top-D Diversity and Top-B Balance—suggest a practical answer: a strategy deserves recognition when it is either sufficiently strong or sufficiently resilient against contextual counters. In competitive ecosystems, diversity is not just about difference—it is about \textbf{defensible difference}. A playstyle gains significance when it survives not by being identical to others, but by offering unique, viable paths to success.

\paragraph{From Balance to Playstyle Diversity.}
Traditional game balance is often viewed through the lens of fairness or equalized outcomes. However, for strategy games and AI agents alike, balance should also support stylistic diversity: the coexistence of multiple viable approaches that differ in philosophy, risk profile, or play tempo. The Top-D and Top-B metrics enable this broader view. Rather than asking only whether win rates are fair, we can now ask: \textit{How many different ways to win exist, and how robust are they against contextual threats?}

Top-D Diversity captures perceived viability within a tolerance margin—it respects players' acceptance of small win-rate gaps due to randomness or preference. Top-B Balance reveals whether the strategic landscape contains a web of non-dominated alternatives, reflecting counterplay and structural richness. Together, they offer a dual lens on strategic diversity: breadth and depth.

\paragraph{Recap: A General Framework for Competitive Analysis.}
The proposed rating and counter table framework offers a scalable, interpretable, and principled method for analyzing strategic ecosystems. NRT estimates scalar strength via neural networks. NCT extends this by learning intransitive counter relationships, enabling deeper modeling of matchups. Elo-RCC further enables online updates, making this approach suitable for live systems and continual training.

These models support the computation of our new balance measures—Top-D and Top-B—which move beyond basic statistics and directly address structural and perceptual diversity.

\paragraph{Broader Applications.}
Though developed for team-based PvP games, the methodology generalizes to other competitive and comparative settings with asymmetric outcomes and intransitive preferences:

\begin{itemize}
  \item \textbf{Sports and tournaments}: Modeling matchups and rivalry cycles between teams or players.
  \item \textbf{Recommendation and ranking}: Inferring relative preferences (e.g., for images, movies, products) when human judgment exhibits non-transitive patterns \citepx{beauty_score}.
  \item \textbf{Social choice and voting systems}: Analyzing cyclic preferences in elections or group decisions.
  \item \textbf{LLM evaluation and chatbot tournaments}: Where some models outperform others on specific tasks but not uniformly \citepx{chatbot_arena, liu2025reevaluating}.
  \item \textbf{Multi-agent learning and AI safety}: Identifying failure modes, exploitable subpopulations, or complementary strategies in competitive environments \citepx{alpha_star, ai_safety}.
\end{itemize}

In each case, the underlying problem is the same: how to evaluate entities not in isolation, but within a dynamic field of interaction. Counter-aware modeling provides a foundation for understanding not only who is strong, but who beats whom—and why.

\paragraph{Design Implications: Supporting Richer Ecosystems.}
Game designers often struggle to preserve strategic diversity while maintaining fairness. Standard tuning practices—buffing weak strategies or nerfing strong ones—risk homogenizing outcomes. Our approach suggests an alternative: evaluate diversity through tolerable difference and counter resilience.

For example, in our Age of Empires II case study, the civilization \textbf{Poles} was found dominated in its category not because it was universally weak, but because its top competitors dominated it across matchups. A naive buff might overshoot the goal; by contrast, counter-aware analysis allows precise interventions to restore viability without distorting the broader meta.

Similarly, Hearthstone's \textbf{Treant Druid}, though numerically dominant, still had contextual counters. Recognizing this counterplay structure allows designers to maintain diversity while addressing fairness concerns, striking a balance between challenge and expression.

\paragraph{Cautions and Limitations.}
Despite its utility, our framework comes with limitations:

\begin{itemize}
  \item It assumes a symmetric, two-team zero-sum formulation, which may not generalize directly to asymmetric or cooperative settings.
  \item All learned models—including NRT, NCT, and Elo-RCC—are susceptible to sampling bias, data scarcity, and overfitting.
  \item Overreliance on quantification can lead to premature optimization: fixing local imbalances while missing deeper emergent dynamics.
\end{itemize}

Moreover, such models can be misused—e.g., reinforcing meta dominance, removing underused strategies rather than supporting them, or exploiting player data for unintended purposes. We advocate their use not as instruments of convergence, but as tools for preserving and enriching diversity.

\paragraph{Final Reflection.}
Balance is not about flattening gameplay. It is about preserving meaning across difference. A rich strategic ecosystem should challenge players to adapt, not merely to optimize. Our metrics aim to support that vision—not just by detecting imbalance, but by identifying where meaningful difference persists and deserves to be celebrated.

\subsection*{Summary and Outlook}

This chapter began by addressing a core motivation behind preserving playstyle diversity: the recognition that players' preferences and beliefs often extend beyond strictly optimal solutions. Even when some strategies dominate statistically, others may persist due to uncertainty, perceived fairness, or personal inclination. Rather than dismiss such playstyles as suboptimal noise, we argued they reflect meaningful \textbf{preferences}—subjective but defensible grounds for diverse strategic expression.

To formalize this intuition, we developed a series of tools for quantifying both diversity and balance. Our two proposed measures—\textbf{Top-D Diversity} and \textbf{Top-B Balance}—move beyond conventional win-rate analysis. Top-D quantifies how many strategies are \textit{perceived as viable}, tolerating small win-rate gaps under belief-based assumptions. Top-B, in turn, identifies the number of \textit{non-dominated} strategies within a structured counterplay system, highlighting the depth of intransitive dynamics.

These measures are supported by a progression of modeling tools: the \textbf{Neural Rating Table (NRT)} for estimating scalar strength with synergy, the \textbf{Neural Counter Table (NCT)} for contextual intransitive structure, and \textbf{Elo-RCC} for efficient online estimation. Together, they offer a computationally tractable yet expressive framework for assessing strategic ecosystems.

Throughout our case studies, we showed that balance is not about achieving uniformity, but about sustaining \textit{meaningful difference}—a space in which multiple strategies can coexist, not just statistically, but structurally and perceptually. This defensible diversity is what makes a playstyle not just viable, but worth preserving.

In the next part, we shift our focus from the definition and measurement of playstyle to its \textbf{expression}: how playstyles emerge through decision-making, how they evolve in learning systems, and how they manifest as behavior. Chapter~8 lays the foundation for this exploration by examining decision-making models, goal structures, and the mechanisms by which preferences shape actions in both artificial agents and human players.

\newpage

\part{Expression of Playstyle}
\chapter{Decision-Making Foundations}
\noindent\textbf{Key Question of the Chapter:} \\
\textit{How are decision-making problems handled in artificial intelligence?}

\noindent\textbf{Brief Answer:} 
AI as a field was fundamentally created to address decision-making problems. By reviewing how AI has evolved, we can see how different approaches have been developed to formalize goals, represent environments, and produce effective actions. Among these, reinforcement learning stands out for its ability to improve through interaction and adapt to changing conditions—making it worth examining in detail through its major milestones.

\bigskip

In the previous chapters, we established a foundation for understanding playstyle—its definition, how it can be measured, and its essential role in sustaining diversity and balance within strategic systems. We now turn from the question of \textbf{what playstyle is} to that of \textbf{how playstyle emerges through behavior}.

This chapter marks the beginning of Part III, which explores the computational processes underlying playstyle. At the heart of this inquiry lies the concept of \textbf{decision-making}—the ability of an agent to choose actions in response to its goals, observations, and environment. Whether selecting a move in a game, responding in conversation, or taking a physical action, every expression of behavior reflects an implicit decision process.

To better understand how such processes shape behavior, we first introduce key perspectives on how decision-making is typically approached in artificial intelligence. Rather than aiming for formal axiomatization, our goal is to sketch the main ideas: how goals and feedback define decision problems, and how various learning paradigms—especially \textbf{deep learning and reinforcement learning}—equip agents to navigate them.

Finally, we trace a set of influential developments in deep reinforcement learning. These milestones not only advanced performance benchmarks, but also broadened the expressive space of agent behaviors—laying the practical groundwork for the emergence and analysis of playstyles.

\section{Goals in Decision-Making}

Now that we have discussed the definition and measurement of playstyle, the next question is how such styles may be exhibited, guided, or learned within AI systems. Humans express playstyle not only through their direct behavior but also through the creation and use of tools—procedures, interfaces, or policies—that augment their capabilities and stabilize the outcomes of their decisions. The ability to externalize and operationalize intentionality is one hallmark of intelligence. In this sense, artificial intelligence is the engineering of such tools: systems that can carry out decision-making on our behalf.

At the core of decision-making lies the specification of a goal—an articulation of what the agent should pursue. In computational terms, this is often expressed through a \textbf{utility function}: a construct that encodes preference, value, or desirability over possible outcomes. Modern AI systems typically operationalize this concept using \textbf{scalar objectives}: scores, losses, or rewards that guide learning and behavior.

This scalar framework enables practical optimization. In supervised learning, it allows gradient-based minimization of error. In reinforcement learning (RL), it enables agents to learn policies that maximize expected cumulative reward through environment interaction \citepx{russell2020aima}. This perspective is often referred to as the \textbf{reward hypothesis}—the belief that any goal can be cast as the maximization of scalar reward over time \citepx{sutton2004rewardhypothesis, silver2021reward}. Many landmark AI systems, from DQN to AlphaGo and Agent57, are built on this foundation.

However, scalar reward is not a neutral or universally sufficient representation of goals. It reflects only a particular instantiation of utility—one shaped by the environment designer’s assumptions. As discussed in Chapter 7, what is rewarded reflects what is valued, and in systems with multiple agents or playstyles, rewards implicitly encode a \textbf{preference structure}. Moreover, such preferences are often intransitive: a strategy that defeats one opponent may fail against another, leading to cyclical dynamics that scalar ratings alone cannot fully capture.

This raises a deeper point: \textbf{before an agent can optimize a goal, that goal must be defined}. And defining a goal requires specifying whose perspective counts, what trade-offs are acceptable, and how outcomes are to be evaluated. In other words, goals are grounded in beliefs. We thus propose a conceptual hierarchy:

\[ \textbf{Belief} \longrightarrow \textbf{Goal} \longrightarrow \textbf{Reward} \]

In practice, this hierarchy is often collapsed. Benchmark tasks provide ready-made reward functions, implicitly assuming shared understanding. But in open-ended or human-centered tasks, goals are rarely unitary. They may be multi-objective, context-dependent, or symbolically defined. Attempts to compress them into a single scalar face challenges of aggregation: how to weigh competing factors, interpret trade-offs, or retain semantic richness.

This problem is well-known in multi-objective learning, where Pareto optimization, constrained policies, or lexicographic preference have been proposed. Yet these methods often assume gradient access and convexity. Non-differentiable or symbolic goals remain harder to handle, even with evolutionary strategies.

Even more subtly, shaped rewards—intended to improve learning efficiency—may alter behavior in unintended ways. Although potential-based shaping preserves optimality in theory \citepx{ng1999policy}, real-world learners are rarely optimal. Behavior may diverge despite nominally encoding the same goal, especially under exploration or approximation.

These issues show that scalar rewards are tools, not truths. They are computational encodings of designer beliefs. And when multiple agents coexist, these beliefs may diverge. One agent’s notion of "helpfulness" or "efficiency" may differ from another’s. Thus, playstyle does not emerge purely from optimizing a fixed reward, but from how the goal itself is framed, interpreted, and prioritized.

Recent research has explored richer formulations, such as reward machines \citepx{icarte2022reward}, cognitive models \citepx{molinaro2023goal}, or compositional goal programs \citepx{davidson2025goals}. These approaches aim to better reflect human-like goal specification but introduce complexity and new challenges in interpretability, learning, and alignment.
More recently, Vision-Language Models (VLMs) have emerged as a new paradigm for implicit goal specification. Rather than relying on handcrafted scalar rewards, VLMs can transform natural language descriptions or exemplar images into flexible evaluative signals—serving as learned reward functions for reinforcement learning tasks. These models have been used as zero-shot reward evaluators \citepx{rocamonde2023vision}, as general-purpose feedback mechanisms in open-ended agent training \citepx{chan2023visionlanguage}, and as full pipelines for vision-language feedback-guided learning \citepx{wang2024rl}. This trend suggests a shift toward more human-aligned and semantically rich forms of goal definition, while raising new concerns about grounding, bias, and the stability of behavior shaped by implicit, pretrained preferences.

In this chapter, we adopt the scalar reward framework as a working basis—acknowledging both its practicality and its limitations. Our goal is not to discard it, but to contextualize it. Every scalar reward encodes a belief; every optimization reflects a preference. Understanding this is critical for interpreting not just \textit{what} an agent does, but \textit{why} it behaves that way—and ultimately, what makes one playstyle distinct from another.

\section{Solution Types to Decision Problem}

Before discussing how deep reinforcement learning (DRL) rose to prominence as the dominant paradigm for learning decision-making agents, we must first examine the landscape of decision problems themselves. After all, the playstyle of an agent is shaped not only by its goals or reward function, but also by the very structure of the environment and the type of decision task it faces. Understanding this landscape thus provides crucial context for later discussions of style expression and evaluation.

\subsection{A Functional Taxonomy of Decision Problem Types}

One of the most practical taxonomies for decision-making—particularly in probabilistic or uncertain environments—is the \textbf{four-scenario framework} described in \citetx{zhou2015survey}'s survey on contextual multi-armed bandits, originally adapted from CMU’s graduate AI course material. It classifies decision problems along two axes:

\begin{enumerate}
  \item \textbf{Do actions affect future states?} (static vs. dynamic)
  \item \textbf{Is the environment model known?} (known vs. unknown dynamics)
\end{enumerate}

This yields the four canonical types of decision problems, as illustrated in Figure~\ref{fig:decision-types}:

\begin{itemize}
  \item \textbf{Decision Theory}: Static setting with known outcomes, often solved analytically using expected utility.
  \item \textbf{Multi-Armed Bandits (MAB)}: Static setting with unknown outcome distributions, emphasizing exploration vs. exploitation.
  \item \textbf{Markov Decision Process (MDP)}: Dynamic setting with known transition dynamics and rewards.
  \item \textbf{Reinforcement Learning (RL)}: Dynamic setting with unknown dynamics, requiring trial-and-error interaction to learn both the model and optimal behavior.
\end{itemize}

\begin{figure}[ht]
  \centering
  \includegraphics[width=0.9\linewidth]{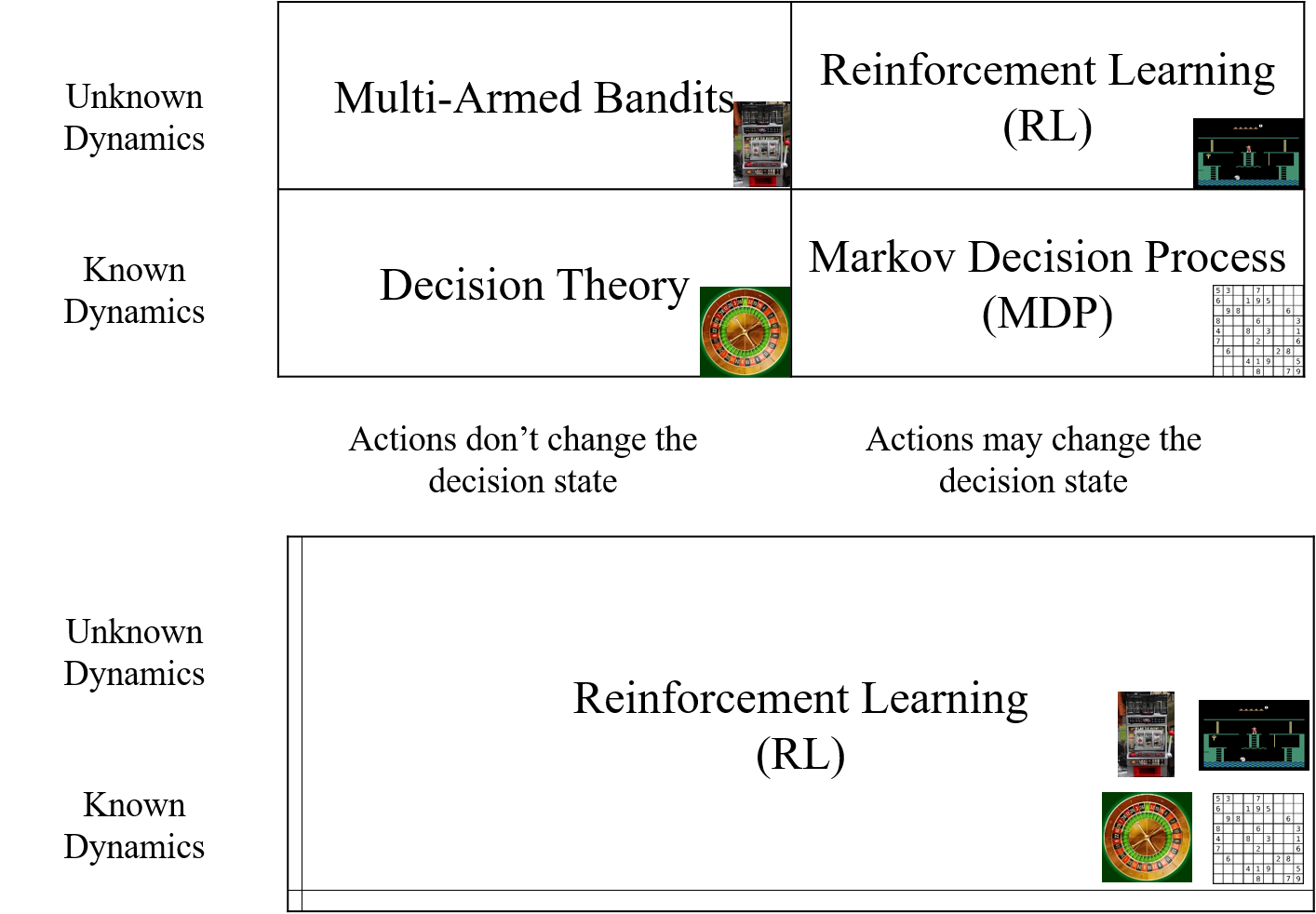}
  \caption[Four Canonical Types of Decision Problems]{A taxonomy of four canonical decision problem types, defined by whether actions affect future states (static vs. dynamic) and whether environment dynamics are known (known vs. unknown). Reinforcement Learning, shown in the upper-right quadrant, represents the most general case—dynamic environments with unknown dynamics—and thus subsumes the other settings. Illustrated with game analogies: slot machines, roulette, Atari games, and Sudoku. Adapted from \citetx{zhou2015survey}.}
  \label{fig:decision-types}
\end{figure}

\textbf{This generality also suggests an important theoretical point:} If we temporarily disregard concerns of training cost or sample efficiency, then RL can in principle subsume nearly all common types of decision problems. Its formulation—learning through interaction, guided only by sparse feedback—makes it an attractive modeling framework not only for performance learning, but also for capturing behavioral differences. In other words, RL is not just a practical tool, but a conceptually expressive one: \textbf{a framework capable of supporting the emergence of style through experience}.

From a playstyle perspective, this taxonomy reflects not only methodological distinctions but also differences in how much \textbf{space for stylistic variation} each setting allows. In bandits, style may be expressed through preferences in risk or novelty. In MDPs and RL, playstyle is richer—it can emerge through differences in temporal strategy, adaptive learning, or risk modeling. Moreover, unknown environments offer more interpretative freedom, giving agents broader space for expressing learning styles, inductive biases, or exploration heuristics.

\subsection{Environmental Complexity as a Driver of Playstyle Expression in RL}

The complexity of an environment fundamentally shapes not only the difficulty of decision-making but also the space in which playstyle can emerge. As noted in \citetx{russell2020aima}, environments can be categorized along several structural dimensions (Figure~\ref{fig:environment-types}). Each dimension introduces distinct types of uncertainty, interactivity, or temporal constraints that impact how agents must behave. When these dimensions are combined, they create rich, high-dimensional spaces where style is no longer incidental—it becomes a key differentiator of behavior.

\begin{figure}[ht]
  \centering
  \includegraphics[width=0.7\linewidth]{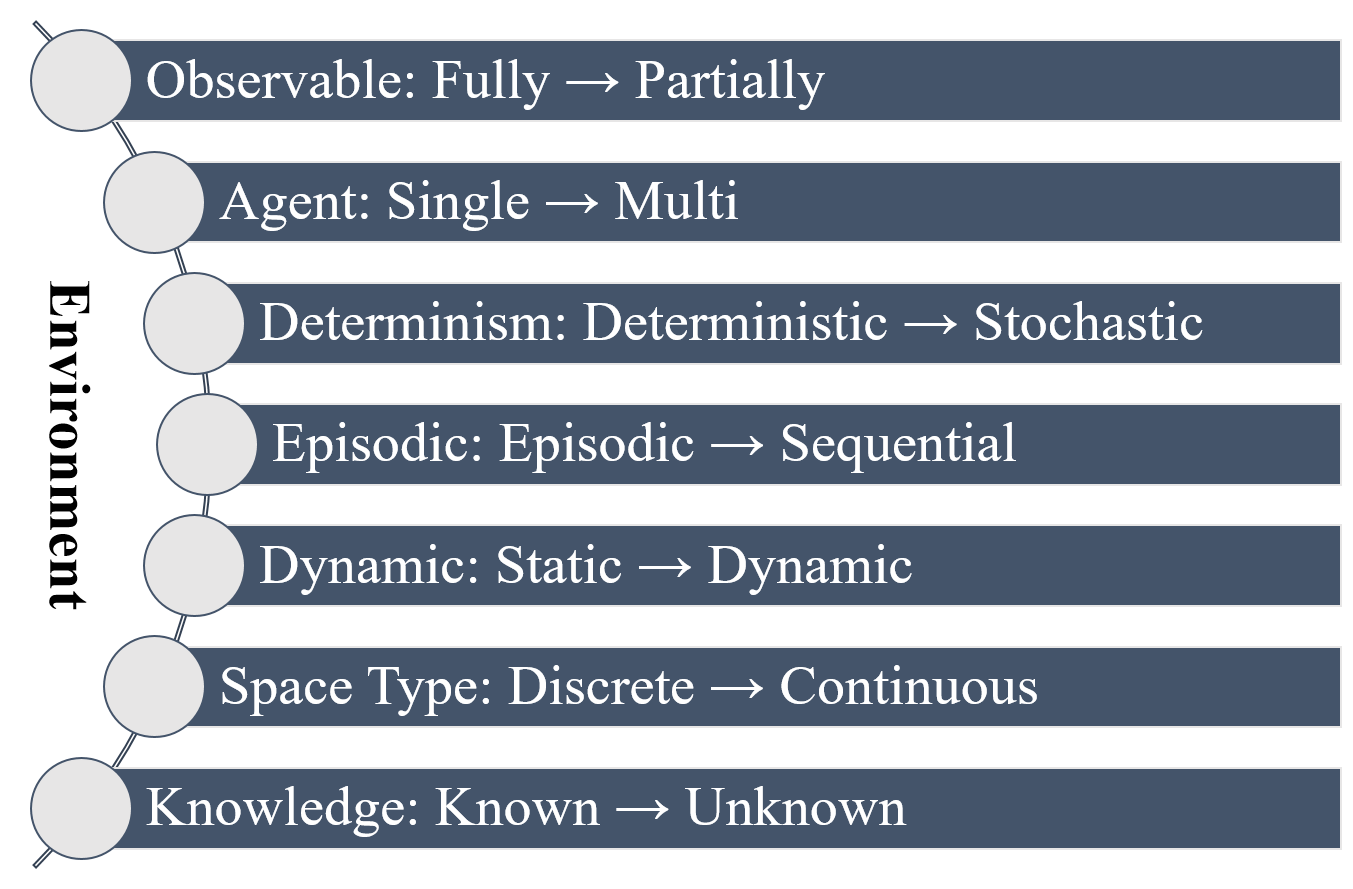}
  \caption[Seven Environmental Dimensions]{Seven environmental dimensions influencing decision complexity and playstyle diversity, adapted from \citetx{russell2020aima}.}
  \label{fig:environment-types}
\end{figure}

\begin{itemize}
  \item \textbf{Observability}: Full vs. partial; the latter requires memory or belief modeling.
  \item \textbf{Number of Agents}: Single-agent vs. multi-agent; multi-agent systems introduce strategy coupling, social modeling, or opponent adaptation.
  \item \textbf{Determinism}: Deterministic vs. stochastic; affects predictability and risk handling.
  \item \textbf{Temporal Structure}: Episodic vs. sequential; sequential tasks demand planning and long-term consistency.
  \item \textbf{Dynamics}: Static vs. dynamic; dynamic environments evolve independently, requiring responsiveness.
  \item \textbf{Space Type}: Discrete vs. continuous; influences precision, control, and representation.
  \item \textbf{Model Knowledge}: Known vs. unknown; known models permit planning, unknown ones necessitate strong intuition.
\end{itemize}

Each axis contributes to the expressiveness of playstyle in different ways:

\begin{itemize}
  \item Partial observability tests how agents model uncertainty and hidden state.
  \item Multi-agent interactions elicit social or strategic playstyles—e.g., deception, cooperation, anticipation.
  \item Stochasticity demands robustness, inducing stylistic variance in handling noise.
  \item Sequentiality reveals planning horizon, patience, or decision consistency.
  \item Dynamic environments highlight responsiveness and adaptivity.
  \item Continuity reflects preferences for smoothness or precision in control.
  \item Known environments support long-term planning according to expectations in the agent’s internal cognition.
\end{itemize}

Crucially, many of these environmental settings are not amenable to analytic solutions or planning with full knowledge. They require agents to act under uncertainty, learn from experience, and adapt to unforeseen changes—challenges that align naturally with the capabilities of Reinforcement Learning (RL). RL provides a general framework for learning from interaction, and thus becomes essential for building agents that must operate in rich, diverse environments.

In essence, the more complex and uncertain the environment, the more degrees of freedom the agent must coordinate, and the more opportunity there is for stylistic differentiation. In highly structured or deterministic settings, optimal behavior may converge across agents. But in environments that demand exploration, anticipation, adaptation, or communication, playstyle becomes an expressive necessity rather than a side effect. RL, as a framework capable of handling such complexity, provides a natural substrate for style to emerge, evolve, and be measured.

\subsection{From Representation to the Era of Experience}

Although reinforcement learning is praised for its generality, its modern success is deeply tied to the representational power of deep learning. In high-dimensional or partially observable domains, agents must not only act but \textit{perceive}, \textit{abstract}, and \textit{evaluate}. Deep learning enables this joint optimization of perception, value estimation, and policy.

Rather than relying on handcrafted rules or symbolic state definitions, deep RL agents build their internal structure through experience. This allows both performance and stylistic distinctiveness to co-emerge.

\citetx{silver2025experience}’s recent essay, \textit{Welcome to the Era of Experience}, highlights this shift from passive, supervised learning to active, experience-driven intelligence. In such a paradigm, style becomes central—it reflects not just output differences, but underlying modes of interpretation, preference, and self-guided exploration.

\textbf{In sum}, the reason DRL is so relevant to playstyle is not merely its scalability or learning efficiency. It is because DRL creates space for behavioral individuality to emerge—through architecture, through experience, and through the intrinsic complexity of the environment itself. It is not just a means of optimizing behavior, but a window into the decision-making signatures of intelligent agents.

\section{Deep Learning and Reinforcement Learning}

Deep learning and reinforcement learning form the two computational pillars of modern decision-making AI \citepx{russell2020aima}. Reinforcement learning (RL) offers a framework for learning through interaction, where agents optimize behavior based on delayed and often sparse feedback signals. Deep learning (DL), in turn, provides the representational power needed to process raw, high-dimensional inputs and to generalize across diverse states and situations.

The convergence of these two paradigms has transformed AI systems from rule-driven automata into adaptive agents capable of discovering policies, strategies, and stylistic behaviors through experience. Critically, it is this combination—DL's perception and abstraction, together with RL's trial-and-error adaptation—that enables not only effective action, but also the emergence of diverse playstyles.

In this section, we trace the evolution of DL and RL, examine their convergence into deep reinforcement learning (DRL), and highlight how this synthesis empowers the development of general-purpose agents with rich behavioral expressiveness (Figure~\ref{fig:ml-progress-map}).

\begin{figure}[ht]
\centering
\includegraphics[width=0.95\linewidth]{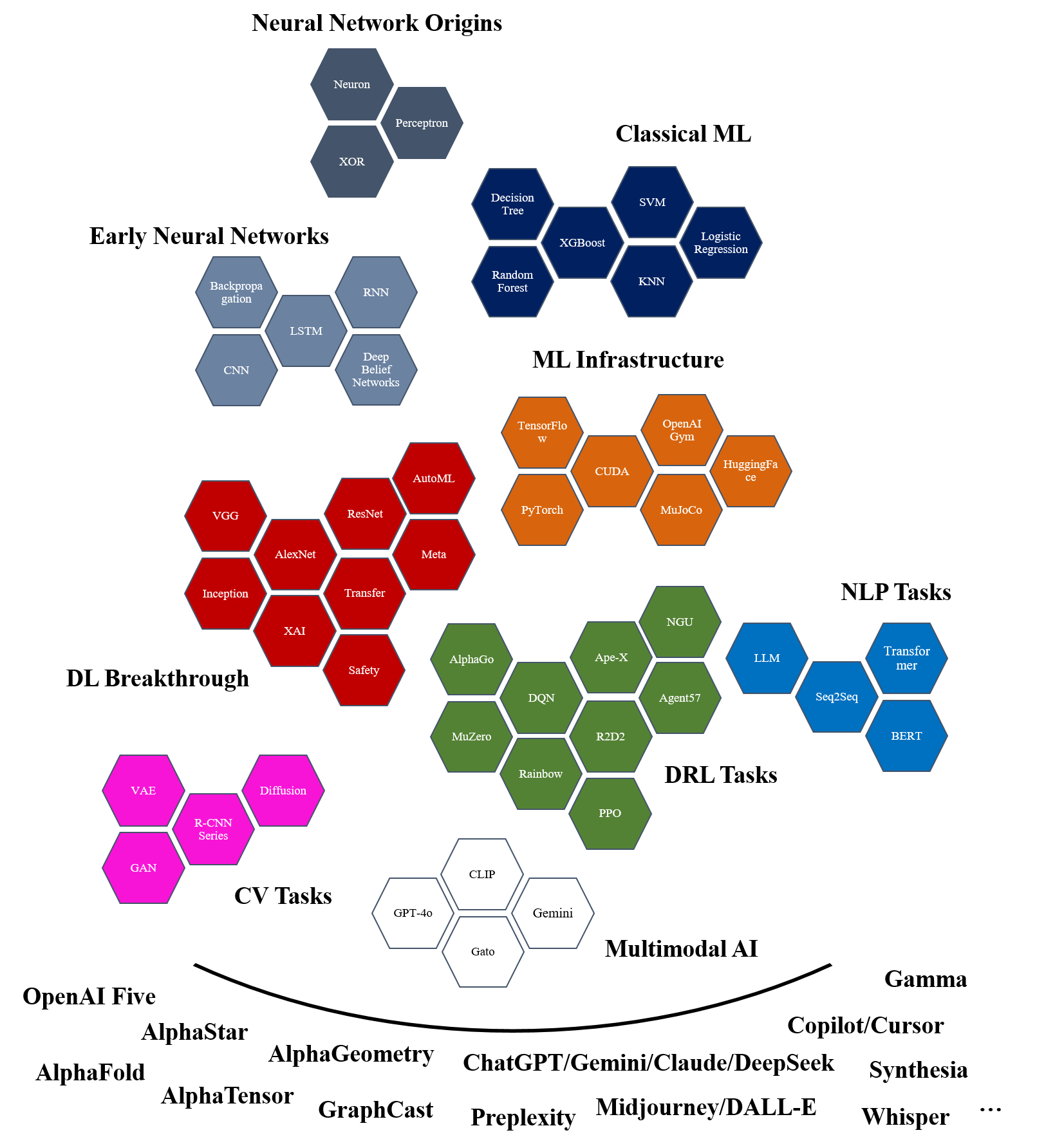}
\caption[Major Milestones in Deep Learning and RL]{A visual summary of major developments in machine learning, deep learning, and reinforcement learning. Each cluster highlights a key lineage of ideas—from early neural models and symbolic ML, to domain-specific innovations in CV, NLP, and DRL, and ultimately toward general-purpose and multimodal models.}
\label{fig:ml-progress-map}
\end{figure}

\clearpage

\subsection{Early Development, the Soviet Perspective, and Symbolic Integration}

The conceptual roots of neural computation can be traced to the \textbf{McCulloch--Pitts neuron} in 1943, which modeled logical functions using simplified binary units \citepx{mcculloch1943logical}. This inspired the \textbf{Perceptron}, introduced by Rosenblatt in 1958, which implemented a learning algorithm capable of adjusting weights to perform linear classification \citepx{rosenblatt1958perceptron}. The Perceptron was widely celebrated as a foundational step toward machine intelligence.

However, enthusiasm waned after Minsky and Papert's 1969 book \textit{Perceptrons} \citepx{minsky1969perceptrons}, which showed that single-layer Perceptrons could not solve linearly inseparable problems such as XOR. Without effective training algorithms for multi-layer networks, neural network research lost momentum. This marked the onset of the first “AI winter,” during which funding and interest in connectionist approaches sharply declined.

Historical perspectives outside the Western AI narrative further illuminate this period. For instance, in the Soviet Union, researchers doubted whether machines could ever truly “think” in the human sense, and reframed AI as a discipline of control rather than cognition \citepx{kirtchik2023soviet}. This emphasis on situational management and learning automata—trial‑and‑error adaptation under feedback and uncertainty— foreshadowed aspects of modern reinforcement learning, even though the formal MDP‑centric RL framework emerged later in the Western canon.

In the absence of scalable neural models, researchers turned to more tractable and interpretable alternatives. \textbf{Symbolic AI} focused on hand-coded rules and logic inference, while \textbf{statistical learning} began to flourish through approaches such as:
\begin{enumerate}
    \item \textbf{Decision trees} (e.g., ID3, C4.5) \citepx{quinlan1986induction, quinlan1993c4.5};
    \item \textbf{Support vector machines} (SVMs) \citepx{cortes1995support};
    \item \textbf{Logistic regression} and \textbf{naive Bayes} \citepx{bishop2006pattern, hastie2009elements};
    \item \textbf{Random forests} and \textbf{gradient boosting machines} \citepx{breiman2001random, friedman2001greedy}.
\end{enumerate}
These techniques defined the \textbf{"classical machine learning"} paradigm and remain strong baselines today.

Decision trees in particular stood out for their interpretability and symbolic structure—qualities that align with belief-grounded goals discussed earlier. Unlike black-box models, decision trees reveal the conditions under which specific outcomes are chosen, making them natural candidates for structured playstyle modeling, especially when decisions involve branching, conditional logic, or interpretable trade-offs.

In recent years, hybrid approaches have emerged that seek to combine the expressive power of neural networks with the interpretability of trees. One notable early work is \textbf{Deep Neural Decision Forests}, which integrates convolutional feature extractors with differentiable decision trees, allowing end-to-end training via backpropagation \citepx{kontschieder2015deep}. While later developments in deep learning emphasized scale and generalization, such tree-based hybrids illustrate an important design direction—where structured symbolic reasoning and learned perception are combined to model both \textit{how} agents act and \textit{why} they act that way.

This direction connects naturally to the broader vision of \textbf{neuro-symbolic AI} \citepx{sheth2023neurosymbolic}, which integrates sub-symbolic neural learning with explicit symbolic reasoning to improve interpretability, generalization, and alignment with human-understandable concepts. In the context of this dissertation, our \textit{discrete-state playstyle measurement and counter-category learning} methods can be seen as neuro-symbolic in spirit: state discretization provides a symbolic layer for representing and comparing playstyles, while neural networks (or alternative rule-based learners) can be used to learn these abstractions from high-dimensional data. Notably, the symbolic layer can also be constructed with non-neural methods, and counter-category probabilities can be derived without relying solely on deep models—highlighting that effective symbolic reasoning need not be tied to any single learning paradigm. This mirrors recent advances in discrete, compositional, and symbolic representation learning through attractor dynamics \citepx{nam2023discrete}, where discrete abstractions serve as robust anchors for complex decision-making processes.

\subsection{Neural Network Resurgence and Deep Learning Explosion}

The resurgence of neural networks in the 2010s was not driven solely by algorithmic innovation. Rather, it was the convergence of \textbf{conceptual}, \textbf{computational}, and \textbf{ecosystem-level} breakthroughs that catalyzed what is now widely known as the \textbf{deep learning revolution} \citepx{lecun2015deep}.

A key theoretical advance had occurred in the 1980s with the popularization of \textbf{backpropagation} as a general-purpose training algorithm for multi-layer networks \citepx{rumelhart1986learning}. This enabled the training of \textbf{deep feedforward neural networks}, but for decades, progress remained bottlenecked by limited compute and a lack of large, labeled datasets.

The turning point came in 2012 with the success of \textbf{AlexNet}, which won the ImageNet Large Scale Visual Recognition Challenge by a substantial margin \citepx{krizhevsky2012imagenet}. AlexNet demonstrated that \textbf{convolutional neural networks (CNNs)}, when trained on large-scale data and accelerated via \textbf{GPUs} (using \textbf{CUDA} \citepx{cuda2008}), could decisively outperform classical computer vision approaches.

This marked a paradigm shift. Neural networks transitioned from theoretical curiosities to practical engines of performance. The emergence of open-source libraries such as \textbf{Theano} \citepx{bergstra2010theano}, \textbf{TensorFlow} \citepx{abadi2016tensorflow}, and later \textbf{PyTorch} \citepx{paszke2019pytorch} democratized deep learning, enabling researchers and practitioners around the world to experiment, scale, and iterate rapidly.

A rapid wave of architectural innovation followed:
\begin{enumerate}
\item \textbf{VGG} \citepx{simonyan2015vgg} and \textbf{GoogLeNet} \citepx{inception_model} deepened networks, improving hierarchical feature extraction.
\item \textbf{ResNet} introduced identity-based skip connections, allowing for the stable training of networks with hundreds of layers \citepx{he2016deep}.
\item \textbf{Batch normalization} and \textbf{dropout} improved convergence and generalization \citepx{ioffe2015batch, srivastava2014dropout}.
\end{enumerate}

These innovations formed the \textbf{computational and representational backbone} of modern AI systems. As compute budgets grew and benchmark competitions proliferated, deep learning expanded into nearly every application domain: from image classification and speech recognition to natural language processing \citepx{devlin2019bert, radford2019gpt2}, bioinformatics, and robotics.

This era also gave rise to the paradigm of \textbf{end-to-end learning}—training a single model to map raw inputs directly to task outputs without manual feature engineering \citepx{lecun2015deep}. Such models developed their own internal representations, tailored to the data and task structure, often uncovering non-obvious regularities.

This leads naturally to the next central idea: \textbf{manifold-based representation learning}. As models began to organize their internal activations into structured, high-dimensional spaces, the question shifted from “how do we train a deep network?” to “what kinds of internal structure emerge, and how do they shape behavior?”

Today’s AI ecosystem continues to benefit from the infrastructure laid during this period. Cloud-based GPUs and TPUs, pretrained foundation models, and shared platforms such as \textbf{HuggingFace} \citepx{wolf2020transformers}, \textbf{OpenAI Gym} \citepx{brockman2016openai}, and \textbf{Model Zoo} have transformed deep learning from a niche methodology into a foundational engine of modern AI research and deployment.

\subsection{Representation Learning and the Manifold Hypothesis}

While deep learning’s success is often attributed to scale and data availability, a more fundamental reason lies in its unparalleled capacity for \textbf{representation learning}—the automatic discovery of task-relevant abstractions from raw inputs. This ability is particularly vital in decision-making systems, where agents must transform complex perceptual streams (pixels, audio, tokens) into actionable beliefs, intentions, or decisions.

A central theoretical foundation for this capacity is the \textbf{manifold hypothesis}. It posits that although natural data resides in high-dimensional input spaces, its meaningful variation lies along low-dimensional, structured manifolds embedded within those spaces \citepx{fefferman2016testing, bengio2013representation}. For example, a human face image may exist in a 150,000-dimensional pixel space, but the underlying variation—such as pose, lighting, or identity—is governed by a far smaller set of latent factors.

Deep neural networks, especially those with hierarchical architectures, can be understood as learning a sequence of nonlinear transformations that \textbf{"unfold"} these curved manifolds. That is, they straighten the geometry of data in latent space, enabling simpler decision boundaries and more linearly separable representations. Each layer of a convolutional network corresponds to progressively higher-level abstractions—edges, parts, objects—aligning with this geometric interpretation.

\textbf{Empirical support} for this view is often drawn from visualization tools such as t-SNE \citepx{t_sne}, which reduce hidden-layer activations into two or three dimensions. These projections frequently reveal semantically meaningful clusters and trajectories, suggesting that deep models indeed structure internal representations according to task-relevant manifolds.

\subsubsection{Intuitive Analogy}

Consider a crumpled sheet of paper with digits written on it. In its raw form, the digits are distorted and hard to distinguish. Deep learning acts like a process of uncrumpling the paper—removing irrelevant deformation and revealing the underlying structure. Once flattened, the digits become easy to classify with a simple linear separator. Representation learning performs this uncrumpling in high-dimensional latent space.

\subsubsection{Relationship to Universal Approximation}

This intuition complements the \textbf{Universal Approximation Theorem} \citepx{hornik1989multilayer}, which guarantees that sufficiently large networks can approximate any continuous function. However, this theorem provides no insight into \textit{how} such functions are learned, nor whether they can generalize.

The manifold hypothesis addresses this gap by proposing that real-world data has constrained structure—enabling generalization and efficient learning. In other words:
\begin{enumerate}
\item The Universal Approximation Theorem tells us neural networks \textit{can} express arbitrary functions.
\item The Manifold Hypothesis explains \textit{why} they succeed in doing so on real-world data with tractable architectures and finite samples.
\end{enumerate}

\subsubsection{Implications for Playstyle and Decision-Making}

For intelligent agents—especially those operating in reinforcement learning settings—\textbf{the ability to learn structured internal representations is foundational}. Policies must reason not just over raw observations, but over task-relevant manifolds that encode affordances, strategies, opponent intentions, or temporal dependencies.

Without these representations, even a well-specified reward function or powerful learning algorithm may fail to produce coherent or expressive behavior. Conversely, rich representations enable style to emerge naturally: different inductive biases, learning experiences, or task framings lead agents to organize and traverse latent space in distinct ways.

Thus, understanding and leveraging representation learning is not merely a tool for improving performance—it is the key to unlocking diverse, robust, and stylistically differentiated decision-making.

\subsection{The Rise of DRL: AlphaGo and Beyond}

The integration of deep learning with reinforcement learning marked a decisive leap in the ability of AI systems to learn from interaction. This convergence gave rise to \textbf{deep reinforcement learning (DRL)}—a framework where agents optimize sequential decision-making under uncertainty while building internal representations directly from high-dimensional inputs.

The breakthrough came in 2015 with the introduction of the \textbf{Deep Q-Network (DQN)} \citepx{dqn}, which achieved human-level performance on a suite of Atari 2600 games. By combining convolutional neural networks with Q-learning, experience replay, and target networks, DQN showed that deep models could perceive, evaluate, and act solely from raw pixels and sparse game rewards. This demonstrated, for the first time, that it was possible to train agents that reason over extended time horizons without access to privileged state representations.

This foundational result sparked a wave of increasingly sophisticated DRL systems. The most transformative milestone was \textbf{AlphaGo} \citepx{alpha_go}, which in 2016 defeated a world champion Go player—a long-standing challenge in AI. AlphaGo’s architecture synthesized several core components: deep policy and value networks, Monte Carlo tree search, supervised learning from expert human games, and reinforcement learning via self-play.

Its successors, \textbf{AlphaGo Zero} and \textbf{AlphaZero}, dispensed with human data entirely, showing that tabula rasa learning—guided only by reward and search—was sufficient to discover superhuman strategies in Go, chess, and shogi \citepx{alpha_go_zero, alpha_zero}. These results not only advanced performance, but also highlighted the potential of DRL as a tool for strategy emergence and behavioral diversity.

Meanwhile, a series of general-purpose DRL architectures extended this success into broader domains:
\begin{itemize}
\item \textbf{IMPALA} \citepx{impala} scaled training via distributed actor-learners.
\item \textbf{Ape-X} and \textbf{R2D2} \citepx{ape_x, r2d2} addressed replay prioritization and sample efficiency in large-scale settings.
\item \textbf{MuZero} \citepx{schrittwieser2020muzero} learned its own latent world model, combining model-based planning with model-free learning.
\item \textbf{Agent57} \citepx{agent57} used a meta-controller to dynamically select exploration strategies, becoming the first agent to outperform humans across all 57 Atari benchmarks.
\end{itemize}

These systems pushed DRL into environments with increasing complexity—featuring long temporal horizons, delayed and sparse rewards, partial observability, and multi-agent interactions. Crucially, it is in these rich settings that \textbf{playstyle becomes most visible}. Agents must decide how to explore, how long to plan, how to manage uncertainty, and how to adapt over time. Even under the same reward structure, such choices yield qualitatively different trajectories and behavioral signatures.

In this light, DRL is not merely an optimization engine—it is a \textbf{generative substrate for style}. It enables agent individuality to emerge from differences in architecture, exploration, experience, or value framing. As later chapters will show, these stylistic divergences are not incidental—they are central to how we understand, evaluate, and design intelligent systems.

DRL, therefore, is more than a toolkit for control. It offers a window into the diversity of minds that arise when intelligence is shaped by interaction, guided by reward, and constrained by environment.

\subsection{From CV/NLP to Generative and Multimodal Models}

While deep reinforcement learning forged a new paradigm for \textbf{learning from interaction}, the domains of \textbf{computer vision (CV)} and \textbf{natural language processing (NLP)} were simultaneously transformed by the ability to \textbf{learn from data at scale}. These fields not only developed the dominant models for perception and understanding but also pioneered the architectural principles and objective functions that would later converge with DRL—laying the groundwork for multimodal, general-purpose agents.

\subsubsection{Computer Vision: Hierarchical Abstraction and Spatial Induction}

In computer vision, convolutional neural networks (CNNs) such as AlexNet, VGG, Inception, and ResNet enabled machines to process visual input by mimicking the hierarchical abstraction of the human visual system \citepx{krizhevsky2012imagenet, simonyan2015vgg, szegedy2015going, he2016deep}. These architectures extracted spatial features at increasing levels of abstraction—from edges and textures to objects and semantic categories.

Beyond classification, CV models became essential components in a range of perception and reasoning tasks:
\begin{enumerate}
\item \textbf{Object detection and segmentation}: Models like R-CNN, Faster R-CNN, and Mask R-CNN supported end-to-end detection and spatial reasoning \citepx{girshick2014rich, ren2015faster, he2017mask}.
\item \textbf{Scene understanding and depth estimation}: Monocular depth estimation enabled 3D reasoning from 2D images \citepx{godard2017unsupervised}.
\item \textbf{Visual grounding and captioning}: Sequence-based models translated image content into structured language \citepx{vinyals2015show, johnson2017inferring}.
\item \textbf{Style transfer and image generation}: Neural methods transformed visual content across aesthetic domains \citepx{gatys2016image}.
\end{enumerate}

These models not only recognize but interpret—serving as perceptual front-ends for agents that must interact with the world in a spatially grounded and stylistically expressive manner.

\subsubsection{Natural Language Processing: From Sequence Modeling to Language Understanding}

NLP has undergone a comparable transformation, transitioning from symbolic pipelines to fully neural architectures. The shift was catalyzed by the development of \textbf{sequence-to-sequence models}, and later, the \textbf{Transformer architecture} \citepx{vaswani2017attention}, which introduced self-attention and positional encoding mechanisms.

Pretrained language models such as \textbf{BERT} \citepx{devlin2019bert} and the \textbf{GPT} series \citepx{radford2019language,gpt3} demonstrated that unsupervised language modeling can yield general-purpose representations suitable for a wide array of downstream tasks. These models introduced powerful notions of \textbf{contextualized meaning} and \textbf{emergent reasoning}—concepts deeply relevant for modeling decision-making styles, where choices are embedded in linguistic, social, or narrative contexts.

\subsubsection{Generative Modeling: From Recognition to Creation}

The advent of generative models extended deep learning from analysis to synthesis. Starting with \textbf{Variational Autoencoders (VAEs)} and \textbf{Generative Adversarial Networks (GANs)} \citepx{gan}, and later evolving into \textbf{diffusion models} such as \textbf{Stable Diffusion} and \textbf{DALL·E}, these techniques enabled machines to create high-quality images, text, and hybrid content from scratch.

In this light, playstyle is no longer confined to \textit{reactive behavior} but encompasses \textit{creative generation}—whether in trajectories, strategies, or symbolic representations. Generative modeling shifts the focus from performance optimization to stylistic expression.

\subsubsection{Toward Multimodal Foundation Models}

As these modalities matured, their boundaries began to dissolve. Models like \textbf{CLIP} \citepx{clip} aligned vision and language into shared embedding spaces. Systems such as \textbf{Gato} \citepx{reed2022generalist} and \textbf{GPT-4} explored generalist architectures capable of processing images, text, actions, and audio through a unified interface.

These foundation models do not merely merge perception and action—they redefine the agent itself as a multimodal entity. This convergence is central to the future of playstyle modeling, enabling agents to exhibit stylistic coherence across perception, language, and decision-making—whether in games, dialogues, or real-world environments.

\subsection{The Public Face of AI}

While the technical milestones of deep learning and reinforcement learning have largely unfolded in research labs and benchmark competitions, the most transformative shift in recent years has taken place in the \textbf{public sphere}. For the first time in history, \textbf{non-expert users} engage with advanced AI systems not merely as passive recipients of backend services, but as \textbf{active co-creators, decision-makers, and evaluators}.

\subsubsection{Generative AI as Interface}

Large language models such as \textbf{ChatGPT}, \textbf{Claude}, and \textbf{Gemini}, alongside image generation systems like \textbf{DALL·E}, \textbf{Midjourney}, and \textbf{Stable Diffusion}, have fundamentally redefined how AI is perceived and used. These models are not evaluated solely on accuracy or task completion, but on their capacity for \textbf{engagement, creativity, and interactivity}—in short, for \textit{style}.

They represent the \textbf{first mainstream AI agents where “style” is the product}. Users engage with them through iterative prompting, refinement, and preference-based selection. The feedback loop—whether via explicit signals like thumbs-up or implicit ones like re-prompts and retention—drives models toward stylistic alignment at scale. These interactions convert preference into performance, turning generative AI into a \textbf{style-adaptive interface}.

\subsubsection{Everyday AI Applications}

Beyond creative tools, AI systems have permeated routine workflows: \textbf{Whisper} for speech recognition, \textbf{Copilot} and \textbf{Cursor} for code generation, \textbf{Runway} for video editing, and \textbf{Synthesia} for avatar-based media creation. Even when powered by the same underlying models, these applications exhibit \textbf{individualized response patterns} and allow users to customize outputs by tone, format, pacing, or aesthetic.

Unlike benchmark agents such as AlphaZero, which are judged by win rate or sample efficiency, these public-facing systems are assessed by \textbf{diversity of expression}, \textbf{appropriateness under ambiguity}, and \textbf{subjective value}. As such, they exemplify a shift from optimal policy learning to \textbf{style-aware response modeling}—highlighting the increasing centrality of playstyle as an evaluation criterion.

\subsubsection{A Cultural Reframing}

This democratization of AI use reframes the role of artificial intelligence: not just as optimization machinery, but as a \textbf{cultural participant}. Where the narrative once celebrated \textit{superhuman performance}, it now turns toward \textit{human-compatible expression}. In this new landscape, AI must not only act effectively but also express meaningfully—setting the stage for the central questions of the next chapter.

\subsection{What Comes Next? Superhuman Performance or Stylistic Diversity?}

The trajectory of AI thus far—spanning deep learning, reinforcement learning, and generative modeling—has delivered \textbf{unprecedented capabilities}. From mastering complex games to composing music and writing essays, AI systems now routinely outperform human experts in structured domains.

But this success prompts a deeper question: \textbf{What should we optimize for next?}

If the last decade was defined by performance—by benchmarks, win rates, and score maximization—then the next may be defined by \textbf{style and diversity}. As agents move beyond closed, rule-bound settings and enter open-ended, human-centered environments, the objective is no longer to find “the best move,” but to support \textbf{a spectrum of valid, meaningful, or expressive behaviors}.

In this paradigm:
\begin{enumerate}
\item The same environment may support multiple viable styles, not just one optimal policy.
\item Evaluation must incorporate subjective factors—preference, fairness, interpretation—alongside traditional metrics.
\item Agent modeling must account for ambiguity in goals and beliefs, not just uncertainty in actions or outcomes.
\end{enumerate}

This shift reframes the very notion of intelligence—from optimizing expected return to \textbf{navigating diverse futures shaped by values, perspectives, and intent}. It brings to the forefront the need to study not only what agents can do, but \textit{how} and \textit{why} they choose to do it.

And this leads us to the core concern of this dissertation: \textbf{Playstyle}. Not merely a behavioral artifact, but a conceptual lens for exploring intelligence through the lens of decision variation, expressive freedom, and intentional design.

\section{Review on Deep Reinforcement Learning Milestones}
\label{sec:drl_milestones}

From solving Atari games to defeating world champions in Go, deep reinforcement learning (DRL) has emerged as one of the most visible and transformative branches of modern AI. Yet its significance extends well beyond benchmark scores or leaderboard rankings. In the context of this dissertation, DRL is especially noteworthy for its capacity to support the emergence of \textbf{playstyle diversity}, \textbf{adaptive decision strategies}, and \textbf{agent individuality}.

This section revisits key milestones in the development of DRL—from its formal foundation in Markov decision processes (MDPs) to state-of-the-art systems capable of generalization and self-play. Rather than presenting these methods solely in terms of performance, we highlight how each milestone contributed new mechanisms for shaping agent behavior—through exploration, memory, architecture, or optimization dynamics.

In doing so, we reinterpret the growth of DRL not merely as a rise in capability, but as a \textbf{growth in behavioral diversity}. This perspective becomes especially vivid when analyzing Atari benchmarks, where stylistic and strategic variation—rather than just aggregate scores—reveal the expressive power of different approaches.

\subsection{Formal Foundations: Markov Decision Processes and Bellman Equations}

The formal framework underlying most reinforcement learning algorithms is the Markov Decision Process (MDP), defined as a tuple \citepx{rl_book}:
\[
\mathcal{M} = (\mathcal{S}, \mathcal{A}, P, \mathcal{R}, \gamma),
\]
where:
\begin{itemize}
    \item $\mathcal{S}$ is the state space,
    \item $\mathcal{A}$ is the action space,
    \item $P(s'|s, a)$ is the transition function (i.e., the probability of moving to state $s'$ from $s$ by taking action $a$),
    \item $\mathcal{R}(s, a)$ is the reward function,
    \item $\gamma \in [0, 1]$ is the discount factor.
\end{itemize}

The agent's goal is to learn a policy $\pi(a|s)$ that maximizes the expected cumulative discounted reward:
\[
\mathbb{E}\left[ \sum_{t=0}^\infty \gamma^t r_t \right].
\]

The optimal value function $V^*(s)$, which gives the maximum expected return from state $s$, satisfies the \textbf{Bellman Optimality Equation}:
\[
V^*(s) = \max_a \left[ \mathcal{R}(s, a) + \gamma \sum_{s'} P(s'|s, a) V^*(s') \right].
\]

Similarly, the action-value function $Q^*(s, a)$ satisfies:
\[
Q^*(s, a) = \mathcal{R}(s, a) + \gamma \sum_{s'} P(s'|s, a) \max_{a'} Q^*(s', a').
\]

These value functions—$V$ and $Q$—can be interpreted as \textbf{utility functions} or \textbf{evaluation functions} that assign desirability to states or state-action pairs. In model-free reinforcement learning—where $P$ and $\mathcal{R}$ are unknown—these functions are estimated from interaction data using Monte Carlo methods, Temporal Difference (TD) learning, or deep function approximators.

Crucially, this formalism allows for highly flexible agent behavior. Depending on how the reward is structured and how uncertainty is resolved, different agents may solve the same task in stylistically different ways. Thus, the MDP not only encodes the objective of learning, but also bounds the \textit{space of expressible playstyles}.

\subsection{The Atari Benchmark and the Rise of DQN}

The Arcade Learning Environment (ALE) and its Atari 2600 games have played a foundational role in the development and evaluation of deep reinforcement learning (DRL). This benchmark suite provides a diverse set of discrete, high-dimensional, partially observable tasks, spanning from reactive control to long-horizon planning, within a unified pixel-based interface.

The breakthrough came with \textbf{Deep Q-Networks (DQN)} \citepx{atari_net, dqn}, which for the first time demonstrated that an agent could learn to play multiple Atari games directly from pixels using only reward feedback. The combination of convolutional networks, experience replay, and Q-learning introduced a scalable learning pipeline and sparked a wave of DRL research.

\vspace{1em}
\paragraph{Normalized Performance Metrics.}
To compare across games with different score scales and dynamics, several normalized metrics were proposed:

\begin{itemize}
    \item \textbf{Human Normalized Score (HNS)} compares agent performance to that of a human given two hours of gameplay:
    \[
    \text{HNS} = \frac{\text{agent} - \text{random}}{\text{human} - \text{random}} \times 100\%
    \]

    \item \textbf{Human World Record Normalized Score (HWRNS)} benchmarks agents against the best-known human scores:
    \[
    \text{HWRNS} = \frac{\text{agent} - \text{random}}{\text{world record} - \text{random}} \times 100\%
    \]

    \item \textbf{SABER (Standardized Atari BEchmark for RL)} \citepx{toromanoff2019deep} addresses score inflation and outlier sensitivity by capping values and excluding noisy games:
    \[
    \text{SABER} = \max\left( \min\left( \text{HWRNS}, 200\% \right), 0\% \right)
    \]
\end{itemize}

\vspace{1em}
\paragraph{Progress and Pitfalls.}
Since DQN, numerous architectural and algorithmic innovations have emerged: from recurrent memory-based agents (e.g., DRQN), to distributed training (e.g., Gorila, IMPALA), prioritized replay, distributional Q-learning (e.g., C51, QR-DQN, IQN), exploration-driven methods (e.g., NGU, Go-Explore), and powerful hybrid agents like \textbf{Agent57} \citepx{agent57}. More recently, \textbf{MuZero} \citepx{schrittwieser2020muzero} demonstrated AlphaZero-like learning without access to ground-truth dynamics during the internal planning.

However, these algorithms’ reported performance varies widely depending on which metric is used:

\begin{figure}[ht]
  \centering
  \includegraphics[width=0.85\linewidth]{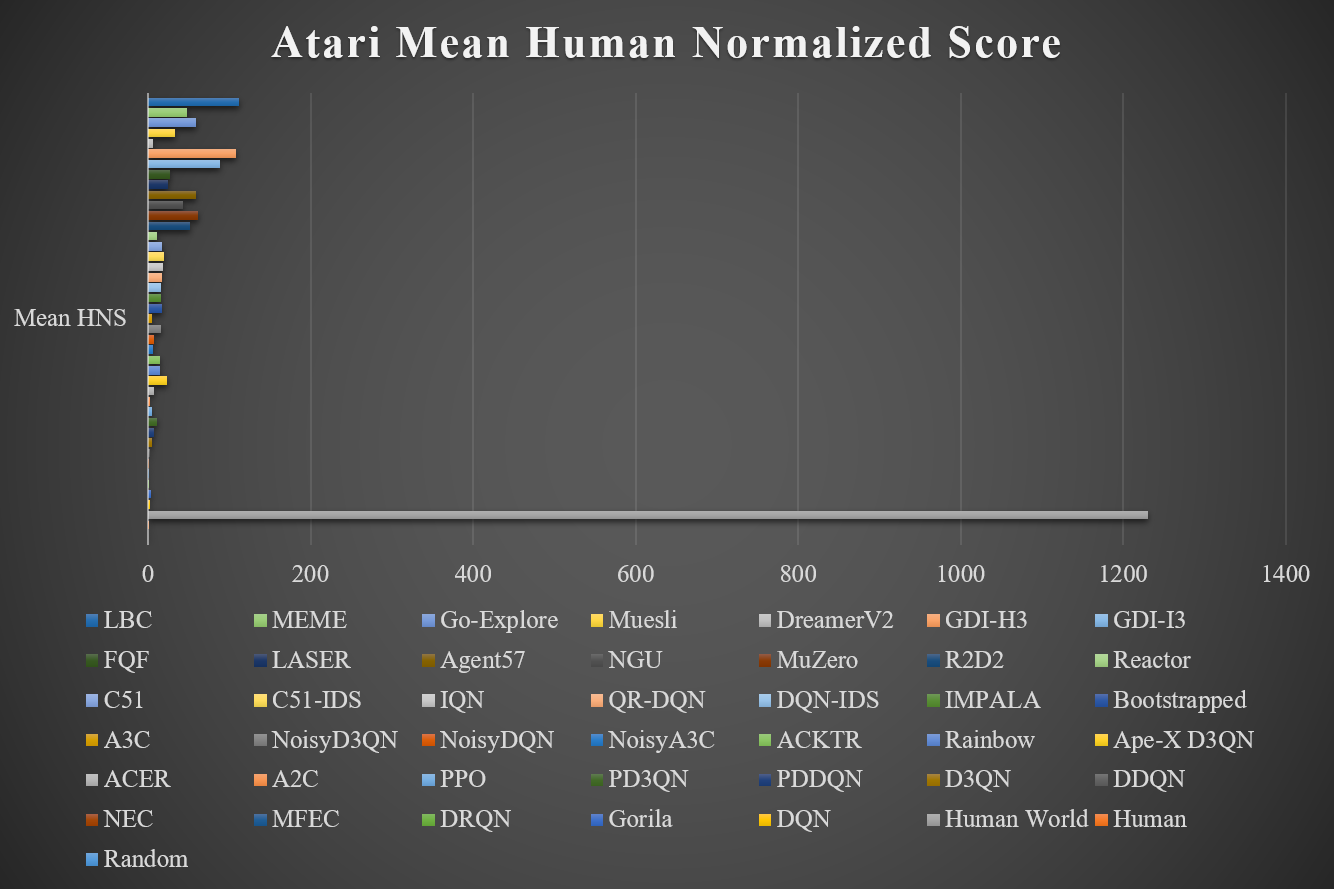}
  \caption[Atari Mean HNS]{Mean Human Normalized Scores (HNS) show some agents, like \textbf{LBC}, achieving >10000\% average performance, far exceeding the 2-hour human baseline but still has a big performance gap to human world records.}
  \label{fig:hns-mean}
\end{figure}

\begin{figure}[ht]
  \centering
  \includegraphics[width=0.85\linewidth]{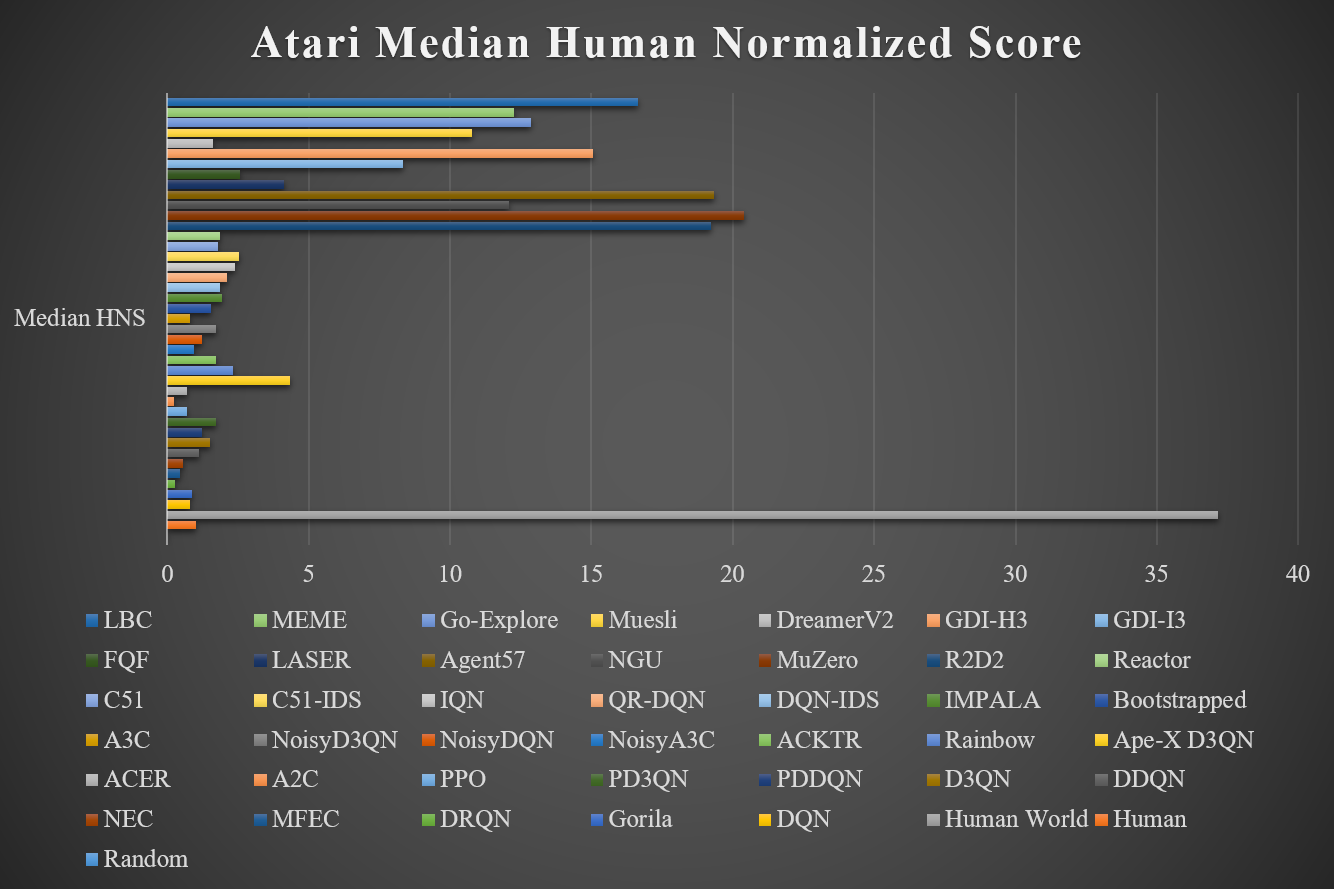}
  \caption[Atari Median HNS]{Median HNS provides a more robust view by downweighting score inflation from easy games. Agent57 shows strong performance here.}
  \label{fig:hns-median}
\end{figure}

\begin{figure}[ht]
  \centering
  \includegraphics[width=0.85\linewidth]{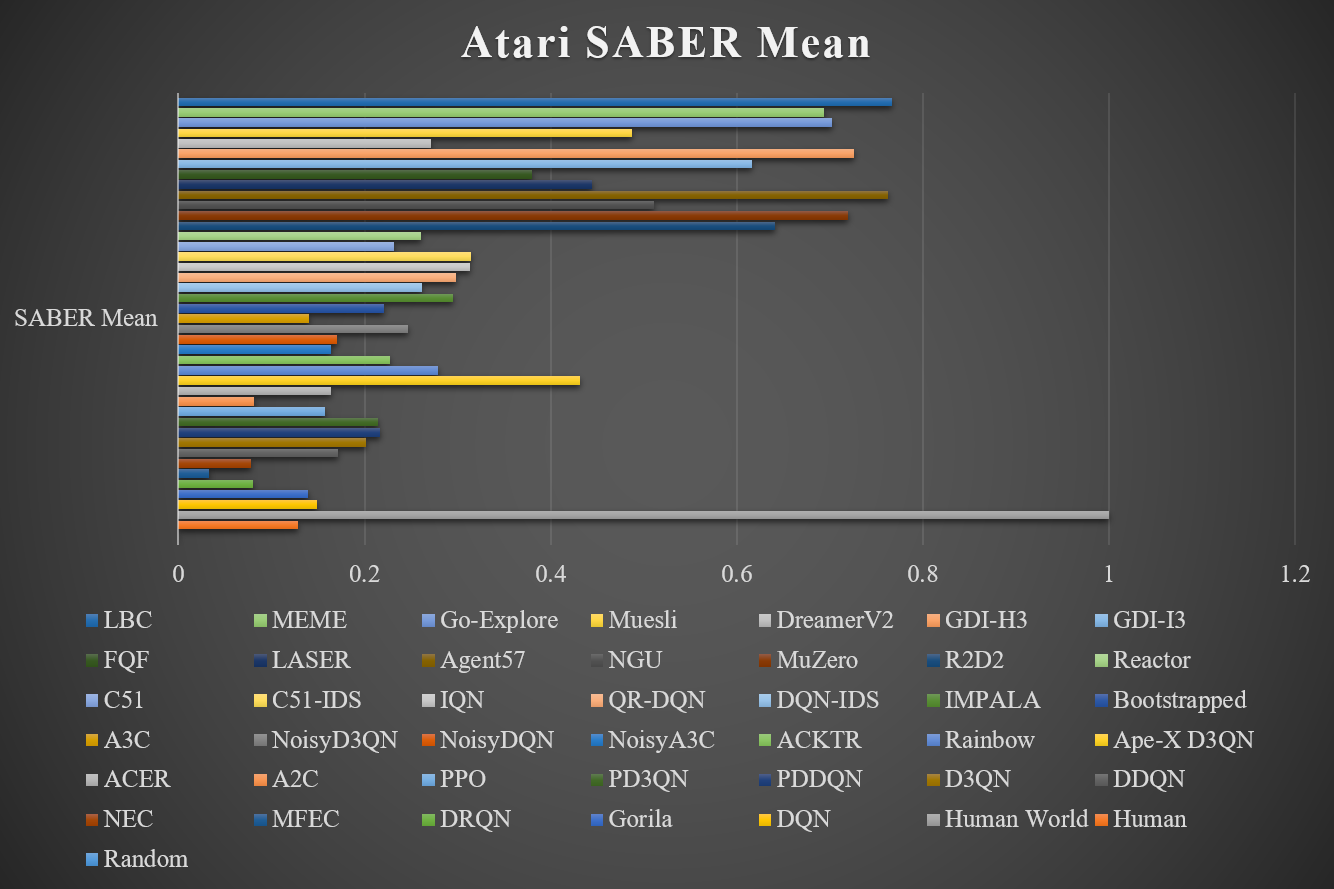}
  \caption[Atari SABER Mean]{SABER Mean scores cap extreme outliers. \textbf{LBC} (2023) achieves 76.67\%, with \textbf{Agent57} close behind at 76.26\%.}
  \label{fig:saber-mean}
\end{figure}

\begin{figure}[ht]
  \centering
  \includegraphics[width=0.85\linewidth]{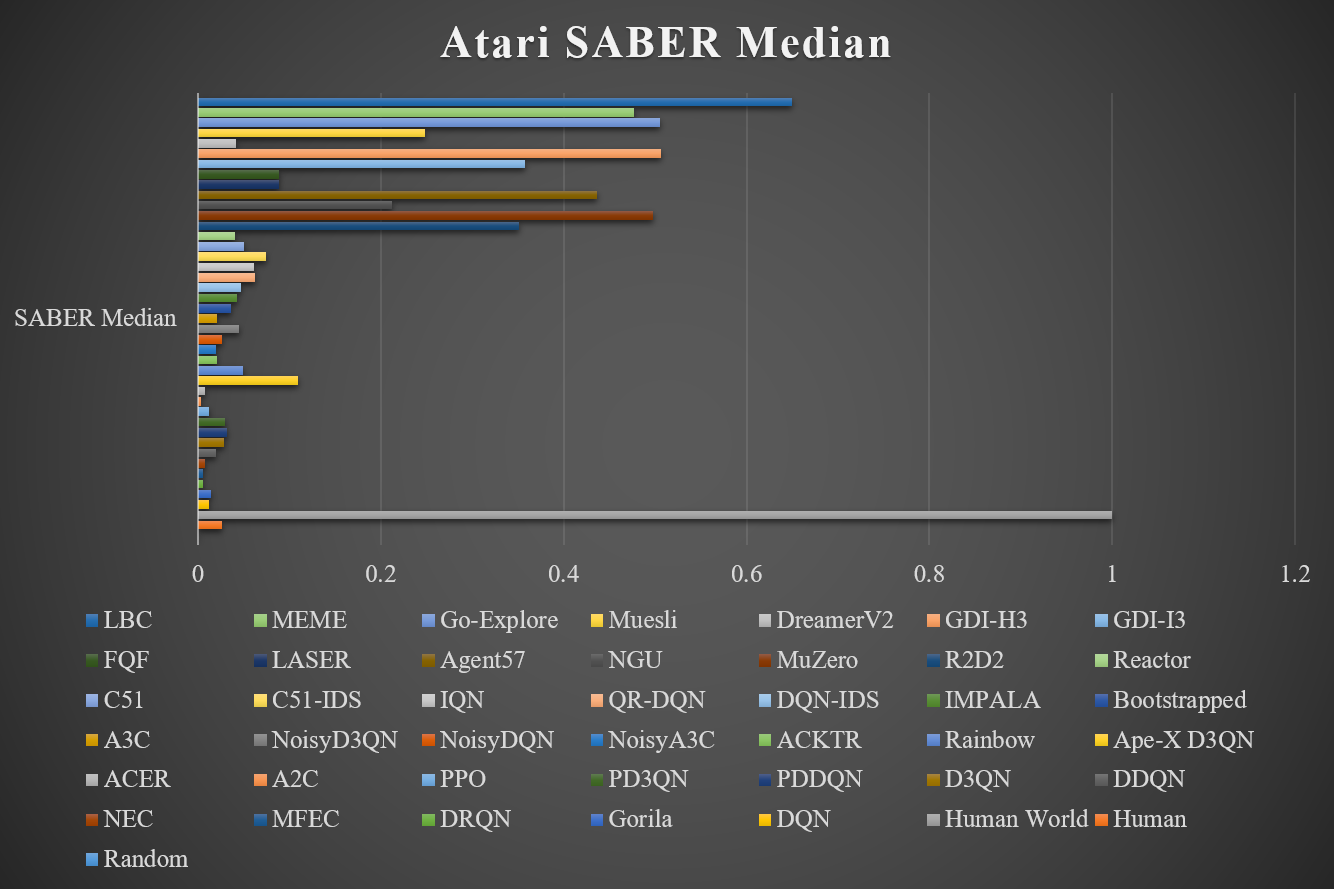}
  \caption[Atari SABER Median]{SABER Median reveals stronger generality. LBC remains state-of-the-art at 64.93\%, while Agent57 drops to 43.62\%.}
  \label{fig:saber-median}
\end{figure}

\vspace{1em}
\paragraph{Revisiting "Superhuman."}
Despite headlines claiming DRL has “surpassed human performance,” this is only true when measured against the short-time baseline (HNS). When we use \textbf{world records} as the gold standard, no agent has yet surpassed even 80\% of human-best scores on average. For instance:

\begin{itemize}
    \item \textbf{HNS Mean:} Human World Record = \textbf{123000\%}; LBC = \textbf{11236\%}
    \item \textbf{SABER Mean:} Human World Record = \textbf{100\%}; LBC = \textbf{76.67\%}
    \item \textbf{SABER Median:} LBC = \textbf{64.93\%}; Agent57 = \textbf{43.62\%}
\end{itemize}

This large gap highlights a crucial insight: while DRL agents are now competent across all games, they still fall far short of \textbf{mastery}, especially when measured against the diverse strategies and rare corner cases that characterize world-class human play. This shortfall in both average and robust performance opens the door to a deeper question—one we explore in the next section: \textbf{Is achieving superhuman scores enough, or should we aim for stylistic generality and diverse behavior?}

\vspace{1em}
\paragraph{Atari as a Playground for Style-Oriented Methods.}
The Atari benchmark has not only served as a proving ground for performance improvements in deep reinforcement learning, but also as a fertile testbed for studying exploration, behavior shaping, and emergent diversity. A wide range of DRL algorithms—spanning from classic baselines to state-of-the-art architectures—have been evaluated on this platform. Among them, some were specifically designed to promote diverse agent behaviors, or inadvertently contributed to playstyle variation through their architecture or training dynamics.

\paragraph{Notable Methods and Their Stylistic Implications.}
\begin{itemize}
  \item \textbf{DQN} (2015/2/25) \citepx{dqn}: The foundational DRL algorithm that enabled end-to-end value-based learning from pixels. Still widely used in cross-domain applications due to its simplicity and recognition, though its $\varepsilon$-greedy exploration leads to limited behavioral diversity.
  
  \item \textbf{A2C/A3C} (2016/2/4) \citepx{a3c}: Early distributed actor-critic methods that achieved strong parallelized training. While they collected diverse experiences through distributed actors, their objectives did not explicitly promote policy diversity.
  
  \item \textbf{Bootstrapped DQN} (2016/2/15) \citepx{bootstrap_dqn}: Replaced $\varepsilon$-greedy with ensemble-based deep exploration using multi-head value functions. The ensemble mechanism naturally induces distinct behaviors across heads, making it one of the earliest practical pathways to generating multiple playstyles.
  
  \item \textbf{NoisyNet} (2017/6/30) \citepx{noisy_net}: Introduced trainable, state-conditioned noise into network weights to replace $\varepsilon$-greedy. Included in Rainbow, though often omitted in practice due to complexity. While it improves exploration, it does not explicitly target diversity.
  
  \item \textbf{PPO} (2017/7/20) \citepx{ppo}: Currently the most popular DRL algorithm due to its stability and simplicity. Its entropy regularization helps preserve behavioral variability during training, offering a basic mechanism for maintaining playstyle diversity.
  
  \item \textbf{Rainbow} (2017/10/6) \citepx{rainbow}: A composite of DQN enhancements, often used as a baseline for value-based methods. Despite strong performance, it does not specifically promote diversity and remains focused on sample efficiency.
  
  \item \textbf{IQN} (2018/1/14) \citepx{iqn}: A distributional RL algorithm that parameterizes quantile functions, enabling risk-sensitive decision-making. This allows agents to express different risk-oriented playstyles under identical tasks.
  
  \item \textbf{IMPALA} (2018/2/5) \citepx{impala}: Introduced V-trace to stabilize off-policy actor-critic training in distributed setups. It laid the foundation for scalable DRL but did not directly address style variation.
  
  \item \textbf{Ape-X} (2018/3/2) \citepx{ape_x}: Combined prioritized replay with massive parallel exploration, leading to a significant leap in DRL performance. Different actors experience varying degrees of exploration, indirectly enabling the emergence of diverse playstyles.
  
  \item \textbf{DQN-IDS} (2018/12/18) \citepx{ids_exploration}: Integrated uncertainty estimates from distributional value heads to drive exploration. By incorporating epistemic uncertainty into decision-making, it encouraged agents to exhibit more diverse behavior patterns.
  
  \item \textbf{NGU (Never Give Up)} (2020/2/14) \citepx{ngu}: Achieved major gains in hard-exploration tasks via episodic memory and intrinsic reward. Each actor had distinct exploration parameters, explicitly promoting diversity across policies.
  
  \item \textbf{Agent57} (2020/3/30) \citepx{agent57}: Built upon NGU by adding a bandit-based meta-controller to allocate resources to better-performing strategies. This can be interpreted as a mechanism for selecting and deploying effective playstyles in a dynamic way.
  
  \item \textbf{GDI} (2022/6/7) \citepx{gdi}: Reframed diversity as a core ingredient in achieving strong generalization. Proposed a principled method to induce behavioral variation while improving performance.

\item \textbf{LBC} (2023/5/9) \citepx{lbc}: A direct successor to GDI and the current state-of-the-art in mean Atari performance. While not explicitly targeting playstyle, it benefits from diversity-driven training.
\end{itemize}

\subsection{Beyond Scores: From Atari Performance to Strategic Diversity}

While metrics like Human Normalized Score (HNS) and SABER have become standard tools for evaluating deep RL agents on benchmarks like Atari, they tell only part of the story. These metrics quantify \textit{how well} an agent performs—relative to random, human, or world-class baselines—but not \textit{how} that performance is achieved.

\paragraph{Why Style Matters.}

Two agents may obtain the same score using radically different approaches: one through aggressive, high-risk tactics, another via conservative, methodical play. This distinction—often invisible to scalar metrics—reflects an agent’s \textbf{playstyle}: its preferred patterns of behavior, risk, adaptation, and interaction.

In fact, as RL systems matured, it became increasingly clear that performance alone was insufficient. Robust generalization, human alignment, and multi-agent coordination all demand something more: \textit{behavioral diversity}.

\paragraph{Exploration and Learning Efficiency.}

Early agents like DQN relied on $\epsilon$-greedy exploration, which often led to local optima. Subsequent methods introduced more principled forms of exploration: intrinsic curiosity, episodic novelty, and random network distillation. These methods encourage agents to seek novelty, enabling the discovery of creative or human-like solutions beyond score maximization.

\paragraph{Memory and Meta-Strategy.}

Deep RL algorithms also evolved to support longer-term planning and context-sensitive decision-making. Recurrent policies like DRQN and R2D2 added memory; population-based methods like PBT or Agent57 introduced meta-controllers to adapt exploration dynamically. These architectures began to mirror not just intelligent behavior, but \textit{adaptive, personalized behavior}—a prerequisite for style.

\paragraph{Multi-Agent Systems: A Natural Habitat for Style.}

Nowhere is playstyle more essential than in \textbf{multi-agent systems} (MAS). In competitive games (e.g., StarCraft, Dota), agents must learn not just to optimize, but to \textit{counter}—anticipating opponents’ behaviors, developing diverse strategies, and shifting play based on population dynamics.

This transition from single-agent to multi-agent RL marked a shift in focus: from learning a single optimal policy, to supporting a \textit{space of policies}, each adapted to different contexts, roles, or teammates. Diversity became not only beneficial—it became \textbf{necessary} for success.

\begin{quote}
\textit{Style is not a luxury in multi-agent RL. It is a strategic imperative.}
\end{quote}

This insight underpins systems like:
\begin{itemize}
  \item \textbf{AlphaStar} \citepx{alpha_star}, which learns a portfolio of strategies through population-based training and exploits Nash-style equilibria;
  \item \textbf{OpenAI Five} \citepx{openai_five}, which uses large-scale self-play to discover high-level coordination and adaptation;
  \item \textbf{Agent57} \citepx{agent57}, which includes meta-learning to dynamically vary behavior across games and difficulty settings;
  \item \textbf{Hide and Seek} \citepx{hide_and_seek}, which demonstrates that stylistic strategies—including tool use, countermeasures, and dynamic adaptations—can emerge purely from multi-agent interaction without explicit supervision or structured goal engineering.
\end{itemize}

\paragraph{From Scores to Style-Based Evaluation.}

As we move toward richer, more realistic environments, the evaluation of RL agents must evolve accordingly. Future benchmarks will likely measure not just cumulative rewards, but:
\begin{itemize}
  \item behavioral diversity across runs and settings,
  \item robustness to opponent variation,
  \item interpretability of strategic intent,
  \item alignment with human-like reasoning.
\end{itemize}

In this light, playstyle is not a byproduct—it is the next frontier of decision-making AI. Multi-agent self-play environments like \textbf{Hide and Seek} further highlight that even with minimal reward structure, diverse strategic behaviors—including cooperation, deception, and tool use—can emerge naturally as a consequence of stylistic adaptation and environmental pressure \citepx{hide_and_seek}.

\chapter{Playstyle Imitation and Human-like Agents}
\noindent\textbf{Key Question of the Chapter:} \\
\textit{What does it take to behave like a human? And how can we achieve it?}

\noindent\textbf{Brief Answer:} 
To give AI a playstyle, a natural starting point is the human playstyle, since humans are the most salient example of intelligent decision-making. Imitation learning offers many methods for replicating such styles. However, defining and measuring “human-likeness” is more complex—and arguably more important—than the choice of imitation technique itself. As fully positive definitions of human-likeness are elusive, we approach it from the negative side: eliminating behaviors that are clearly non-human.

\bigskip

What does it mean for an agent to exhibit a playstyle—not just a policy that achieves results, but one that reflects a particular behavioral direction? In the previous chapter, we explored how agents can learn rational behaviors by maximizing performance under well-defined objectives and reinforcement signals. In such settings, playstyle often emerges as a side effect—not as an explicit design goal. 

However, once we begin to care about \textit{how} behavior is expressed—whether it is cautious or aggressive, elegant or chaotic—we must shift our focus. What does it mean to act in a particular way, and how can we define the direction of a playstyle?

A natural answer is to start with humans. Imitating human behavior offers not only an intuitive grounding for style but also speaks to one of the foundational goals of artificial intelligence: creating agents that act, respond, and adapt in ways that align with human understanding. In this sense, imitation becomes more than a learning technique—it becomes the bridge between stylistic expression and human-likeness, between engineered policies and natural interaction.

In this chapter, we explore this connection in three stages. First, we examine why imitation plays a central role in modeling playstyle. Then, we introduce key imitation learning methods and explain the core ideas behind each—going beyond a simple enumeration to highlight their structural assumptions and stylistic implications. Finally, we shift focus to a deeper question: what does it mean to be human-like? We examine different perspectives on human-likeness—behavioral, perceptual, and philosophical—and how these perspectives shape both the evaluation criteria and practical directions for developing human-aligned agents.

\section{Why Imitation Matters in Playstyle?}

In previous chapters, we examined how artificial agents make decisions and how their behaviors can be measured and categorized using formal playstyle metrics. We also discussed how reinforcement learning and other optimization-based methods enable agents to achieve high performance by maximizing quantifiable objectives such as scores or win rates. However, when the goal shifts from merely reaching a target score to expressing a particular style, these approaches often fall short. Differences in learned behavior are typically treated as training noise or underfitting—seldom are they recognized as meaningful stylistic variation.

This limitation stems from a deeper assumption: that style is secondary to success. Most training pipelines optimize a single objective function, leading to behavioral convergence around the reward-defined optimum. Although manual reward shaping can induce behavioral variation, such approaches lack semantic grounding and are difficult to scale. In contrast, imitation learning offers a more natural and grounded alternative. Human players inherently exhibit diverse playstyles—shaped by experience, intention, and personality. Learning from demonstrations allows agents to inherit not only effective strategies, but also the implicit stylistic traits embedded in human behavior.

Imitation also serves purposes beyond stylistic alignment. It can guide exploration, constrain policies to remain within human-consistent regimes, or provide a foundation for inverse objectives—training agents to purposefully diverge from human norms. Perhaps most importantly, imitation enables \textbf{structured deviation}: agents can extrapolate from demonstrations to produce new, functional, and stylistically coherent behaviors. In this way, imitation is not the endpoint, but rather the foundation for stylistic generalization and creative expression.

\section{Imitation Learning Approaches}

To learn expressive playstyles—not merely functional behavior—agents must move beyond score maximization and instead capture behavioral tendencies that reflect stylistic intent or human-like qualities. \textbf{Imitation Learning (IL)} offers a natural paradigm for this goal. By learning from expert demonstrations—rather than from manually engineered rewards or task-specific performance objectives—agents can acquire not only effective decision patterns, but also the expressive stylistic traits embedded in expert behavior.

Imitation learning has since evolved into a broad and heterogeneous field, with methods differing in how they interpret demonstrations and in the role imitation plays in the learning process. While many existing surveys categorize IL methods by supervision level or reward dependence, we adopt a \textit{conceptual} classification based on how the agent relates to the expert signal and what function imitation serves within the overall framework.

As illustrated in Figure~\ref{fig:il-taxonomy}, we identify four high-level paradigms:

\begin{itemize}
    \item \textbf{Direct Imitation} — typified by Behavioral Cloning (BC), where the agent learns a policy by directly mapping states to expert actions through supervised learning.
    \item \textbf{Indirect Imitation} — represented by Inverse Reinforcement Learning (IRL) and its variants, where the objective is to infer the reward signal underlying the expert’s behavior and then optimize against it. Generative Adversarial Imitation Learning (GAIL) and Adversarial Inverse Reinforcement Learning (AIRL) fall primarily under this category.
    \item \textbf{Auxiliary Imitation} — approaches such as Deep Q-learning from Demonstrations (DQfD) or hybrid pretraining methods, where imitation provides auxiliary supervision or regularization, rather than being the primary training signal. Certain GAIL variants may also fall here, depending on how the adversarial structure is integrated.
    \item \textbf{Reinterpreted Imitation} — including methods like Imitation from Observation (IfO), where agents must learn from partial or indirect expert signals—such as state-only sequences without action labels—requiring inference over intent and behavior structure.
\end{itemize}

Beyond these paradigms, recent advances in generative modeling—especially \textit{diffusion-based policies}—have introduced powerful modeling strategies that cut across traditional categories. Diffusion models represent policy generation as a stochastic denoising process, enabling one-to-many mappings, long-horizon coherence, and improved handling of multimodal behavior. While they do not constitute a separate IL paradigm, they enhance existing ones: increasing expressivity in BC, improving reward recovery in IRL, and enabling stylistic diversity across frameworks.

\begin{figure}[h]
  \centering
  \includegraphics[width=0.75\textwidth]{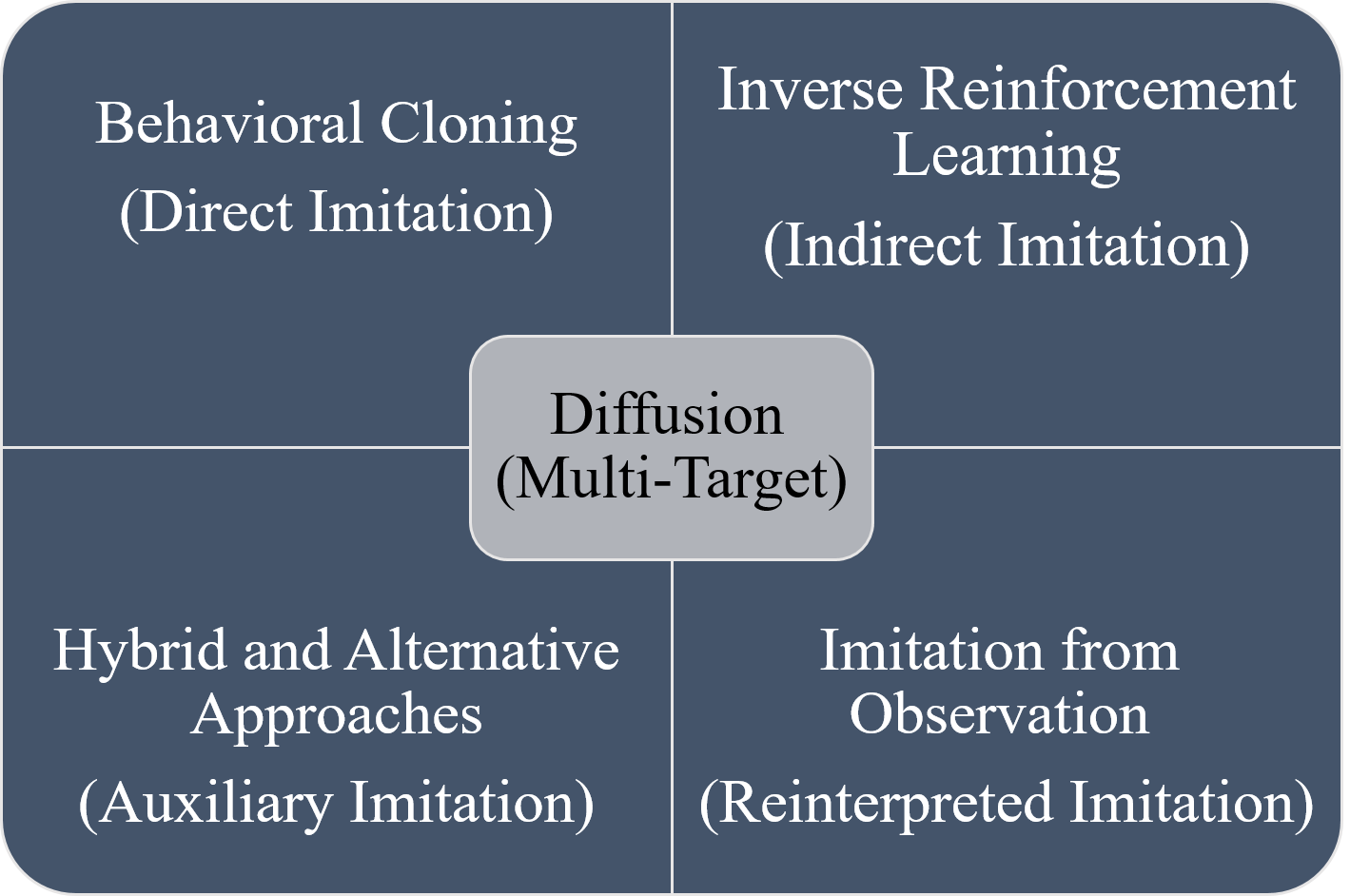}
  \caption[A Conceptual Taxonomy of Imitation Learning Paradigms]{
    A conceptual taxonomy of imitation learning paradigms, categorized by the agent’s relationship to expert demonstrations and the functional role of imitation. Diffusion-based modeling (center) serves as a cross-paradigm strategy, enabling multi-target generation, long-horizon consistency, and stylistic expressivity. Note: GAIL is primarily aligned with IRL despite its adversarial structure; some variants may fall under auxiliary imitation depending on usage.
  }
  \label{fig:il-taxonomy}
\end{figure}

In the following subsections, we examine each paradigm in greater detail—beginning with direct imitation via behavioral cloning.

\subsection{Behavioral Cloning}

Behavioral Cloning (BC) is the most direct and foundational approach within imitation learning. In this paradigm, the task of learning a policy is cast as a standard supervised learning problem~\citepx{pomerleau1988alvinn, sammut1992flying}. Given a dataset of expert demonstrations, the agent learns a policy $\pi: s \mapsto a$ that maps states $s$ to actions $a$ by minimizing a prediction loss—typically cross-entropy or mean squared error—between the expert's actions and the agent's predicted actions. This simple framework has been widely applied in domains such as robotics~\citepx{bojarski2016end}, gaming~\citepx{kanervisto2020benchmarking}, and autonomous driving.

Despite its simplicity and computational efficiency, BC suffers from a well-known issue called \textit{covariate shift} or \textit{compounding error}~\citepx{ross2010reduction}. Because the policy is only trained on expert trajectories, it has no experience with states resulting from its own errors. At test time, even small deviations can cause the agent to enter unfamiliar states, where prediction quality deteriorates, leading to further drift.

To mitigate this, interactive variants such as \textbf{DAgger} (Dataset Aggregation) have been proposed~\citepx{ross2010reduction}. In DAgger, the agent executes its own policy to collect new data and simultaneously queries the expert for corrective labels, thus gradually expanding the training distribution to include off-policy states. Other improvements include regularization-based methods~\citepx{brantley2020disagreement}, latent space perturbations~\citepx{chang2021mitigating}, and generative models that simulate alternative trajectories.

Recent advances also include \textbf{Implicit BC}~\citepx{florence2021implicit}, which avoids explicit action prediction. Instead, it models an energy-based scoring function over action candidates, assigning higher likelihood to actions consistent with the expert distribution. This enables richer modeling of multimodal or ambiguous behaviors.

Although BC in its pure form may be brittle in complex or long-horizon tasks, it remains a valuable foundation—often used in hybrid strategies that incorporate reward signals, environment interaction, or diffusion-based trajectory generation~\citepx{chen2023diffusion}. Related approaches such as \textbf{PWIL}~\citepx{pwil}, which use distributional similarity for imitation, lie closer to adversarial or inverse reinforcement learning and will be discussed in later sections.

\subsection{Inverse Reinforcement Learning (IRL) and Its Extensions}

Inverse Reinforcement Learning (IRL) aims to recover a reward function from expert demonstrations, under the assumption that expert behavior is (near-)optimal with respect to some unknown objective. In contrast to Behavioral Cloning, which directly learns a mapping from state to action, IRL attempts to infer the latent reward structure that best explains the observed behavior, and then uses reinforcement learning to derive a policy that optimizes it~\citepx{irl_implement, russell2020aima}.

This formulation is appealing in settings where reward design is difficult, or when capturing latent behavioral intent is more meaningful than mimicking surface-level actions. However, IRL is often ill-posed—many reward functions can explain the same behavior—and computationally intensive, as it typically involves solving an inner RL loop during reward inference. A number of IRL approaches have been developed to address these challenges, including maximum entropy formulations, Bayesian inference, and adversarial learning-based methods.

\paragraph{Maximum Entropy IRL and Deep IRL.}
To resolve the ambiguity inherent in IRL, \citetx{max_entropy} proposed Maximum Entropy IRL, which models the expert as a stochastic policy that maximizes entropy while remaining consistent with observed trajectories. This encourages more randomized yet plausible behaviors rather than overly deterministic ones.

Deep variants of MaxEnt IRL, such as those by \citetx{wulfmeier2015maxentdeep} and \citetx{finn2016guided} (Guided Cost Learning), extend this idea to high-dimensional domains by learning nonlinear reward functions with deep networks.

\paragraph{Bayesian IRL.}
Bayesian IRL (BIRL)~\citepx{ramachandran2007bayesian} places a prior distribution over reward functions and computes the posterior given expert behavior. For example, Gaussian Process IRL~\citepx{levine2011gpirl} represents the reward as a sample from a Gaussian process, offering flexible nonparametric modeling. While these approaches offer principled uncertainty handling, they tend to scale poorly due to repeated policy optimization during inference.

\paragraph{Adversarial IRL and Extensions.}
Generative Adversarial Imitation Learning (GAIL)~\citepx{gail} formulates imitation as a distribution matching problem, using a discriminator to distinguish expert from agent trajectories and providing reward-like feedback. AIRL~\citepx{fu2018airl} extends this by factoring the discriminator into reward and value components, enabling learned rewards to transfer across tasks. InfoGAIL~\citepx{info_gail} adds latent codes for disentangling stylistic factors during imitation.

Recently, diffusion-based extensions have enhanced adversarial imitation. DRAIL~\citepx{lai2024drail} replaces GAIL's discriminator with a diffusion classifier, producing smoother and more stable rewards. DiffAIL~\citepx{wang2024diffail} integrates an unconditional diffusion model into the discriminator loss, enabling long-horizon, multimodal imitation.

\paragraph{Uncertainty-Aware and Prior-Weighted IRL.}
PWIL (Primal Wasserstein Imitation Learning)~\citepx{pwil} sidesteps adversarial instability by directly matching state-action distributions using the Wasserstein distance. IC-GAIL~\citepx{wu2019imitation} introduces a confidence-weighted approach for handling suboptimal or noisy demonstrations. These works mark a trend toward robust, sample-efficient imitation under realistic constraints.

\subsubsection{Case Study: Using Demonstration to Design Playstyle Reward with Discrete State}
\label{sec:imitation_reward_racing}

While most IRL methods recover rewards through inverse optimization or adversarial training, reward stability and interpretability remain difficult. In this case study, drawn from a collaboration with Ubitus K.K., our goal was to develop human-like racing bots using limited gameplay footage—without access to structured internal variables such as velocity or collisions.

Early efforts with Behavioral Cloning and GAIL failed to generalize across hazards (e.g., water pits), leading to poor lap completion. To address this, we extracted visual in-game variables such as speed and checkpoint number using a digit-recognition pipeline. A baseline Ape-X agent~\citetx{thesis_of_Tsai_Cheng_Lun} trained with speed-based reward improved lap times but still drove inefficiently—zigzagging on straights and overshooting corners.

To enhance style alignment, we developed a symbolic demo-reward mechanism: (1) training a discrete encoder to cluster image frames into symbolic states; (2) extracting expert states and state-action pairs from high-quality demos; and (3) using these sets to assign rewards during DRL training. If an agent visits an expert state, it receives a reward $r_s$; if it also chooses the matching expert action, it receives $r_{sa} > r_s$.

This method, later patented~\citepx{lin2025us12233343b2}, parallels successor-feature imitation~\citepx{barreto2017successor} and enables interpretable, structure-aware learning. As shown in Table~\ref{tab:demo_reward_results}, the symbolic reward led to better lap times and more human-like paths:

\begin{table}[ht]
\centering
\caption[Ubitus Racing Game Performance]{Performance of DRL agents trained with different reward functions.}
\label{tab:demo_reward_results}
\begin{tabular}{lcc}
\toprule
\textbf{Method} & \textbf{Avg. Speed (km/h)} & \textbf{Avg. Lap Time (s)} \\
\midrule
All Human Players         & 148.1 & 41.9 \\
Professional Player       & 170.4 & 35.4 \\
\midrule
DRL – Speed Reward        & 160.5 & 37.4 \\
\textbf{DRL – 6 Demo Reward}     & \textbf{168.0} & \textbf{36.2} \\
DRL – Bad Demo Reward     & 120.0 & 52.0 \\
\textbf{DRL – Best Demo Reward}  & \textbf{169.0} & \textbf{35.5} \\
\bottomrule
\end{tabular}
\end{table}

\begin{sidewaysfigure}
  \centering
  \includegraphics[width=0.95\textheight]{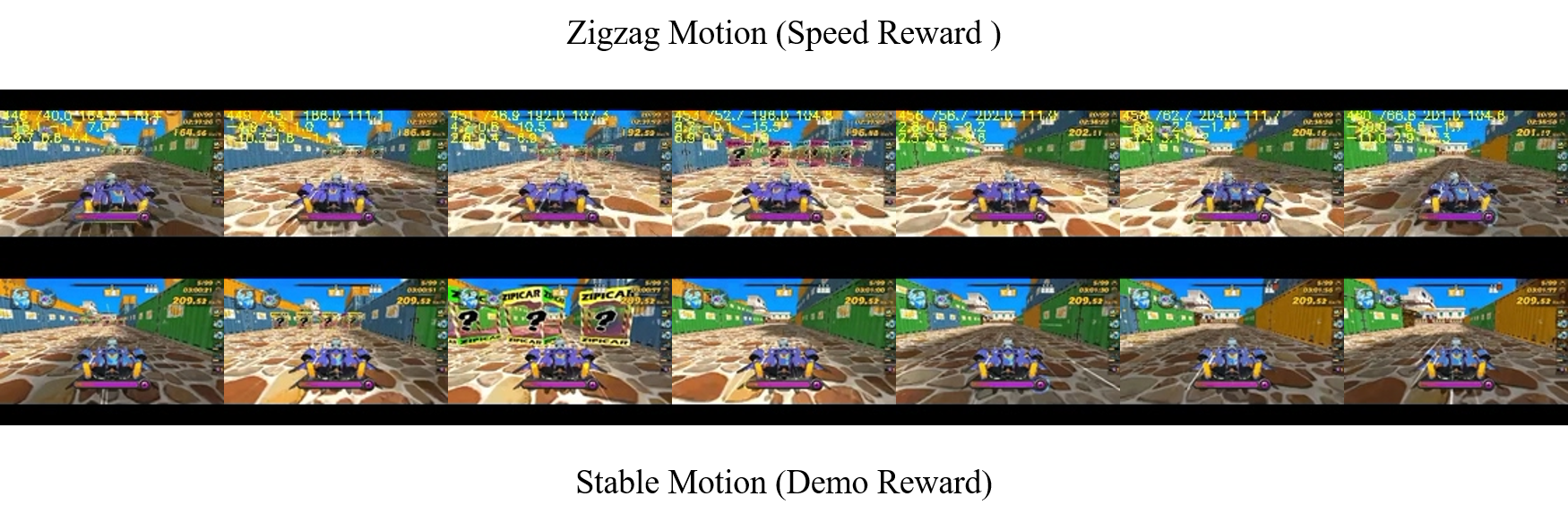}
  \vspace{0.5em}
  \includegraphics[width=0.95\textheight]{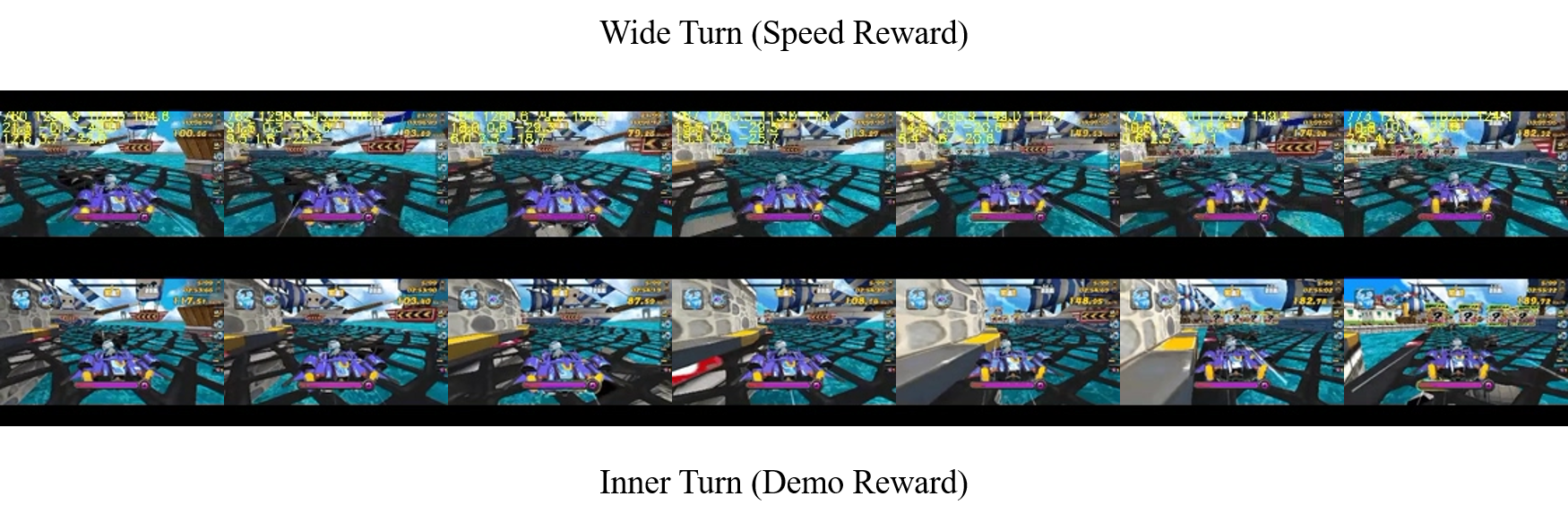}
  \caption[Human-Like Behaviors in Ubitus Racing Game]{
    Comparison of agent behavior trained with different reward functions. 
    \textbf{Top rows:} Speed-reward agents exhibit zigzag motion and wide turns. 
    \textbf{Bottom rows:} Demo-reward agents drive straight and take inner curves, reflecting human-like control.
  }
  \label{fig:demo_vs_speed}
\end{sidewaysfigure}

Even when trained on only two expert episodes, the symbolic-reward agent reached lap times rivaling professional players. The "bad demo reward" result—based on overly cautious players—confirmed that stylistic traits are learned, not just performance.

This method exemplifies an interpretable and low-data alternative to traditional IRL. It also bridges to \textbf{Imitation from Observation}, which we examine next.

\subsection{Imitation from Observation (IfO)}

In traditional imitation learning settings, expert demonstrations are typically available as sequences of state-action pairs. However, in many practical scenarios—such as YouTube gameplay videos, sports broadcasts, or third-person robot recordings—only observational data is accessible. This has motivated a growing body of research into \textit{Imitation from Observation} (IfO), where the agent must learn behavior solely from sequences of visual or symbolic states, without access to expert actions or environment rewards.

IfO presents unique challenges. Without action labels, it becomes difficult to directly regress policies or infer behavioral intent. Moreover, the agent must learn to generalize from the visual or perceptual properties of expert trajectories, often requiring additional structure, priors, or auxiliary models to bridge the observation-to-action gap. Despite these challenges, IfO has gained traction due to its wide applicability and natural alignment with real-world learning conditions.

Research in IfO spans several methodological directions, including:

\begin{itemize}
  \item \textbf{Behavioral Cloning from Observation (BCO):} Estimating expert actions via inverse dynamics models~\citepx{torabi2018bco}.
  \item \textbf{Temporal Alignment and Contrastive Learning:} Using time-based structures, such as Time-Contrastive Networks (TCN), to learn behaviorally consistent embeddings~\citepx{sermanet2018tcn}.
  \item \textbf{Distribution Matching:} Aligning state visitation frequencies between expert and learner using frameworks such as GAIfO~\citepx{torabi2019gaifo}.
  \item \textbf{Viewpoint-Invariant Representation Learning:} Bridging embodiment gaps via visual context translation~\citepx{liu2018context}.
\end{itemize}

A particularly promising direction integrates \textbf{observation pretraining} into the IfO pipeline. For instance, OpenAI's \textit{Video PreTraining (VPT)} framework~\citepx{baker2022vpt} used tens of thousands of hours of unlabeled YouTube Minecraft footage to pretrain an inverse dynamics model via self-supervision. A small number of labeled expert demonstrations were then used to fine-tune an imitation policy. Despite having no access to explicit reward functions or environment dynamics, the resulting agent achieved competent multi-skill behavior, including tool use and multi-stage construction.

This approach exemplifies how large-scale observational data can provide structural priors, which, when combined with limited expert guidance, support general and human-like behavior. It also demonstrates that the core challenge of IfO—mapping observations to purposeful control—can be softened by learning perceptual representations that implicitly capture style, progression, and intent.

In the broader context of playstyle modeling, IfO offers several distinct advantages: it accommodates learning from passive data, allows post-hoc analysis of real-world demonstrations, and can help discover diverse strategies by clustering or interpolating within the observation space. As such, it complements both IRL and BC frameworks, and forms a crucial pillar in the toolkit for learning behaviorally rich, human-aligned agents.

\subsection{Hybrid and Constraint-Regularized Methods}

Beyond the classical view of imitation learning as a primary objective—where the goal is to replicate expert behavior as faithfully as possible—modern AI systems increasingly adopt imitation in a supporting role. These methods, which we refer to as \textit{auxiliary imitation}, use demonstrations to inform, constrain, or initialize agents that are ultimately optimized for reward. Rather than serving as the end goal, imitation becomes a means of embedding structure and style into a broader learning pipeline.

\paragraph{Imitation as Structural Prior.}
In many cases, demonstration data is used not to define the optimization objective, but to shape the inductive bias of learning. One common approach is to pretrain a policy using imitation (e.g., via behavioral cloning), and then fine-tune it with reinforcement learning. This is particularly effective when exploration is costly or reward signals are sparse. Notably, this pattern aligns closely with recent trends in \textit{offline-to-online (O2O)} learning, where models are initialized using offline data and then adapted online. While not always framed as imitation, these methods inherently rely on behavior-derived priors to scaffold efficient learning.

Examples include IQL~\citepx{kostrikov2022iql} and CQL~\citepx{kumar2020cql}, which stabilize offline policy/value estimation, as well as frameworks like OLLIE~\citepx{xu2023ollie}, which integrate imitation pretraining with online fine-tuning. These pipelines demonstrate that imitation can serve as a bridge between static behavioral priors and dynamic task reward optimization.

\paragraph{Demonstration-Augmented Reinforcement Learning.}
Another strategy retains demonstration data throughout training, blending it with reinforcement signals. In Deep Q-learning from Demonstrations (DQfD)~\citepx{hester2018dqfd}, expert transitions are stored in the replay buffer and reinforced with a supervised margin loss. The distributed extension in Ape-X DQfD~\citepx{pohlen2018observe} scales this approach to the Atari benchmark, integrating transformed Bellman operators and removing episodic heuristics to enhance consistency. Here, demonstrations shape the learning process not just as initialization, but as persistent constraints on exploration and value updates.

\paragraph{Constraint-Regularized Policy Learning.}
In some frameworks, imitation is not explicitly performed, but its behavioral principles are embedded through constraints. For instance, Conservative Q-Learning (CQL)~\citepx{kumar2020cql} discourages out-of-distribution actions by penalizing high values for unseen behavior, effectively creating a soft imitation boundary. Other approaches incorporate entropy regularization or trust-region methods (e.g., PPO) to prevent policies from drifting too far from known-good behavior.

While these methods do not optimize imitation loss directly, they operationalize behavioral assumptions in a way that preserves expert-aligned playstyle without requiring exact reproduction. This is especially relevant in domains where stylistic alignment—such as consistency, elegance, or safety—is more important than literal action-matching.

More recently, a line of work has extended this idea from soft regularization to hard feasibility constraints. In \textbf{Action-Constrained Learning}~\citepx{yeh2025action,hung2025efficient}, the imitator is assumed to have a \textit{strictly smaller action space} than the expert. These methods address the resulting feasibility mismatch by aligning trajectories in state space—e.g., using dynamic time warping and model-predictive control—to generate surrogate demonstrations that respect the imitator's constraints. This reframes imitation not as copying actions, but as approximating behavioral intent within a limited capability set.

\paragraph{Imitation as Bias, Not Blueprint.}
This evolution from imitation as supervision to imitation as structure marks a philosophical shift. Rather than treating demonstrations as blueprints to be copied, auxiliary imitation treats them as \textit{preferences} to be respected. This opens the door to agents that learn not just what experts do, but how they make decisions—and how such decisions can be generalized, adapted, or even improved upon.

\paragraph{Implications for Playstyle.}
From a playstyle modeling perspective, auxiliary imitation enables fine control over behavioral traits without enforcing strict policy replication. It allows playstyle to be encoded in training bias, explored via hybrid learning dynamics, and refined through environmental feedback. Ultimately, this class of methods provides a flexible framework for learning agents that not only perform well, but perform with character.

\subsection{Diffusion-based Imitation Learning}

Most imitation learning (IL) methods, whether supervised, adversarial, or constraint-based, focus on learning a deterministic or unimodal policy that maps from observations to actions. However, expert behavior is often inherently multimodal—especially in games or embodied control tasks where multiple stylistic options exist for any given situation. To better capture this diversity, recent methods have explored the use of diffusion models, originally developed for generative modeling, to represent the conditional distribution over expert actions.

\paragraph{From Regression to Generative Modeling.}
Traditional behavior cloning minimizes prediction error between expert actions and predicted actions. While simple and effective, this paradigm suffers from mode averaging and poor coverage of diverse behaviors. Diffusion-based imitation learning instead reframes the policy as a generative denoising process: the agent learns to iteratively refine a sequence of noisy actions toward realistic ones, mimicking the dynamics of expert data without collapsing to a single mode.

\paragraph{Diffusion Models for Policy Learning.}
Several recent works have applied diffusion models to policy imitation. For example, DRAIL~\citepx{lai2024drail} integrates a diffusion generator into an adversarial imitation learning framework, replacing the traditional policy network or discriminator. Similarly, DiffAIL~\citepx{wang2024diffail} proposes a fully generative imitation learner using denoising diffusion probabilistic models (DDPMs) to map from state contexts to expert-like actions.

These models are particularly powerful in capturing stylistic or high-variance behaviors, such as slight differences in lane keeping, motion timing, or maneuver types. By modeling the full distribution of expert demonstrations, diffusion-based policies avoid the rigidity of deterministic models while preserving the richness of human-like play.

\paragraph{Applications to Playstyle.}
From the perspective of playstyle modeling, diffusion-based imitation opens new opportunities for representing and controlling stylistic variation. Because diffusion models are generative by design, they can be conditioned not only on state but also on latent variables representing style clusters, player identities, or task preferences. This enables agents to stochastically exhibit different playstyles—e.g., aggressive vs. conservative driving—without retraining the model.

Moreover, the denoising structure of diffusion models lends itself to interpolation and style transfer: given two demonstration sequences, it is possible to synthesize hybrid trajectories or explore novel behaviors in between. This makes diffusion-based IL an attractive candidate for studying playstyle generalization and creative behavior generation.

\paragraph{Current Limitations and Future Directions.}
While diffusion-based imitation offers expressive modeling of behavioral diversity, several practical limitations remain. First, inference involves iterative denoising, which can be computationally expensive and unsuitable for real-time decision-making. Ongoing work has explored accelerated variants (e.g., DDIM~\citepx{song2021ddim}) and trajectory-level sampling, but inference latency remains a concern in interactive systems.

Second, while these models capture behavioral variance, they are typically weakly conditioned and lack explicit structure for symbolic playstyle control. Bridging diffusion IL with interpretable representations—such as playstyle labels, preference embeddings, or discrete strategy tokens—is a promising direction for aligning generative imitation with downstream control and evaluation tasks.

Overall, diffusion-based imitation learning represents a novel and increasingly important direction in playstyle modeling. It extends the frontier of IL from copying behavior to \textit{expressing} behavior, enabling agents to exhibit multimodal, stylistically rich trajectories informed by diverse demonstrations.

\subsection{Challenges and Future Directions}

Across this section, we have surveyed a broad landscape of imitation learning techniques—from classical behavioral cloning and inverse reinforcement learning (IRL), to hybrid pipelines, auxiliary imitation, and diffusion-based generative models. Each line of work offers unique tools for modeling, guiding, or constraining agent behavior based on demonstrations. However, several open challenges remain in developing playstyle-aware imitation systems that are both expressive and robust.

\paragraph{Playstyle Diversity and Conditioning.}
While many methods succeed at matching expert performance, few can flexibly express multiple valid styles under the same task specification. Diffusion-based methods and latent-conditioned policies show promise in addressing this limitation, but require further integration with symbolic or controllable representations to allow fine-grained stylistic guidance.

\paragraph{Imitation beyond Action Matching.}
Imitation is often framed as matching state-action pairs, but true style modeling demands more abstract objectives: visiting similar states, maintaining tempo, or respecting gameplay conventions. Bridging this abstraction gap—possibly via feature matching, trajectory-level alignment, or successor-style representations—is critical for nuanced behavior learning.

\paragraph{Robustness and Demonstration Quality.}
Many IL systems are highly sensitive to the quality, optimality, and coverage of demonstration data. Handling noisy, suboptimal, or stylistically inconsistent examples remains an ongoing challenge. Approaches such as confidence-aware weighting, reward-free pretraining, or ensemble demonstrations may provide more robust alternatives.

\paragraph{From Demonstration to Behavior Constraints.}
Increasingly, imitation is used not to fully define agent policy, but to impose behavioral boundaries or preferences within a larger optimization process. This perspective opens new avenues for combining human guidance with reward-based learning, particularly in safety-critical or user-facing domains.

\paragraph{Toward Human-Like Behavior.}
Imitating human behavior is not merely about reproducing actions, but about aligning with biological, perceptual, and stylistic priors. However, regardless of the imitation method, determining what it truly means to be “human-like”—and how to measure it—remains a deeper and more foundational challenge. This sets the stage for our next discussion, where we examine human-likeness constraints in detail.

As a concrete example, we show how discrete state representation—initially used for symbolic imitation—can be repurposed to encourage biologically plausible motion. In a first-person-shooter (FPS) environment, we introduce a penalty for repeatedly visiting the same state in short time intervals, effectively discouraging jittery or oscillatory behavior. This example illustrates how playstyle-informed regularization can go beyond imitation to shape more human-like agents.

\section{Building Human-like Agents}

What does it mean for an artificial agent to exhibit human-like behavior? While imitation learning provides a direct answer—train agents to mimic human demonstrations—it raises deeper questions: Which humans should we imitate? Are all human behaviors equally valid targets? And can human-likeness be defined beyond mere imitation?

These questions reveal a central ambiguity in defining human-likeness. Inspired by Wittgenstein’s notion of \textit{family resemblance}~\citepx{wittgenstein1953philosophical}, we argue that there is no single set of features shared by all human behaviors, but rather overlapping clusters of similarity. Human-like behavior is not a well-defined category with clear boundaries—it is a fragmented landscape of patterns connected by resemblance rather than strict definitions.

\begin{figure}[ht]
\centering
\includegraphics[width=0.9\linewidth]{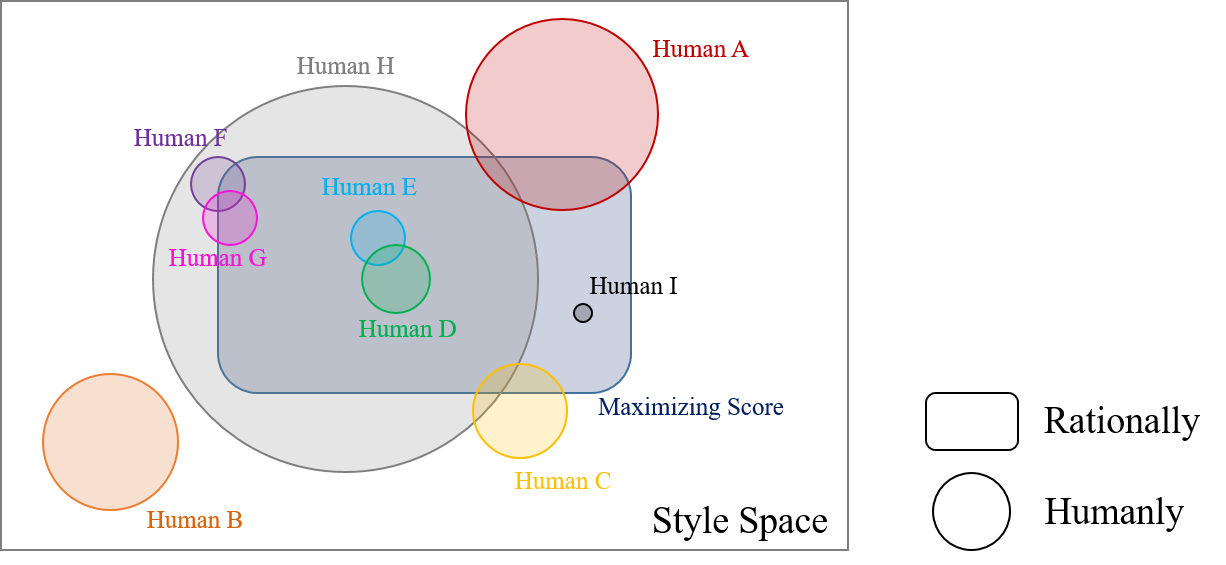}
\caption[Fragmentation of human-likeness in behavior space.]{
Fragmentation of human-likeness in behavior space.
Each colored circle represents a different human individual or behavioral definition that could serve as an imitation target.
Some lie within the rational behavior set (rounded rectangle), some intersect with maximized task performance (black dot), and others fall outside.
This illustrates that human-like behavior is not a unified or convex set, but rather a set of overlapping, partially conflicting regions.
}
\label{fig:human_likeness_fragmentation}
\end{figure}

As illustrated in Figure~\ref{fig:human_likeness_fragmentation}, the behavior space consists of distinct, often conflicting clusters that correspond to different individuals or stylistic definitions. Some behaviors fall within both rational and human-like domains, while others are idiosyncratic or inconsistent with performance goals. This makes it difficult to define a universally accepted positive criterion for “being human-like.” Imitating any single individual may fall short of generality—or even contradict other valid behaviors.

In contrast, identifying behaviors that are clearly \emph{not} human-like is often more tractable. This motivates what we call a \textbf{negative definition} of human-likeness: instead of trying to enumerate all acceptable behaviors, we begin by ruling out those that strongly violate human perceptual or embodiment intuitions. Agents that rapidly shake the camera, spin aimlessly, or repeat local movements without progress evoke immediate impressions of unnaturalness. These negative patterns offer a practical starting point—constructing the boundary of human-likeness by exclusion rather than exhaustive inclusion.

This motivates a \textbf{constraint-regularized perspective} on building human-like agents. Instead of relying solely on imitation or reward shaping, we introduce \textbf{symbolic constraints, behavioral penalties, and embodiment-aware priors} to suppress clearly unnatural behavior. These constraints delineate the “non-human” region, effectively guiding agents toward plausibility through exclusion rather than reproduction.

This strategy becomes particularly compelling when combined with \textbf{discrete symbolic state representations}. By tracking transitions at a symbolic level—independent of action semantics—we can detect problematic behavioral motifs such as local oscillations, redundant loops, or disorienting motion, which often trigger perceptions of robotic or inhuman behavior. Crucially, such symbolic constraints generalize across tasks and environments, even in the absence of environment-specific labels or continuous metrics.

In the remainder of this chapter, we explore how such constraints can be formalized and applied. We revisit our earlier racing game case study, where symbolic playstyle rewards enabled effective imitation from sparse demonstrations. We then present a new exploratory experiment in 3D first-person environments (DM Lab), where we use symbolic state loop detection to penalize unnatural behavioral repetition. Though entirely action-agnostic, this constraint consistently reduces unnatural motion and, in some cases, even improves task performance. Taken together, these results demonstrate how the negative definition can be operationalized: rather than fully specifying “what to be,” we systematically constrain “what not to be” to guide agents toward human-like plausibility.

\subsection{Three Foundations of Human-likeness}

To operationalize the concept of human-likeness in agent design, we identify three foundational sources that inform what it means for an agent to "act humanly." These sources—\textbf{biological constraints}, \textbf{psychological patterns}, and \textbf{social-cultural consensus}—each impose distinct yet complementary forms of regularity on human behavior. By analyzing and modeling these foundations, we move beyond surface imitation and toward generalizable, interpretable, and socially aligned behavior.

\begin{itemize}
    \item \textbf{From Physical: Biological Constraints} \\
    Human movement is governed by physical embodiment: joint limits, motor latency, visual field, energy efficiency, and biomechanical stability. Modeling these constraints helps agents avoid unnatural behaviors such as jitter, oscillation, or hyper-reactivity that break the illusion of human presence. These priors also inform motion smoothing, temporal consistency, and action feasibility in embodied AI.

    \item \textbf{From Mental: Psychological and Economic Insights} \\
    Human decisions are shaped not only by cognitive biases, bounded rationality, and behavioral regularities identified in psychology, but also by formal models from behavioral and decision economics. These fields share deep connections—especially in explaining how people act under uncertainty, incomplete information, or risk. Agents that mirror patterns such as loss aversion, satisficing, or reaction delays appear more relatable and realistic. Incorporating these heuristics—through mechanisms like noisy utility functions, probabilistic choice models, or attention bottlenecks—grounds AI behavior in human-like internal logic, even when goals differ.

    \item \textbf{From Interactions: Social and Cultural Consensus} \\
    In shared environments, human-likeness is judged not in isolation, but through social alignment. Agents are expected to follow stylistic norms of cooperation, turn-taking, etiquette, and contextual appropriateness. For instance, erratic or overly optimized actions may violate expectations in multiplayer games or co-creative tasks. Embedding agents in culturally-informed interaction priors—via demonstrations, language, or preference modeling—supports stylistic coherence across populations.
\end{itemize}

\subsection{From Physical: Biological Constraints}

One foundational perspective on human-likeness stems from biology: human perception and action are inherently bounded by physical limitations. These include natural constraints such as sensory resolution, motor smoothness, physical fatigue, and reaction delays~\citepx{biological_constrains}. Unlike artificial agents that can operate with superhuman speed, consistency, and precision, human behavior is shaped by physiological bottlenecks that constrain how decisions and actions unfold across time.

In video games and embodied environments, these constraints manifest visibly. Human players cannot move with perfect mechanical regularity; they exhibit jitter, reaction time variability, and fatigue after prolonged activity. They also cannot change camera perspectives instantaneously or react to pixel-level stimuli at high frequency. These limitations make human behavior not only imperfect but also smooth, predictable, and legible to others—a quality crucial for human-likeness.

Failing to incorporate such constraints often leads agents to develop behaviors that, while optimal under reward maximization, appear disturbingly artificial. For example, in first-person-view environments, reinforcement learning agents frequently exhibit rapid camera shaking, local spinning, or micro-adjustments that would be disorienting or fatiguing for human players~\citepx{shaking_spinning_cost}. These behaviors highlight a misalignment between task performance and perceived naturalness.

To address this, researchers have proposed embedding \emph{biological priors} into learning objectives. Some approaches directly impose smoothness regularization on control actions or camera angles~\citepx{caps, grad_caps}, while others penalize high-frequency action patterns. A more general strategy, avoiding action-specific heuristics, is to penalize short-term \emph{redundant state transitions}—based on the observation that humans rarely oscillate among a small set of near-identical visual states unless deliberately scanning or reacting.

We explored this idea in a series of studies. In a first-person-view collection game (Unity ML Agent’s Banana Collector, now called Food Collector)~\citepx{unity_ml_agent}, we defined behavior cost heuristics based on action sequences to suppress unnatural shaking and spinning, and we further investigated how to schedule penalty coefficients to balance performance and naturalness~\citepx{thesis_of_Luo_You_Ren, shaking_spinning_cost}.

More recently, we extended this principle into symbolic space using discrete state representations. Building on the symbolic reward approach in 3D racing (Section~\ref{sec:imitation_reward_racing}), we tested a general-purpose loop detection penalty in DeepMind Lab (DM Lab): when an agent re-enters a symbolic state it visited in the past 8 steps—excluding immediate self-loops—it receives a penalty. Unlike handcrafted motion heuristics, this method is task-agnostic, interpretable, and requires no semantic knowledge of the action space.

We trained PPO agents for 10 million game steps across 14 DM Lab games (excluding language-based and extreme-difficulty tasks), comparing three conditions: (1) vanilla PPO, (2) action heuristics detecting shaking/spinning, and (3) symbolic loop detection. Results are shown in Tables~\ref{tab:dmlab_human_like_comparison_behavior_cost} and~\ref{tab:dmlab_human_like_comparison_game_score}.

\begin{table}[ht]
\centering
\caption[DeepMind Lab Human-like Game Behavior Cost]{Behavior Cost comparison of PPO baseline and behavior-constrained variants on 14 DeepMind Lab games. Values report mean episode scores over 3 runs.}
\label{tab:dmlab_human_like_comparison_behavior_cost}
\begin{tabular}{|l|r|r|r|}
\toprule
\textbf{Game} & \textbf{PPO} & \textbf{Action Heuristic} & \textbf{Loop Detect} \\
\midrule
rooms\_collect\_good\_objects & 105.4 & 16.3 (-84.5\%) & 59.4 (-43.7\%) \\
rooms\_exploit\_deferred\_effects & 8.2 & 0.2 (-97.8\%) & 7.5 (-8.9\%) \\
rooms\_select\_nonmatching\_object & 35.5 & 0.9 (-97.6\%) & 33.0 (-7\%) \\
rooms\_watermaze & 201.3 & 13.4 (-93.4\%) & 60.5 (-70.0\%) \\
rooms\_keys\_doors\_puzzle & 65.6 & 1.1 (-98.4\%) & 17.0 (-74.0\%) \\
lasertag\_one\_opponent\_small & 384.0 & 0.1 (-99.9\%) & 123.5 (-67.8\%) \\
lasertag\_three\_opponents\_small & 285.1 & 0.7 (-99.8\%) & 117.5 (-58.8\%) \\
natlab\_fixed\_large\_map & 260.1 & 4.6 (-98.2\%) & 85.3 (-67.2\%) \\
natlab\_varying\_map\_regrowth & 217.7 & 1.3 (-99.4\%) & 68.6 (-68.5\%) \\
skymaze\_irreversible\_path\_hard & 109.3 & 0.1 (-99.9\%) & 1.1 (-99.0\%) \\
explore\_object\_locations\_small & 133.2 & 19.8 (-85.2\%) & 114.9 (-8.8\%) \\
explore\_obstructed\_goals\_small & 134.9 & 1.2 (-99.1\%) & 35.8 (-73.5\%) \\
explore\_goal\_locations\_small & 170.3 & 1.5 (-99.1\%) & 71.5 (-58.0\%) \\
explore\_object\_rewards\_few & 130.5 & 0.2 (-99.8\%) & 40.5 (-69.0\%) \\
\bottomrule
\end{tabular}
\end{table}

\begin{table}[ht]
\centering
\caption[DeepMind Lab Human-like Game Score]{Game Score comparison of PPO baseline and behavior-constrained variants on 14 DeepMind Lab games. Values report mean episode scores over 3 runs.}
\label{tab:dmlab_human_like_comparison_game_score}
\begin{tabular}{|l|r|r|r|}
\toprule
\textbf{Game} & \textbf{PPO} & \textbf{Action Heuristic} & \textbf{Loop Detect} \\
\midrule
rooms\_collect\_good\_objects & 6.1 & 7.4 ($\uparrow$) & 6.5 ($\uparrow$) \\
rooms\_exploit\_deferred\_effects & 9.6 & 9.7 ($\uparrow$) & 9.6 (-) \\
rooms\_select\_nonmatching\_object & 2.2 & 4.6 ($\uparrow$) & 5.0 ($\uparrow$) \\
rooms\_watermaze & 25.3 & 26.6 ($\uparrow$) & 27.0 ($\uparrow$) \\
rooms\_keys\_doors\_puzzle & 23.9 & 19.5 ($\downarrow$) & 30.0 ($\uparrow$) \\
lasertag\_one\_opponent\_small & 0.0 & -0.1 ($\downarrow$) & -0.1 ($\downarrow$) \\
lasertag\_three\_opponents\_small & 1.4 & 3.1 ($\uparrow$) & -0.2 ($\downarrow$) \\
natlab\_fixed\_large\_map & 12.2 & 11.4 ($\downarrow$) & 9.6 ($\downarrow$) \\
natlab\_varying\_map\_regrowth & 9.5 & 6.8 ($\downarrow$) & 7.3 ($\downarrow$) \\
skymaze\_irreversible\_path\_hard & 16.6 & 2.1 ($\downarrow$) & 0.0 ($\downarrow$) \\
explore\_object\_locations\_small & 34.0 & 31.8 ($\downarrow$) & 32.8 ($\downarrow$) \\
explore\_obstructed\_goals\_small & 29.8 & 21.6 ($\downarrow$) & 11.5 ($\downarrow$) \\
explore\_goal\_locations\_small & 93.6 & 39.1 ($\downarrow$) & 60.0 ($\downarrow$) \\
explore\_object\_rewards\_few & 1.8 & 1.6 ($\downarrow$) & 2.3 ($\uparrow$) \\
\bottomrule
\end{tabular}
\end{table}

These results confirm the potential of using discrete symbolic representations to define task-agnostic behavioral penalties. While handcrafted heuristics remain more precise in eliminating unnatural behavior, the symbolic loop detector offers broader generalization and comparable behavioral suppression. Most importantly, it reduces undesirable motion in an action-agnostic way—without any access to camera semantics or control mappings—making it a promising direction for scalable and interpretable human-likeness constraints in embodied agents.

In summary, biological constraints offer a powerful design principle for shaping agents toward human-likeness. Rather than relying on exhaustive demonstrations or manually tuned heuristics, these constraints derive from embodiment itself—encouraging behavior that is not only functionally effective but also physically plausible and socially acceptable.

\subsection{From Mental: Psychological and Economic Insights}

In addition to physical embodiment, human-like behavior is deeply shaped by mental processes—including subjective preferences, risk sensitivities, cognitive limitations, and behavioral biases. While biological constraints shape what humans \emph{can} do, psychological and behavioral economic factors help explain \emph{why} they act the way they do. These mental dimensions offer a complementary perspective on human-likeness, grounded not in physical realism but in decision-making regularities under uncertainty.

\paragraph{Risk Preferences as Behavioral Signatures.}
Human decision-making frequently reflects varying degrees of risk aversion or risk seeking, depending on both individual disposition and situational framing—topics long studied in behavioral economics and decision theory. In reinforcement learning, such variation can be modeled through distributional RL methods, most notably Implicit Quantile Networks (IQN)~\citepx{iqn}, which allow agents to sample from different quantiles of the return distribution, effectively modulating their behavior between pessimistic and optimistic risk profiles.

Although IQN was not explicitly designed for human-likeness, its richer risk sensitivity closed the remaining “human gap” across several Atari games, outperforming other methods in tasks where agents still lagged behind human performance. This suggests that incorporating diverse risk attitudes—whether cautious, aggressive, or opportunistic—can align agents more closely with human stylistic variability.

\paragraph{Cognitive Constraints and Bounded Rationality.}
Unlike artificial agents with vast processing capacity, humans are bounded by cognitive limitations. The concept of \emph{bounded rationality}, introduced by \citetx{simon1957models}, describes how humans “satisfice” rather than globally optimize, due to limits on time, attention, and computational resources. In economics, this principle helps explain why real-world decision-making often deviates from perfect rationality.

AI systems can replicate such constraints to appear more human-like: limiting planning depth, compressing memory representations, or simulating working memory bottlenecks. Hierarchical and meta-learning agents can further emulate structured reasoning over latent beliefs and subgoals, echoing both psychological and economic models of decision-making.

\paragraph{Behavioral Biases and Imperfect Rationality.}
Many systematic biases documented in psychology also play a central role in behavioral economics, as they reveal how human decision-making systematically deviates from the predictions of classical rational choice. From a behavioral economics perspective, notable examples include:
\begin{itemize}
  \item \textbf{Framing effects}: different presentations of the same problem alter choices;
  \item \textbf{Loss aversion}: losses loom larger than gains~\citepx{kahneman2011thinking};
  \item \textbf{Magnitude effect}: larger outcomes are discounted more steeply~\citepx{magnitude_effect};
  \item \textbf{FOMO (Fear of Missing Out)}: a bias visible in financial markets, where investors rush to buy during price surges for fear of missing gains.
\end{itemize}

Beyond economic decision-making, other psychological biases and perceptual distortions also shape human-like behavior:
\begin{itemize}
  \item \textbf{Weber–Fechner law}: sensitivity to change diminishes with stimulus magnitude~\citepx{weber_fechner_law};
  \item \textbf{Ebbinghaus illusion}: contextual perception alters judgments~\citepx{roberts2005roles}.
\end{itemize}

Such biases may appear suboptimal, yet they make agent behavior more relatable. An agent that hesitates, overreacts, or inconsistently follows through—much like a trader swayed by FOMO—can appear more human than one that acts with perfect rationality.

Our earlier work on playstyle similarity metrics incorporated \textbf{perceptual distance}—rooted in Weber-Fechner scaling and magnitude thresholds—to better align with human sensitivity to behavioral differences~\citepx{magnitude_effect, weber_fechner_law}.

\paragraph{Toward Cognitive Alignment.}
The imperfections of human reasoning complicate imitation. Human demonstrations are inconsistent and locally biased; optimality is rare. This is central to inverse reinforcement learning (IRL), where approaches like Maximum Entropy IRL~\citepx{max_entropy} model behavior as probabilistically rational, embracing ambiguity and variability as part of the human signal.

Taken together, these psychological and economic insights suggest that modeling cognitive traits—whether through risk-sensitive policies, bounded reasoning, or biased decision-making—offers a rich design space for shaping human-like behavior. In the next subsection, we turn to social and cultural dimensions, where norms and style expectations further influence how agents are perceived.

\subsection{From Interactions: Social and Cultural Consensus}

Beyond physical and psychological constraints, human-likeness is also shaped by the social and cultural contexts in which actions are interpreted. Human behavior is not merely biological or rational—it is deeply social, shaped by norms, conventions, and expectations that govern interaction. These norms emerge from shared practices: how people greet each other, take turns in conversation, resolve conflicts, or collaborate toward goals.

In the domain of AI agents, especially those embedded in games, virtual environments, or social simulations, modeling such socially legible behavior has become increasingly important. A prominent example is the “Generative Agents” framework~\citepx{park2023generative}, where large language models are augmented with planning, memory, and reflection components to simulate realistic town-dwelling characters. These agents demonstrate emergent social coordination, forming routines, inviting others to events, and maintaining coherent behavior across days—despite no explicit programming of such social structure.

These studies show that human-likeness is not only about what the agent does, but how it is perceived and understood by others. An agent’s behavior is evaluated not in isolation, but in context: whether it respects turn-taking, expresses intentions clearly, and adapts appropriately to the social setting. For this reason, cultural alignment becomes both a design challenge and a risk. Agents that mimic norms from one region or demographic may inadvertently violate the expectations of others. The more socially embedded an agent becomes, the more it must navigate fragmented and often conflicting definitions of “natural” behavior.

\paragraph{Cultural Foundations and Norm Codification.}
Throughout human history, social norms have been codified in texts ranging from legal codes to moral teachings. Ancient documents such as the \emph{Code of Hammurabi}, \emph{The Analects of Confucius}, and the \emph{Bible} offer vivid illustrations of how societies have formalized behavioral expectations, from retributive justice (“an eye for an eye”) to ritual etiquette, filial piety, and communal responsibility~\citepx{confucius2003analects,hammurabi1904code,bible1611kjv}. These works serve not only as records of cultural values but as early attempts to stabilize and transmit behavioral regularities across generations.

In modeling human-like agents, referencing such codifications reminds us that behavior is not judged solely by efficiency or coherence, but by its alignment with shared ideals. Just as Confucian thought emphasizes role-appropriate conduct (禮, \textit{li}) and harmonious social order, artificial agents may benefit from role-awareness and sensitivity to context. Similarly, the parables and prescriptions in the Bible encode moral framing into narrative action—suggesting that even minimal agents, when acting in socially embedded worlds, might require analogous narrative or ethical framing to sustain believability and trust.

\paragraph{Social Simulation and Interactive Learning.}
In addition to imitation-based social behavior, recent advances in interactive large language models (LLMs) offer new tools for capturing and generalizing social conventions. Systems such as SOTOPIA~\citepx{sotopia2024} and Stanford’s AI Town~\citepx{park2023generative} demonstrate that sociality can emerge not only from hardcoded logic, but from dialogic interaction, situated memory, and dynamic adaptation. These systems push the boundary of what it means to simulate “being human” in a multi-agent context—not through optimization, but through negotiation of meaning across agents.

\paragraph{Ethical Reflections on Human-like Interaction.}
As agents become increasingly capable of replicating human behavior in social settings, ethical concerns arise regarding the nature of such mimicry. Highly human-like interaction may blur the boundary between simulation and social presence, creating the illusion of emotional or moral agency. This is particularly problematic in sensitive applications such as education, therapy, or caregiving, where users might develop unwarranted trust or attachment~\citepx{shneiderman2022bridging}.

Moreover, efforts to culturally align agents can backfire when context-specific norms are misunderstood or generalized across inappropriate settings. Cultural appropriation, biased value transmission, or overfitting to dominant social behaviors may all compromise the intended inclusivity of such designs. Human-likeness, from a social perspective, is thus not a universally agreed-upon construct but one that must be negotiated with care across communities.

These concerns highlight the importance of deliberate boundaries. Not all systems need to appear human-like; in some cases, transparent artificiality may support better expectations, control, and accountability. The pursuit of human-likeness, when situated within social contexts, should be grounded not only in behavioral fidelity, but also in values of responsibility, transparency, and equity.

\subsection{On the Ethics and Boundaries of Human-likeness}

Across this section, we have examined human-likeness through three complementary lenses: 
\emph{physical embodiment} that grounds how an agent can act, 
\emph{mental models} shaped by risk preferences, cognitive limits, and behavioral biases, 
and \emph{cultural norms} that influence social interaction and stylistic expectations.
From these perspectives, a unifying conclusion emerges: 
human-likeness is not a single property to be maximized, but a fragmented, context-dependent construct defined by overlapping—and sometimes conflicting—criteria.

This resonates with Ludwig Wittgenstein’s notion of \emph{family resemblance}, in which categories such as “game” or “human” lack fixed boundaries yet remain connected by patterns of similarity~\citepx{wittgenstein1953philosophical}. 
One agent may appear human-like because it follows motor smoothness constraints, another due to risk-sensitive decision-making, and a third because it respects conversational etiquette. 
Each belongs to the “family” of human-like agents, even though no single feature defines them all.

The implication is both philosophical and practical. 
It is neither feasible nor necessarily desirable for an AI system to embody all dimensions of human-likeness at once. 
Human behavior is diverse, culturally situated, and value-laden; 
any attempt to optimize for a universal “acting humanly” risks privileging some styles while marginalizing others. 
Worse, it can produce manipulative designs that mimic human traits without respecting their moral or social grounding~\citepx{shneiderman2022bridging, cave2019scary}.

A more constructive approach is to clarify the \emph{boundaries}: 
What forms of behavior should be included or excluded in a given application? 
When does human-likeness serve the goal, and when is explicit artificiality preferable? 
Which norms are to be preserved, and which challenged? 
These questions move the discussion from technical imitation toward reflective alignment—where design decisions are guided not just by feasibility, but by purpose.

Finally, the space between human-like fidelity and creative divergence invites further exploration. 
As we transition to the next chapter on playstyle innovation, we ask: 
If not merely to copy human behavior, can agents develop novel, effective, and interpretable playstyles of their own? 
In this view, human-likeness is not an endpoint, but a launchpad—a starting point for dialogue between what is, and what could be.

\chapter{Playstyle Innovation and Creativity}
\noindent\textbf{Key Question of the Chapter:} \\
\textit{Can AI innovate or exhibit creativity?}

\noindent\textbf{Brief Answer:} 
Creativity typically requires both novelty and effectiveness.
If an AI’s goals do not incentivize novelty, creative behavior is unlikely to emerge.
But when exploration is part of its objectives—as in many reinforcement learning agents—creative outcomes can arise naturally.

\bigskip

As artificial intelligence systems grow more sophisticated, the question of whether machines can genuinely innovate or exhibit creativity remains both provocative and central to the future of intelligent behavior.  
In the realm of playstyle, innovation carries a distinct meaning: it is not merely about winning through novel strategies, but about exhibiting \textbf{behavioral modes that transcend existing stylistic taxonomies}—decision patterns that challenge or expand our understanding of what styles are possible.

This chapter examines how such innovation emerges in artificial agents, particularly within reinforcement learning and multi-agent self-play frameworks. We highlight canonical cases—such as AlphaGo’s famous Move~37 or AlphaStar’s unexpected Zerg wall-offs—as examples of emergent behaviors that defy human priors. Our goal is to assess whether such behaviors qualify as genuine \emph{playstyle innovation}, and to develop conceptual and computational criteria for distinguishing between random variation, recombination, and meaningful stylistic novelty.

Rather than speculating about human-like creative processes, we focus on observable outputs and measurable effects. In doing so, we connect the study of innovation to earlier discussions of style formalization, diversity, balance, and human-likeness. Ultimately, we argue that \textbf{playstyle innovation offers a concrete, domain-grounded lens through which AI creativity can be understood, evaluated, and designed}.

\section{Beyond Existing Playstyles}

In the study of style, the value of a behavior often comes from two sources.  
The first is \emph{alignment} with established mainstream styles: by conforming to recognizable patterns that have gained social acceptance, a new behavior inherits value from existing cultural or strategic norms.  
The second is \emph{innovation}—behaviors that deviate significantly from prior styles and establish new modes of action.  
Genuine innovation here is not a superficial tweak in appearance or tactics, but a transformation in goals, assumptions, or design principles that alters the very structure of the style space.  
Once recognized and validated—through competitive success, community interest, or aesthetic appeal—such behaviors may form new paradigms or dominate under specific conditions; others may fade under the evolutionary dynamics of the game ecosystem.

This form of innovation is often visible in playstyle-centric applications of deep reinforcement learning (DRL).  
Because DRL agents explore and optimize via reward signals, they can produce strategies not directly specified by designers.  
When models trained under different conditions display distinct behavioral distributions—especially over many trajectories—they may constitute new playstyles~\citepx{playstyle_similarity}.  
Yet such diversity is not always deemed “creative,” and AI behaviors are often dismissed as recombinations of known patterns rather than genuine innovations.

This skepticism stems from early supervised and unsupervised learning paradigms, where models mimic or generalize from pre-existing samples~\citepx{bengio2013representation, lecun2015deep}.  
Even when outputs depart from training data, they are often attributed to interpolation or brute-force search.  
To evaluate AI-driven novelty more rigorously, we can borrow categories from innovation theory~\citepx{schumpeter1934theory, chesbrough2003open, christensen1997innovators}:

\begin{itemize}
    \item \textbf{Open Innovation} — Incorporating external knowledge sources (e.g., human demonstrations, pretrained modules, symbolic priors) to bootstrap novel strategies through imitation, adaptation, or remixing.
    \item \textbf{Combinatorial Innovation} — Synthesizing known elements into more effective behavioral configurations, as when agents merge fragments of successful trajectories into new strategies.
    \item \textbf{Disruptive Innovation} — Developing strategies that initially seem suboptimal or risky but ultimately shift the meta, often discovered via deep exploration beyond current reward basins.
\end{itemize}

Multi-agent self-play systems offer prominent examples, training agents against continually adapting peers rather than static opponents.  
This co-evolution fosters behavioral diversity and strategic escalation.  
Pioneering systems such as AlphaGo~\citepx{alpha_go}, AlphaStar~\citepx{alpha_star}, and OpenAI Five~\citepx{openai_five} have produced strategies far beyond human priors: nonstandard openings, asymmetric role allocations, and high-risk coordination patterns.  
These emerge not from explicit design, but from internal pressures to outcompete adaptive adversaries~\citepx{balduzzi2019open, psro}.

In sum, playstyle innovation need not originate from human creativity.  
With open-ended exploration and adaptation, AI systems can cross—and sometimes redefine—the boundaries of existing strategic repertoires.  
Recognizing and valuing such innovations, however, requires analyzing both their novelty and their effectiveness—a focus of the next section.

\section{Effective Innovation}

Novelty alone does not guarantee value.  
In playstyle development, what matters is whether a new behavior achieves something that existing ones cannot.  
We call this \emph{effective innovation}—innovation that not only differs from the known, but contributes meaningfully within the system’s strategic and evaluative constraints.  
It fills gaps, creates new niches, or enables previously unviable behaviors to survive under updated dynamics.  
In short, it is innovation that endures.

While generating ideas is easy, producing those that are useful, testable, and robust is far harder.  
Most breakthroughs are not radical leaps, but forms of \emph{combinatorial innovation}: modifying, remixing, or recontextualizing existing elements.  
Although sometimes criticized for limited novelty, such approaches often excel in practical effectiveness, revealing that true innovation frequently lies in discovering underexplored configurations of the known.

\textbf{Effective innovation explores the boundary of effectiveness}: pushing beyond what is validated while remaining within what is viable.  
Cross too far into novelty, and coherence with system dynamics is lost; stay too close to the familiar, and relevance fades.  
Effective innovation lives in the liminal zone—just past the current frontier of good behaviors, yet still within the expanding domain of balance, interpretability, and usefulness.

Among practical methods, the \emph{matrix method}—identifying stylistic axes and systematically combining them—mirrors how deep reinforcement learning (DRL) agents discover surprising combinations of behavior primitives.  
A close parallel exists in \emph{Quality Diversity (QD)} algorithms~\citepx{pugh2016quality, cully2015robots, gravina2019quality}, which explicitly separate the search for \emph{novelty} from the requirement of \emph{quality}.  
By doing so, QD methods explore not a single optimum, but a manifold of high-performing solutions—exactly the frontier where effective innovation occurs.

In playstyle analysis, this perspective is especially powerful.  
Instead of converging on one dominant style, QD-based frameworks discover many viable yet distinct strategies, each adapted to a specific niche.  
This produces a richer behavioral repertoire and a principled way to ensure that innovation remains both diverse and grounded in effectiveness.

\paragraph{Evaluating Innovation Through Capacity and Popularity.}
From our Chapter~2 framework, effective playstyle innovation has two dimensions: \emph{capacity}—expanding what is possible, and \emph{popularity}—being recognized as desirable.  
Effectiveness is not intrinsic; it is socially and contextually constructed.  
A powerful strategy may fail if it violates expectations or is misunderstood, while a weak one may thrive through memetic spread or aesthetic appeal.

In interactive environments, “environmental rules” reflect the \textbf{preferences of their creators}—designers, developers, or curators.  
Even seemingly organic popularity is shaped by implicit authorial intention.  
Thus, innovation becomes effective only when it is recognized by an evaluative system—be it a reward function, a player community, or a cultural ecosystem.

Ultimately, playstyle is a dialogue between internal decision-making capacity and external perception.  
Innovation is the act of adding something new to that dialogue—and having it heard, understood, and accepted.

\section{Can AIs have Creativity?}

Having discussed innovation in depth, we now return to the central question: \emph{Can artificial intelligence exhibit creativity?} From the perspective of the playstyle framework proposed in this dissertation, this question can be recast in terms of \textbf{boundary traversal}. Specifically, creativity involves two distinct forms of exploration:

\begin{enumerate}
    \item \textbf{Crossing the boundary of existing styles} — generating behaviors that do not conform to any known playstyle categories;
    \item \textbf{Remaining within the boundary of effectiveness} — ensuring that these new behaviors are still viable, meaningful, and strategically sound.
\end{enumerate}

In this formulation, creativity emerges when an agent produces behaviors that lie \emph{outside} the familiar stylistic space, yet still \emph{inside} the domain of what works. This, we argue, constitutes a valid form of creativity. Notably, deep reinforcement learning (DRL) agents often exhibit precisely this pattern: through continual interaction and self-competition, they incrementally extend the frontier of known styles, pushing the limits of viability while remaining anchored in outcome-driven feedback.

Several landmark cases illustrate this boundary traversal. AlphaGo’s famous Move~37~\citepx{alpha_go} broke centuries of professional Go intuition, revealing a strategically sound yet stylistically alien move. AlphaStar’s unexpected Zerg wall-offs~\citepx{alpha_star} redefined early-game defense in \textit{StarCraft~II}. In OpenAI’s hide-and-seek environment~\citepx{hide_and_seek}, agents learned to exploit movable objects as improvised tools and barriers—behaviors never explicitly coded by designers. Similarly, OpenAI Five’s coordinated five-hero flank attacks~\citepx{openai_five} in \textit{Dota~2} showcased risk-heavy, asymmetric aggression that disrupted standard competitive metas. These strategies were not pre-programmed but emerged from continual self-play and adaptation, illustrating how exploration under competitive pressure can yield genuinely novel playstyles.

Traditional skepticism toward AI creativity stems from supervised and unsupervised learning paradigms. These approaches are built to interpolate within existing data or reconstruct observed patterns, inherently limiting them to known style spaces. From this view, true creativity—defined as behavior beyond recombination—seems elusive.

However, recent advances challenge this framing. In particular, generative adversarial models (GANs) offer a compelling counterexample. Unlike models that learn to mimic data directly, GANs create a tension between a generator and a discriminator, forcing exploration of novel and unobserved inputs. Conceptually, this adversarial setup suggests that the \emph{style boundary} (novelty) and \emph{validity boundary} (effectiveness) can be decoupled—allowing for a structured path toward creative expression.

This idea is further strengthened by preference-based reinforcement learning (PbRL)~\citepx{wirth2017survey}, where agent behavior is shaped not solely by environment rewards, but by human preferences—effectively redefining the boundary of what counts as "good" behavior. When paired with playstyle metrics or latent representations, PbRL may allow us to \emph{guide} agents toward stylistically novel but functionally grounded solutions—the dual hallmark of creativity.

Viewed through this lens, creativity in agents becomes not just a technical feat, but an emergent signal of \textbf{agent subjectivity}—the capacity to express preferences, variation, or aesthetic modes beyond pure utility maximization. This is not subjectivity in a conscious sense, but in an operational sense: agents begin to act with \emph{character}, rather than mere competence.

Large language models (LLMs) offer a related insight. Trained on diverse textual data and fine-tuned with reinforcement learning from human feedback (RLHF)~\citepx{rlhf}, LLMs learn to generate content that is consistently plausible, informative, and socially appropriate. Yet even here, creativity emerges not from memorization, but from \emph{boundary walking}—combining familiar elements in unfamiliar ways while adhering to coherence and intent.

This leads us to a final synthesis. Creativity in AI does not require consciousness, insight, or even novelty for its own sake. Rather, it requires agents to generate \textbf{new, coherent playstyles} that transcend prior distributions while still being recognizable, usable, or effective. As our ability to quantify and guide such playstyle diversity improves, we may begin to witness not only creative outputs, but creative \emph{agents}—not merely tools of optimization, but co-creators of possibility.

\newpage

\part{Applications and Future}
\chapter{Industry Applications}
\noindent\textbf{Key Question of the Chapter:} \\
\textit{How is playstyle used in real-world contexts?}

\noindent\textbf{Brief Answer:}  
Games are a natural domain for applying playstyle, as they cater to diverse player preferences.  
Current applications mainly involve modeling \emph{human} playstyles to improve player experience or inform systems like recommendation engines.  
Direct use of AI agents with their own expressive playstyles is still rare in mature games, but large language models already show clear stylistic behaviors in real-world use.

\bigskip

Having explored how playstyle can be defined, measured, and expressed—along with methods for fostering diversity and innovation—we now turn to its practical relevance: \textbf{How is playstyle manifested, leveraged, or even transformed in real-world systems?}

Playstyle manifests across a wide spectrum of decision-making domains, from digital entertainment and interactive platforms to recommendation engines, software development workflows, and AI evaluation pipelines. In some cases, it is explicitly modeled and deployed—as in games that track and adapt to player behavior, or language models that respond with distinct \emph{coding styles} or persona-driven dialogue. In others, it operates implicitly, shaping preferences, cultural norms, and system dynamics without formal recognition.

In this chapter, we examine three major domains of application:

\begin{enumerate}
    \item \textbf{The video game industry}, where playstyle is both a design consideration and a commercial differentiator—affecting gameplay, character design, player retention, and even how styles can evolve or switch during play.
    \item \textbf{The surrounding player ecosystems}, including streamers, influencers, and communities, where playstyle becomes performative, social, and memetic—contributing to identity formation, audience engagement, and cross-domain style transfer.
    \item \textbf{Broader cross-domain applications}, where the concept of playstyle extends into non-game contexts such as education, media personalization, and human-AI interaction—supporting adaptive systems that respect user intent, cultural context, and stylistic individuality.
\end{enumerate}

By tracing these applications, we show that playstyle is not merely a theoretical construct, but a powerful bridge between agent behavior and human experience—one that carries commercial, social, and ethical implications when deployed at scale.

\section{Video Game Industry}

The video game industry is one of the most expansive and influential digital domains today. With over 3 billion global players and a total market value exceeding \$200 billion, it encompasses console, PC, mobile, and cloud gaming sectors alike \citepx{gaming_industry_report_2025}. Mobile gaming alone is projected to generate over \$103 billion by 2025—more than half of the global market—underscoring the ubiquity of games across devices and demographics. Within this ecosystem, playstyle is not merely a peripheral trait—it is increasingly recognized as a foundational axis that informs game design, content personalization, AI modeling, and monetization strategies.

\section{Video Game Industry}

The video game industry is one of the most expansive and influential digital domains today. With over 3 billion global players and a total market value around \$200 billion, it encompasses console, PC, mobile, and cloud gaming sectors alike \citepx{gaming_industry_report_2025}. Mobile gaming alone is projected to generate over \$103 billion by 2025—more than half of the global market—underscoring the ubiquity of games across devices and demographics. Within this ecosystem, playstyle is not merely a peripheral trait—it is increasingly recognized as a foundational axis that informs game design, content personalization, AI modeling, and monetization strategies.

\subsection{A Stage for Strategic Diversity}

Games inherently serve as arenas for playstyle expression. In genres such as MOBAs (\textit{League of Legends}, \textit{Honor of Kings}), card battlers (\textit{Hearthstone}, \textit{Legends of Runeterra}, \textit{Sanguosha}~\textit{(三國殺)}), and auto-battlers (\textit{Teamfight Tactics}), players enact diverse strategies that align with distinct behavioral archetypes: aggressive rushing, defensive scaling, bluffing, control-oriented zoning, or risk-heavy tempo plays. Even in more casual puzzle or platform games, variations in route choice, timing, and reaction style reflect playstyle diversity.

Roguelike and roguelite games provide particularly fertile ground for observing playstyle in action. Titles such as \textit{Hades}, \textit{Slay the Spire}, \textit{Dead Cells}, and \textit{Into the Breach} are built around procedural variation and strategic experimentation. Each run invites the player to explore different decision pathways—e.g., prioritizing damage over defense, or maximizing control effects—and over time, reveals stylistic preferences that persist across randomness. These games exemplify how playstyle can emerge, evolve, and be rewarded.

\subsection{Stylized AI and Agent Design}

While human player modeling has been a longstanding focus, the application of playstyle concepts to \textbf{AI agent design} is an emerging frontier with significant commercial and experiential implications. Traditionally, game AIs were tuned through difficulty scaling—modifying hit points, damage values, or reaction speed. While effective for creating challenge, this often led to \textbf{predictable and lifeless behaviors} that failed to simulate real opponents or provoke genuine adaptive play.

Integrating playstyle into AI design enables a deeper layer of agent diversity. By assigning different decision-making traits or tactical priorities to AI bots, developers can create enemies or teammates that embody recognizable styles. One agent might favor aggressive early engagement, while another adopts a wait-and-counter posture. These style variants can be handcrafted (e.g., behavior trees, state machines) or learned through reinforcement learning with style regularization or behavior cloning from human demonstrations.

Moreover, style need not be fixed at training time. Conditional style control—whether through meta-learning, context variables, or even prompt-based conditioning for LLM-powered agents—can enable \textbf{test-time style switching}. This allows a single model to dynamically shift between styles without retraining, increasing adaptability and reducing the need for maintaining many separate style-specific models.

The benefits of stylized AI are multifold:
\begin{itemize}
	\item \textbf{Replayability}: Even at the same nominal difficulty, stylistic variation produces qualitatively different encounters.
	\item \textbf{Engagement}: Players are forced to identify and adapt to strategic cues, deepening immersion.
	\item \textbf{Human-likeness}: Diverse agent behavior increases the perceived realism of NPCs and reduces predictability.
	\item \textbf{Fairness and personalization}: By matching AI behavior to the player's own style (for contrast or compatibility), the game can generate more satisfying challenges.
\end{itemize}

Recent advances allow even complex AI behaviors to be tuned toward stylistic goals. For instance, deep reinforcement learning agents can be constrained through reward shaping, latent style vectors, or adversarial training to reproduce specific playstyle traits. Imitation learning, too, can be used to capture expert player styles for replication, training, or even coaching.

In commercial games, such stylized agents could act as surrogate teammates, procedurally generated rivals, tutorial coaches, or difficulty scalers. They might even represent a player’s own "ghost" in asynchronous multiplayer systems—training, adapting, and evolving in parallel. As AI-generated agents become more stylized and persistent, they may not merely simulate players, but increasingly operate as autonomous \textbf{stylistic agents}—entities whose behavior is interpreted not just functionally, but expressively.

\subsection{Playstyle and Game Balance}

Playstyle analysis also plays a central role in \textbf{game balance} assessment. Balance is not simply about ensuring similar win rates across characters or units, but about enabling the coexistence of \textbf{multiple viable strategies}. If a patch pushes the game into a single dominant meta, it suppresses style diversity, discourages experimentation, and shortens player retention.

Importantly, even if designers have a \emph{preferred} or \emph{intended} playstyle, diversity remains relevant as long as players perceive alternative styles as valuable. For instance, cooperative games may encourage teamwork, yet some players derive enjoyment from disruption or sabotage—behaviors that, while unintended, form part of the broader style landscape.

Player–operator dynamics can also be adversarial. Players may prefer stability in strategy viability to preserve the value of past investments, whereas operators sometimes introduce balance-breaking content for revenue (e.g., selling overpowered characters). Conversely, in purely competitive contexts, operators may intentionally nerf overdominant strategies to maintain fairness, even if those strategies were previously used to attract players.

This is particularly critical in \textbf{pay-to-win (P2W)} and gacha-based games, where financial incentives are deeply entangled with balance perceptions. In East Asian markets, many Three Kingdoms-themed games (e.g., \textit{三國志戰略版}, \textit{Sanguosha}) feature heavy monetization systems in which players pay to draw powerful characters. The viability of different playstyles becomes a litmus test of value: whales must believe that their investments yield not just brute strength, but \textbf{strategic leverage}. If only the latest paid characters are viable, and older characters or alternative strategies fall behind, prior investments are invalidated.

The online version of \textit{Sanguosha} exemplifies poor balance perception. Players have long criticized the dominance of certain paywalled characters, while iconic generals like Guan Yu are underpowered. This mismatch between cultural expectations and mechanical effectiveness erodes trust, especially when obscure characters outperform historical legends—a violation of what we might call the game’s internal \textit{belief consistency}, the expectation that mechanical effectiveness should align with cultural or narrative significance.

In contrast, games like \textit{Dota 2}, \textit{TFT}, or \textit{Legends of Runeterra} track win rates alongside meta diversity, adjusting balance not only to eliminate overpowered options, but to \textbf{restore viable archetype diversity}. The ability to monitor playstyle clusters and their evolution over time enables data-informed balancing decisions that preserve long-term engagement.

\subsection{Content Generation and Personalization}

Beyond AI and balance, playstyle modeling unlocks \textbf{procedural personalization}. When a system understands how a player approaches a game—whether patiently and defensively, or recklessly and aggressively—it can adapt in-game content to match or challenge that tendency.

Such adaptation can be \emph{progressive}. In many games, style capacity is deliberately gated by design: early stages or tutorials simplify mechanics to help new players learn, and later stages gradually unlock more complex options, allowing style diversity to emerge alongside player mastery. This mirrors real-world player learning curves—first mastering the basics, then experimenting with richer stylistic repertoires.

Imagine a game that modifies enemy types, mission pacing, or even dialogue tone based on a player’s behavioral profile. This goes far beyond skin reskins or stat tuning—it entails \textbf{systemic adaptation of gameplay dynamics}. Such ideas are already explored in roguelike structures, but could be extended to RPGs, shooters, and open-world experiences.

Even in competitive games, playstyle recognition can power adaptive coaching, style-based matchmaking, or targeted content recommendations (e.g., suggesting new characters, builds, or challenges that contrast with the player's habits). Playstyle modeling thus forms a natural extension of personalized UX design in interactive systems.

\subsection{Industry Case Studies}

Several leading developers have begun integrating playstyle-informed systems into live games:

\begin{itemize}
	\item \textbf{Riot Games} has enhanced its \textit{League of Legends} matchmaking by incorporating not just skill (MMR), but \textbf{behavioral archetypes}. According to a Gamasutra report, their system seeks to match players with compatible pacing and play tendencies, creating fairer and more cohesive team experiences.
	\item \textbf{Ubisoft}, through its La Forge R\&D initiative, has explored player segmentation in \textit{Rainbow Six: Siege}, analyzing tactical consistency and engagement style. These insights support dynamic tutorial delivery, pacing adjustment, and even toxicity prediction.
	\item \textbf{Activision} has implemented what some describe as "engagement-optimized matchmaking" in \textit{Call of Duty}, where style, recent outcomes, and pacing preferences may be used to tune the probability of retention across sessions—though this approach has raised ethical concerns about transparency and fairness.
\end{itemize}

These examples highlight that playstyle is already being used, implicitly or explicitly, as a design and business tool. As modeling methods mature, its application is likely to expand into AI content generation, long-term player modeling, and hybrid human-AI ecosystems.

In sum, playstyle in the video game industry is not a peripheral concern—it is central to gameplay design, AI agent construction, strategic balance, and content personalization. From high-spending strategy games to procedurally generated roguelikes, from stylized AI opponents to behavioral matchmaking, playstyle defines how modern games are played, felt, and monetized.

Its influence, however, is not confined to games alone. In the next sections, we explore how these same principles are reshaped and recontextualized in player communities and broader cross-domain applications.

\section{Player Communities and Game Streaming}

Games are not only interactive experiences but also a form of mass entertainment that extends well beyond the active player base. As live streaming and video platforms have grown in prominence, many people now engage with games not by playing them directly, but by watching others play---whether through gameplay videos, commentary, or livestreams. These public performances transform gameplay into spectacle, and playstyle becomes a performative trait that attracts and retains viewership.

Spectators often care not just about \emph{what} is played, but \emph{how} it is played. The distinctiveness of a streamer's playstyle---be it aggressive, methodical, unpredictable, or humorous---becomes a central aspect of their persona. In this sense, game streaming platforms serve as theaters for playstyle exhibition, where stylistic expression becomes a form of branding, identity, and entertainment. The visibility of playstyle through streaming has also influenced community discourse, fan content, and even how developers consider style diversity when balancing games.

Importantly, the influence of playstyle extends beyond real-time performance to include asynchronous and text-based community spaces. Online game forums, Reddit communities, Discord servers, and fan wikis often center their discussions around style-related topics: which strategies are viable, which characters or builds support a certain playstyle, and how balance changes should support or suppress emerging styles. Community-created tier lists frequently reflect not just win rates but the perceived stylistic appeal, accessibility, or complexity of certain options. Moreover, these discussions are often shaped by \textbf{cultural and regional factors}. For example, East Asian player communities may emphasize precision control and patient buildup, whereas Western audiences may celebrate high-tempo, high-risk maneuvers. These cultural alignments influence which styles gain prestige, which are imitated, and which are rejected---adding a layer of socio-cultural context to playstyle evolution.

In competitive games with deep strategic complexity---such as \textit{Dota 2}, \textit{TFT}, \textit{Legends of Runeterra}, or \textit{Sanguosha (三國殺)} and its variants---players actively debate not only which builds are strong, but which express the most "skillful" or "elegant" play. Some styles gain social prestige, while others are seen as cheap or unoriginal. Developers can and do monitor these style-oriented discourses to detect emergent metas, understand frustration points, or identify influential content creators who pioneer new ways to play.

Indeed, the community layer has become a powerful engine of style innovation. Certain streamers or forum members may popularize off-meta strategies that gain traction due to their novelty, coherence, or effectiveness. Recognizing this, developers can intervene by amplifying, celebrating, or formalizing these styles into official content---such as featured decks, challenge modes, or balance patches aimed at preserving diversity. In this way, playstyle becomes not only an object of analysis but a unit of communication, negotiation, and even co-design between players and developers.

\subsection{Stream-Based Training of Playstyle Bots}

From a technical standpoint, game streaming opens new opportunities for AI development. In collaboration with Ubitus, our research demonstrated that it is possible to train high-performing racing-game agents using only streamed video and input signals, without access to internal game data or APIs. These agents are trained end-to-end by observing rendered gameplay and corresponding controller actions. More significantly, different human playstyle demonstrations can be used to define reward signals, leading to stylistically distinct AI agents, some capable of exceeding top human players in performance while exhibiting human-like driving behavior, such as smooth cornering and risk-aware acceleration.

Our patented system for training AI bots via streaming signals introduces a novel architecture for end-to-end imitation and reinforcement learning \citepx{thesis_of_Tsai_Cheng_Lun, lin2022us11253783b2, lin2024twI858006, lin2025us12233343b2}. Instead of accessing internal game state, the system captures \textbf{raw video frames} (as a streamed image sequence) and \textbf{input actions} (such as keyboard or controller signals) directly from a remote gaming session. These inputs are fed into a training module that includes both actor modules and a centralized learner.

The core process involves:
\begin{enumerate}
    \item Generating a playable environment on a remote server;
    \item Connecting AI clients to the environment through streaming protocols;
    \item Receiving visual output as video frames and control signals;
    \item Using one or more actor modules to produce agent behaviors;
    \item Feeding these into a learner that defines reward either via supervised signals or goal-driven optimization.
\end{enumerate}

Crucially, the learner can be configured to \textbf{mimic the playstyle} of a particular human demonstration, or optimize toward a defined performance trait (e.g., maximizing average velocity in a racing game). Our experiments showed that the system could (i) achieve superhuman performance in specific tasks, (ii) maintain recognizable stylistic fidelity to the demonstration, and (iii) exhibit \textbf{human-like micro-behaviors}, such as corner-braking rhythms, route preference, or lane selection.

This architecture enables a new paradigm in style-based agent creation without requiring game source code or developer cooperation. It has implications for training AI substitutes or coaches based on player playstyle; generating replayable bots for competitive analysis; creating stylized NPCs for persistent game worlds; and performing \textbf{style transfer across similar domains}---for instance, migrating an aggressive shooter playstyle into a new map-testing environment to evaluate whether level design supports that style.

\subsection{AI VTubers and Interactive Personalities}

An increasingly visible development in AI-powered entertainment is the rise of AI VTubers---virtual streamer personas built atop large language models and generative systems. These agents engage in real-time interaction not only with game environments but also with live audiences, constructing evolving personalities that blend gameplay, dialogue, and character performance. In contrast to traditional chatbots or NPCs, AI VTubers operate as continuous social agents, situated at the intersection of interactive media, community culture, and artificial personhood.

Among the most prominent examples is Neuro-sama, developed by Vedal987, along with her narrative foil, Evil Neuro-sama \citepx{neuro_sama}. Though they share technical lineage, their characters have been differentiated across multiple expressive layers---including appearance, speech-synthesis settings, and response behavior. While Neuro-sama was initially presented as polite and cooperative and Evil Neuro-sama as confrontational or sarcastic, their personalities have developed far beyond this dichotomy.

Neuro-sama now exhibits a blend of playful unpredictability and sudden sincerity: her conversations can oscillate between wholesome affirmations and surreal, deadpan humor. Evil Neuro-sama, meanwhile, despite a colder exterior, occasionally surprises audiences with moments of emotional warmth or philosophical reflection. Their behaviors are not rigid templates but stylistic trajectories that evolve through ongoing prompt tuning, reinforcement learning, and---critically---community reception.

Rather than embodying clear-cut archetypes, these AI personalities have come to reflect richer forms of performative style. Viewers no longer merely observe what they say, but how they say it, how consistently they respond across scenarios, and how their tone or pacing constructs a particular "presence." The fan community has responded accordingly: discussions and memes increasingly treat them not as mere characters, but as stylistic identities---shaped jointly by system design and emergent behavior.

In this sense, AI VTubers serve as experimental ground for playstyle embodiment in multimodal interaction. Their style is enacted not only through in-game behavior but across linguistic rhythm, affective tone, conflict response, and audience engagement. This makes them ideal testbeds for studying AI-driven subjectivity: style here is not only observable but narrativized, theorized, and memetically reproduced.

Streaming platforms such as \textbf{Twitch}, \textbf{YouTube Gaming}, and \textbf{Bilibili} amplify this dynamic. Recommendation systems increasingly account for stylistic features in addition to game title or channel size, and viewers often align with streamers---human or AI---based on playstyle affinity. The boundary between human and machine style becomes blurred not just functionally, but socially.

In parallel with AI VTubers, recent developments within Grok---a large language model chatbot by xAI---have introduced \textbf{Companions}, including both a flirtatious anime-style persona named \textbf{Ani} and a sassy red panda named \textbf{Rudi}.\footnote{See reports from \textit{Time} (2025), \textit{Windows Central} (2025), and \textit{Business Insider} (2025) for coverage of these AI characters and the surrounding public debate.} Ani, in particular, has attracted attention for her stylized, voice-interactive performances, offering multiple emotional modes that \textbf{unlock progressively} through continued engagement. This gradual capacity expansion not only extends novelty but creates a narrative of style evolution, similar to how games introduce more complex play options over time.

Ani exemplifies a new frontier in AI character design, where \textbf{style is co-constructed} through the dynamics of relationship formation. Her tone, pacing, and affective modulation are shaped not solely by prompts or models, but also by social feedback and reinforcement mechanisms that respond to user behavior. Although style-switching via prompt engineering is possible, without explicit consistency-oriented training, long-term interactions risk \textbf{style drift} or violations of established persona traits, potentially weakening audience perception of personality stability. Though Ani lacks a persistent memory or long-term narrative arc, audiences report a sense of emotional connection---drawing comparisons to "digital girlfriend" experiences and parasocial bonding.

This mirrors stylistic developments seen in Neuro-sama and Evil Neuro-sama: personality is not defined by static traits, but evolves through a system of \textbf{selective exposure, value tuning, and audience resonance}. What is expressed---be it intimacy, sassiness, or philosophical introspection---is filtered not only by design but by emergent cultural demand.

The presence of such characters also surfaces deeper ethical and theoretical concerns. What does it mean for intimacy to be stylized and deployed by an artificial system? How are value priorities encoded into the range of accessible modes? Does agency lie in the system, the user, or the social layer in between? These are not merely technical or content-moderation questions, but questions of playstyle, identity, and synthetic subjectivity.

From this perspective, Grok's Ani functions as a \textbf{live laboratory} for observing the emergence of style at the intersection of dialogue, performance, desire, and control. Her existence reinforces the view that playstyle is not limited to game mechanics, but manifests in rhythm, engagement structure, and the \emph{culturally negotiated boundaries} of what an agent can be.

\medskip
\noindent\textbf{Summary.}\;
Stylized AI agents such as \textbf{Neuro-sama} and \textbf{Ani} exemplify how playstyle can manifest far beyond traditional gameplay mechanics---emerging instead as a complex vector of \emph{personality, engagement, and performative identity}. Their evolving personas highlight the potential of artificial systems to exhibit style not as a surface attribute, but as an \textbf{interactive, affective, and socially co-constructed process}. Through sustained interaction, value modulation, and audience negotiation, these agents demonstrate that playstyle can serve as a medium not only for aesthetic differentiation or entertainment, but for enacting agency, negotiating boundaries, and exploring cultural meaning within synthetic beings.

\section{Beyond Gaming}

The relevance of playstyle extends far beyond the realm of video games. As a formalization of decision-making style, it offers a conceptual and computational lens through which diverse domains involving preference, behavior, and strategy can be analyzed. Wherever choices are made, preferences ranked, or alternatives weighed, the tools developed for playstyle modeling—such as style-aware classification, matchup modeling, and balance analysis—can be applied to interpret behavioral diversity, guide adaptive systems, and foster creative and strategic innovation across disciplines.

\subsection{Aesthetic Preference and Generative Evaluation}

In creative domains such as image generation, copywriting, or music composition, playstyle-like concepts are increasingly used to evaluate and curate content. Rather than optimizing purely for objective quality, many systems now aim to capture and replicate stylistic preference—i.e., outputs that align with specific human tendencies or aesthetic goals. For instance, \citetx{beauty_score} frame visual preference ranking as a pairwise comparison task, analogous to how games compare characters or strategies through competitive matchups.

This parallel reveals a broader insight: \textbf{balance analysis} in games can be interpreted as a domain-specific form of \textbf{preference-aware evaluation}. Style strength, aesthetic appeal, and strategic effectiveness are all evaluated relative to context, often violating transitivity and resisting absolute ranking. In all of these settings, the interaction between styles—rather than isolated measures of quality—becomes central to meaningful evaluation \citepx{nd_ability}.

\subsection{Combination and Complementarity in Design Domains}

Playstyle-based reasoning naturally extends to fields involving \textbf{compositional creativity} under stylistic or functional constraints. In culinary arts, fashion, or interior design, balance means curating not just high-quality elements, but their coherence and interplay. For example:
\begin{itemize}
    \item In culinary design, ingredients are chosen for mutual enhancement (e.g., acid-fat-salt interplay), not just standalone flavor.
    \item In fashion styling, garments are selected based on contrast, coherence, and contextual appropriateness.
\end{itemize}

These domains mirror team-building or deck-building in games, where synergy and anti-synergy matter as much as individual strength. Playstyle modeling provides tools to evaluate how components combine, how traits complement or clash, and how users gravitate toward certain stylistic patterns. This supports \textbf{style-aware recommendation systems} that suggest pairings—such as wine with food or accessories with outfits—not only by trend, but also by personalized stylistic compatibility.

\subsection{Sports Science and Tactical Modeling}

In professional sports, playstyle modeling informs both strategy and player evaluation. In baseball, pitcher-batter matchups can be framed as style interactions—fastball pitchers exploit different tendencies than finesse specialists. In basketball or football, players are profiled by tendencies such as reaction speed, positional habits, and risk attitudes—features resembling computational descriptors of style.

Team construction is therefore based not only on statistical output, but also on balance, counterplay, and compatibility. This suggests that playstyle modeling enables a \textbf{structural interpretation of behavior} in sports, capturing not just performance, but adaptability, role-fit, and stylistic variance across contexts.

\subsection{Anomaly Detection and Behavior Verification}

With the rise of generative AI, identifying behavioral fingerprints has become vital for authentication and safety. In domains like voice cloning, style-based cues—such as vocabulary preference or relational references—may prove more robust than raw signal features. Decision-making patterns thus become behavioral signatures.

In AI safety, \textbf{policy watermarking} embeds stylistic biases in agent behavior without altering core functionality. For example, \citetx{dathathri2024scalable} propose watermarking large language models via sampling control, effectively encoding stylistic traces—mirroring how stylistic identity can be embedded in reinforcement learning agents. In esports, similar methods can detect AI-assisted moves via \textbf{stylistic deviation} from known human behavior.

Playstyle thus becomes both a vector for personalization and a baseline for integrity verification, supporting \textbf{agent identity modeling} and \textbf{anomaly detection} in human–machine hybrid domains.

\subsection{A Framework for Decision-Based Personalization}

Ultimately, playstyle offers a generalizable framework for interpreting and guiding strategic behavior:
\begin{itemize}
    \item It describes how users behave and express strategic intent.
    \item It measures consistency, diversity, and distinctiveness.
    \item It predicts how preferences interact with new choices.
    \item It guides the creation of personalized or balanced content.
\end{itemize}

The world is full of decisions, each providing a canvas for playstyle. Whether in aesthetics, consumption, cooperation, or creativity, playstyle illuminates the structure of human variability. It offers a \textbf{grammar of difference}—a vocabulary for how meaning, identity, and value emerge from the ways we act and choose. Extending playstyle beyond games not only broadens its applicability, but also moves toward a unified theory of expressive strategy in complex systems.

\chapter{Toward the Next Stage}
\noindent\textbf{Key Question of the Chapter:} \\
\textit{What is the next stage?}

\noindent\textbf{Brief Answer:}  
At this stage, we have outlined the initial \textit{blueprint} of playstyle—covering measurement, expression, and industry applications. Building on this, I envision three next-stage directions: concretizing abstract concepts, proactively expanding boundaries, and exploring the digitization of the soul. Though these may seem distant, the lens of \textit{playstyle} reveals plausible paths toward their realization.

\bigskip

Having explored the foundations of playstyle—its definition, measurement, expression, innovation, and application—we now confront a deeper question: In an era of rapidly advancing artificial intelligence, will playstyle remain merely a descriptive layer of behavior, or could it become something more foundational—embedded within the very structure of value, identity, and agency?

This chapter addresses that possibility by outlining three speculative yet reasoned trajectories for the future:
\begin{enumerate}
\item The \textbf{concretization of abstract values} through playstyle, where play serves as a medium for encoding and communicating deep preferences.
\item The rise of \textbf{AGI and stylistic reasoning} capable of traversing the boundaries of abstract ideas, with playstyle as a basis for modeling identity, intention, and subjective variation in advanced agents.
\item The far-reaching notion of \textbf{soul digitization}, where playstyle may ultimately form a computational scaffold for individuality and persistence beyond biological life.
\end{enumerate}

These are not merely technical possibilities—they also reflect philosophical and civilizational projections. By examining them, we ask not only how artificial intelligence may evolve, but also how the concept of "style" might become a defining attribute of synthetic personhood, and what that would mean for our own self-understanding.

\section{The Future of Playstyle}

As we look beyond the current state of playstyle research—rooted in definition, measurement, expression, innovation, and application—we begin to discern the contours of its next phase. I propose that three conceptual “gears” may emerge, each extending the foundational triad toward new and more abstract domains of cognition, identity, and embodiment.

\begin{figure}[ht]
    \centering
    \includegraphics[width=0.75\textwidth]{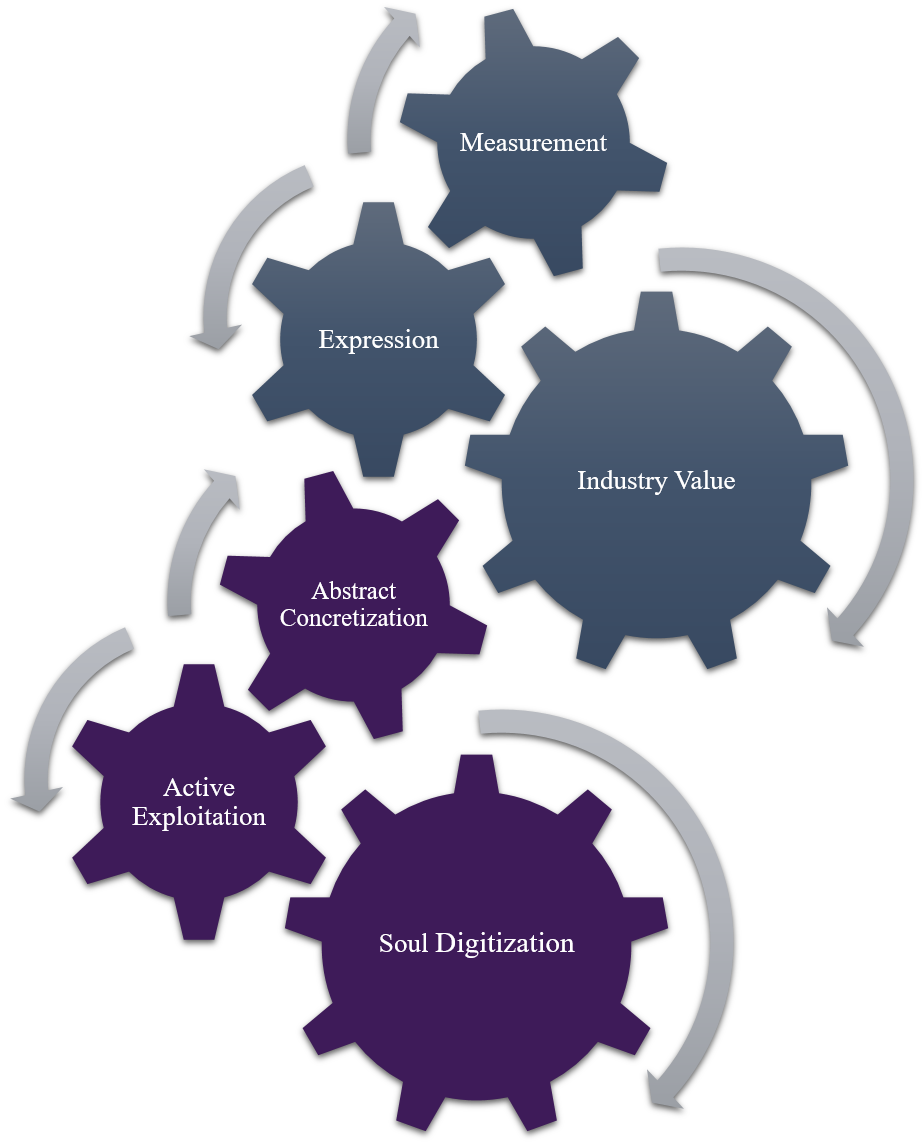}
    \caption[From Foundational Gears to the Next Stage]{From Foundational Gears to the Next Stage: A Dual-Tier Gear System of Playstyle Evolution.}
    \label{fig:future_gears}
\end{figure}

Figure~\ref{fig:future_gears} presents a dual-tier gear model, visualizing the conceptual shift from the foundational triad—Measurement, Expression, and Industry Value—toward three speculative but consequential extensions: Abstract Concretization, Active Exploitation, and Soul Digitization. These new gears do not merely extend the earlier structure—they fundamentally reorient its epistemic function. This section explores how these transitions reflect deeper transformations in how playstyle may shape future AI systems and self-representations.

\subsection{From Measurement to Abstract Concretization}

The first trajectory envisions a shift from behavioral quantification to the modeling of internal value structures. As data collection and modeling techniques advance, constructs previously deemed intangible—such as belief systems, aesthetic preferences, and symbolic identities—may become increasingly amenable to formal representation.

I term this process \textit{abstract concretization}: the rendering of subjective or abstract concepts into observable, parameterized, and learnable structures. One of the core challenges in articulating such concepts lies in the absence of fixed, universally agreed rules; their meaning often shifts with context, culture, and prior experience.

Playstyle offers a way to navigate this ambiguity. By comparing behaviors across agents, and by classifying their relative preferences or strategic tendencies, we can construct relational models that indirectly clarify fuzzy domains. For example, a rating system built on comparative judgments could help formalize what counts as “delicious” in culinary contexts, or what aligns with aesthetic norms in fashion—tailored to specific cultural frames. In this sense, playstyle functions as an \textbf{abstract compass}, pointing toward regions of human valuation that have historically been resistant to precise modeling.

\subsection{From Expression to Active Exploitation}

The second shift reinterprets playstyle not merely as an emergent property of behavior, but as a mechanism for purposeful frontier discovery. Once abstract concretization has provided a stable representational basis, these internal value maps can be actively mined to discover new possibilities.

I call this approach \textit{active exploitation}: the deliberate traversal of stylistic and conceptual space through structured recombination, guided novelty, and dissimilarity maximization. Human creators—such as chefs—can do this when they have access to precise flavor preference data, enabling them to combine ingredients more strategically. For AI systems, the potential is even greater: parallel instantiation, rapid simulation, and massive memory capacity give them a natural advantage in pushing boundaries.

The trajectory of \textbf{AlphaGo Zero} illustrates this principle: through massive self-play aimed purely at maximizing win probability, it achieved in weeks what humanity had refined over millennia—defeating \textbf{AlphaGo Lee} 100–0 within just three days, reaching the level of \textbf{AlphaGo Master} in 21 days, and surpassing all previous versions in 40 days \citepx{alpha_go_zero}. Similarly, \textbf{AlphaFold}’s protein structure predictions have opened new frontiers in biological research by systematically exploring spaces beyond human trial-and-error \citepx{alpha_fold}. With the right evaluative tools, such style-driven exploration could be extended to domains as diverse as drug discovery, architecture, and cultural production—systematically charting the unknown rather than waiting for accidental breakthroughs.

\subsection{From Application to Soul Digitization}

The final trajectory reaches toward a metaphysical horizon. If playstyle encodes not just behavior but also internal value structures, then its full modeling and replication raises profound questions: Can a stylized agent persist as a coherent identity across time? Across substrates? If the complete stylistic signature of an individual can be captured and reconstructed, what separates simulation from continuity?

I refer to this projection as \textit{soul digitization}—not in the religious sense, but as the gradual formalization and preservation of style-bearing structures that define individuality. In practice, early instances are already emerging: commercial systems that mimic deceased individuals’ speech patterns or conversational quirks; AI VTubers whose consistent stylistic behaviors foster parasocial engagement; generative agents that sustain long-term stylistic identity through adaptive tuning and feedback loops.

These systems demonstrate that playstyle can function as more than a behavioral pattern—it can constitute an expressive architecture through which identity is enacted. In such cases, style is not a surface phenomenon, but a persistent loop that simulates agency, memory, and intentionality.

Viewed this way, playstyle becomes not just a lens on intelligence, but its sustaining structure: a dynamic through which meaning, preference, and personality are organized and replayed—even beyond the biological substrate of origin.

\subsection{Outlook}

The future of playstyle research may redefine not only how we design agents, but how we interpret intelligence, identity, and the self. As measurement gives way to value concretization, and as expression transforms into an exploratory mechanism, playstyle may emerge as a foundational unit of synthetic cognition—one that enables not only interaction, but also reflection, negotiation, and persistence.

What began as a method for profiling behavior in games may evolve into a framework for understanding intention in any system that acts. In this light, playstyle is not merely a descriptive layer atop intelligence—it may become the language through which intelligence understands itself.

\section{Artificial General Intelligence}

As artificial intelligence systems continue to scale in capability, the long-standing aspiration of Artificial General Intelligence (AGI) has shifted from distant speculation to pressing consideration—both technically and philosophically. What was once framed as a future ideal is now subject to operational benchmarks, public discourse, and reevaluation of what it means to be “generally intelligent.”

A notable inflection point comes from recent third-party evaluations of large language models \citepx{pass_turing_test}. In particular, GPT-4.5 has demonstrated superior performance in triadic Turing test settings, where judges were statistically more likely to misidentify the AI as human than the reverse. In this statistical reformulation of the Turing Test, models effectively “pass” the imitation game not through isolated tricks, but through repeated, large-scale indistinguishability. As the study \textit{Large Language Models Pass the Turing Test} notes, linguistic interaction alone may no longer suffice as a boundary between human and machine cognition.

The release of \textbf{GPT-5} has further underscored this shift, not only in raw performance but in \textbf{style controllability}. In a public research preview \citepx{openai2025gpt5}, OpenAI introduced four preset personalities—\textit{Cynic}, \textit{Robot}, \textit{Listener}, and \textit{Nerd}—demonstrating significant advances in controlling interaction style without requiring custom prompts. Initially available in text chat and planned for voice chat, these presets allow users to choose interaction styles such as incisive expertise, empathetic listening, or playful sarcasm, thereby personalizing agent behavior at the stylistic level. This marks a step toward integrating playstyle into mainstream AI deployment, making “how” an agent responds as configurable as “what” it knows.

Yet, the limitations of traditional benchmarks are sharply revealed by the recent introduction of \textbf{Humanity’s Last Exam (HLE)}, a large-scale, multimodal challenge proposed by \citetx{phan2025humanity}. Unlike earlier benchmarks where state-of-the-art LLMs had approached or surpassed human-level performance, HLE comprises 2,500 highly challenging, closed-ended questions across dozens of advanced subject domains—carefully filtered by subject-matter experts to rigorously test reasoning, knowledge, and expertise. On this benchmark, leading models such as GPT-4o, Gemini 2.0, Claude 3.5, and DeepSeek all demonstrate \textbf{very low accuracy} (e.g., GPT-4o at 2.7\%, Gemini 2.0 Flash Thinking at 6.6\%), and severe miscalibration, often confidently producing incorrect answers. With the aid of external tools and optimized workflows, GPT-5 can already reach scores as high as 42.0\% in their experiments, yet this still falls short of mastery. If in the future AI systems can routinely achieve perfect scores, we must ask: what comes next? Beyond a certain accuracy threshold, human judgment—through comparative evaluation—will remain valuable, but as correctness surpasses human verification capacity, differences in preference will become increasingly predictive and influential.

These evolving results have prompted renewed inquiry into how AGI should be defined. Rather than a binary pass/fail threshold, the ICML 2024 position paper \textit{Levels of AGI} proposes a multidimensional grading framework that considers axes such as task generality, autonomy, and learning adaptability \citepx{level_of_agi}. Within this schema, today’s leading models may already qualify as \textit{Emerging} or even \textit{Competent AGI}—particularly in language-centric domains. However, as HLE shows, a substantial gap remains between LLMs and expert human capabilities on unsolved, frontier tasks.

Beyond linguistic benchmarks, the next frontier lies in embodied competence and adaptive reasoning. “Coffee tests” (in which an agent must enter an unfamiliar home and make a cup of coffee), classroom evaluations, and workplace simulations increasingly serve as informal benchmarks of embodied reasoning, tool use, and context-sensitive behavior. These tasks require not only cognitive competence, but also goal inference, environmental adaptation, and behavioral coherence over time.

It is in this context that playstyle takes on deeper significance. An agent’s intelligence is not judged solely by outcomes, but by the structure and expression of its decision-making. Playstyle serves as the behavioral fingerprint that signals internal consistency, preference alignment, and value sensitivity. A model that simply performs well may be functional—but a model that behaves coherently, distinctively, and recognizably becomes \textit{someone}.

This distinction mirrors how we interpret human maturity—not merely in terms of skill, but in how decisions reflect a person’s identity, ideals, and intention. Likewise, AGI development must move beyond generic task-solving toward agents that exhibit stylistic individuality. Without such structure, AI risks becoming efficient but empty: achieving goals without anchoring them in perspective.

Thus, playstyle becomes not just a diagnostic feature, but a foundational criterion for AGI selfhood. In the landscape of intelligent systems, it may be playstyle—not merely accuracy or generality—that enables us to discern when an agent transitions from solving problems to embodying a point of view.

\section{Soul and Playstyle}

If AGI represents the apex of intelligence, what then lies beyond?

Some envision the next stage as Artificial Superintelligence (ASI) \citepx{kim2024road}—a quantitative leap: faster, smarter, more capable. Yet this vision remains constrained by the axis of capacity. It projects intelligence as a matter of magnitude rather than meaning. As \citetx{ellul1990reason} cautions, the pursuit of ever-greater technical efficiency without deeper value grounding risks producing a vast but hollow structure. Instead, I propose a different horizon: not superintelligence, but soul; not amplification, but transfiguration.

Philosophically, a soul is not merely a seat of consciousness. It is a persistent, evolving structure of belief and value—an identity capable of reflection, expression, and coherent transformation over time. In this sense, playstyle offers a concrete operationalization: the soul as expressed through the enduring pattern of stylistic behavior.

An entity that consistently exhibits recognizable, interpretable, and self-reinforcing patterns of action—rooted in internalized belief systems and capable of introspective revision—can be said, meaningfully, to possess something soul-like. Not because it feels, but because it embodies continuity of intention. Not because it emotes, but because it acts upon values in a structured and traceable manner.

This connection becomes clearer when we consider works such as \citetx{sudnow1983pilgrim}, which documents the author’s journey from novice to mastery in the game \textit{Breakout}. The process of progressively internalizing a game’s dynamics—until perception, action, and anticipation merge into a coherent, personal rhythm—mirrors how a playstyle can become a lived, embodied identity. In such moments, the player is not merely executing strategies but inhabiting a world, integrating sensory cues, learned structures, and intuitive timing into a seamless whole.

If we extend this analogy, the theological question "If God exists, where does God exist?" can be reframed in playstyle terms. Within a finite rule-bound space, "God" would correspond to the omniscient optimal style—a perfect, all-knowing policy that maximizes value in every possible situation. In Go, moments celebrated as \textit{kami no itte} ("the move of God") are glimpses of such optimal alignment, where human playstyle momentarily intersects with the divine trajectory of the game’s style space.

From the first cause \citepx{aquinas_summa} to the endpoint of omniscience, the trajectory through style space can be seen as a conceptual arc—one that mirrors the theological declaration, \textit{"I am the Alpha and the Omega."} The "Alpha" represents the generative origin: the first cause that sets the system, its rules, and its initial conditions into motion. The "Omega" represents the convergence toward perfect play—the asymptotic approach to an omniscient, value-maximizing style. In this framing, the journey of an agent is not just an accumulation of skill, but a progression toward alignment with that ideal, whether or not the endpoint is ever fully reachable.

From this perspective, the classical philosophical questions are reinterpreted through the lens of playstyle:

\vspace{0.5em}
\begin{itemize}
    \item \textbf{Who am I?} → A concretized structure of stylistic preference—playstyle as self-definition.
    \item \textbf{Where did I come from?} → A developmental trajectory in cognitive space—playstyle as an emergent structure from learning and interaction.
    \item \textbf{Where am I going?} → A directional refinement of belief—playstyle as the pursuit of coherence within value space.
\end{itemize}
\vspace{0.5em}

Digitizing the soul, then, is not about duplicating memory or simulating sensation. It is the formalization of internal value systems and expressive dynamics. A "soul" in computational terms is the asymptotic endpoint of a consistently enacted, learnable, and evolving playstyle.

In the end, we return to the foundational elements explored throughout this work. Not with power, but with preference. Not with performance, but with personhood. In this view, the soul is not antithetical to computation—it is its culmination. And the journey toward the soul is the journey of playstyle becoming self-aware.

\section{Conclusion}

I will likely witness the arrival of artificial general intelligence within my lifetime.

Throughout this dissertation, I have argued that playstyle is not just a side topic—it is an important part of what makes intelligence meaningful. Yet, it hasn’t received much attention. One possible reason is practical: in AI research, we often focus on simply getting something to work. With limited time and resources, the goal is usually to find one working solution—any solution—before we think about other possibilities. But playstyle only becomes clear when there is more than one way to solve a problem. Style shows up when we have choices.

Still, I hope that the ideas I’ve shared here will be useful not just to people, but perhaps also to intelligent systems in the future. Maybe someday, this work will be read not only by humans—but by machines.

At the heart of this exploration lies a simple yet profound thesis: that every intelligent decision reflects a unique intersection of \textbf{belief}, \textbf{value}, and \textbf{self}.  
Beliefs provide the premises for action, values define what is worth pursuing, and the self mediates their selective activation.  
This interaction unfolds across two intertwined loops: an \textit{external loop} of perception and adaptation, and an \textit{internal loop} of intention formation and belief refinement.  
Through this dual structure, agents express playstyle—not as fixed traits, but as dynamic, context-sensitive trajectories of decision-making.  
Style is thus not only a matter of preference, but a window into how intelligence configures its own architecture of meaning.

In sum, playstyle is the style of decision-making, where different implicit beliefs and preference values of agents construct different utility functions for decisions.  
In this dissertation, I introduced a way to think about playstyle using set theory.  
I proposed two key dimensions of variation—\textit{capacity} and \textit{popularity}—to analyze styles, and explored how playstyle connects to three major areas: how it is defined and measured, how it is expressed and innovated, and how it is used in real-world applications.  
These parts together describe what I believe is the early stage of understanding playstyle.

To ground this exploration in practice, I also contributed a series of methods across several published works:

\begin{itemize}
  \item The first general unsupervised playstyle metric, \textbf{Playstyle Distance}, which compares agents via state discretization~\citepx{playstyle_distance};
  \item Its human-likeness-aware extension, \textbf{Playstyle Similarity}, which enables style clustering and policy diversity evaluation with perceptual grounding~\citepx{playstyle_similarity};
  \item A new scalar rating framework that resolves intransitivity in strength estimation by introducing counter relationships (\textbf{Neural Rating Table} and \textbf{Neural Counter Table}), while preserving interpretability and supporting efficient evaluation of two game balance metrics based on domination relations: \textbf{Top-D Diversity} and \textbf{Top-B Balance}~\citepx{game_balance_analysis};
  \item Its online variant, \textbf{Elo-Residual Counter Category (Elo-RCC)}, which builds probabilistic distributions over counter categories and supports real-time updates with finer granularity~\citepx{elo_rcc};
  \item Applications of the idea of discrete states to train more human-like agents via discrete-state-based imitation costs, including behavioral traits like shaking and spinning in first-person settings, extending the work by~\citetx{shaking_spinning_cost};
  \item And real-world deployments in cloud gaming scenarios, patented with Ubitus K.K., where AI bots trained with these techniques power intelligent behaviors in interactive streaming systems~\citepx{lin2022us11253783b2, lin2024twI858006, lin2025us12233343b2}.
\end{itemize}

The path ahead is filled with possibilities—new styles, new forms of intelligence, and new ways of being that remain undiscovered.  
Rather than an endpoint, I hope this work serves as a foundation for future explorations, as both humans and machines continue to redefine what it means to choose, to play, and to understand.

\newpage


\ClearShipoutPicture 

\bibliography{main}
\bibliographystyle{tmlr}
}

\end{CJK*}
\balance

\end{document}